\documentclass{article}
\usepackage{graphicx} % for inserting images
\usepackage{amsthm}
\usepackage{amsmath}
\usepackage{amssymb}
\usepackage{algorithm}
\usepackage{algorithmic}
\usepackage{booktabs} 
\usepackage{threeparttable}
\usepackage[hidelinks]{hyperref}
\usepackage{subcaption}
\usepackage{mathtools}
\usepackage{tikz}
\usetikzlibrary{arrows.meta,positioning,calc}
\usepackage{threeparttable}
\usepackage{float}
\usepackage{placeins}
\usepackage{multirow}
\usepackage{tcolorbox}
\tcbuselibrary{skins}
\usepackage{lipsum}
\usepackage{booktabs,subcaption}
\usepackage{tabularx}

\newcommand{\inc}{\(\uparrow\)}
\newcommand{\dec}{\(\downarrow\)}

\usepackage{fancyhdr}
\newtheorem{theorem}{Theorem}
\newtheorem{corollary}{Corollary}
\newtheorem{definition}{Definition}

\newtheorem{proposition}[theorem]{Proposition}

\theoremstyle{remark}
\newtheorem{remark}{Remark}[section]

\title{Sampling via Gaussian Mixture Approximations}
\author{Yongchao Huang \footnote{Author email: yongchao.huang@abdn.ac.uk}}
\date{July 2025}

\begin{document}

\maketitle

\begin{abstract}
We present a family of \textit{Gaussian Mixture Approximation} (GMA) samplers for sampling unnormalised target densities, encompassing \textit{weights-only GMA} (W-GMA), \textit{Laplace Mixture Approximation} (LMA), \textit{expectation-maximization GMA} (EM-GMA), and further variants. GMA adopts a simple two-stage paradigm: (i) initialise a finite set of Gaussian components and draw samples from a proposal mixture; (ii) fit the mixture to the target by optimising either only the component weights or also the means and variances, via a sample-based KL divergence objective that requires only evaluations of the unnormalised density, followed by stratified resampling. The method is gradient-free, and computationally efficient: it leverages the ease of sampling from Gaussians, efficient optimisation methods (projected gradient descent, mirror descent, and EM), and the robustness of stratified resampling to produce samples faithful to the target. We show that this optimisation-resampling scheme yields consistent approximations under mild conditions, and we validate this methodology with empirical results demonstrating accuracy and speed across diverse densities.
\end{abstract}

\tableofcontents

\newpage
\section{Introduction} \label{sec:intro}

We are concerned with the Bayesian inference or sampling problem \footnote{Sampling based inference finds a way to generate samples or quasi samples which, in the limit, represent the target distribution $p(\mathbf{z})$.}: how to effectively and efficiently draw samples from a distribution, called the target distribution \footnote{In this and many other contexts, we specifically refer distribution to \textit{probability density function} (PDF).}, $p(\mathbf{z})$, based which inference can be performed. The target $p(\mathbf{z})$ can be a complex distribution, for example, a posterior distribution derived from Bayes' rule:
\[
p(\mathbf{z} | \mathcal{D}) = \frac{p(\mathcal{D} | \mathbf{z}) p(\mathbf{z})}{p(\mathcal{D})}
\]
where $\mathcal{D}$ represents data, $p(\mathbf{z})$ is the prior distribution, $p(\mathcal{D} | \mathbf{z})$ is the likelihood, $p(\mathcal{D})=\int p(\mathcal{D} | \mathbf{z}) p(\mathbf{z}) d \mathbf{z}$ is the normalising constant (also termed \textit{evidence} or \textit{marginal likelihood}). In Bayesian inference, the posterior distribution $ p(\mathbf{z} | \mathcal{D}) \propto p(\mathcal{D} | \mathbf{z}) p(\mathbf{z}) $ often exhibits complexity (e.g. a hierarchical Bayesian model), such as multi-modality or high dimensionality, rendering direct evaluation \footnote{For example, likelihood may be expensive to evaluate due to e.g. large data volume or involving solving dynamics \cite{huang_classification_2022}.} or sampling intractable. This is particularly the case when the normalising constant $p(\mathcal{D})=\int p(\mathcal{D} | \mathbf{z}) p(\mathbf{z}) d \mathbf{z}$ is analytically intractable due to e.g. high dimensional integration. Approximate inference methods such as Monte Carlo (e.g. MCMC \cite{metropolis_equation_1953,hastings_monte_1970, neal1993probabilistic}) and variational inference (VI \cite{jordan_introduction_1999,kingma_auto-encoding_2022,ranganath_black_2014}) are commonly employed to tackle this. MCMC methods can be gradient-based (e.g. HMC \cite{duane_hybrid_1987}, LMC \cite{roberts_exponential_1996}, etc) or gradient-free (e.g. MH \cite{hastings_monte_1970}, nested sampling \cite{skilling_nested_2004}, slice sampling \cite{Neal2003}, etc), while most VI methods are gradient-based \footnote{
VI turns Bayesian inference into an optimisation problem:
$$
\text{Find } q^*(\mathbf{z}) = \arg\min_{\theta} KL(q_{\theta}(\mathbf{z}) \| p(\mathbf{z} \mid \mathcal{D}))
$$
or equivalently, maximize the ELBO (evidence lower bound):
$$
q^*(\mathbf{z}) = \arg\max_{\theta} \mathcal{L}(q_{\theta}) = \arg\max_{\theta} \mathbb{E}_{q_{\theta}(\mathbf{z})}[\log p(\mathcal{D}, \mathbf{z}) - \log q_{\theta}(\mathbf{z})]
$$
Optimizing $\mathcal{L}(q_{\theta})$ via e.g. gradient-based routines requires computing gradients of the objective \textit{w.r.t.} the parameters $\theta$, usually via stochastic gradient ascent. Common VI methods include mean-field VI \cite{jordan_introduction_1999}, black-box VI \cite{ranganath_black_2014}, amortized VI \cite{kingma_auto-encoding_2022}, particle-based VI such as SVGD \cite{liu_stein_2019}, etc. Gradient-free VI methods include using Gaussian process surrogates \cite{hernandez-lobato_probabilistic_2015}, evolutionary strategies for ELBO optimisation, conditional mixture networks \cite{heins_gradient-free_2025}, etc.}. Gradient-based, specifically, score-based \footnote{Score used in machine learning, in contrast to statistics, refers to $\nabla_{\mathbf{z}} \log p(\mathbf{z})$.} methods get around of the problem of deriving or estimating the normalising constant, as the constant term is eliminated taking the gradient $\nabla_{\mathbf{x}} \log p(\mathbf{x}) = \nabla_{\mathbf{x}} \log \hat{p}(\mathbf{x})$ where $\hat{p}(\mathbf{x})$ is the un-normalised density. We can then exploit the score, i.e. plugging it into the Hamiltonian or Langevin dynamics or other gradient flow methods (e.g. SVGD), to arrive at an approximate parametric (classic VI) or non-parametric (MCMC and ParVI \footnote{Particle-based VI (ParVI) methods such as SVGD \cite{liu_stein_2019}, SMC \cite{doucet_introduction_2001}, EVI \cite{wang_particle-based_2021}, EParVI \cite{huang_electrostatics-based_2024}, MPM-ParVI \cite{huang_variational_2024-1}, SPH-ParVI \cite{huang_variational_2024}, are gaining popularity.}) representation of the target density $p(\mathbf{x})$.

MCMC methods are equipped with theoretical guarantees (e.g. invariant distribution and ergodicity \cite{o_ruanaidh_numerical_1996}), given proper setup and sufficiency of simulation time. VI methods, either approximating the target empirically (ParVI) or parametrically, are generally fast. Both classes of methods can be used in low to high dimensions, however, they vary in their accuracy \footnote{The accuracy of approximating a distribution can be measured using distributional distances, depending on if samples are available, such as KL divergence \cite{kullback_information_1951}, Jensen-Shannon divergence \cite{lin1991divergence}, Wasserstein distances \cite{villani_optimal_2009}, kernelized stein discrepancy \cite{liu_kernelized_2016}, maximum mean discrepancy \cite{gretton2007kernel}, etc.} and efficiency. MCMC approaches generally require long-time runs \cite{geyer_practical_1992,dasgupta_correlation_2000,neal1993probabilistic,brooks2011handbook} and possibly some kind of tuning of hyper-parameters such as discretization step size \cite{homan_no-u-turn_nodate}, while VI methods can be less accurate \cite{turner2011two,blei_variational_2017,wainwright_graphical_2008}. 

\paragraph{Function approximation}
Function approximation is a fundamental technique in statistics and machine learning \footnote{The universal approximation theorem \cite{cybenko_approximation_1989,hornik_multilayer_1989,augustine_survey_2024}, for example, states that feedforward neural networks can approximate any continuous function on a compact domain, given sufficient width and appropriate activation functions. Deep neural networks can also be universally used to approximate probability distributions \cite{lu_universal_2020}.} \cite{cybenko_approximation_1989,hornik_multilayer_1989,hornik_new_1993}, enabling the representation of complex functions such as probability density functions (PDFs), through simpler basis functions. A general function $f(\mathbf{x})$ can be approximated as $f_N(\mathbf{x}) = \sum_{i=1}^N w_i \cdot \phi_i(\mathbf{x}; \theta_i)$, where $\phi_i(\mathbf{x}; \theta_i)$ are basis functions (e.g. sigmoids, B-splines, or Fourier exponentials), $w_i$ are weights, and $\theta_i$ are parameters specific to each basis function.
For example, sigmoidal basis function \footnote{A sigmoid function (s-shaped) is typically a bounded, measurable function satisfying $ \phi(z) \to 1 $ as $ z \to \infty $ and $ \phi(z) \to 0 $ as $ z \to -\infty $, e.g. the logistic sigmoid $ \sigma(x) = \frac{1}{1 + e^{-x}} $.}. $\phi_i(a_k \cdot x + b_k)$, parameterized by $a_k \in \mathbb{R}^d$ and $b_k, c_k \in \mathbb{R}$, are widely used in feedforward neural networks to approximate arbitrary continuous functions on bounded sets \cite{barron1993universal}. B-splines, which are piecewise polynomials with local support, offer smooth and numerically stable approximations by controlling the degree and knot placement \cite{Deboor1978splines,yang_tracking_2023}. Fourier-based methods represent functions via integrals or discrete sums of complex exponentials, capturing periodic or complex patterns effectively. 
Gaussian mixture models (GMMs), or in general mixture distributions, are a widely used statistical model for density approximation - they can approximate arbitrary probability density functions (PDFs) with sufficient number of components \cite{li_estimation_1999,norets_approximation_2010,morningstar_automatic_2020}. GMMs, tracing back to the early work of Pearson \cite{pearson1894contributions}, with the basis $\phi_i(\mathbf{x}; \theta_i)=\mathcal{N}(\mathbf{x};\boldsymbol{\mu}_i,\Sigma_)$ with normalised non-negative weights $w_i$, admit multi-modes and have been widely used to approximate complex densities \cite{bishop_pattern_2006,blei_variational_2017}. Research has been focused on developing efficient and accurate estimation techniques, e.g. maximum likelihood and the expectation-maximization (EM) algorithm \cite{dempster_maximum_1977}, and their convergence properties \cite{west1993approximating,li_mixture_1999,balakrishnan_statistical_2017,yan_convergence_2017,zhao_statistical_2020,dasgupta_learning_nodate,segol_improved_2021}. These approximation techniques are central to density estimation and sampling addressed in this paper.   

Machine learning-based black-box and non-parametric methods can also be used for function approximation (e.g. approximating a regressor \cite{huang_rl_2025} or classifier \cite{huang2025rlclassifier}); they can learn complex mappings from inputs to outputs without assuming a fixed, predefined functional form \cite{bishop_pattern_2006}. Non-parametric approaches such as kernel regression \cite{nadaraya_estimating_1964,watson_smooth_1964}, Gaussian processes \cite{rasmussen_gaussian_2004}, and splines \cite{Deboor1978splines,yang_tracking_2023} allow the model complexity to grow with the data, adapting flexibly to intricate (e.g. highly non-linear) patterns while trading interpretability for adaptability. Many machine learning models, e.g. neural networks \cite{hornik_multilayer_1989}, tree-based models \cite{breiman_classification_2017}, and support vector machines \cite{cortes_support-vector_1995}, operate as black-box function approximators \cite{lu_universal_2020}, capturing high-dimensional, nonlinear relationships directly from data but may offer limited transparency into their internal workings. For example, a regression function can be approximated by a reinforcement learning model which learns a policy to map input to output \cite{huang_rl_2025}.

\paragraph{GMM based sampling}
Bringing together the ideas of PDF approximation and sampling, one naturally comes up with the idea of fitting a complex, possibly unnormalised $p(\mathbf{z})$ using a GMM
\footnote{
To ensure $q_{\boldsymbol{\theta}}(\mathbf{z})$ a valid PDF, it should satisfy     
$q_{\boldsymbol{\theta}}(\mathbf{z}) \geq 0$ 
and $\int_{\mathbb{R}^d} q_{\boldsymbol{\theta}}(\mathbf{z}) d\mathbf{z} = \int_{\mathbb{R}^d} \sum_{i=1}^N w_i \mathcal{N}(\mathbf{z} \mid \boldsymbol{\mu}_i, \boldsymbol{\Sigma}_i) d\mathbf{z} = \sum_{i=1}^N w_k \int_{\mathbb{R}^d} \mathcal{N} (\mathbf{x} \mid \boldsymbol{\mu}_i, \boldsymbol{\Sigma}_i) d\mathbf{z} = \sum_{k=1}^K w_k = 1$.
} 
$q_{\boldsymbol{\theta}}(\mathbf{z}) = \sum_{i=1}^N w_i \mathcal{N}(\mathbf{z} | \boldsymbol{\mu}_i, \boldsymbol{\Sigma}_i)$, i.e. $p(\mathbf{z}) \approx q_{\boldsymbol{\theta}}(\mathbf{z})$, then finding a way to generate samples from $q_{\boldsymbol{\theta}}(\mathbf{z})$, hoping that sampling from the approximate density can either be easier, more accurate or advantageous. A standard routine for generating samples from $q_{\boldsymbol{\theta}}(\mathbf{z})$ is again using Monte Carlo methods \cite{newman_fast_2025}, which doesn't take advantage of the convenience of sampling from each Gaussian component. In fact, it adds an extra layer of complexity compared to direct sampling $p(\mathbf{z})$. The question them arises: how to efficiently sample the surrogate density $q_{\boldsymbol{\theta}}(\mathbf{z})$? We come up with the idea that, we first generate samples from each Gaussian component, then ensemble these samples with stratified sampling \footnote{An analysis of stratified sampling can be found in Appendix.\ref{app:simple_and_stratified_random_sampling}.}, i.e. each sample in the pool has a component label and its probability to be sampled is proportional to its (normalised) component weight (samples come from the same Gaussian component is uniformly sampled). It turns out that, these ensembled samples resemble the target shape, i.e. they are indeed samples from the target density. This two-stage sampling strategy, while being simple and straightforward, has not been discussed elsewhere to the best of the author's knowledge. The contribution of this Gaussian mixture approximation (GMA) based sampling method lies in its novelty: it's simple, flexible \footnote{Being simple and flexible sometimes contradict with each other. For example, having too many tunable hyper-parameters enables a method flexibly adapt to different scenarios, while making it more complicated and less usable. Simplicity of this GMA sampling method refers to its simple logic and ease of implementation, while flexibility points to the fact that the user has the space for moving the Gaussian components around (i.e. adjusting the means and variances to adapt to specific tasks) in the initialisation stage.}, efficient, and accurate \footnote{Claiming high efficiency and accuracy at the same time, again, can be dangerous, as there is no free lunch. However, as we observe from later experiments, GMA sampling does offer reasonably well (moderate) efficiency and accuracy.}. 

This work introduces a Gaussian mixture approximation (GMA) based method for sampling arbitrary density (e.g. unnormalised PDFs). Based on the universal density approximation property of GMMs \cite{li_mixture_1999}, we issue a GMM with finite, fixed components, sample each component, and fit the target density by only optimising the weights via minimizing the KL divergence (evaluated at fixed sample positions). Then we re-sample the samples as per the categorical distribution of the optimised weights using stratified sampling - these samples is guaranteed to be distributed as per the target distribution. The GMA sampling method features a set of fixed Gaussian components of user choice $\mathcal{N}(\mathbf{z} | \boldsymbol{\mu}_i, \boldsymbol{\Sigma}_i)$, as well as drawing finite, fixed samples from each Gaussian component for discrepancy evaluation. Re-sampling the sample pool is easy and straightforward: the optimised weights are used to weight each basket, and within a basket the samples are uniformly sampled. This GMA sampling method differs itself from previous work by its simplicity, efficacy and efficiency. Empirical and theoretical evidences confirm its validity.

\section{WGMA sampling: methodology} \label{sec:methodology}

Here we propose a simple weights-only GMA method, balancing efficiency and efficacy, for inferring the target. The idea is intuitive: we exploit the fact that, drawing samples from Gaussian (or more strictly, standard Gaussian) densities are fast and straightforward \footnote{Generating samples from Gaussian are both theoretically and practically mature for low to high dimensional Gaussians. Direct sampling methods such as Box-Muller transform (for univariate Gaussians) which transforms two uniform random variables into two standard normal variables:
$$
z_1 = \sqrt{-2\log U_1} \cos(2\pi U_2), \quad z_2 = \sqrt{-2\log U_1} \sin(2\pi U_2)
$$
where $U_1, U_2 \sim \mathcal{U}(0,1)$.  
Also, Cholesky decomposition: for $\boldsymbol{z} \sim \mathcal{N}(\boldsymbol{\mu}, \Sigma)$, sample $\boldsymbol{\epsilon} \sim \mathcal{N}(0, I)$, then:
$$
\boldsymbol{z} = \boldsymbol{\mu} + L\boldsymbol{\epsilon}, \quad \text{where } L \text{ is such that } LL^\top = \Sigma
$$
Others include eigen decomposition and reparameterization trick (as used in variational inference).
In practice, one can use \textit{np.random.randn()} in NumPy, or \textit{torch.randn()} in PyTorch, to sample from normals.}, and approximate the unnormalised target $\bar{p}(\mathbf{z})$ using a dynamically evolving Gaussian mixture model (GMM \cite{dempster_maximum_1977,dempster_maximum_1977,hand_mixture_1989}) with $N$ Gaussian components:

\begin{equation} \label{eq:GMM}
    q_{\mathbf{w}}^{(k)}(\mathbf{z}) = \sum_{i=1}^N w_{i}^{(k)} \cdot \mathcal{N}(\mathbf{z}; \boldsymbol{\mu}_i,\Sigma_i)
\end{equation}
in which $q_{\mathbf{w}}^{(k)}$ represents the configuration of the GMM at iteration $k$. $\mathbf{w} = [w_1^{(k)}, w_2^{(k)}, \ldots, w_N^{(k)}]^\top \in \mathbb{R}^N$ is the weight vector, with $0 \leq w_i^{(k)} \leq 1$ for all $ i = 1, 2, \ldots, N $, and the weights are subject to the normalization constraint $\sum_{i=1}^N w_i^{(k)} = 1$ to ensure $ q_{\mathbf{w}}^{(k)}(\mathbf{z}) $ is a valid probability density function.
$\mathcal{N}(\mathbf{z}; \boldsymbol{\mu}_i, \Sigma_i)$ is the probability density function of a multivariate normal distribution with mean $\boldsymbol{\mu}_i$ and covariance matrix $\Sigma_i$.

We first draw $M$ samples $\{\mathbf{s}_{i,j}\}, i=1,2,...,N, j=1,2,...,M$ from each Gaussian component $\mathcal{N}(\mathbf{z}; \boldsymbol{\mu}_i,\Sigma_i)$. Unlike particle-based VI methods which evolve a set of samples (particles) towards the target, we fix these sample positions, and dynamically optimise the GMM weights $\{w_{i}^{(k)}\}_{i=1}^N$. The initial weights $\{w_{i}^{(0)}\}_{i=1}^N$ can be randomly or uniformly initalised, and we aim to evolve $\mathbf{w}$ to minimize the \textit{exclusive} \footnote{Minimizing the reverse/exclusive $KL(q_{\mathbf{w}} \| p)$ w.r.t. $\mathbf{w}$ yields the \textit{zero-forcing/mode-seeking} behaviour, while minimizing the forward/inclusive $KL(p \| q_{\mathbf{w}})$ yields the \textit{mass-covering/mean-seeking} behaviour. See e.g. \cite{Le2017reversekl,vieira2014kldivergence} for an explanation.} KL divergence between the GMM and the target (see Appendix.\ref{app:gradient_of_exclusive_KL_divergence} for details):

\begin{equation} \tag{cc.Eq.\ref{eq:KL_divergence_objective4}} \label{eq:KL_divergence_objective4_cc}
    \mathbf{w}^* = \arg\min_{\mathbf{w}}   KL\left(q_{\mathbf{w}}(\mathbf{z}) \| p(\mathbf{z})\right)
    =\arg\min_{\mathbf{w}} \left[  \sum_{\mathbf{z}} q_{\mathbf{w}}(\mathbf{z}) \log q_{\mathbf{w}}(\mathbf{z}) - \sum_{\mathbf{z}} q_{\mathbf{w}}(\mathbf{z}) \log \bar{p}(\mathbf{z}) \right]
\end{equation}

Each gradient element $\nabla_{w_i} KL$, estimated using $M$ samples $\{\mathbf{s}_{i,j}\}$ drawn from $\mathcal{N}_i$, is:
\begin{equation} \tag{cc.Eq.\ref{eq:KL_divergence_objective4_gradient_GMM1}} \label{eq:KL_divergence_objective4_gradient_GMM1_cc}
    \nabla_{w_i} KL \approx 1 + \sum_{j=1}^M \mathcal{N}(\mathbf{s}_{i,j}; \boldsymbol{\mu}_i, \Sigma_i) [\log q_{\mathbf{w}}(\mathbf{s}_{i,j}) - \log \bar{p}(\mathbf{s}_{i,j})] 
\end{equation}

Recognising the second term in the above is a probability weighted discrepancy, it can be computed using a Monte Carlo estimator:
\begin{equation} \tag{cc.Eq.\ref{eq:KL_divergence_objective4_gradient_GMM1_MC}} \label{eq:KL_divergence_objective4_gradient_GMM1_MC_cc}
    \nabla_{w_i} KL \approx 1 + \frac{1}{M} \sum_{j=1}^M \left[ \log q_{\mathbf{w}}(\mathbf{s}_{i,j}) - \log \bar{p}(\mathbf{s}_{i,j}) \right]    
\end{equation}
with $\mathbf{s}_{i,j} \sim \mathcal{N} (\mathbf{z};\boldsymbol{\mu}_i, \Sigma_i)$. 
By this, we are essentially minimizing the KL divergence between these two distributions evaluated at the pre-sampled, fixed sample positions $\{\mathbf{s}_{i,j}\}_{i=1,2,...,N}^{j=1,2,...,M}$.
We observe that, using the exclusive KL divergence objective is convenient as we can plug in the un-normalised target $\bar{p}(\mathbf{z})$, and it is convex \footnote{The KL divergence objective is nonconvex in general; however, in the case of GMMs, it is convex. See Appendix.\ref{app:gradient_of_exclusive_KL_divergence} a statement of convexity.}. The disadvantage being, minimizing the exclusive KL divergence can miss some isolated modes (see \cite{Le2017reversekl} for a discussion).

Collecting all $w_i$ into $\mathbf{w}$, we can numerically find the (locally) optimal $\mathbf{w}^*$ using projected gradient descent
\footnote{For constrained optimisation problem, typically \textit{projected gradient descent} or\textit{ Lagrange multipliers} are used to enforce the $\sum w_i = 1$ constraint. If we use a standard gradient descent followed by heuristics, i.e. clipping the weights (zero truncation) to ensure $w_i \geq 0$, and re-normalising all weights in each GD iteration to ensure $\sum_{i=1}^N w_i=1$, it won't be the same as a true projection onto the probability simplex. Re-normalizing the weights after the update step alters the gradient direction and invalidates the formal convergence guarantees of standard gradient descent for constrained convex problems. A GD based, heuristic GMA sampling algorithm is implemented in Algo.\ref{algo:GMA-sampling-heuristic}.}
(pGD \footnote{pGD combines the standard gradient descent (GD) step with a projection onto a feasible set defined by the constraints. Essentially, it takes a step in the direction of the negative gradient (as in regular gradient descent) and then 'projects' the result back into the feasible region if it falls outside. This ensures that all intermediate and final solutions satisfy the given constraints.}), which consists of 2 steps:
a \textit{standard gradient descent step} is taken to find an intermediate vector $\mathbf{v}^{(k)}$:
\begin{equation}
    \mathbf{v}^{(k)} = \mathbf{w}^{(k-1)} - \eta_k \cdot \nabla_{\mathbf{w}} KL(q_{\mathbf{w}^{(k-1)}}(\mathbf{z}) \| p(\mathbf{z}))
\end{equation}
where $\eta_k$ is the learning rate in $k$-th iteration. To satisfy the \textit{Robbins-Monro conditions} \cite{robbins1951stochastic} (\(\sum \eta_k = \infty\), \(\sum \eta_k^2 < \infty\)), we use a diminishing \footnote{The use of a diminishing (decaying) step size scheme can be found in optimisation literature e.g. \cite{bottou_optimization_2018} and similar Gaussian mixture inference framework \cite{guo_boosting_2017}. Our experimental trials show that constant learning rate also works.} learning rate  $\eta_k = \frac{\eta_0}{k}$ to ensure the weights converge to a local minimum of the convex, exclusive $KL$ divergence objective.
$\mathbf{w}^{(k-1)}= [w_{1}^{(k-1)}, \dots, w_{N}^{(k-1)}]^\top$ is the weight vector from the previous iteration.
Then a \textit{projection step}, where the intermediate vector $\mathbf{v}^{(k)}$ is projected back onto the feasible set, which is the probability simplex $\Delta$:
\begin{equation}
    \mathbf{w}^{(k)} = \text{Proj}_{\Delta}(\mathbf{v}^{(k)})
\end{equation}
where the \textit{projection} operator $\text{Proj}_{\Delta}$ on $\Delta$ is defined as
\begin{equation}
    \text{Proj}_{\Delta}(x) = \arg\min_{x' \in \Delta} \| x - x' \|
\end{equation}
with $\| \cdot \|$ denotes the Euclidean norm, the set of constraints lie in the set $\Delta \in \mathbb{R}^d$.
In the GMM optimisation case, the projection operator $\text{Proj}_{\Delta}$ finds the closest point in the specific probability simplex
\footnote{The notation $\Delta^{N-1}$ denotes the $(N-1)$-dimensional probability simplex, i.e.\ the set of all 
$\mathbf w \in \mathbb R^N$ with $w_i \ge 0$ and $\sum_{i=1}^N w_i = 1$. 
Although $\mathbf w$ has $N$ components, the sum-to-one constraint removes one degree of freedom, leaving an $(N-1)$-dimensional space.}
$\Delta^{N-1} = \{ \mathbf{w} \in \mathbb{R}^N \mid w_i \ge 0 \text{ and } \sum_{i=1}^N w_i = 1 \}$ to the intermediate vector. This projection step ensures that the updated weight vector $\mathbf{w}^{(k)}$ strictly satisfies the necessary constraints at every iteration.

After the weight vector converges to $\mathbf{w}^*$, we ensemble these $N \times M$ samples and use \textit{stratified sampling} \cite{morningstar_automatic_2020} to decide which samples to be selected to approximate the samples drawn from the target distribution, i.e. we re-draw $N \times M$ samples from these $N$ bins with repetition, and each bin has probability $0 \leq w_i \leq 1$ to be selected. Within each bin, all $M$ samples have equal probability $1/M$ to be chosen. 

As a result, we have the two-stage GMA sampling method: we first sample $M$ samples from each fixed Gaussian basis (totalling $NM$ samples), and then re-sample them using the optimised weights. This method is summarised in the below pseudo-code, and Algo.\ref{algo:WGMA-sampling-pgd} which implements GMA sampling for inferring the target distribution $ p(\mathbf{z}) $ using a $N$-component GMM with dynamically evolving weights.
\\

\fbox{%
\begin{minipage}{0.95\linewidth}
\begin{center}
\textbf{WGMA sampling (vanilla, pseudo-code)}
\end{center}

\texttt{Step 1. Initialise GMM and sample bank.} \\[4pt]
Create $N$ Gaussian components 
$\{\mathcal N(\mu_i,\Sigma_i)\}_{i=1}^N$ (e.g. Uniform/Normal, or informed by a trained mode), 
draw $M$ fixed samples from each component. \\

\texttt{Step 2. Optimise weights.} \\[4pt]
Update $\mathbf w \in \Delta^{N-1}$ 
by gradient descent (e.g. pGD/MD) to minimise $KL(q_{\mathbf w} \| p)$, with 
$q_{\mathbf w}(z)=\sum_{i=1}^N w_i \mathcal N(z;\mu_i,\Sigma_i)$. \\

\texttt{Step 3. Re-sample (via stratified sampling).} \\[4pt]
Form an ensemble by stratified sampling: 
pick component $i \sim \mathrm{Cat}(\mathbf w)$ and a stored sample 
within its $M$ draws; repeat to obtain the final set.
\end{minipage}
}

\begin{algorithm}[H]
\footnotesize
\caption{WGMA-sampling: sampling via Gaussian mixture approximation (with pGD)}
\label{algo:WGMA-sampling-pgd}
\textbf{Input:} Number of Gaussian components $N$; number of samples per component $M$; number of iterations $K$; target unnormalised density \(\bar{p}(\mathbf{z})\); initial means \(\{\boldsymbol{\mu}_i\}_{i=1}^N\); initial covariance matrices \(\{\Sigma_i\}_{i=1}^N\); initial learning rate \(\eta_0\). \\
\textbf{Output:} Ensemble of selected samples \(\{\mathbf{s}_{\text{selected}}\}\) approximating samples from $p(\mathbf{z})$.

\vspace{1mm}\hrule\vspace{1mm}

\begin{algorithmic}[1]
\STATE Initialize an empty set \(\mathcal{S} = \{\}\) to store selected samples. \hfill\textit{\(\mathcal{O}(1)\)}
\STATE Initialize weight vector \(\mathbf{w}^{(0)}\) on the probability simplex, e.g. $w_i^{(0)} \sim \mathcal{U}(0, 1)$ and normalize. \hfill\textit{\(\mathcal{O}(N)\)}
\STATE Draw $M$ samples \(\{\mathbf{s}_{i,j}\}_{j=1}^M\) from each Gaussian \(\mathcal{N}(\boldsymbol{\mu}_i, \Sigma_i)\) for $i = 1, 2, \ldots, N$ using standard Gaussian sampling (e.g. \(\boldsymbol{\epsilon} \sim \mathcal{N}(0, I)\), \(\mathbf{s}_{i,j} = \boldsymbol{\mu}_i + L_i \boldsymbol{\epsilon}\), where $L_i L_i^\top = \Sigma_i$). \hfill\textit{\(\mathcal{O}(N M d^2)\), where $d$ is dimension}
\FOR{$k = 1$ to $K$}
    \STATE Compute gradient vector \(\mathbf{g} = [g_1, g_2, \ldots, g_N]^\top\):
    \FOR{$i = 1$ to $N$}
        \STATE Compute GMM density \(q_{\mathbf{w}}^{(k-1)}(\mathbf{s}_{i,j}) = \sum_{i=1}^N w_l^{(k-1)} \cdot \mathcal{N}(\mathbf{s}_{i,j}; \boldsymbol{\mu}_l, \Sigma_l)\) for all $j = 1, \ldots, M$. \hfill\textit{\(\mathcal{O}(N M d^2)\)}
        \STATE Compute gradient component for $w_i$ (\ref{eq:KL_divergence_objective4_gradient_GMM1_cc}): 
        \[
        g_i = 1 +  \sum_{j=1}^M \mathcal{N}(\mathbf{s}_{i,j}; \boldsymbol{\mu}_i, \Sigma_i) [\log q_{\mathbf{w}}^{(k-1)}(\mathbf{s}_{i,j}) - \log \bar{p}(\mathbf{s}_{i,j})]
        \]
        \hfill\textit{\(\mathcal{O}(M d^2)\)}
    \ENDFOR
    \STATE Take gradient descent: \(\mathbf{v}^{(k)} = \mathbf{w}^{(k-1)} - \frac{\eta_0}{k} \cdot \mathbf{g}\). \hfill\textit{\(\mathcal{O}(N)\)}
    \STATE Project onto simplex: \(\mathbf{w}^{(k)} = \text{Proj}_{\Delta}(\mathbf{v}^{(k)})\), where \(\text{Proj}_{\Delta}\) is the projection onto the set \(\{\mathbf{w} | w_i \ge 0, \sum w_i = 1\}\). \hfill\textit{\(\mathcal{O}(N \log N)\)}
\ENDFOR
\STATE Set final weights \(\mathbf{w}^* = \mathbf{w}^{(K)}\), and component selection probabilities \(\mathbf{p} = \mathbf{w}^*\). \hfill\textit{\(\mathcal{O}(1)\)}
\STATE Generate ensemble samples: for each $m = 1$ to $N \cdot M$, draw index $i_m \sim \text{Categorical}(\mathbf{p})$ and append $\mathbf{s}_{i_m, j_m}$ (where $j_m \sim \text{Uniform}(\{1, \ldots, M\})$) to \(\mathcal{S}\). \hfill\textit{\(\mathcal{O}(N M)\)}
\STATE Return the set \(\mathcal{S}\). \hfill\textit{\(\mathcal{O}(1)\)}
\end{algorithmic}
\end{algorithm}
{\footnotesize
\noindent\textbf{Note:} for the outer loop, instead of specifying a fixed number of $K$ iterations, we can use early stop with criteria such as $\|\mathbf{w}^{(k)}-\mathbf{w}^{(k-1)}\| < \epsilon$.
}

\paragraph{Complexity}
The overall computational complexity of the GMA-sampling algorithm is $\mathcal{O}(K N^2 M d^2)$, which is dominated by the nested loops within the iterative weight update (Step 4). This cost arises from two main factors. First, evaluating a single multivariate Gaussian PDF for a $d$-dimensional sample costs $\mathcal{O}(d^2)$ due to the matrix operations in the \textit{Mahalanobis distance}. Second, it requires a loop through $N$ components (the $i$ loop), and for each of the $M$ samples within, the mixture density calculation $q$ requires another sum over all $N$ components (the $l$ index). This nested $N \times N$ operation within the $k$-loop results in the $N^2$ term.

The initial sample generation (Step 3) contributes a one-time cost \footnote{Assuming $\Sigma\_i$ is provided, and its Cholesky factorization ($L\_i$) is precomputed. If computed within Step 3, the cost would include $\mathcal{O}(N d^3)$ for factorization, but this is a one-time cost.} of $\mathcal{O}(N M d^2)$, and the final ensemble sampling (Step 14) adds $\mathcal{O}(N M)$, since each re-sampled index can be obtained by a single categorical draw and a uniform sub-index selection. For $K>1$, both are subsumed by the primary loop's cost. The majority of the computational expense is incurred in the repeated density and gradient computations, which scale quadratically with both the number of components $N$ and the dimensionality $d$. The cost of the gradient projection step ($\mathcal{O}(N \log N)$) is subsumed by the much more expensive mixture density calculation ($\mathcal{O}(NMd^2)$) within each iteration.

Memory complexity is $\mathcal{O}(N M d + N d^2 + N^2M)$, primarily due to storing the $N \cdot M$ samples of dimension $d$, and the $N$ $d$-dimension covariance matrices. $N d^2$ can be large if e.g. we are inferring the $d$ parameters in a neural network. For Storage for the weight and gradient vectors is negligible at $\mathcal{O}(N)$.

\section{GMA variants} \label{sec:GMA_variants}

While the vanilla GMA sampling method in Algo.\ref{algo:WGMA-sampling-pgd} is effective, we further propose some improved variants which trade-off accuracy and efficiency. 

\subsection{Speed acceleration}

Three strategies are identified to accelerate the sampling speed: 

\paragraph{Single or batch sample gradient estimator}
When implementing the gradient of the KL divergence (\ref{eq:KL_divergence_objective4_gradient_GMM1_cc}), instead of summing over all samples within the same Gaussian component, we can accelerate the inference by estimating the gradient $g_i$ using either a stochastic or mean sample. 
For example, randomly picking one sample $s$ from the sample pool $\mathbf{s}_i=\{\mathbf{s}_{i,j}\}_{j=1}^M$ of the Gaussian cluster $i$:
\[
g_i = \left[ \nabla_{\mathbf{w}} KL \right]_i \approx 1 + \mathcal{N}(s; \boldsymbol{\mu}_i, \Sigma_i) [\log q_{\mathbf{w}}(s) - \log \bar{p}(s)]
\]
where $s$ is randomly (with probability $1/M$, i.e. uniformly) sampled from $\mathbf{s}_i=\{\mathbf{s}_{i,j}\}_{j=1}^M$. After this, one can update $v_i^{(k)} = w_i^{(k-1)} - \eta \cdot g_i$ following projected gradient descent on. This turns the full-batch gradient calculation into a stochastic one with a batch size of one. The cost of the gradient calculation loop would be reduced by a factor of $M$, which is a significant speed-up. However, it also results in high variance and bias, i.e. using a single sample introduces an enormous amount of noise into the gradient estimate. The optimisation process can become very unstable, struggling to converge and oscillating erratically around a minimum. The random single sample estimator is also biased.

Alternatively, we can just take a mean sample $s=\frac{1}{M} \sum_{j=1}^M \mathbf{s}_{i,j}$ within the same Gaussian cluster $i$, which reduces the cost to computing a single mean and then a single gradient evaluation per component. However, this also introduces severe bias, as we cannot interchange the expectation (approximated by the sum) with non-linear functions - both the logarithm ($\log q_{\mathbf{w}}(\mathbf{z})$) and the Gaussian PDF ($\mathcal{N}(\mathbf{z}; \dots)$) are non-linear. In general, for a non-linear function $f$: $f(\mathbb{E}[\mathbf{z}]) \neq \mathbb{E}[f(\mathbf{z})]$. By using the mean sample, we are calculating $f(\mathbb{E}[\mathbf{s}_{i}])$ instead of an approximation of $\mathbb{E}[f(\mathbf{s}_{i})]$. This introduces a large, uncontrolled bias that completely changes the optimisation landscape. It effectively collapses all information about the distribution of the $M$ samples into a single point, defeating the purpose of using them to approximate an integral.

A theoretically sound and common way to achieve a speed-up is to use mini-batching, as used in stochastic gradient descent (SGD). Instead of using just one sample, we randomly sample a small batch of $B$ samples (where $1 \ll B < M$) from the pool of $M$ samples within the same Gaussian cluster at each iteration, then the gradient estimator becomes:
$g_i \approx 1 + \frac{1}{B} \sum_{s \in \text{batch}} \left[ \log q_{\mathbf{w}}(s) - \log \bar{p}(s) \right]$, It provides an unbiased estimate of the true gradient while significantly reducing computational cost, and having much lower variance than a single-sample estimate.

\paragraph{Pre-computing the PDFs}.
Since the samples $\mathbf{s}_{i,j}$ are generated once and are fixed, the values of the Gaussian PDFs $\mathcal{N}(\mathbf{s}_{i,j};\boldsymbol{\mu}_l,\Sigma_l)$, as well as the target PDFs $\bar{p}(\mathbf{s}_{i,j})$, are constant throughout all optimisation iterations. By pre-computing these values outside the main iterative loop, we can avoid the most expensive calculations (evaluating the Gaussian PDFs costs $d^2$) being repeated at every iteration. This can dramatically improves computational performance. Based on this, we present an improved algorithm in Algo.\ref{algo:GMA-sampling-precomputing-pgd} in Appendix.\ref{app:precomputing_density_values}.

\paragraph{Monte Carlo gradient estimator}
As mentioned before, we can further save (although marginal) computational efforts by replacing the gradient estimator \ref{eq:KL_divergence_objective4_gradient_GMM1_cc} by the Monte Carlo estimator \ref{eq:KL_divergence_objective4_gradient_GMM1_MC_cc}, avoiding the multiplication of the normal PDFs $\mathcal{N}(\mathbf{s}_{i,j}; \boldsymbol{\mu}_i, \Sigma_i)$ in Step 8 (these PDFs are evaluated when computing $q_{\mathbf{w}} (\mathbf{s}_{i,j})$ anyway), which presents Algo.\ref{algo:GMA-sampling-MCEstimator} in Appendix.\ref{app:MC_gradient_estimator}.

\paragraph{Combining PDFs pre-computation and MC gradient estimation} Combining both strategies yields the optimal GMA sampling method which is theoretically sound and practically efficient, as presented in Algo.\ref{algo:GMA-sampling-optimal} in Appendix.\ref{app:combine_precomputing_and_MC_gradient_estimator}.

\subsection{Stabilizers: entropy regularization, tempering, convex mixing (momentum), Polyak averaging} \label{subsec:stabilizers}

Although mirror descent resolves the simplex constraint elegantly, the algorithm can still suffer from weight collapse, where nearly all mass concentrates onto a single component, leading to poor posterior approximation and unstable forecasts. To address this, several stabilisation techniques are introduced (see Appendix.\ref{app:mirror_descent}): 

(i) \textit{tempering}, where the target log-density is flattened during early iterations by introducing a temperature parameter $\beta \in (0,1)$:
\[
\log \bar{p}_\beta(\mathbf{z}) = \beta   \log \bar{p}(\mathbf{z})
\]
which reduces gradient magnitudes and allows weights to adapt more gradually. Annealing $\beta \to 1$ restores the true posterior.

(ii) \textit{entropy regularization}, which augments the optimisation objective with an entropy penalty $-\lambda H(\mathbf{w}) = \lambda \sum_i w_i \log w_i$, leading to an additional gradient term (derivations see Appendix.\ref{app:mirror_descent})
\[
g_i^{\text{entropy}} = \lambda (1 + \log w_i)
\]
that counteracts collapse by pushing weights toward a more uniform distribution.

(iii) \textit{convex mixing} (or momentum), where the updated weights $\mathbf{w}^{(k)}$ are smoothed with the previous iterate to damp oscillations:
\[
\mathbf{w}^{(k)}  \leftarrow  \alpha   \mathbf{w}^{(k)} + (1-\alpha)   \mathbf{w}^{(k-1)}, 
\quad \alpha \in (0,1)
\]
so that large jumps in weight allocations are suppressed and stability is improved.

(iv) \textit{Polyak (iterative, tail) averaging} \cite{granziol_iterative_2021}, where instead of taking the last iterate, the final weights are averaged across the last $L$ iterations:
\[
\bar{\mathbf{w}}^{(K)} = \frac{1}{L} \sum_{k=K-L+1}^{K} \mathbf{w}^{(k)}
\]
which reduces variance, smooths out noisy fluctuations, and prevents single-iteration instabilities from dominating the outcome.

Together these modifications yield an anti-collapse mirror descent GMA sampler that maintains diversity across mixture components, ensures robustness to noisy gradients, and delivers more reliable posterior samples for downstream forecasting. 

After applying these strategies, the effective optimisation problem at iteration $k$ becomes
\[
\min_{\mathbf{w}\in\Delta}   
\Bigg\{ 
    \underbrace{\Big\langle g^{(k)}, \mathbf{w} \Big\rangle}_{\text{gradient term}} 
    + \tfrac{1}{\eta_k} \underbrace{\mathrm{KL}  \big(\mathbf{w} \| \mathbf{w}^{(k)}\big)}_{\text{mirror descent prox}}
    + \underbrace{\lambda_k \sum_{i=1}^N w_i \log w_i}_{\text{entropy regularisation}}
\Bigg\}
\]
where 
\begin{itemize}
    \item $g^{(k)}$ is the stochastic gradient evaluated on tempered targets $\beta_k \log \bar{p}(\mathbf{z})$,
    \item the KL term enforces the mirror descent geometry and temperature scaling $\tau_k$,
    \item the entropy penalty $\lambda_k$ counteracts collapse, 
    \item convex mixing $\alpha_k$ smooths the raw update $\mathbf{w}^{\text{prop}}$, and
    \item Polyak averaging forms the final output $\bar{\mathbf{w}}^{(K)}$ across the tail of the trajectory.
\end{itemize}
This composite objective highlights how each stabilisation schedule enters explicitly into the optimisation: tempering rescales the target gradients, entropy regularisation shapes the objective, and convex mixing / Polyak averaging modify the update rule.

\subsection{Weights optimisation via mirror descent} 

An alternative to projected gradient descent is to employ a mirror descent (MD) update with KL geometry, also known as the multiplicative weights update (see Appendix.\ref{app:mirror_descent}). This replaces the Euclidean projection step by an entropic mirror map, ensuring the weights remain strictly positive and normalized on the simplex. The update takes the form $w_i^{(k)} \propto w_i^{(k-1)} \exp(-\eta_k g_i^{(k)})$, which can be interpreted as a natural-gradient step in the space of probability vectors. Besides being more principled, it also reduces the per-iteration complexity by avoiding the $\mathcal{O}(N \log N)$ projection step. In practice, the square-root decaying schedule $\eta_k = \eta_0 / \sqrt{k + k_0}$ is adopted to maintain stability in the presence of noisy Monte Carlo gradients, even though it no longer satisfies the Robbins-Monro conditions. This mirror descent variant (Algo.\ref{algo:GMA-sampling-mirror} in Appendix.\ref{app:mirror_descent_GMA}) inherits the same pre-computation and gradient estimation advantages as the pGD algorithm, while being slightly more efficient and more natural for probability distributions.

\paragraph{Mirror-descent update with $\tau_k,\ \lambda_k,\ \alpha_k$ and tempering $\beta_k$.}
Let $g^{(k)}$ denote the stochastic gradient computed on the tempered target
$\beta_k \log \bar{p}(\mathbf{z})$ (with the standardization used in our GMA estimator; see Appendix.\ref{app:mirror_descent}).
Augment it by the entropy gradient:
\[
\tilde{g}^{(k)}_i = g^{(k)}_i  +  \lambda_k\big(1 + \log w^{(k)}_i\big)
\]
A single KL-prox mirror step with temperature $\tau_k \ge 1$ yields the
\textit{exponentiated-gradient} proposal
\[
w^{\text{prop}}_i
=
\frac{\exp  \left(\frac{\log w^{(k)}_i - \eta_k \tilde{g}^{(k)}_i}{\tau_k}\right)}
     {\sum_{j=1}^N \exp  \left(\frac{\log w^{(k)}_j - \eta_k \tilde{g}^{(k)}_j}{\tau_k}\right)}
=
\frac{\big(w^{(k)}_i\big)^{1/\tau_k} \exp  \left(-\frac{\eta_k}{\tau_k} \tilde{g}^{(k)}_i\right)}
     {\sum_{j=1}^N \big(w^{(k)}_j\big)^{1/\tau_k} \exp  \left(-\frac{\eta_k}{\tau_k} \tilde{g}^{(k)}_j\right)}
\]
Finally, apply convex mixing (momentum) to stabilize the iterate:
\[
\mathbf{w}^{(k+1)} = (1-\alpha_k) \mathbf{w}^{(k)}  +  \alpha_k \mathbf{w}^{\text{prop}}
\]
Here $\eta_k$ is the step size, $\tau_k$ flattens the update early (annealed to $1$), $\lambda_k$ controls the entropy regularization
(annealed to $0$), $\alpha_k$ dampens rapid changes via convex mixing, and $\beta_k$ (used inside $g^{(k)}$) anneals the target from a tempered surrogate to the true posterior. The whole workflow of MD, with these stabilization strategies applied, is described in Fig.\ref{fig:mirror-descent-flow}.

\subsection{Multi-stage GMA: progressively refining GMA sampling} \label{sec:refined_GMA}
To further improve inference accuracy for complex posteriors (e.g. the posterior distributions of the parameters in a Lotka-Volterra system, see Section.\ref{sec:LV_system}), we introduce a multi-stage \textit{refining GMA method}. A two-stage GMA as an example, we first perform a coarse exploration of the parameter space with a moderately broad Gaussian mixture bank, and then refine the mixture centres and variances based on the most informative components discovered from the first stage. Concretely, Stage.1 proceeds as in the baseline GMA sampling: we initialise $N$ Gaussian components with anisotropic diagonal covariances tailored to each parameter dimension, generate $M$ fixed samples per component, and optimise the component weights $\mathbf{w}$ via projected gradient descent (pGD) under the exclusive KL divergence. In Stage.2, we select the top-$k$ components with the largest posterior weights, and re-centre a new Gaussian bank around them with shrunk covariance scales to better capture the high-probability region of the posterior. A small jitter is added to the re-centred means to maintain diversity. This yields a refined GMM which is re-optimised with the same pGD scheme, now operating on a more localised sample bank.

\subsection{Laplace mixture approximation (LMA)} \label{subsec:LMA}

\paragraph{Laplace’s method as local Gaussian approximation.}
Laplace’s method\footnote{Laplace here refers to \textit{Pierre-Simon Laplace} and his asymptotic method for integrals, not the Laplacian operator.} approximates a sharply peaked, smooth integrand by a second-order expansion of its log around a mode $\hat{\mathbf w}$. Specifically, it fits a Gaussian to the logarithm of the function at its maximum, using a second-order Taylor expansion ($\nabla \log f(\hat{\mathbf w})=0$ at a mode):
$$
\log f(\mathbf{w}) \approx \log f(\hat{\mathbf{w}}) - \frac{1}{2} (\mathbf{w} - \hat{\mathbf{w}})^\top H (\mathbf{w} - \hat{\mathbf{w}})
$$
where $H = -\nabla^2 \log f(\mathbf{w}) \big|_{\mathbf{w} = \hat{\mathbf{w}}}$ is the (negative) Hessian at the mode. Exponentiating both sides yields a local Gaussian approximation of $f(\mathbf{w})$:
$$
f(\mathbf{w}) \approx f(\hat{\mathbf{w}}) \cdot \exp\left( -\frac{1}{2} (\mathbf{w} - \hat{\mathbf{w}})^\top H (\mathbf{w} - \hat{\mathbf{w}}) \right)
$$
Integrating this Gaussian gives the \textit{classic} Laplace approximation (LA) \cite{Mackay1998Choice}:
$$
\int \mathrm{d}^k \mathbf{w}   f(\mathbf{w}) 
\approx f(\hat{\mathbf{w}}) (2\pi)^{k/2} |H|^{-1/2}
$$

LA is widely adopted in Bayesian inference, especially for approximating marginal likelihoods (i.e. model evidences, normalising constant), because it provides a tractable way to approximate integrals that would otherwise require costly sampling. The resulting estimate balances model fit $f(\hat{\mathbf{w}})$ with a complexity penalty via the determinant $|H|$, offering a natural trade-off useful for model comparison tasks \cite{Mackay1998Choice}.

\paragraph{LMA as a variant of GMA}

Inspired by the classic, unimodal Laplace approximation, we approximate a target density by a multi-modal \textit{Laplace mixture}, built from local quadratic expansions at posterior modes. 
Given modes $\{\hat\theta_j\}_{j=1}^J$ found via e.g. multi-start optimisation of the log-unnormalised target $\ell(\theta)=\log \bar p(\theta)$, set
\[
\mu_j=\hat\theta_j,\qquad 
H_j = -\nabla^2 \ell(\hat\theta_j),\qquad 
\Sigma_j = H_j^{-1}
\]
One can add flattening and stabilisers: $\Sigma_j \leftarrow \kappa^2 \Sigma_j + \lambda I$, with inflation $\kappa\ge 1$, floor $\lambda>0$.

Initialise weights by local Laplace evidence:
\[
\tilde w_j \propto \exp\{\ell(\hat\theta_j)\} (2\pi)^{d/2} |\Sigma_j|^{1/2},\qquad
w_j=\frac{\tilde w_j}{\sum_{r=1}^J \tilde w_r}
\]
This yields the fitted mixture (See Appendix.\ref{app:LMA} for derivations):
\[
q_0(\theta)=\sum_{j=1}^J w_j \mathcal N(\theta;\mu_j,\Sigma_j)
\]

We can then use this fitted Laplace mixture for:
\begin{enumerate}
\item \textbf{Direct sampling (stratified).} Either (a) sample i.i.d. draws by first picking a component $j\sim\mathrm{Cat}(w_{1:J})$ and then $\theta\sim\mathcal N(\mu_j,\Sigma_j)$; or (b) form a fixed bank with $M$ draws per component and \textit{stratified sampling}: choose a component $j\sim\mathrm{Cat}(w)$, then pick uniformly among its $M$ stored draws.

\item \textbf{Warm start for GMA.} Use $\{(\mu_j,\Sigma_j,w_j)\}_{j=1}^J$ to initialise the GMM used by GMA (LMA refinement with weights only optimisation): seed components at the Laplace centres (optionally with small jitter $\xi\sim\mathcal N(0,\tau^2\Sigma_j)$) and keep $\Sigma_j$ fixed; then run the reverse-KL weight optimisation (Section.\ref{sec:methodology}) on a bank of samples to adapt the mixture weights, or expand to $N>J$ components by assigning “parent modes” proportionally to $w$.
\end{enumerate}

The advantage of Laplacian mixture aprroximation being, it locates the position of the Gaussian means and estimates their spreads (variances) at the same time. LMA, as a special case of GMA, is employed for direct sampling in our later epidemic model inference (Section.\ref{subsec:SIR_model_inference}).

\subsection{Improving GMM approximation with EM} \label{subsec:em_GMA}

The KL objective in \ref{eq:KL_divergence_objective4_gradient_GMM1_cc} can be \textit{biased} when evaluated only at a fixed, pre-sampled bank $\{\mathbf{s}_{i,j}\}$: if component locations or scales are poorly chosen, the bank may under-cover high-density regions of $p(\mathbf z)$, not representing the optimal or steepest gradient direction for weights optimisation and yielding noisy and systematically misdirected weight gradients. A more principled, though more costly, alternative is to optimise \textit{all} GMM parameters $\boldsymbol{\theta}=\{w_i,\boldsymbol{\mu}_i,\Sigma_i\}_{i=1}^N$ simultaneously via (population) EM, which maximises $\mathbb{E}_{p}[\log q_{\boldsymbol{\theta}}(\mathbf{z]})]$ and equivalently, minimises the inclusive $\mathrm{KL}(p\|q_{\boldsymbol{\theta}})$ \footnote{Each EM sweep costs $\mathcal{O}(NMd^2)$; see Appendix.\ref{app:em_GMA}.}. 

We term this variant \textit{population EM} because its E/M updates are defined as expectations under the target (population) density \(p\), rather than empirical averages over a finite dataset \footnote{Note that, traditional EM is a data-driven, alternating optimisation approach, it requires observed data.}. For comparison, classical (data-based) EM for GMMs maximises the \textit{empirical log-likelihood}
\[
\widehat{\mathcal L}_N(\boldsymbol{\theta})
= \sum_{n=1}^N \log q_{\boldsymbol{\theta}}(\mathbf x_n)
= \sum_{n=1}^N \log \Bigg[\sum_{k=1}^K w_k \mathcal N(\mathbf x_n;\boldsymbol{\mu}_k,\Sigma_k)\Bigg]
\]
with responsibilities \(r_{nk}=r_k(\mathbf x_n;\boldsymbol{\theta}) \propto w_k \mathcal N(\mathbf x_n;\boldsymbol{\mu}_k,\Sigma_k)\), yielding the standard updates
\[
N_k=\sum_{n=1}^N r_{nk},\qquad
\boldsymbol{\mu}_k=\frac{1}{N_k}\sum_{n=1}^N r_{nk} \mathbf x_n,\qquad
\Sigma_k=\frac{1}{N_k}\sum_{n=1}^N r_{nk} (\mathbf x_n-\boldsymbol{\mu}_k)(\mathbf x_n-\boldsymbol{\mu}_k)^\top 
\quad
, w_k=\frac{N_k}{N}
\]

In contrast, \textit{population EM} replaces the dataset by the \textit{population} $p$ and maximises $\mathbb E_{p}[\log q_{\boldsymbol\theta}(\mathbf Z)]$ (equivalently, minimises $\mathrm{KL}(p\|q_{\boldsymbol\theta})$), which leads to the population analogues.

\paragraph{Population EM updates.}
Define mixture responsibilities
\[
r_k(\mathbf z;\boldsymbol\theta)
=\frac{w_k \mathcal N(\mathbf z;\boldsymbol\mu_k,\Sigma_k)}
      {\sum_{\ell=1}^K w_\ell \mathcal N(\mathbf z;\boldsymbol\mu_\ell,\Sigma_\ell)}
\]
The population E/M updates are the population analogues of the data-based formulas:
\[
\bar N_k=\mathbb E_{p} \big[r_k(\mathbf Z;\boldsymbol\theta)\big],\quad
\boldsymbol\mu_k'=\frac{\mathbb E_{p} \big[r_k(\mathbf Z;\boldsymbol\theta) \mathbf Z\big]}{\bar N_k},\quad
\Sigma_k'=\frac{\mathbb E_{p} \big[r_k(\mathbf Z;\boldsymbol\theta) (\mathbf Z-\boldsymbol\mu_k')
(\mathbf Z-\boldsymbol\mu_k')^\top\big]}{\bar N_k},\quad
w_k'=\frac{\bar N_k}{\sum_{\ell=1}^K \bar N_\ell}
\]
These are fixed points of $\max_{\boldsymbol\theta}\mathbb E_{p}[\log q_{\boldsymbol\theta}]$.

\paragraph{Self-normalised importance sampling (SNIS) implementation.}
Because exact expectations under \(p\) are intractable, we approximate them with self-normalised importance sampling (SNIS): we draw a bank $\{\mathbf z_m\}_{m=1}^M \sim r(\mathbf z)$ (typically $r=q_{\boldsymbol\theta^{(t)}}$), compute unnormalised weights $\tilde\omega_m=\bar p(\mathbf z_m)/r(\mathbf z_m)$, normalise $\omega_m=\tilde\omega_m/\sum_{s=1}^M \tilde\omega_s$, and plug the estimator $\mathbb E_{p}[f(\mathbf Z)]\approx \sum_{m=1}^M \omega_m f(\mathbf z_m)$ into the population updates:
\[
\bar N_k \approx \sum_{m=1}^M \omega_m  r_k(\mathbf z_m;\boldsymbol\theta),\quad
\boldsymbol\mu_k' \approx \frac{\sum_{m=1}^M \omega_m  r_k(\mathbf z_m;\boldsymbol\theta) \mathbf z_m}{\bar N_k}
\]
\[
\Sigma_k' \approx \frac{\sum_{m=1}^M \omega_m  r_k(\mathbf z_m;\boldsymbol\theta) (\mathbf z_m-\boldsymbol\mu_k')(\mathbf z_m-\boldsymbol\mu_k')^\top}{\bar N_k},\quad
w_k' \approx \frac{\bar N_k}{\sum_{\ell=1}^K \bar N_\ell}
\]
This yields a generalised EM: the E/M steps are exact for $\mathbb E_{p}$, and we approximate those expectations via SNIS using only its evaluations.

In practice, one can stabilise covariance updates by adding a small ridge, $\Sigma_k' \leftarrow \Sigma_k' + \lambda I$, and by capping the condition number of $\Sigma_k'$ to prevent degeneracy. To improve robustness early on, we can use an annealed E-step in which responsibilities are tempered and then renormalised, $r_k \leftarrow r_k^{1/T}$ with $T>1$ gradually reduced to $1$, and we discard or floor very small responsibilities to avoid noisy updates from vanishing components. A single EM sweep costs $\mathcal O(K M d^2)$ with full covariances (or $\mathcal O(K M d)$ if diagonals are used). Initialising EM-GMA from LMA components (means/covariances from local Laplace fits, with evidence-proportional weights) provides good posterior coverage from the outset; EM-GMA then reallocates mixture mass and reshapes covariances using the SNIS-weighted sufficient statistics, yielding a coherent, mass-covering refinement of $q_{\boldsymbol\theta}$ toward $p$.

The EM-fitted mixture can then be used in two ways: (i) \textit{direct sampling}, by drawing from each component and using stratified sampling with mixture weights; and (ii) as a \textit{warm start} for weights-only GMA (WGMA) refinement. This EM-GMA approach is different from traditional EM in that, standard EM fits a GMM to observed data by maximising the empirical log-likelihood (MLE), whereas EM-GMA replaces data with a proposal-drawn bank and self-normalised importance weights to approximate expectations under $p(\mathbf z)\propto \bar p(\mathbf z)$, requiring no prior dataset \footnote{EM-GMA never sees a dataset; it only queries $log \bar{p}(\mathbf z)$.} and yielding a mass-covering approximation. See Appendix.\ref{app:em_GMA} for more details about EM-GMA. 

\subsubsection*{\textit{A toy test}}

To test the EM-GMA sampling method, we construct a 2D ground-truth GMM with three components (weights $[0.45,0.25,0.30]$, means $(0,0)$, $(3,1.5)$, $(-2,3)$, non-diagonal covariances) and draw $N=5000$ samples for reference. We then fit two mixtures:
(i) \textit{traditional data-based EM} (maximum-likelihood EM on the observed samples); and
(ii) \textit{EM-GMA} (population EM) that minimises $\mathrm{KL}(p\|q_\theta)$ using self-normalised importance sampling (SNIS, see Appendix.\ref{app:em_GMA}) with a refreshed bank per sweep (bank size $M=4096$, proposal $r^{(t)}=q_{\theta^{(t)}}$). 
The GMM in EM-GMA is initialised randomly and does not use the dataset; it only has access to the unnormalised target $\bar p(\mathbf z)$.

\paragraph{Results.}
As observed in Fig.\ref{fig:em_gma_toy} and Table.\ref{tab:em_gma_toy}, both methods closely recover the ground truth: component means for EM-GMA are within $\leq 0.04$ in each coordinate of the true centres, and mixture weights differ by at most $\approx 0.02$; the learned $2\sigma$ covariance ellipses largely overlap the true ones. As expected, the data-based EM fit is slightly sharper in the densest regions, while EM-GM, which optimises the inclusive KL, exhibits a mildly more \textit{mass-covering} behaviour (some variances a bit larger, others comparable), despite starting from random parameters and never using the data. This supports EM-GMA as a strong, data-free surrogate $q_\theta$ for direct sampling or as a warm start for WGMA. More examples of applying EM-GMA can be found in Section.\ref{subsec:star_shaped_density}, and Section.\ref{subsec:sensor_network_localization}.

\begin{figure}[H]
    \centering
    \includegraphics[width=0.5\linewidth]{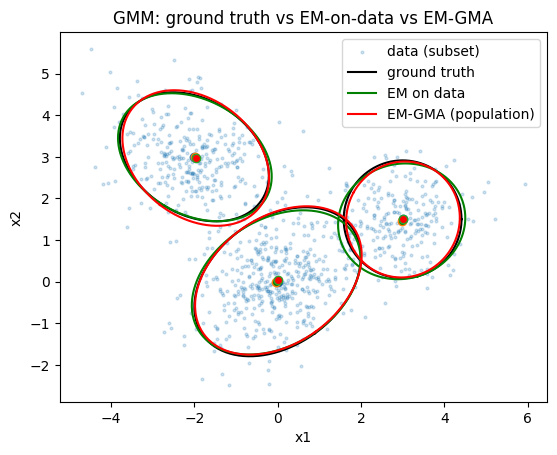}
    \vspace{-2mm}
    \caption{Ground truth (black) \textit{vs} traditional data-based EM (green) \textit{vs} EM-GMA (red). 
    Scatter shows a subset of the generated data (context only; \textit{not used} by EM-GMA). 
    Curves are $2\sigma$ covariance ellipses; dots mark component means.}
    \label{fig:em_gma_toy}
\end{figure}

\begin{table}[H]
\centering
\small
\setlength{\tabcolsep}{6pt}
\begin{tabular}{l l c c c c c}
\toprule
\textbf{Comp.\ (by $\mu_x$)} & \textbf{Method} & $\boldsymbol{w}$ & $\boldsymbol{\mu_x}$ & $\boldsymbol{\mu_y}$ & $\boldsymbol{\sigma_x^2}$ & $\boldsymbol{\sigma_y^2}$ \\
\midrule
\multirow{3}{*}{C1 ($\mu_x \approx -2$)} 
  & Ground truth        & 0.300 & -2.000 & 3.000  & 0.800 & 0.600 \\
  & Data-based EM       & 0.297 & -1.988 & 2.990  & 0.851 & 0.590 \\
  & EM-GMA (population) & 0.296 & -1.967 & 2.971  & 0.775 & 0.662 \\
\midrule
\multirow{3}{*}{C2 ($\mu_x \approx 0$)} 
  & Ground truth        & 0.450 &  0.000 & 0.000  & 1.000 & 0.800 \\
  & Data-based EM       & 0.448 & -0.035 & -0.016 & 1.026 & 0.750 \\
  & EM-GMA (population) & 0.462 &  0.017 & 0.031  & 1.000 & 0.793 \\
\midrule
\multirow{3}{*}{C3 ($\mu_x \approx 3$)} 
  & Ground truth        & 0.250 & 3.000 & 1.500  & 0.500 & 0.500 \\
  & Data-based EM       & 0.256 & 2.978 & 1.451  & 0.583 & 0.483 \\
  & EM-GMA (population) & 0.242 & 3.009 & 1.495  & 0.465 & 0.485 \\
\bottomrule
\end{tabular}
\caption{Quantitative comparison for the toy experiment. 
EM-GMA is fit without using the dataset; it only accesses the unnormalised target $\bar p(\mathbf z)$.}
\label{tab:em_gma_toy}
\end{table}

\section{Experiments}\label{sec:experiments}

We evaluate WGMA, LMA and EM-GMA against a broad set of established baselines. Our experimental plan has two parts. First, we probe \textbf{sampling on complex synthetic targets} in one and two dimensions to address mode coverage and geometry: 1D connected and isolated tri-modal bumps; 2D four-modal Gaussians; moon-, banana-, wave-, and star-shaped densities, and Neal’s funnel. Second, we test \textbf{real-world probabilistic models} spanning generalized linear models and time-series to dynamical systems and language modeling: Bayesian logistic regression and hierarchical BLR (Radon), hierarchical Bayesian symbolic regression for pendulum dynamics, Bayesian LSTM for mortality forecasting, a head-only Bayesianization of a language model, Lotka-Volterra and SIR inference, Bayesian optimal experimental design for logistic dose-response (via an LMA prior), and sensor-network localisation with LMA and EM-GMA posteriors.

Across all tasks we compare GMA to commonly used MCMC and VI methods, including Metropolis-Hastings (MH \cite{hastings_monte_1970}, Hamiltonian Monte Carlo (NUTS \cite{hoffman2014nuts}), and Langevin Monte Carlo (LMC \cite{roberts_exponential_1996}), SVGD \cite{liu_stein_2019}, MFVI-ADVI \cite{kucukelbir_automatic_2016}, GM-ADVI \cite{morningstar_automatic_2020}, S-ADVI \cite{shao_nonparametric_2024} and EVI \cite{Wang2021EVI}. 
All experiments were run in \textit{Google Colab}. In a CPU-only environment, they were run on a 64-bit Intel Xeon VM (4 physical cores, 8 threads) under KVM, with 55 MB shared L3 cache, 54.75 GB RAM, and 242.49 GB disk; in an additional A100 GPU environment, they were on a virtualized x86\_64 environment running on a single-socket Intel Xeon CPU @ 2.20 GHz with 6 physical cores and 12 logical threads. 
Implementation details, hyperparameters, and additional diagnostics are provided in the appendix.

\subsection{Sampling complex densities}

To empirically evaluate the performance of the proposed GMA sampling method, we conduct a series of experiments on seven challenging target densities. The chosen target distributions range from simple 1D multi-modal cases to complex 2D densities with non-Gaussian geometries, high curvature, and strong variable dependencies, such as the moon, double banana, wave, and Neal's funnel shapes. In each experiment, we assess both the quality of the generated samples and the computational efficiency of the 8 samplers.

\subsubsection{1D connected \& isolated tri-modal bumps}
We first tested the GMA sampling method (Algo. \ref{algo:WGMA-sampling-pgd}) on two unnormalised, 1D multi-modal densities: one with 3 connected modes, and the other with 3 isolated modes. Both tests share the same methodological settings. 

\paragraph{Samplers set-up}
For GMA, $N=10$ Gaussian components are initialized with means uniformly distributed across the range $[-6, 6]$ and variances linearly interpolated between $0.5^2$ and $0.7^2$. The optimization runs for $K=120$ iterations\footnote{$K=120$ was selected by trial and error, enabling stable convergence of the final weights.} with $M=200$ fixed samples per component, using a iteration-dependent, diminishing learning rate $\eta_0/k$ with an initial value of $\eta_0 = 0.5$ for connected modes and $\eta_0 = 0.01$ for isolated modes. Metropolis-Hastings (MH) employs $N \times M = 2000$ iterations, starting from an initial point of 0.0 with a proposal standard deviation of 1.0. Hamiltonian Monte Carlo (HMC) utilizes $2000$ samples after 1000 warm-up steps, with a step size of 0.01 and 20 integration steps. Langevin Monte Carlo (LMC) executes $2000$ steps with a learning rate of 0.01 and a noise scale of 0.02. Stein Variational Gradient Descent (SVGD) runs for 500 iterations with a step size of 0.01, using two initializations: particles drawn from the initial GMA samples (labeled \textit{SVGD (GMA init)}) or from a standard normal distribution (labeled \textit{SVGD (Std init)}). For Mean-Field VI (MFVI-ADVI \cite{kucukelbir_automatic_2016}), the optimization is performed for 120 steps. For Gaussian Mixture ADVI (GM-ADVI \cite{morningstar_automatic_2020}), we set the number of mixture components to $10$, the stratified samples per component to $200$, and the number of training steps to 120 with a learning rate of 0.001. The weight update for GMA is performed using projected gradient descent, where the weights are updated with a true Euclidean projection onto the probability simplex in each iteration.

\paragraph{Performance metrics}
For both tests, we collect the final samples and record the execution times. Employing the MH samples as a reference, we use five distributional distance and divergence metrics to quantify the quality of the approximations: 1D Wasserstein distance, Kolmogorov-Smirnov (KS) statistic, squared maximum mean discrepancy ($MMD^2$, calculated using an RBF kernel with unity length-scale), total variation (TV) distance, and Kullback-Leibler (KL) divergence. With the exception of KL divergence, all are true distance metrics. All metrics prefer small values. See Appendix.\ref{app:distributional_distance_metrics} for details of these metrics. In producing the histograms in our experiments, 60 bins are used. We also compare execution times (in seconds) as a measure of inference speed.

\subsubsection*{\textit{Connected tri-modal bumps}}
We use the following 1D, unnormalised tri-modal density as the target:
\begin{equation} \label{eq:connected_trimodal_bump}
    \bar{p}(z) = \exp\left(-\frac{(z^2 + 0.1 z^4)^2}{2 \cdot 1.0^2}\right) + 0.3 \cdot \mathcal{N}(z; 3, 0.5^2) + 0.2 \cdot \mathcal{N}(z; -3, 0.6^2)
\end{equation}
where the central term $\exp\left(-\frac{(z^2 + 0.1 z^4)^2}{2}\right)$ represents a banana-shaped distribution with a quartic exponent. The Gaussian terms are scaled by weights 0.3 and 0.2, centered at $z=\pm 3$ with variances 0.25 and 0.36, respectively.
To obtain the target density $p(z)=\bar{p}(z)/Z_p$, one can use e.g. numerical integration\footnote{In our test, the normalisation constant $Z_p$ is obtained using \textit{scipy.integrate.quad()} with 200 equal-distance points within [-6,6].} such as \textit{trapezoidal} or \textit{Simpson's rules}.

% \paragraph{Results}
The results, shown in Fig.\ref{fig:GMA_test1}, Fig.\ref{fig:densities_compare_test1}, Fig.\ref{fig:performance_compariosn_test1}, and Table.\ref{tab:metrics_and_times_test1}, demonstrate the effectiveness of the optimized GMA sampling method. GMA is by far the most computationally efficient method, achieving convergence in just 0.0278 seconds due to its pre-computation strategy. As seen in Fig. \ref{fig:GMA_test1}, the re-sampled GMA samples effectively capture all three connected modes of the target density. The weight evolution plot shows a rapid convergence where a few key Gaussian components are assigned significant weights, while the others are correctly pruned towards zero. This highlights the importance of initializing components across the full support of the density.

In comparison, the MCMC-based methods show varied performance. MH samples align well with the target and serve as our reference baseline. HMC, however, only captures the central mode. LMC performs strongly, capturing the central and left modes accurately, and achieves the best overall distributional distance scores among all benchmark methods. The performance of SVGD is highly dependent on initialization. When initialized with GMA samples that span the full support (SVGD (GMA init)), it captures all three modes. However, when initialized from a standard normal distribution (SVGD (Std init)), it fails to find the side modes. Both SVGD variants are by far the most computationally expensive. The variational methods, MFVI-ADVI and GM-ADVI, also show contrasting results. MFVI-ADVI, with its unimodal approximation, only captures the central mode, yielding poor distance metrics. GM-ADVI successfully identifies all three modes, demonstrating the advantage of its mixture-based approximation, though its accuracy is lower than GMA and LMC.

Overall, GMA sampling provides an excellent balance of speed and accuracy, outperforming all other methods in computational time while achieving sample quality comparable to the best-performing MCMC methods like LMC. This highlights its potential as a highly efficient and effective inference tool for complex, multi-modal distributions.

\begin{table}[H]
\centering
\scriptsize
\begin{threeparttable}
\caption{Quality metrics and execution times for the connected tri-modal bumps experiment.}
\label{tab:metrics_and_times_test1}
\begin{tabular}{p{1.75cm} p{1.3cm} p{1.2cm} p{1.2cm} p{1.2cm} p{1.2cm} p{1.3cm}}
\toprule
\textbf{Method} & Wasserstein distance $\downarrow$ & KS statistic $\downarrow$ & MMD\textsuperscript{2} $\downarrow$ & Total Variation $\downarrow$ & KL Divergence $\downarrow$ & Execution time (s) $\downarrow$ \\
\midrule
GMA & 0.2792 & 0.1230 & 0.0110 & 0.2087 & 0.5227 & \textbf{0.0278} \\
MH & 0.0000 & 0.0000 & 0.0000 & 0.0000 & 0.0000 & 1.2536 \\
HMC & 0.6050 & 0.1720 & 0.0572 & 0.2430 & 0.2675 & 1.5616 \\
LMC & \textbf{0.2223} & \textbf{0.0790} & \textbf{0.0109} & \textbf{0.1445} & \textbf{0.1572} & 1.7020 \\
SVGD (GMA) & 1.3676 & 0.2905 & 0.3068 & 0.5472 & 1.6117 & 838.8171 \\
SVGD (Std) & 0.5960 & 0.3260 & 0.0509 & 0.8040 & 1.8568 & 725.6705 \\
ADVI & 1.6172 & 0.3065 & 0.2142 & 0.4787 & 3.4929 & 1.6419 \\
GM-ADVI & 1.1819 & 0.2490 & 0.1262 & 0.3713 & 2.6153 & 1.7021 \\
\bottomrule
\end{tabular}
\begin{tablenotes}
\item[1] SVGD (GMA init): using GMA initial samples; SVGD (Std init): initial particles drawn from a standard normal distribution.
\end{tablenotes}
\end{threeparttable}
\end{table}

\begin{figure}[H]
    \centering
    \includegraphics[width=0.8\linewidth]{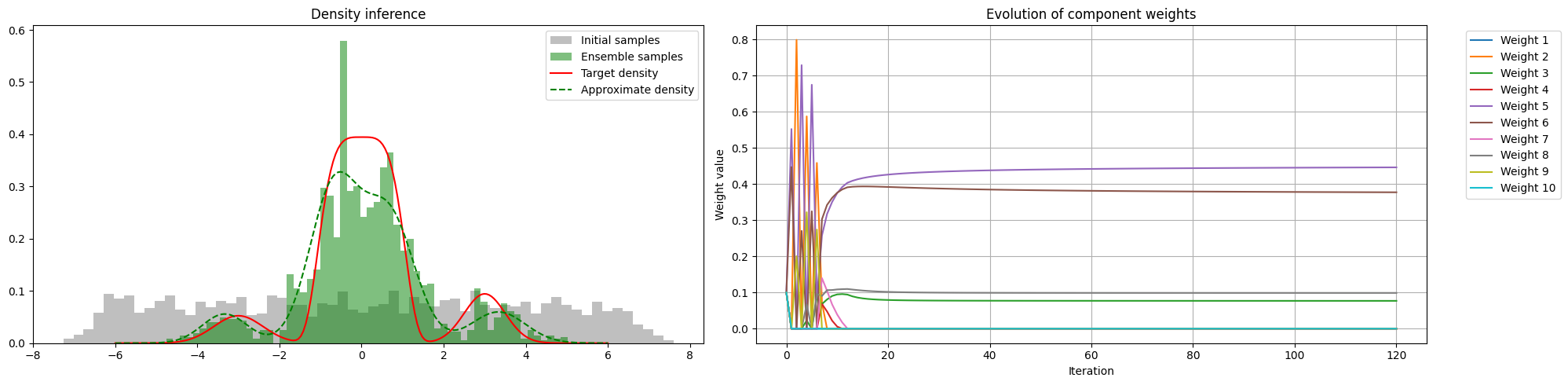}
    \caption{GMA samples and weight trajectories.}
    \label{fig:GMA_test1}
\end{figure}

\begin{figure}[H]
    \centering
    \includegraphics[width=1.0\linewidth]{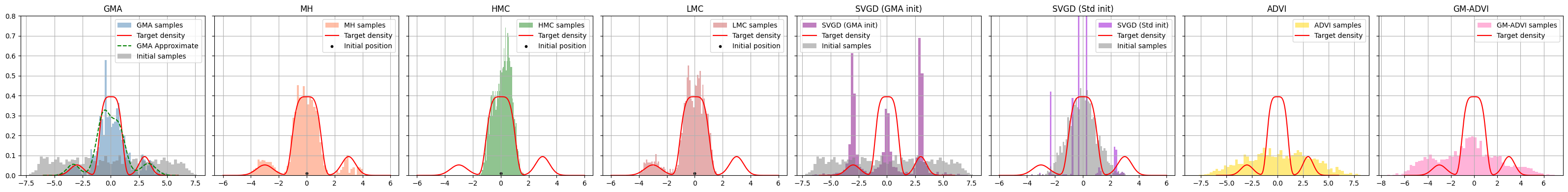}
    \caption{Density inference comparison for all methods.}
    \label{fig:densities_compare_test1}
\end{figure}

\begin{figure}[H]
    \centering
    \includegraphics[width=1.0\linewidth]{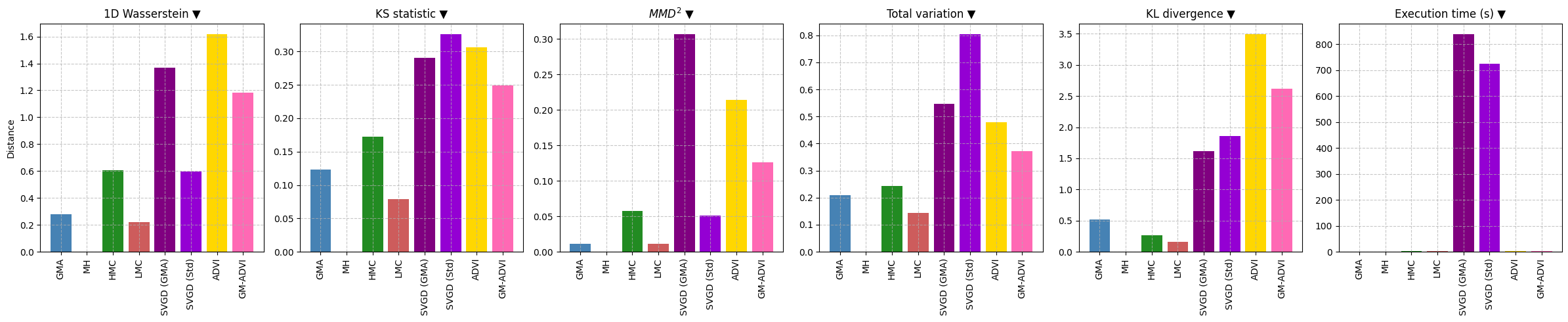}
    \caption{Performance comparison for all methods (MH samples as reference).}
    \label{fig:performance_compariosn_test1}
\end{figure}

\subsubsection*{\textit{Isolated tri-modal bumps}}

Testing the samplers on a density with isolated modes presents a more significant challenge \footnote{Score-based methods such as the classic SVGD are known to be unable to transport particles across isolated modes due to the ineffectiveness of the vanilla KSD objective, when particles are not properly initialised \cite{wenliang_blindness_2021}.}:
\begin{equation} \label{eq:isolated_trimodal_bump}
    \bar{p}(z) = \exp\left(-\frac{(z^2 + 0.1 z^4)^2}{2 \cdot 1.0^2}\right) + 0.3 \cdot \mathcal{N}(z; 5, 0.2^2) + 0.2 \cdot \mathcal{N}(z; -5, 0.2^2)
\end{equation}
This target is similar to the connected density in Eq.\ref{eq:connected_trimodal_bump}, but the side modes are now centered at $z=\pm 5$ with a much smaller variance of $0.04$, creating low-density regions that are difficult for many samplers to traverse.

\paragraph{Results}
The results for the isolated bumps, shown in Fig.\ref{fig:GMA_test2}, Fig.\ref{fig:densities_compare_test2}, Fig.\ref{fig:performance_compariosn_test2}, and Table.\ref{tab:metrics_and_times_test2}. GMA sampling is again highly effective, capturing all three isolated modes while remaining the fastest method by a large margin. The weight evolution plot in Fig.\ref{fig:GMA_test2} confirms this, showing rapid convergence to components located near the modes.

In stark contrast, all MCMC-based methods (MH, HMC, LMC) and the variational methods (MFVI-ADVI, GM-ADVI) fail to discover the side modes, as they all start their exploration from the center of the distribution and do not take large enough steps to cross the low-probability regions. As a result, their samples are concentrated only on the central mode.
SVGD's performance once again hinges on its initialization. When initialized with GMA samples spanning the full support (SVGD (GMA init)), it successfully captures all three modes. However, when initialized from a standard normal distribution centered at zero (SVGD (Std init)), it only captures the central mode. This confirms the known difficulty of SVGD in traversing low-density regions.

Because MH fails to explore the full distribution, using its samples as a reference for the distance metrics is misleading; the metrics in Table.\ref{tab:metrics_and_times_test2} only reflect how well other methods capture the central mode. Visually, however, it is clear from Fig.\ref{fig:densities_compare_test2} that GMA and SVGD (GMA init) are the only methods to successfully approximate the true multi-modal target. This demonstrates the critical advantage of GMA's global initialization strategy in scenarios with isolated modes.

\begin{table}[H]
\centering
\scriptsize
\begin{threeparttable}
\caption{Quality metrics and execution times.}
\label{tab:metrics_and_times_test2}
\begin{tabular}{p{1.5cm} p{1.3cm} p{1.2cm} p{1.2cm} p{1.2cm} p{1.2cm} p{1.3cm}}
\toprule
\textbf{Method} & Wasserstein distance $\downarrow$ & KS statistic $\downarrow$ & MMD\textsuperscript{2} $\downarrow$ & Total Variation $\downarrow$ & KL Divergence $\downarrow$ & Execution time (s) $\downarrow$ \\
\midrule
GMA & 1.2978 & 0.2080 & 0.1237 & 0.2125 & 6.0401 & \textbf{0.0296} \\
MH & 0.0000 & 0.0000 & 0.0000 & 0.0000 & 0.0000 & 0.8624 \\
HMC & 0.1596 & 0.1395 & 0.0134 & 0.1533 & 0.1607 & 0.8995 \\
LMC & 0.1062 & 0.0810 & 0.0049 & 0.1173 & 0.1157 & 1.5460 \\
SVGD (GMA) & 2.4462 & 0.2985 & 0.3646 & 0.6040 & 11.9322 & 853.2492 \\
SVGD (Std) & 0.3159 & 0.3305 & 0.0241 & 0.9195 & 2.2669 & 743.2780 \\
ADVI & 1.9885 & 0.3620 & 0.3835 & 0.2319 & 11.7313 & 60.7535 \\
GM-ADVI & 1.7153 & 0.3650 & 0.2999 & 0.1739 & 9.9653 & 3.0097 \\
\bottomrule
\end{tabular}
\begin{tablenotes}
\item[1] SVGD (GMA init): using GMA initial samples; SVGD (Std init): initial particles drawn from a standard normal distribution.
\end{tablenotes}
\end{threeparttable}
\end{table}

\begin{figure}[H]
    \centering
    \includegraphics[width=0.8\linewidth]{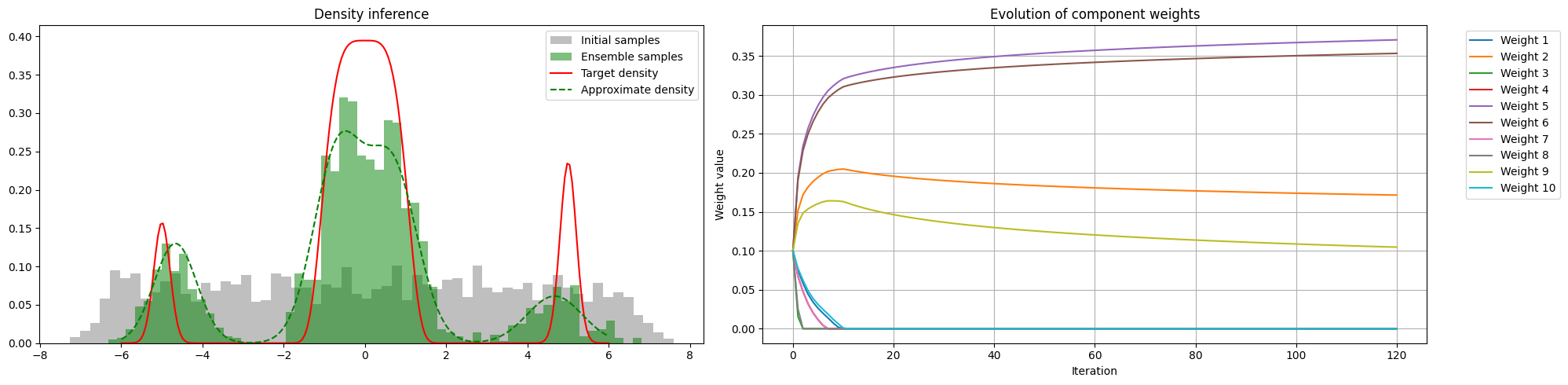}
    \caption{GMA samples and weight trajectories.}
    \label{fig:GMA_test2}
\end{figure}

\begin{figure}[H]
    \centering
    \includegraphics[width=1.0\linewidth]{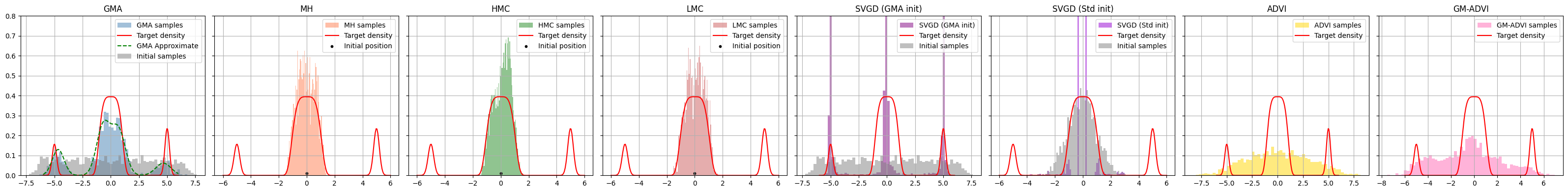}
    \caption{Density inference comparison for all methods.}
    \label{fig:densities_compare_test2}
\end{figure}

\begin{figure}[H]
    \centering
    \includegraphics[width=1.0\linewidth]{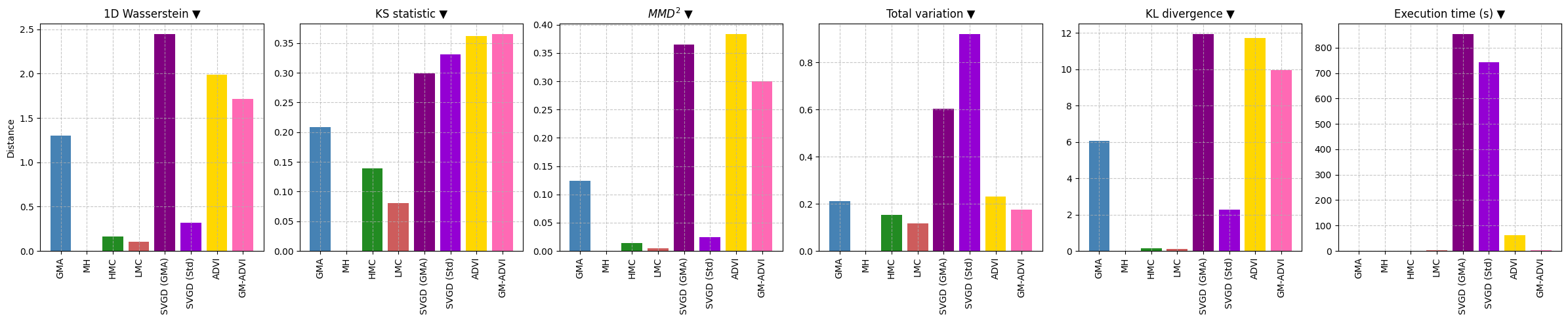}
    \caption{Performance comparison for all methods (MH samples as reference).}
    \label{fig:performance_compariosn_test2}
\end{figure}

\subsubsection{2D 4-modal Gaussians}

To test the samplers in a multi-dimensional scenario, we define a 2D target distribution composed of a mixture of 4 Gaussians. The modes are strategically placed in four quadrants, each with a distinct covariance structure to create a challenging landscape with varied shapes and orientations. The unnormalized target density $\bar{p}(\mathbf{z})$ for $\mathbf{z} = [z_1, z_2]^\top$ is:
\begin{equation} \label{eq:2d_4modal_gaussian}
\bar{p}(\mathbf{z}) = \sum_{i=1}^{4} c_i \cdot \mathcal{N}(\mathbf{z} | \boldsymbol{\mu}_i, \Sigma_i)
\end{equation}
with component weights $c_i$, means $\boldsymbol{\mu}_i$, and covariances $\Sigma_i$ defined as:
\begin{itemize}
    \item $c_1=0.3, \boldsymbol{\mu}_1 = [-3, 3]^\top, \Sigma_1 = \begin{pmatrix} 1.0 & 0.8 \\ 0.8 & 1.0 \end{pmatrix}$
    \item $c_2=0.3, \boldsymbol{\mu}_2 = [3, 3]^\top, \Sigma_2 = \begin{pmatrix} 1.0 & -0.8 \\ -0.8 & 1.0 \end{pmatrix}$
    \item $c_3=0.2, \boldsymbol{\mu}_3 = [-3, -3]^\top, \Sigma_3 = \begin{pmatrix} 1.0 & 0.0 \\ 0.0 & 0.2 \end{pmatrix}$
    \item $c_4=0.2, \boldsymbol{\mu}_4 = [3, -3]^\top, \Sigma_4 = \begin{pmatrix} 0.2 & 0.0 \\ 0.0 & 1.0 \end{pmatrix}$
\end{itemize}

\paragraph{Samplers set-up}
With minimal tuning, the final GMA sampler was configured with $N=225$ Gaussian components on a $15 \times 15$ grid, an initial covariance scale of $0.1$, and was run for $K=300$ iterations with $M=30$ samples per component and $\eta_0=0.1$. For the benchmarks, MH was run for 6750 iterations. HMC used 6750 samples after 1000 warm-up steps, with a step size of 0.1 and 10 integration steps. LMC was run for 6750 steps with a learning rate of 0.1 and a noise scale of 0.02. SVGD was run for 500 iterations with a step size of 0.01 and 800 particles to save computing time. For MFVI-ADVI, the optimization was performed for 5000 steps. For GM-ADVI, the parameters were matched to the GMA setup ($N=225, M=30, K=300$) with a learning rate of 0.01.

\paragraph{Results}
We first investigated the effect of the initial covariance scale for the GMA components. As shown in Fig.\ref{fig:GMA_test3_cov_tuning}, there is a clear trade-off: with same other settings (i.e. no. of Gaussian components $N=25$, , samples per component $M=100$, iterations $K=120$, initial learning rate $\eta=0.05$), a larger covariance (i.e. $cov=1.0$, left figure) ensures the components overlap and cover the target support, but results in a coarse approximation. Conversely, a smaller covariance (i.e. $cov=0.1$, right figure) provides better resolution for capturing the shape of each mode but requires a denser grid of components to ensure sufficient overlap. Even when the components are not perfectly placed, GMA correctly identifies the single nearest component to each mode and assigns it a high weight, demonstrating the robustness of the weight optimization in GMA.

\begin{figure}[H]
    \centering
    \includegraphics[width=0.49\linewidth]{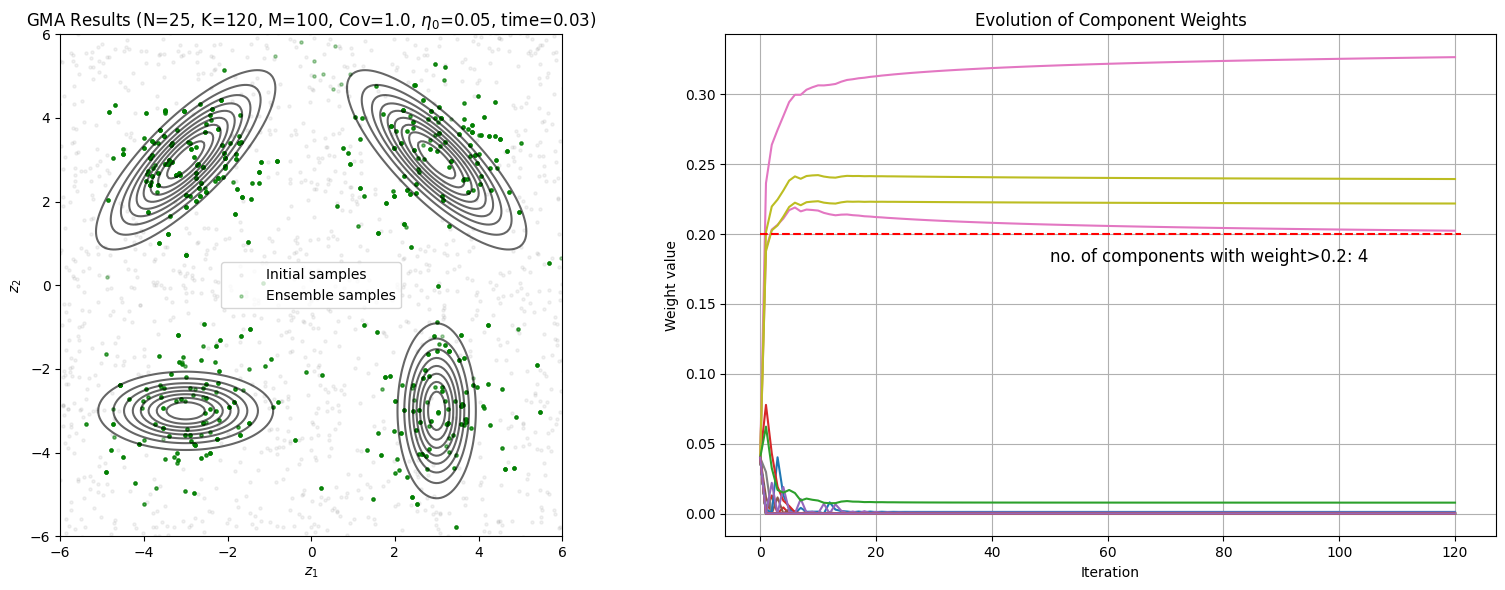}
    \includegraphics[width=0.49\linewidth]{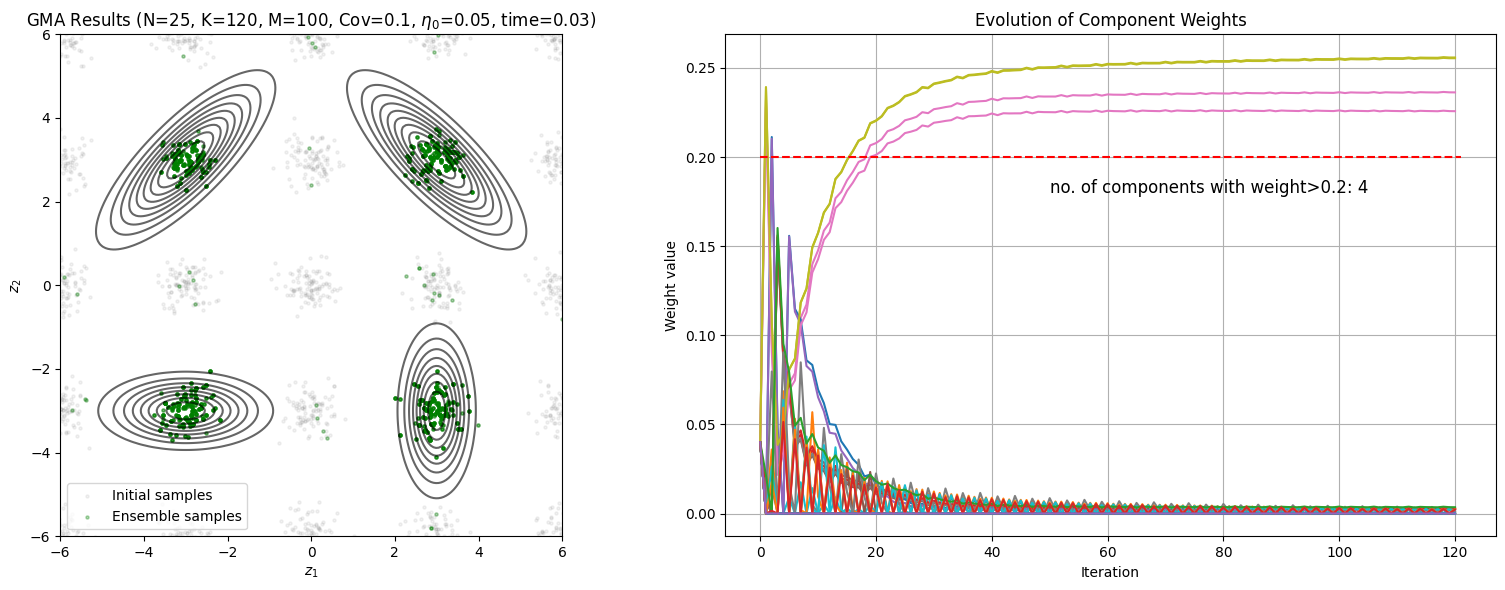}
    \caption{Tuning GMA components covariance. Left: A large covariance (1.0) covers the modes easily but gives a coarse approximation. Right: A small covariance (0.1) gives good resolution but requires more components to infer the full geometry.}
    \label{fig:GMA_test3_cov_tuning}
\end{figure}

Based on this, we selected a small covariance scale (0.1) and increased the number of components. The final results are presented in Fig.\ref{fig:GMA_test3} and Fig.\ref{fig:densities_compare_test3}. The tuned GMA sampler successfully captures all four modes with roughly even proportions across 4 modes. The weight evolution plot shows a more gradual convergence due to the larger number of components, but it ultimately identifies the correct regions of high density. It's also computationally efficient - it takes 3.14 seconds to compute, being the second fastest method.

The performance of the benchmark methods, as shown in Fig.\ref{fig:densities_compare_test3}, was varied. MH again provides a solid baseline, capturing all modes; however, the mode weights [0.3, 0.3, 0.2, 0.2] are not fully captured by the proportions of samples. HMC and LMC both collapse to a single mode. Both SVGD variants successfully find the modes but are computationally very expensive. MFVI-ADVI fails by averaging between the modes, while GM-ADVI captures the modes but with high variance and at a significant computational cost. In this comparison, GMA provides a high-quality approximation of the target density, outperforming most methods in sample quality while remaining computationally competitive.

\begin{table}[H]
\centering
\scriptsize
\caption{Execution times for the tuned 2D 4-modal Gaussians experiment.}
\label{tab:times_test3}
\begin{tabular}{lc}
\toprule
\textbf{Method} & Execution time (s) $\downarrow$ \\
\midrule
GMA & 3.14 \\
MH & 6.47 \\
HMC & \textbf{2.87} \\
LMC & 15.94 \\
SVGD (Std init) & 107.45 \\
SVGD (GMA init) & 107.46 \\
MFVI-ADVI & 50.56 \\
GM-ADVI & 13.91 \\
\bottomrule
\end{tabular}
\end{table}

\begin{figure}[H]
    \centering
    \includegraphics[width=0.8\linewidth]{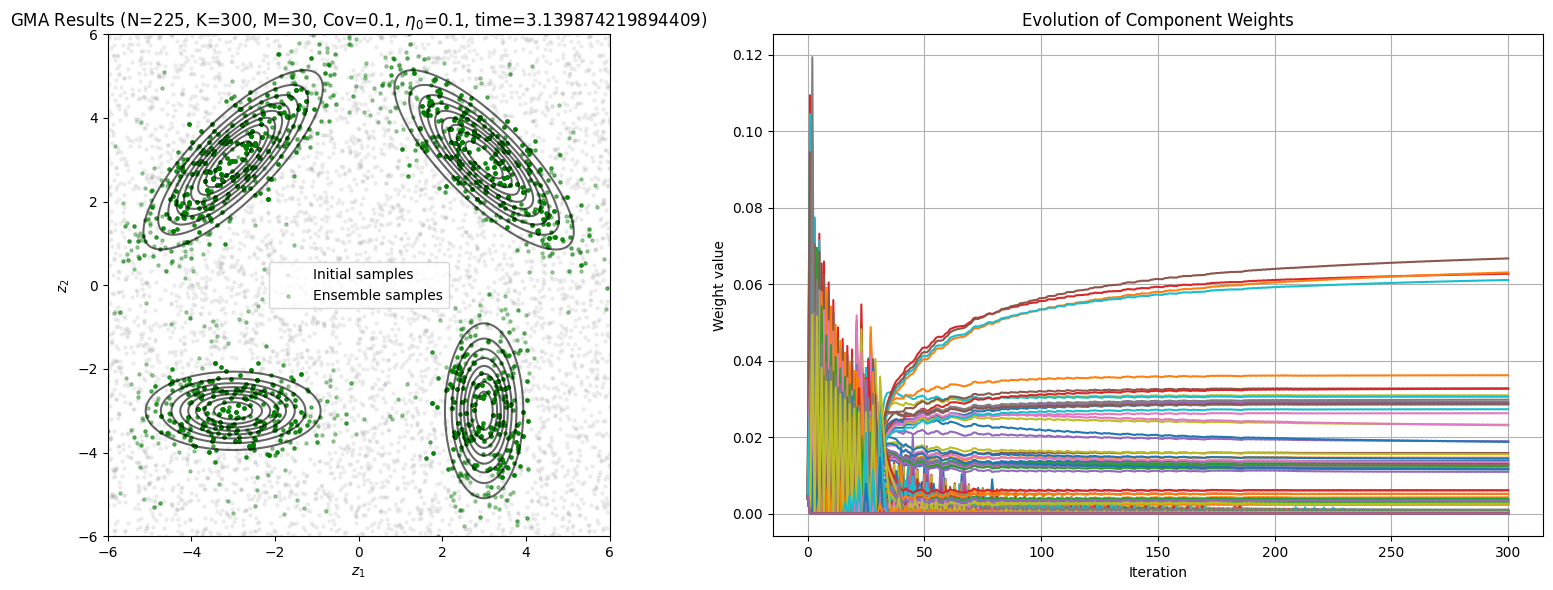}
    \caption{Tuned GMA samples and weight trajectories for the 2D 4-modal target.}
    \label{fig:GMA_test3}
\end{figure}

\begin{figure}[H]
    \centering
    \includegraphics[width=1.0\linewidth]{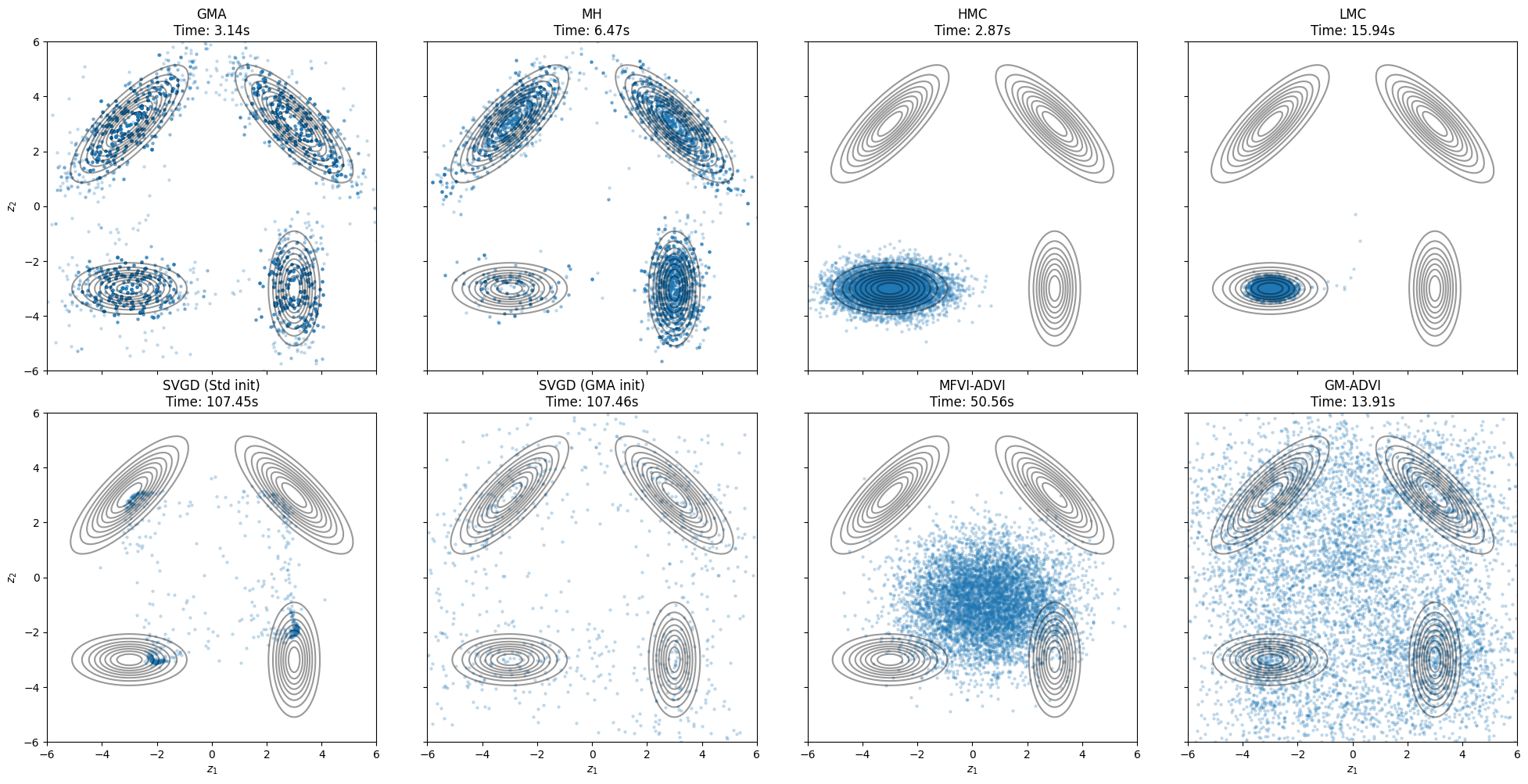}
    \caption{Density inference comparison for all methods on the 2D 4-modal target after tuning GMA.}
    \label{fig:densities_compare_test3}
\end{figure}

\subsubsection{A 2D moon-shaped density}

We further test these samplers on a 2D moon-shaped density, which is known to be difficult for methods that struggle with non-Gaussian geometries and high curvature. The unnormalized target density is defined as:
\begin{equation} \label{eq:2d_moon}
\bar{p}(\mathbf{z}) = \exp\left(-\frac{z_1^2}{2} - \frac{1}{2}(10z_2 + 3z_1^2 - 3)^2\right)
\end{equation}

\paragraph{Samplers set-up}
With minimal tuning, the final GMA sampler was configured with $N=400$ Gaussian components on a $20 \times 20$ grid, each with a small initial covariance scale of $0.04$. The optimization was run for $K=1500$ iterations with $M=20$ samples per component (totaling 8000 initial samples) and an initial learning rate of $\eta_0=0.1$. For the benchmarks, MH was run for 8000 iterations with a proposal covariance of $0.1\mathbf{I}$. HMC used 8000 samples after 1000 warm-up steps, a step size of 0.05, and 20 integration steps. LMC was run for 8000 steps with a learning rate of 0.01. SVGD was run for 500 iterations with a step size of 0.01 and 800 particles to save computing time. MFVI-ADVI was run for 20,000 steps. For GM-ADVI, the number of mixture components and samples per component were matched to the GMA setup ($N=400, M=20$), and it was trained for 1500 steps with a learning rate of 0.01.

\paragraph{Results}
The outcomes for the moon-shaped density test are shown in Fig.\ref{fig:GMA_test4} and Fig.\ref{fig:densities_compare_test4}, with execution times in Table.\ref{tab:times_test4}. The results highlight the flexibility of the GMA method. As seen in Fig.\ref{fig:GMA_test4}, by using a large number of components with small variances, GMA successfully approximates the complex, curved geometry of the target. The weight evolution plot shows that the algorithm correctly identifies and assigns high weights to the components lying along the crescent shape.

The performance of the other samplers was mixed. MH, HMC, and LMC all performed well, successfully capturing the shape of the distribution, with MH being the fastest of all methods at just over one second. The SVGD samplers also accurately captured the target shape but were the most computationally expensive, even with only one-tenth of the 8,000 samples compared to other samplers. The variational methods struggled with the non-Gaussian geometry. MFVI-ADVI produced a poor approximation, unable to represent the curvature. After significant tuning to match the GMA setup, GM-ADVI's performance still produced a noisy and highly dispersed set of samples, failing to conform to the target's shape and taking significantly longer to run. This experiment demonstrates that while classic MCMC methods can be very effective for such densities, GMA provides a strong alternative that can flexibly adapt to complex geometries given appropriate hyperparameter settings.

\begin{table}[H]
\centering
\scriptsize
\caption{Execution times for the 2D moon-shaped density experiment.}
\label{tab:times_test4}
\begin{tabular}{lc}
\toprule
\textbf{Method} & Execution time (s) $\downarrow$ \\
\midrule
GMA & 10.47 \\
MH & \textbf{1.09} \\
HMC & 2.69 \\
LMC & 4.03 \\
SVGD (Std init) & 101.89 \\
SVGD (GMA init) & 116.32 \\
MFVI-ADVI & 4.63 \\
GM-ADVI & 60.16 \\
\bottomrule
\end{tabular}
\end{table}

\begin{figure}[H]
    \centering
    \includegraphics[width=0.8\linewidth]{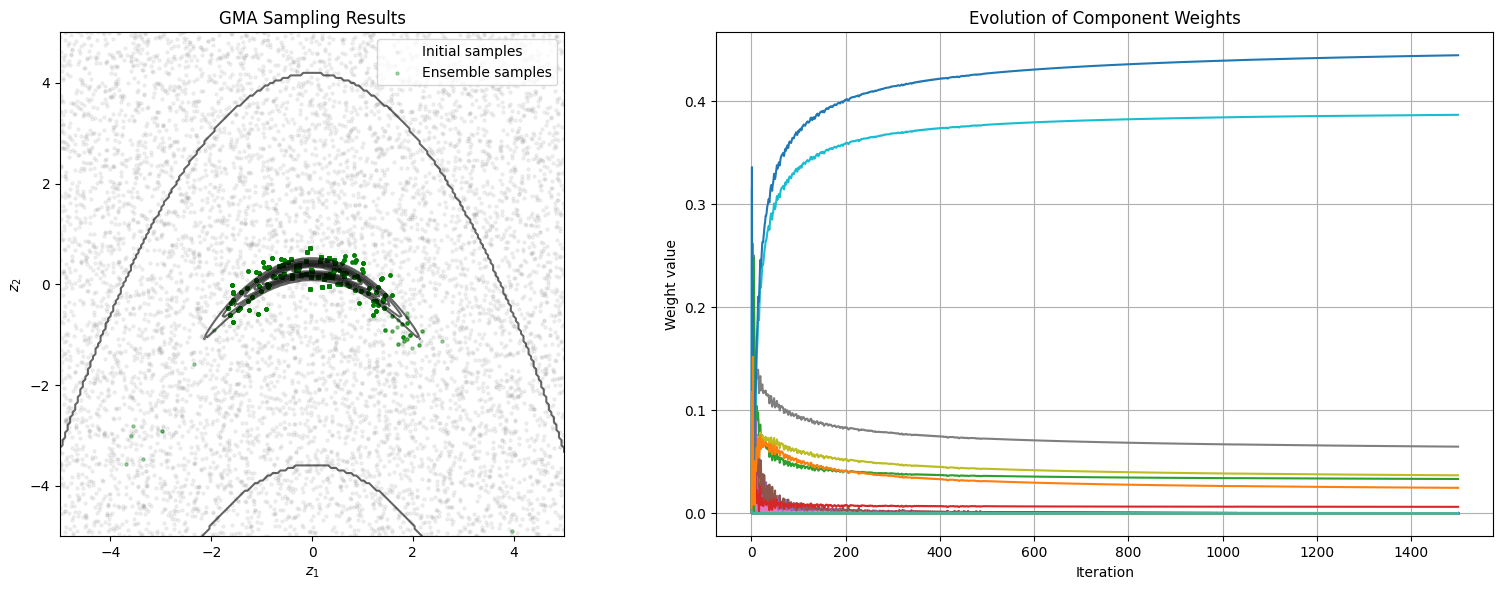}
    \caption{GMA samples and weight trajectories for the 2D moon-shaped target.}
    \label{fig:GMA_test4}
\end{figure}

\begin{figure}[H]
    \centering
    \includegraphics[width=1.0\linewidth]{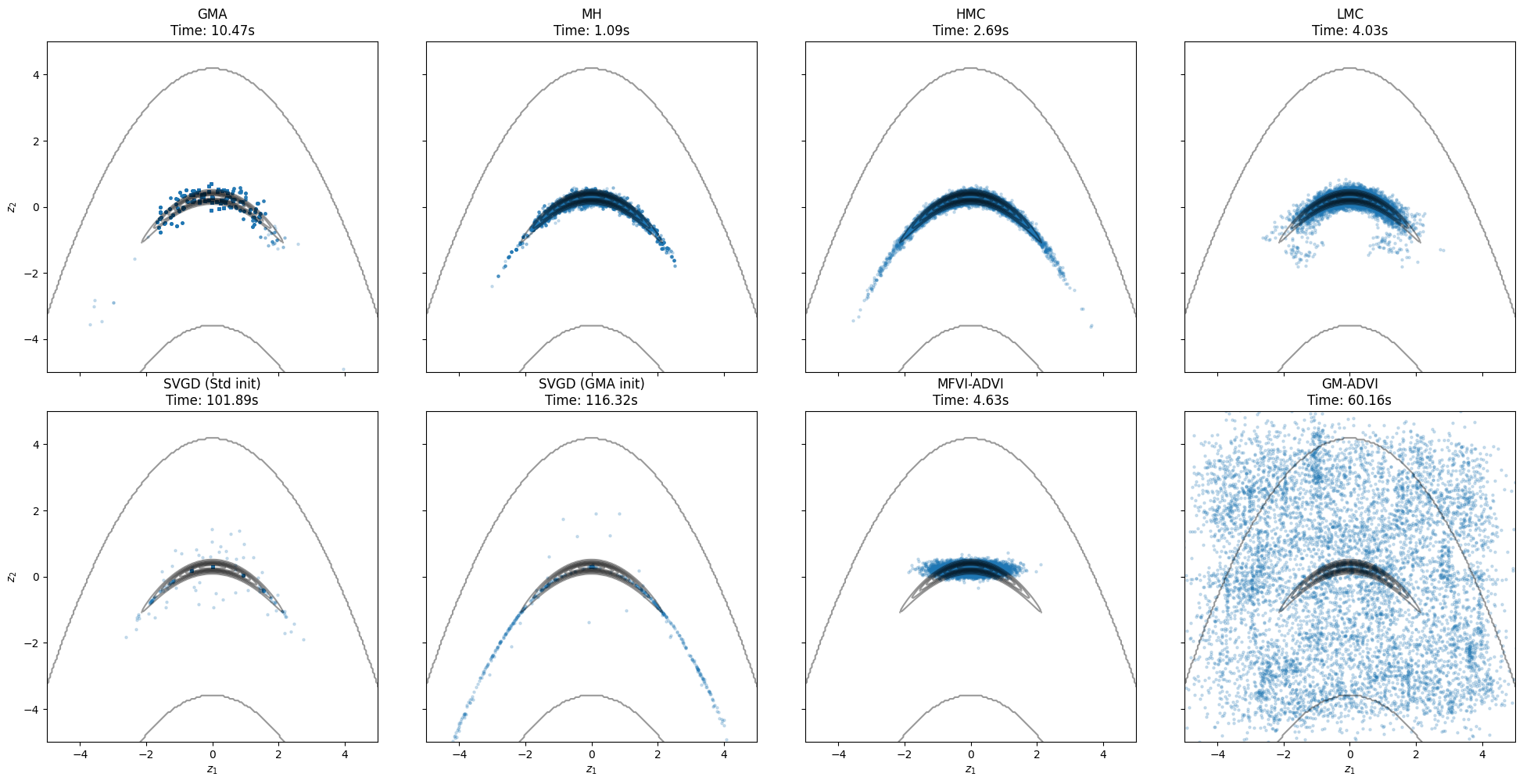}
    \caption{Density inference comparison for all methods on the 2D moon-shaped target.}
    \label{fig:densities_compare_test4}
\end{figure}

\subsubsection{A 2D double banana density}

Our 5-th test case is a 2D double banana-shaped density, a challenging distribution characterized by its two thin, curved, and isolated modes. The unnormalized target density is defined as:
\begin{equation} \label{eq:2d_double_banana}
\bar{p}(\mathbf{z}) = \exp\left(-2((z_1^2 + z_2^2) - 3)^2 + \log(\exp(-2(z_1-2)^2) + \exp(-2(z_1+2)^2))\right)
\end{equation}

\paragraph{Samplers set-up}
After some covariance tuning, the GMA sampler was configured with a dense grid of $N=625$ components ($25 \times 25$), a very small covariance scale of $0.03$, and run for $K=600$ iterations with $M=20$ samples per component and $\eta_0=0.1$. For the benchmarks, MH was run for 12,500 iterations with a proposal covariance of $0.05\mathbf{I}$. HMC used 12,500 samples after 1000 warm-up steps, with a step size of 0.02 and 20 integration steps. LMC was run for 12,500 steps with a learning rate of 0.005. SVGD was run for 500 iterations with a step size of 0.01 and 800 particles. MFVI-ADVI was run for 30,000 steps. For GM-ADVI, the parameters were matched to the GMA setup ($N=625, M=20, K=600$) with a learning rate of 0.01.

\paragraph{Results}
The results for the double banana density are shown in Fig.\ref{fig:GMA_test5} and Fig.\ref{fig:densities_compare_test5}, with execution times in Table.\ref{tab:times_test5}. This experiment highlights the trade-offs inherent in the GMA method when faced with highly complex geometries. To capture the two thin banana shapes, a very large number of components ($N=625$) with a small variance ($cov=0.1$) homogeneously for each Gaussian component was required. As shown in Fig.\ref{fig:GMA_test5}, this configuration allows GMA to successfully approximate the target, with the weight evolution plot confirming that components along both bananas are assigned significant weight.

However, this success comes at a computational cost. While GMA's time of 7.47 seconds is still much faster than the SVGD variants, it is slower than the well-tuned MCMC methods. MH and HMC perform exceptionally well, capturing the target shape accurately, with MH being the fastest overall at 1.62 seconds. LMC misses one mode. The SVGD samplers also produce high-quality samples but remain prohibitively slow. Variational methods in this case were largely unsuccessful: MFVI-ADVI collapsed to a single mode between the two bananas, and GM-ADVI produced a very noisy and dispersed approximation, taking over a minute to run. This experiment illustrates that while GMA is highly flexible, achieving good performance on very complex, non-Gaussian targets requires careful tuning.

\begin{table}[H]
\centering
\scriptsize
\caption{Execution times for the 2D double banana density experiment.}
\label{tab:times_test5}
\begin{tabular}{lc}
\toprule
\textbf{Method} & Execution time (s) $\downarrow$ \\
\midrule
GMA & 7.47 \\
MH & \textbf{1.62} \\
HMC & 3.93 \\
LMC & 9.29 \\
SVGD (Std init) & 123.44 \\
SVGD (GMA init) & 122.67 \\
MFVI-ADVI & 70.36 \\
GM-ADVI & 71.72 \\
\bottomrule
\end{tabular}
\end{table}

\begin{figure}[H]
    \centering
    \includegraphics[width=0.8\linewidth]{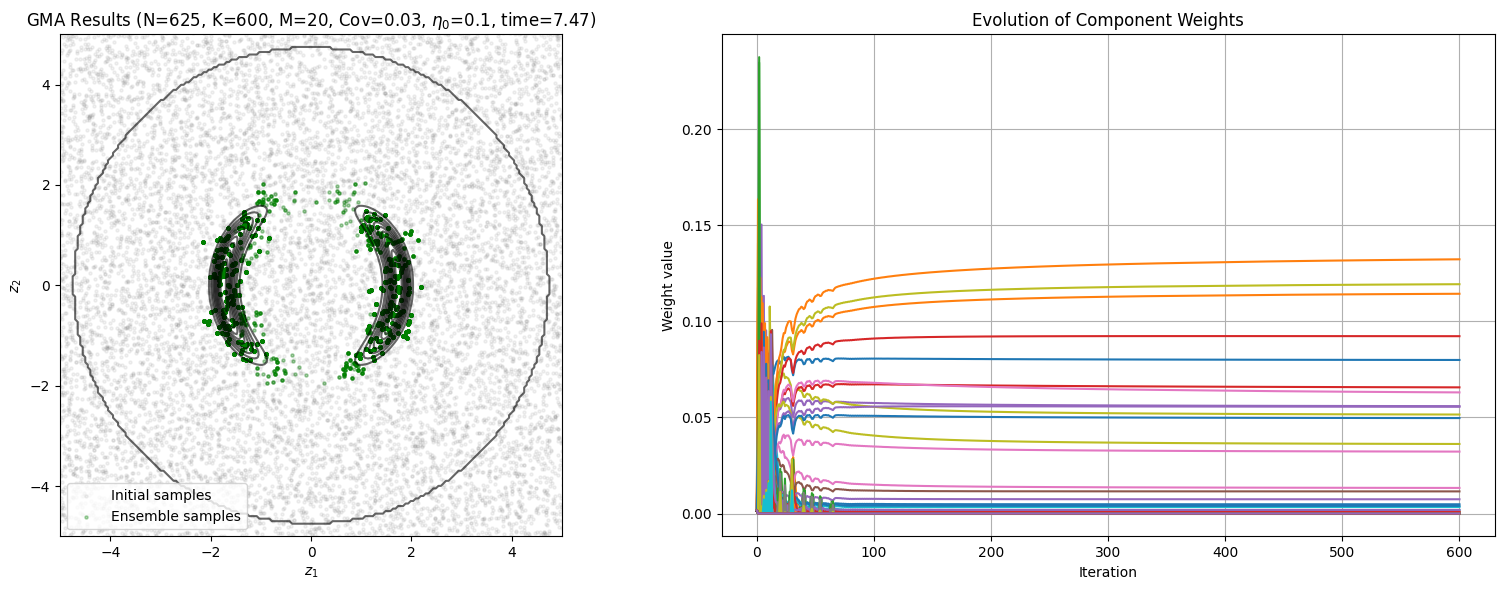}
    \caption{GMA samples and weight trajectories for the 2D double banana target.}
    \label{fig:GMA_test5}
\end{figure}

\begin{figure}[H]
    \centering
    \includegraphics[width=1.0\linewidth]{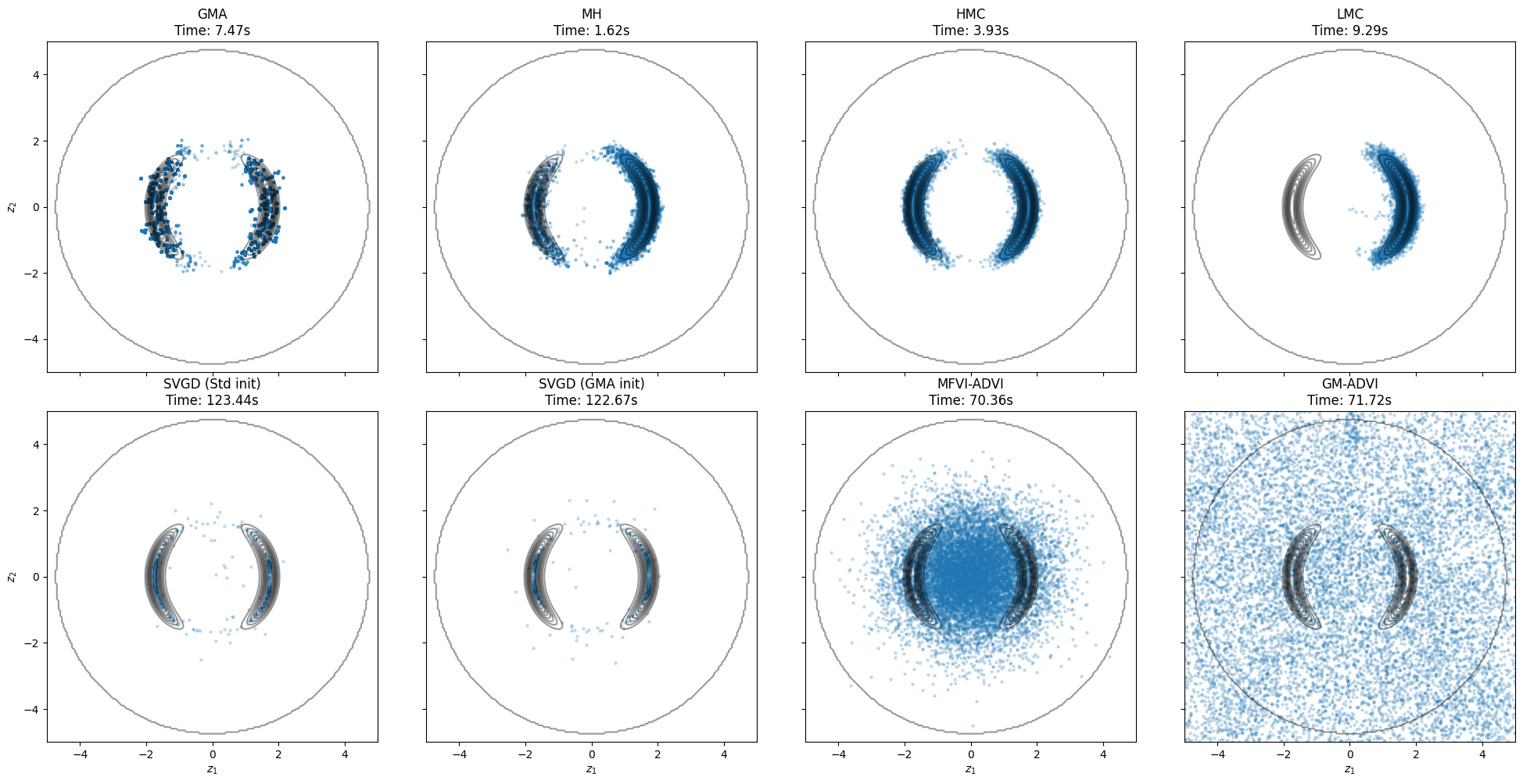}
    \caption{Density inference comparison for all methods on the 2D double banana target.}
    \label{fig:densities_compare_test5}
\end{figure}

\subsubsection{A 2D wave density}

Our 6-th experiment targets a 2D wave-shaped density, which presents a continuous, highly curved, and non-Gaussian manifold that is challenging for many samplers. The unnormalized target density is defined as:
\begin{equation} \label{eq:2d_wave}
\bar{p}(\mathbf{z}) = \exp\left(-\frac{1}{2 \cdot 0.16} (z_2 - \sin(\frac{\pi z_1}{2}))^2\right)
\end{equation}

\paragraph{Samplers set-up}
To approximate this complex shape, the GMA sampler was configured with a very dense grid of $N=900$ components ($30 \times 30$), a small covariance scale of $0.02$, and was run for $K=800$ iterations with $M=15$ samples per component and $\eta_0=0.1$. For the benchmarks, MH was run for 13,500 iterations with a proposal covariance of $0.05\mathbf{I}$. HMC used 13,500 samples after 1000 warm-up steps, with a step size of 0.05 and 20 integration steps. LMC was run for 13,500 steps with a learning rate of 0.01. SVGD was run for 500 iterations with a step size of 0.01 and 800 particles. MFVI-ADVI was run for 30,000 steps. For GM-ADVI, the parameters were matched to the GMA setup ($N=900, M=15, K=800$) with a learning rate of 0.01.

\paragraph{Results}
The results for the wave density are shown in Fig.\ref{fig:GMA_test6} and Fig.\ref{fig:densities_compare_test6}, with execution times in Table.\ref{tab:times_test6}. This test further demonstrates GMA's flexibility. By deploying a large number of fine-grained components, GMA is able to successfully trace the winding path of the wave, as shown in Fig.\ref{fig:GMA_test6}. The final samples, while slightly more dispersed than the best MCMC methods, accurately capture the overall geometry of the target.

The benchmark methods showed a clear division in performance. HMC and LMC performed very well, generating high-quality samples that closely followed the wave. Notably, MH was fast (1.26 seconds) but produces samples covering part of the observed tube. SVGD with full support initialization captures the shape of the tube but the dispersion is also large.
MFVI-ADVI averaged across the waves, producing a wide, uninformative distribution centered on the x-axis, whihle GM-ADVI, being the second slowest method, failed to converge and resulted in samples scattered randomly across the entire domain \footnote{The failure of GM-ADVI may be due to the limited number of optimisation iterations which matches that of GMA.}. This experiment shows that while certain MCMC methods (HMC and LMC) are highly effective for this type of problem, GMA can also succeed given fine tuned hyper-parameters \footnote{The most significant hyper-parameters of GMA are the co-dynamics of the homogeneous covariance $cov$ for each component and the number of components $N$. These two can be quickly tuned in a trial running step. Other parameters include the number of samples per component $M$, number of optimisation iterations $K$, initial learning rate $\eta_0$.}.

\begin{table}[H]
\centering
\scriptsize
\caption{Execution times for the 2D wave density experiment.}
\label{tab:times_test6}
\begin{tabular}{lc}
\toprule
\textbf{Method} & Execution time (s) $\downarrow$ \\
\midrule
GMA & 8.92 \\
MH & \textbf{1.26} \\
HMC & 3.18 \\
LMC & 4.70 \\
SVGD (Std init) & 83.39 \\
SVGD (GMA init) & 89.71 \\
MFVI-ADVI & 54.01 \\
GM-ADVI & 105.63 \\
\bottomrule
\end{tabular}
\end{table}

\begin{figure}[H]
    \centering
    \includegraphics[width=0.8\linewidth]{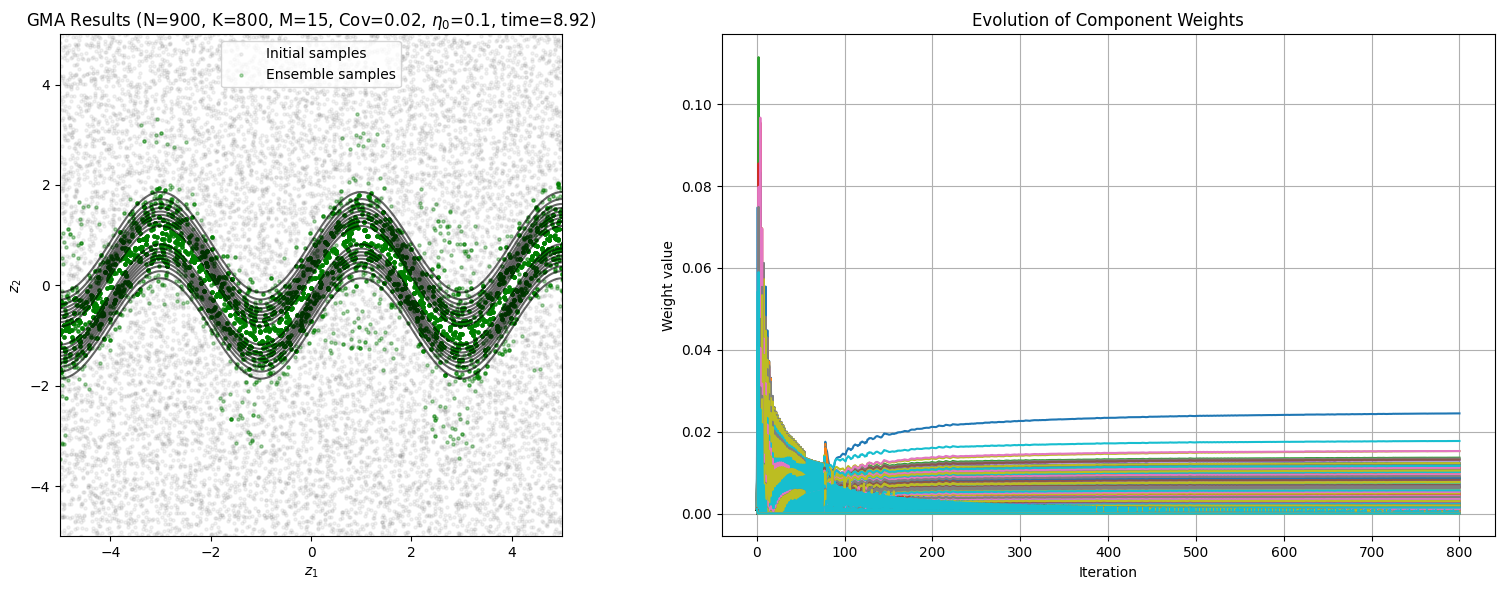}
    \caption{GMA samples and weight trajectories for the 2D wave density target.}
    \label{fig:GMA_test6}
\end{figure}

\begin{figure}[H]
    \centering
    \includegraphics[width=1.0\linewidth]{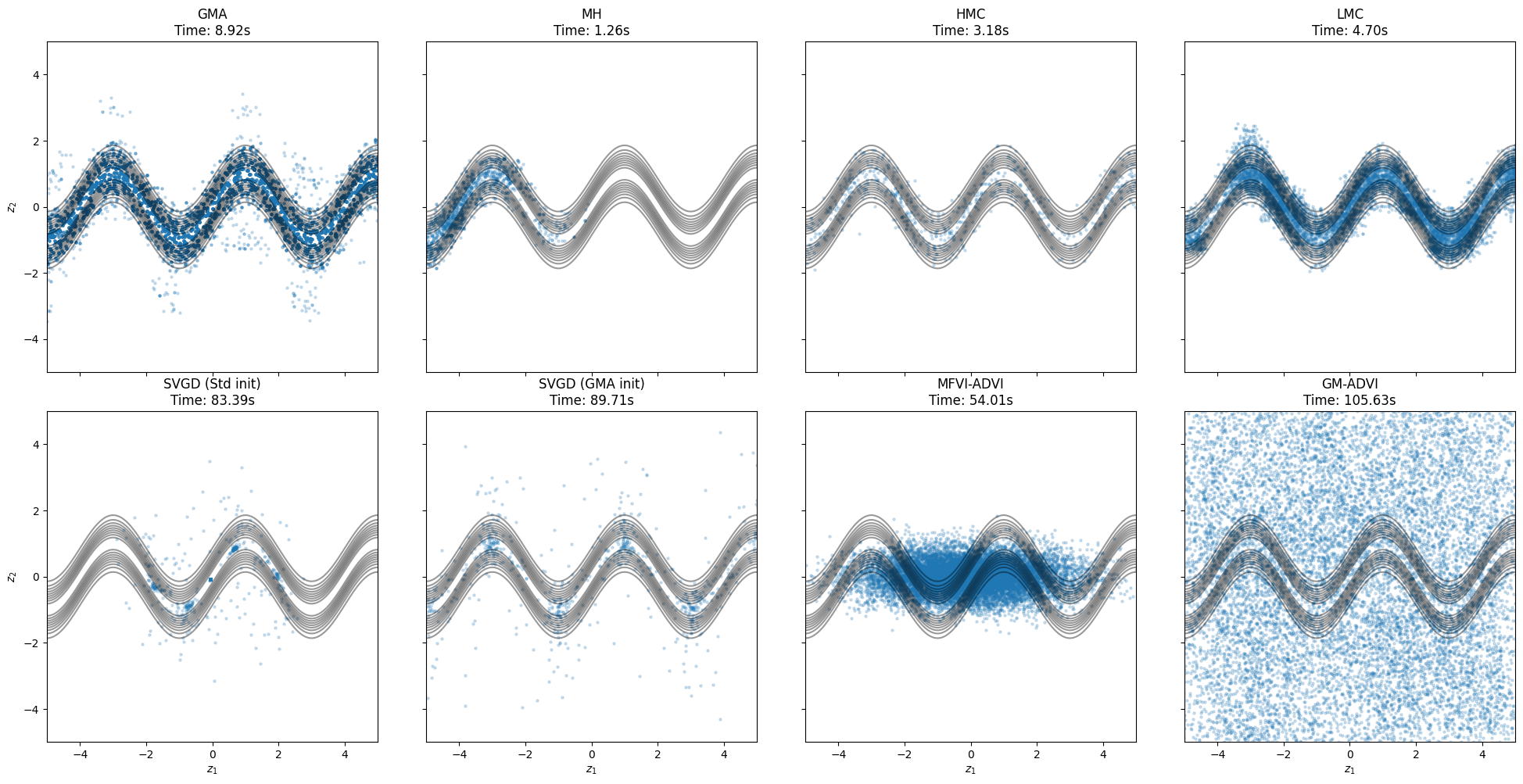}
    \caption{Density inference comparison for all methods on the 2D wave density target.}
    \label{fig:densities_compare_test6}
\end{figure}

\subsubsection{2D Neal’s funnel}

Our 7-th experiment addresses \textit{Neal's funnel}, a classic distribution commonly used to test the limitations of samplers. It is characterized by a strong dependency between its two dimensions, where the variance of one variable ($z_1$) is controlled by the value of the other ($z_2$). This creates a 'funnel' shape with a wide top and an extremely narrow, high-density neck, which is difficult for many sampling algorithms to explore efficiently. The unnormalized target density is defined as \cite{huang_electrostatics-based_2024}:
\begin{equation} \label{eq:2d_funnel}
\bar{p}(\mathbf{z}) = \mathcal{N}(z_2 | 0, 3^2) \cdot \mathcal{N}(z_1 | 0, \exp(z_2))
\end{equation}

\paragraph{Samplers set-up}
To handle the asymmetric geometry, the GMA sampler was configured with an asymmetric grid of $N=900$ components ($30 \times 30$) over the range $[-5, 5]$ for $z_1$ and $[-8, 6]$ for $z_2$. A small covariance scale of $0.03$ was used, and the optimization was run for $K=800$ iterations with $M=15$ samples per component and $\eta_0=0.1$. For the benchmarks, MH was run for 13,500 iterations with a proposal covariance of $0.1\mathbf{I}$. HMC used 13,500 samples after 1000 warm-up steps, with a step size of 0.05 and 20 integration steps. LMC was run for 13,500 steps with a learning rate of 0.01. SVGD was run for 500 iterations with a step size of 0.01 and 800 particles. MFVI-ADVI was run for 30,000 steps. For GM-ADVI, the parameters were matched to the GMA setup ($N=900, M=15, K=800$) with a learning rate of 0.01.

\paragraph{Results}
The results for the Neal's funnel experiment are shown in Fig.\ref{fig:GMA_test7} and Fig.\ref{fig:densities_compare_test7}, with execution times in Table.\ref{tab:times_test7}. With minimal tuning, the GMA method produced exceptionally high-quality samples, proving its capability on sampling this challenging distribution. As seen in Fig.\ref{fig:GMA_test7}, the final ensemble samples successfully capture the entire funnel geometry, from the wide, diffuse top to the narrow, high-density neck. The weight evolution plot shows that a large number of components retain non-trivial weights, which is necessary to collectively approximate the complex, conditional structure of the funnel.

In comparison, the benchmark methods showed varied performance. HMC, which is specifically designed for such geometries, performed best and serves as the gold standard. The quality of the GMA samples is highly competitive and arguably the second-best overall. MH produced reasonable samples but struggled to explore the very bottom of the funnel. LMC's performance was poor, with many samples diverging in the bottom. SVGD variants failed to capture the full geometry, with most particles concentrating in the wider part of the funnel. Variational methods were unsuccessful: MFVI-ADVI produced a simple Gaussian approximation that missed the target's structure, and GM-ADVI, despite its complexity and long runtime, resulted in a scattered, uninformative sample set, given the current parameters. This experiment verifies the impressive flexibility of the GMA sampler, demonstrating its competitive performance with specialized samplers such as HMC even on pathological distributions.

\begin{table}[H]
\centering
\scriptsize
\caption{Execution times for the 2D Neal's funnel experiment.}
\label{tab:times_test7}
\begin{tabular}{lc}
\toprule
\textbf{Method} & Execution time (s) $\downarrow$ \\
\midrule
GMA & 9.27 \\
MH & 4.44 \\
HMC & \textbf{3.03} \\
LMC & 8.70 \\
SVGD (Std init) & 82.97 \\
SVGD (GMA init) & 84.38 \\
MFVI-ADVI & 47.20 \\
GM-ADVI & 82.63 \\
\bottomrule
\end{tabular}
\end{table}

\begin{figure}[H]
    \centering
    \includegraphics[width=1.0\linewidth]{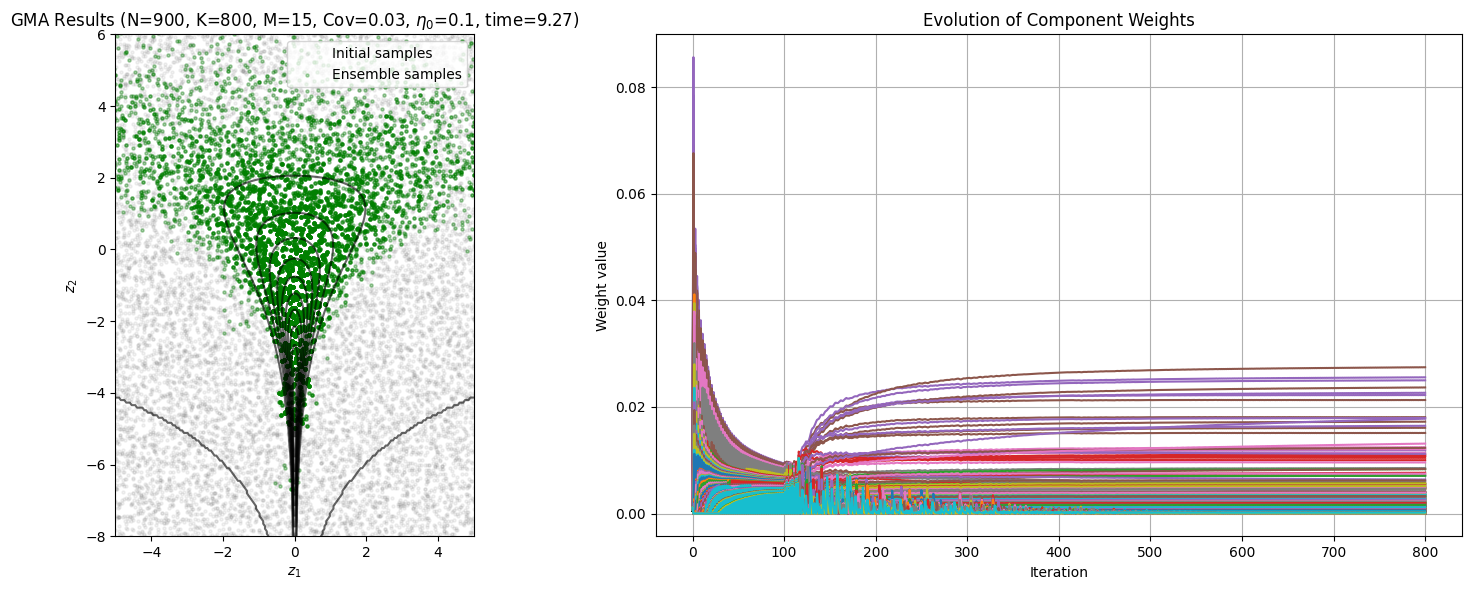}
    \caption{GMA samples and weight trajectories for the 2D Neal's funnel target.}
    \label{fig:GMA_test7}
\end{figure}

\begin{figure}[H]
    \centering
    \includegraphics[width=1.0\linewidth]{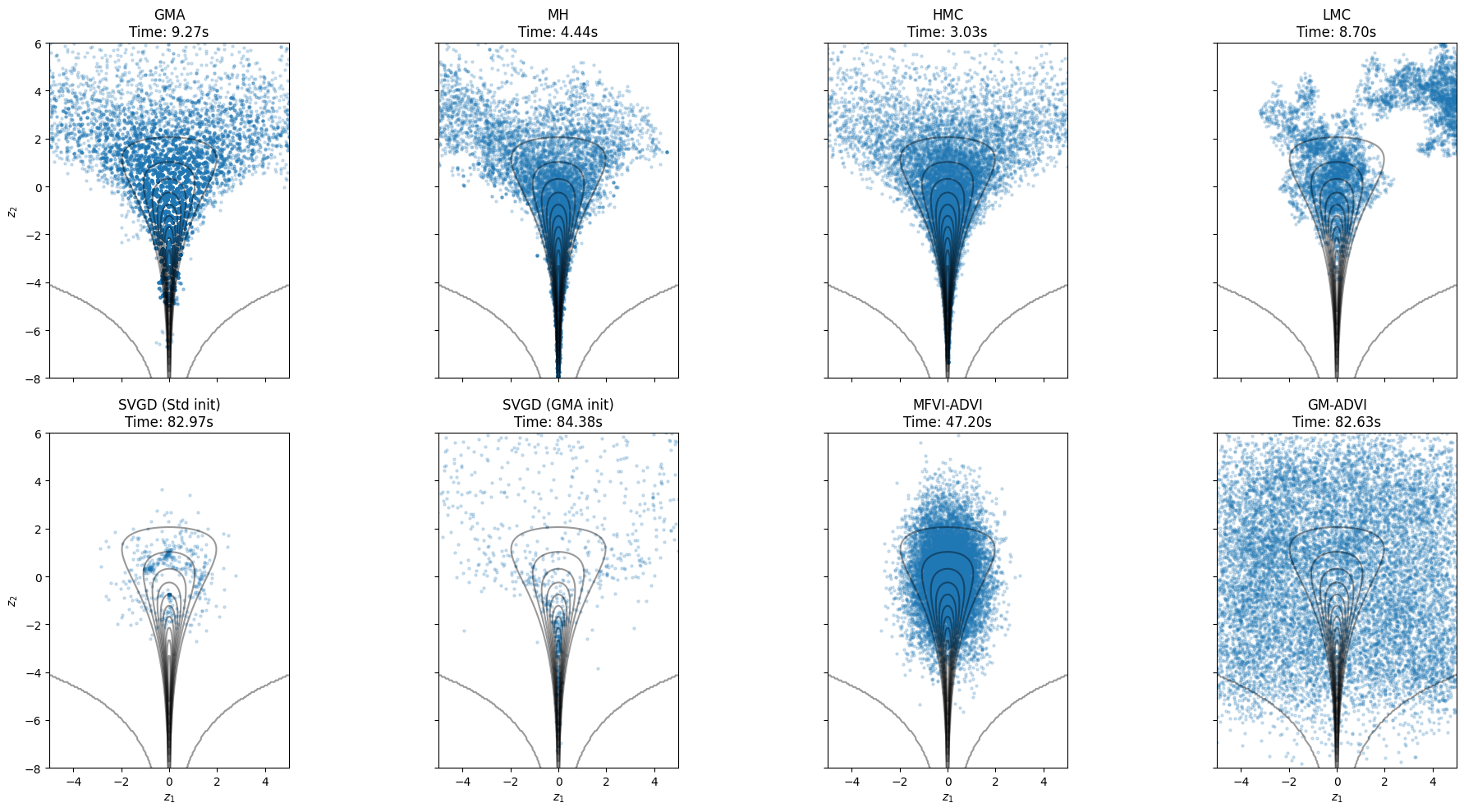}
    \caption{Density inference comparison for all methods on the 2D Neal's funnel target.}
    \label{fig:densities_compare_test7}
\end{figure}

\subsubsection{A 2D star-shaped density}
\label{subsec:star_shaped_density}

We look at another 2D, multi-modal target density \cite{Wang2021EVI} composed of $K=5$ rotated, anisotropic Gaussians arranged on a circle with strong per-arm skew. The (normalized) target is:
\[
p(\mathbf z)=\tfrac{1}{5}\sum_{k=1}^{5}\mathcal N(\mathbf z;\boldsymbol\mu_k,\Sigma_k)
\]
where $\{\boldsymbol\mu_k,\Sigma_k\}$ are generated by successive rotations of a base mean and covariance (details see Appendix.\ref{app:star_shape_implementations}). This geometry stresses both mode coverage and anisotropy.

\paragraph{Setup}
We benchmark eight samplers on inferring this complex density: EM-GMA (Section.\ref{subsec:em_GMA}), MH \cite{hastings_monte_1970}, HMC (NUTS \cite{hoffman2014nuts}), LMC \cite{roberts_exponential_1996}, SVGD \cite{liu_stein_2019}, MFVI-ADVI \cite{kucukelbir_automatic_2016}, GM-ADVI \cite{morningstar_automatic_2020}, and a particle-based energetic variational inference sampler (EVI \cite{Wang2021EVI}).
Each method produces exactly $N_{\text{draw}}=2000$ samples to enable a fair speed-accuracy comparison. Discrepancy is measured by the unbiased squared \textit{Maximum Mean Discrepancy} (MMD \cite{gretton2007kernel}, calculation see Appendix.\ref{app:distributional_distance_metrics}) $\widehat{\mathrm{MMD}}^2$ between a method’s samples and $N_{\text{ref}}=2000$ i.i.d. reference draws from the target $p$ (as the target is known and it's convenient to sample from), using a three-scale RBF kernel with bandwidths $\{0.5,1,2\}\times$ the median-heuristic (Appendix.\ref{app:star_shape_implementations}). 
We record the execution timings including warm-up/optimization (NUTS/ADVI/GM-ADVI/EVI) and the full iteration budget (SVGD/LMC/MH/EM-GMA). \textit{EM-GMA} fits a 5-component GMM to $p$ by maximizing $\mathbb{E}_{p}[\log q_\theta(\mathbf z)]$ with population EM using a self-normalized importance bank (no dataset and no target gradients used). Full hyperparameter settings are summarized in Appendix.\ref{app:star_shape_implementations}.

\paragraph{Results}
Table.\ref{tab:star_results} and Fig.\ref{fig:star_grid} show the time and sample-based $\widehat{\mathrm{MMD}}^2$. We observe that, EM-GMA achieves the best accuracy at the lowest cost (time $1.03$s, $\widehat{\mathrm{MMD}}^2=1.09{\times}10^{-5}$); it learns components that align with, and stretch along, the 5 anisotropic arms; this mass-covering behaviour resulted from EM's inclusive-KL objective. NUTS is next most accurate ($4.67{\times}10^{-3}$) but requires $\sim 18$s due to gradient-based simulation and warm-up. With some tunings for stabilization, GM-ADVI captures multi-arm structure ($1.45{\times}10^{-2}$ in $4.35$s), and particle-based EVI also recovers the star albeit at higher cost ($1.64{\times}10^{-2}$ in $33.85$s). In contrast, SVGD under-covers several arms ($8.14{\times}10^{-2}$), while LMC and MH exhibit substantial bias (\(1.55{\times}10^{-1}\) and \(1.80{\times}10^{-1}\), respectively). MFVI-ADVI collapses centrally (\(1.98{\times}10^{-1}\)), reflecting the known limitations of mean-field families on multi-modal targets.

\begin{table}[H]
\centering
\scriptsize
\caption{Star-shaped density: time and $\widehat{\mathrm{MMD}}^2$ ($N_{\text{draw}}=2000$ samples).}
\label{tab:star_results}
\begin{tabular}{lcc}
\toprule
\textbf{Method} & \textbf{Time (s) $\downarrow$} & $\boldsymbol{\widehat{\mathrm{MMD}}^2}$ \textbf{$\downarrow$} \\
\midrule
EM-GMA        &\underline{1.03}  &\underline{$1.09\times 10^{-5}$} \\
MH            & 1.06  & $1.80\times 10^{-1}$ \\
NUTS (HMC)    & 18.21 & $4.67\times 10^{-3}$ \\
LMC           & 7.26  & $1.55\times 10^{-1}$ \\
SVGD          & 1.17  & $8.14\times 10^{-2}$ \\
MFVI-ADVI     & 10.96 & $1.98\times 10^{-1}$ \\
GM-ADVI       & 4.35  & $1.45\times 10^{-2}$ \\
EVI           & 33.85 & $1.64\times 10^{-2}$ \\
\bottomrule
\end{tabular}
\end{table}

\begin{figure}[H]
    \centering
    \includegraphics[width=1.0\linewidth]{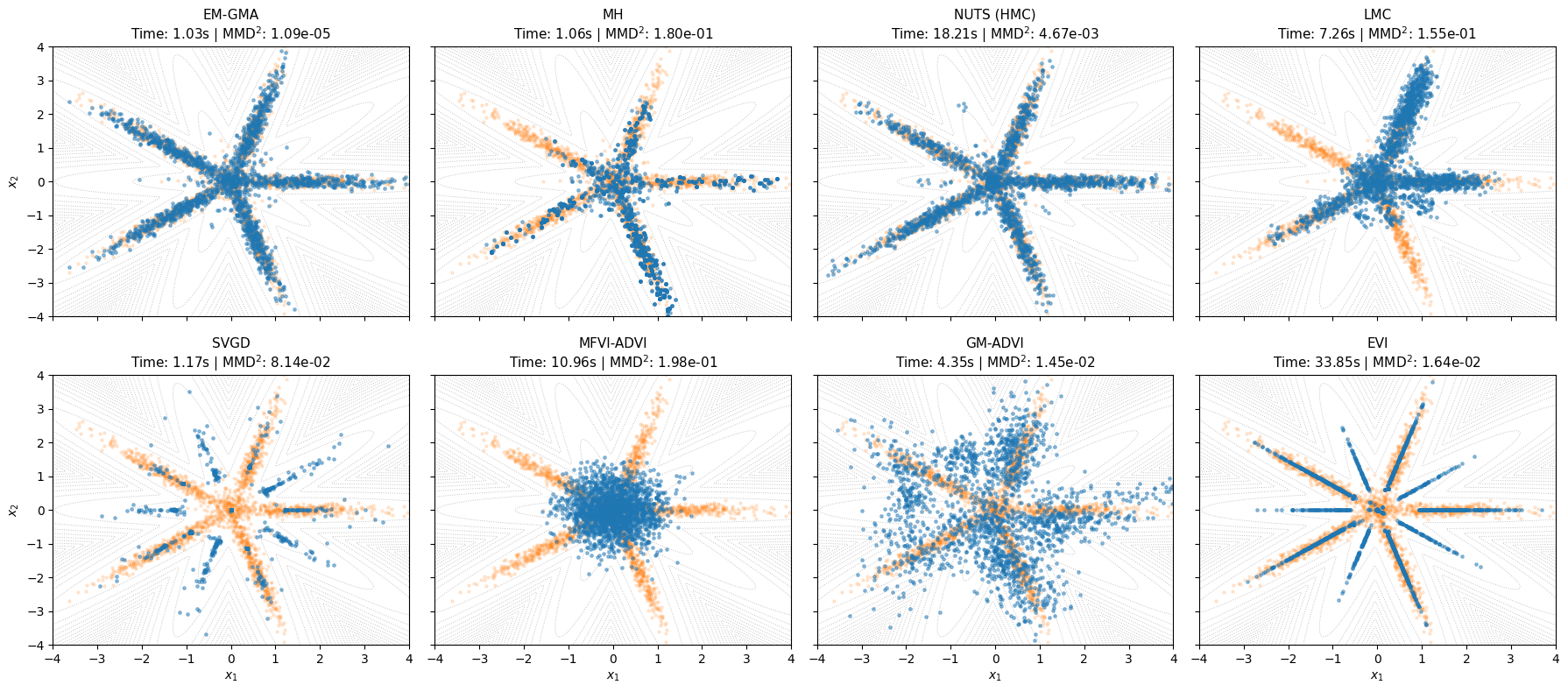}
    \caption{Star-shaped density: sample overlays for all 8 methods (blue) against $N_{\text{ref}}=2000$ (ground truth) reference samples (orange), with target contours (grey).}
    \label{fig:star_grid}
\end{figure}

\subsection{Real-world applications} 

\subsubsection{Bayesian logistic regression (BLR) with \textit{WGMA}}

We apply our method to a real-world problem: inferring the posterior distribution of the weights in a Bayesian logistic regression model. The model is trained on the Iris dataset, which we adapt for binary classification by using only two classes, \textit{Setosa} and \textit{Versicolor}, thereby removing the third class \footnote{Note, in \cite{huang2024EParVI} the author reserves all 150 instances in the original dataset and masked their labels as binary. Results here are thus not comparable to theirs, as the data are treated differently.}. This results in a subset containing 100 data points, which are split into a training set of 70 points and a test set of 30 points. The data has 4 features, making the posterior 4-dimensional (i.e. one coefficient for each feature). We place a standard normal prior $\mathcal{N}(0, \mathbf{I})$ on the regression weights $\boldsymbol{w}=[\hat{w}_1,\hat{w}_2,\hat{w}_3,\hat{w}_4]$. The posterior is then proportional to the product of the prior and the Bernoulli likelihood:
\begin{equation} \label{eq:blr_posterior}
p(\boldsymbol{w} | \mathbf{y}, \mathbf{X}) \propto p(\boldsymbol{w}) \prod_{i=1}^{N} p(y_i | \mathbf{x}_i, \boldsymbol{w}) = \exp\left(-\frac{1}{2}\boldsymbol{w}^T\boldsymbol{w}\right) \prod_{i=1}^{N} \sigma(\boldsymbol{w}^T\mathbf{x}_i)^{y_i} (1 - \sigma(\boldsymbol{w}^T\mathbf{x}_i))^{1-y_i}
\end{equation}
where $\sigma(\cdot)$ is the sigmoid function.

\paragraph{Samplers set-up}
After tuning, the GMA sampler was configured with $N=2000$ isotropic Gaussian components, each with the same small covariance scale of 0.03. The optimization was run for $K=1500$ iterations with $M=60$ samples per component and $\eta_0=0.1$. For the benchmarks, MH, HMC, and LMC were run for 120,000 total samples. SVGD was run for 500 iterations with 3200 particles. MFVI-ADVI was run for 30,000 steps. For GM-ADVI, we used $K_{mix}=100$ components, drawing $M=60$ samples from each component and optimized for 1500 steps. We compare the posterior samples to the Maximum Likelihood Estimate (MLE) as a reference point.

\paragraph{Results}
The results for the Bayesian logistic regression task are shown in Fig.\ref{fig:GMA_test8_gma}, Fig.\ref{fig:GMA_test8_weights}, Fig.\ref{fig:GMA_test8_compare}, and Table.\ref{tab:times_test8}. A quantitative summary of the posterior point estimates is provided in Table.\ref{tab:blr_stats}. In terms of predictive performance, all methods except for GM-ADVI achieved perfect (100\%) accuracy on the test set, indicating that they all successfully located the high-probability region of the posterior.

However, the quality of the posterior approximation varied significantly. As seen in Table.\ref{tab:blr_stats}, the posterior means from GMA, HMC, LMC, MH, SVGD (GMA init) and MFVI-ADVI are all close to the MLE reference, indicating accurate mode-finding. Visually, Fig.\ref{fig:GMA_test8_compare} shows these methods produced high-quality, tightly concentrated posterior distributions. The GMA samples were a bit dispersed, i.e. overestimated posterior uncertainty. The weight evolution plot (Fig.\ref{fig:GMA_test8_weights}) confirms this, showing that a large number of components retained non-trivial weights, leading to a wide final mixture. The SVGD variants were extremely slow; SVGD with GMA init also produced overly dispersed samples with means deviating from the MLE (this may be due to the non-convergence within the 500 iterations). GM-ADVI yields a poor accuracy of 13.33\% and a scattered, uninformative posterior. This task suggests that while GMA can effectively find the correct posterior for real-world problems, it also requires large number of components if the variation of each component is small in higher dimensions.

\begin{table}[H]
\centering
\scriptsize
\caption{Comparison of posterior statistics for the Bayesian logistic regression weights.}
\label{tab:blr_stats}
\begin{tabular}{l cccc cccc}
\toprule
\textbf{Method} & \multicolumn{4}{c}{\textbf{Posterior Mean}} & \multicolumn{4}{c}{\textbf{Posterior Mode}} \\
\cmidrule(lr){2-5} \cmidrule(lr){6-9}
& $\hat{w}_0$ & $\hat{w}_1$ & $\hat{w}_2$ & $\hat{w}_3$ & $\hat{w}_0$ & $\hat{w}_1$ & $\hat{w}_2$ & $\hat{w}_3$ \\
\midrule
MLE (Reference) & 0.803 & -1.000 & 1.435 & 1.442 & \multicolumn{4}{c}{---} \\
\midrule
GMA & 0.833 & -1.091 & 1.693 & 1.463 & 0.521 & -1.184 & 1.659 & 1.001 \\
MH & 0.859 & -1.075 & 1.592 & 1.569 & 0.774 & -1.040 & 1.654 & 1.444 \\
HMC & 0.871 & -1.086 & 1.577 & 1.567 & 0.859 & -1.040 & 1.682 & 1.562 \\
LMC & 0.897 & -1.106 & 1.533 & 1.581 & 0.871 & -1.044 & 1.592 & 1.590 \\
SVGD (Std init) & 0.717 & -0.843 & 1.178 & 1.166 & 0.769 & -1.026 & 1.389 & 1.479 \\
SVGD (GMA init) & 0.821 & -0.918 & 1.333 & 1.306 & 0.731 & -1.028 & 1.330 & 1.322 \\
MFVI-ADVI & 0.889 & -1.051 & 1.574 & 1.590 & 0.844 & -1.048 & 1.499 & 1.574 \\
GM-ADVI & 0.057 & -0.123 & 0.119 & -0.524 & 1.115 & -1.186 & 1.372 & 0.870 \\
\bottomrule
\end{tabular}
\end{table}

\begin{table}[H]
\centering
\scriptsize
\caption{Execution times and test accuracy for the Bayesian logistic regression experiment.}
\label{tab:times_test8}
\begin{tabular}{lcc}
\toprule
\textbf{Method} & Execution time (s) $\downarrow$ & Test Accuracy $\uparrow$ \\
\midrule
GMA & 93.26 & \textbf{1.0000} \\
MH & 11.79 & \textbf{1.0000} \\
HMC & 35.13 & \textbf{1.0000} \\
LMC & 43.10 & \textbf{1.0000} \\
SVGD (Std init) & 1348.02 & \textbf{1.0000} \\
SVGD (GMA init) & 1527.43 & \textbf{1.0000} \\
MFVI-ADVI & \textbf{5.50} & \textbf{1.0000} \\
GM-ADVI & 23.48 & 0.1333 \\
\bottomrule
\end{tabular}
\end{table}

\begin{figure}[H]
    \centering
    \includegraphics[width=0.8\linewidth]{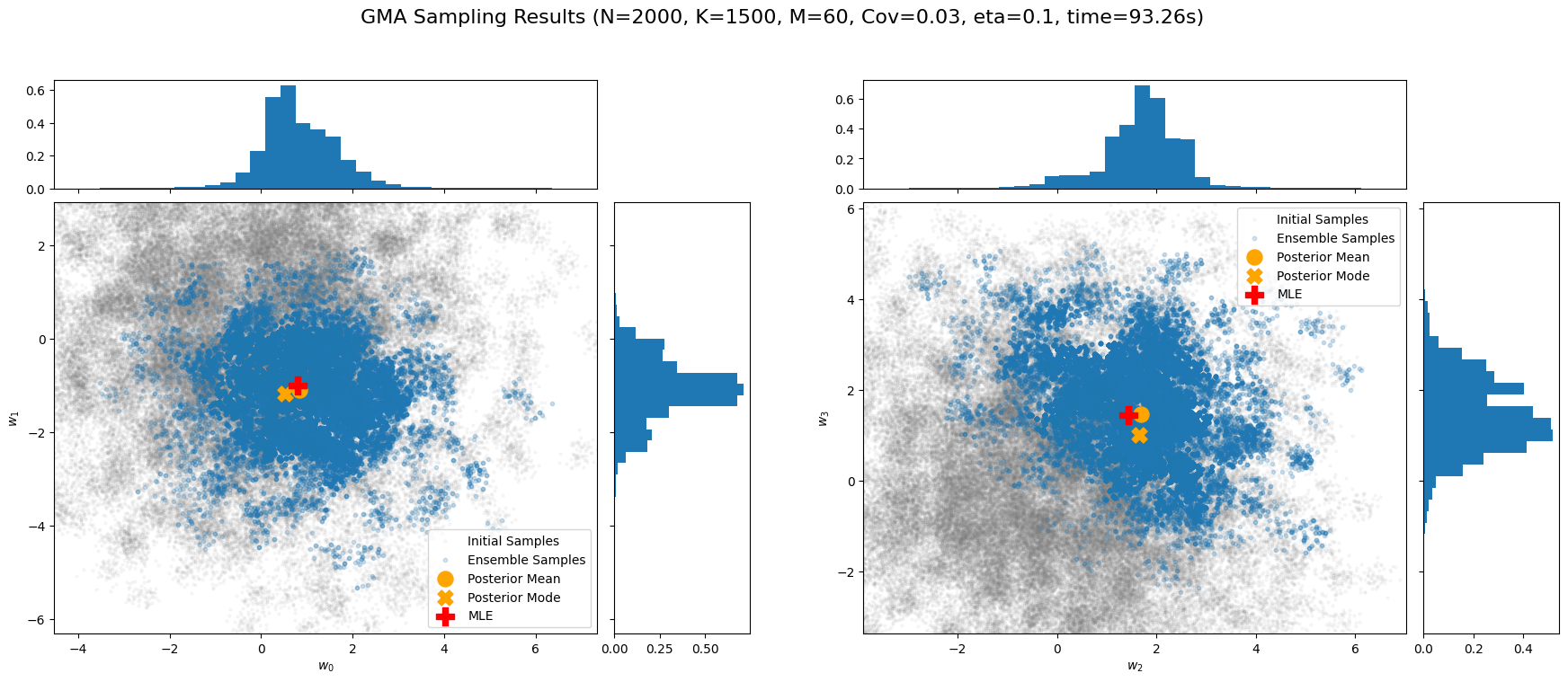}
    \caption{GMA posterior marginals for the Bayesian logistic regression weights.}
    \label{fig:GMA_test8_gma}
\end{figure}

\begin{figure}[H]
    \centering
    \includegraphics[width=0.40\linewidth]{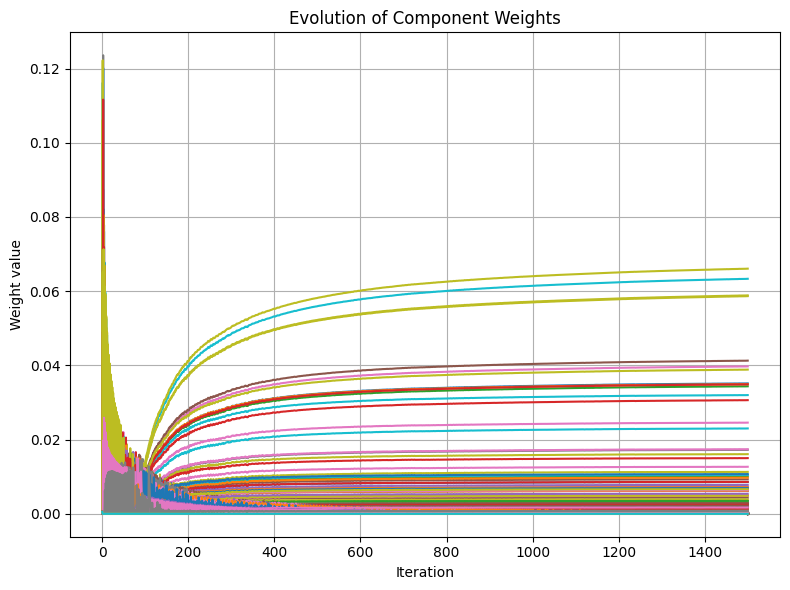}
    \caption{GMA weight evolution for the Bayesian logistic regression task.}
    \label{fig:GMA_test8_weights}
\end{figure}

\begin{figure}[H]
    \centering
    % --- Row 1 ---
    \begin{subfigure}{0.48\linewidth}
        \centering
        \includegraphics[width=\linewidth]{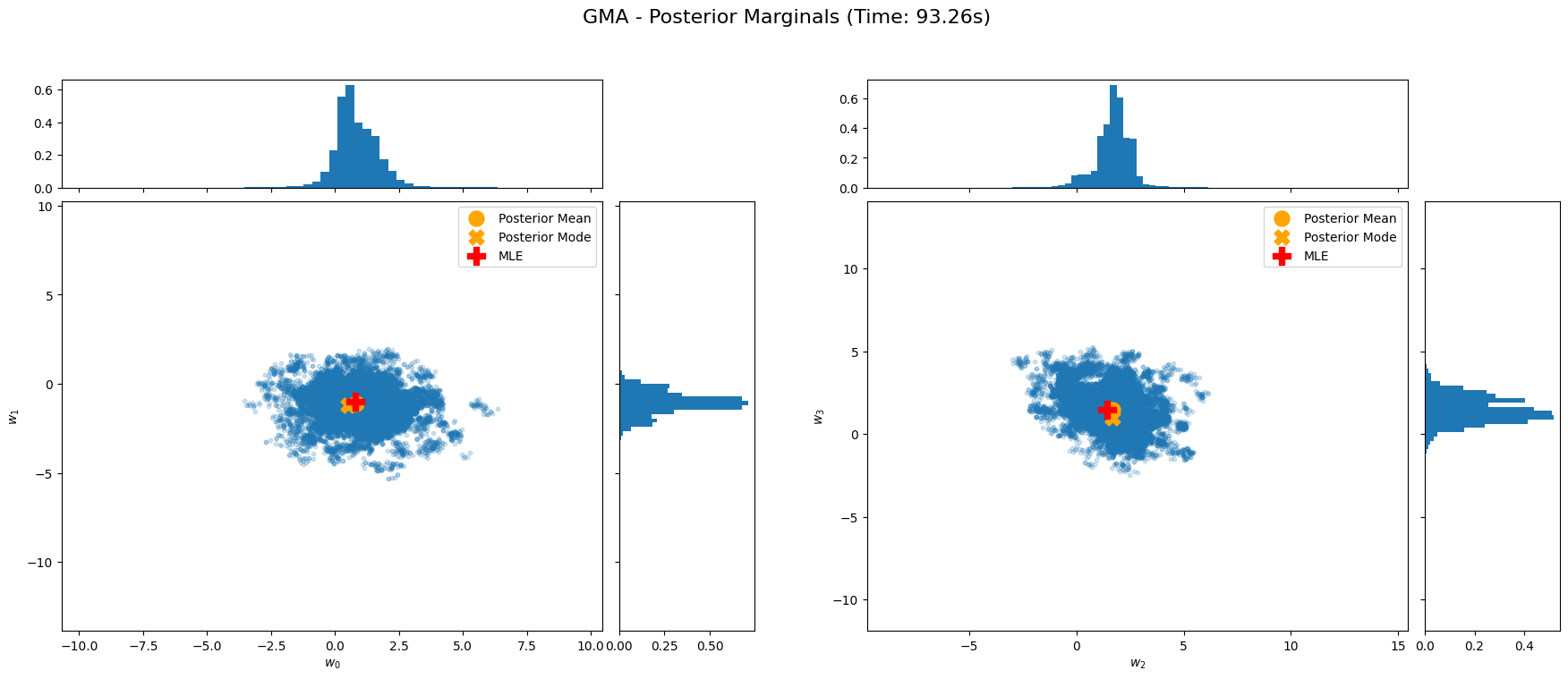}
        \caption{GMA.}
    \end{subfigure}
    \hfill
    \begin{subfigure}{0.48\linewidth}
        \centering
        \includegraphics[width=\linewidth]{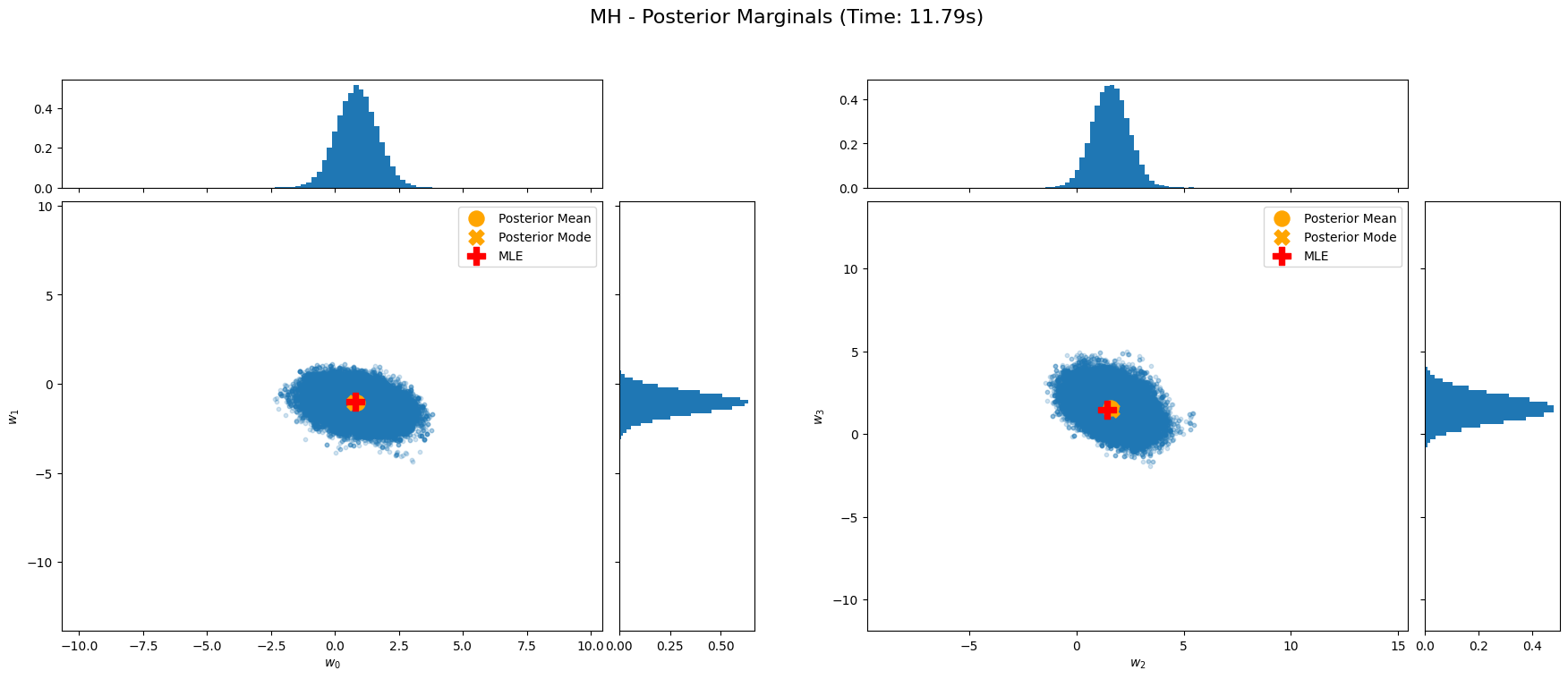}
        \caption{MH.}
    \end{subfigure}

    % --- Row 2 ---
    \begin{subfigure}{0.48\linewidth}
        \centering
        \includegraphics[width=\linewidth]{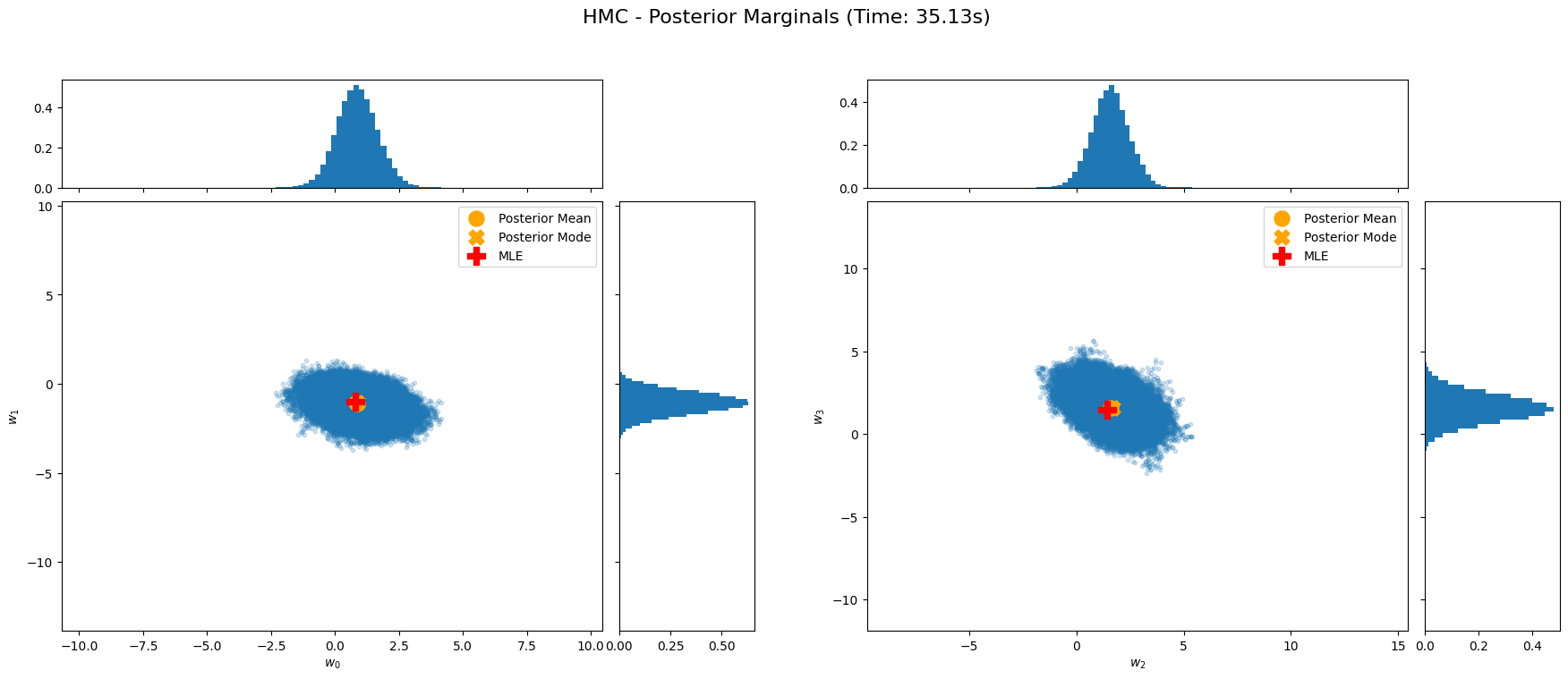}
        \caption{HMC.}
    \end{subfigure}
    \hfill
    \begin{subfigure}{0.48\linewidth}
        \centering
        \includegraphics[width=\linewidth]{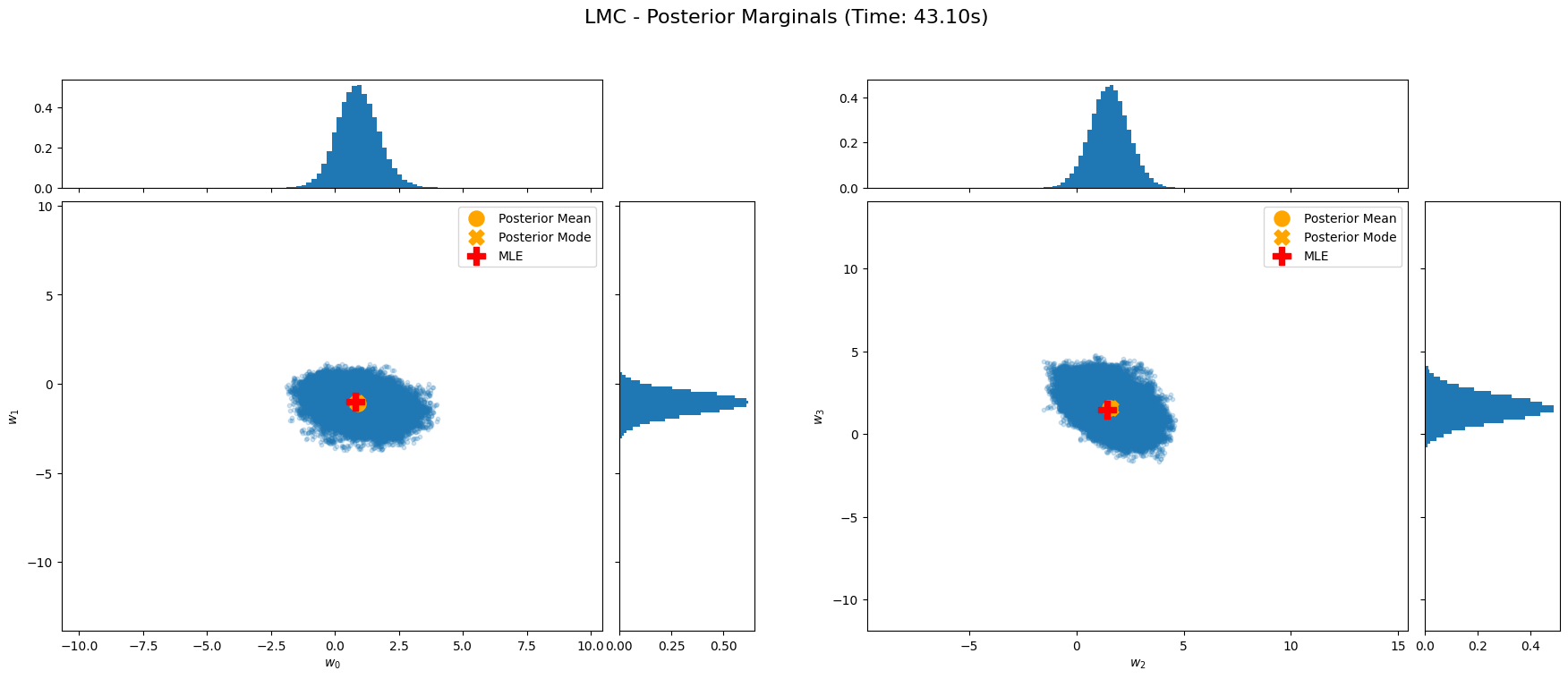}
        \caption{LMC.}
    \end{subfigure}

    % --- Row 3 ---
    \begin{subfigure}{0.48\linewidth}
        \centering
        \includegraphics[width=\linewidth]{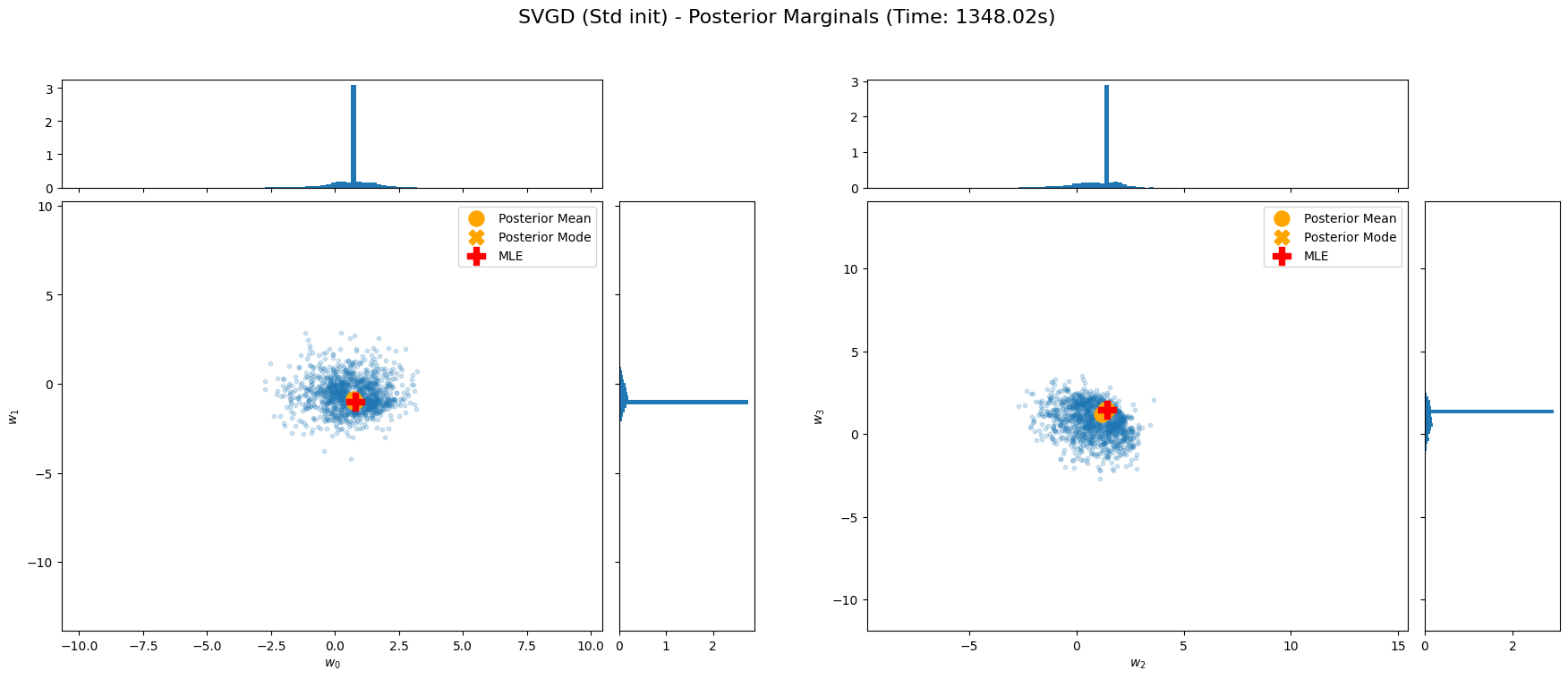}
        \caption{SVGD (Std init).}
    \end{subfigure}
    \hfill
    \begin{subfigure}{0.48\linewidth}
        \centering
        \includegraphics[width=\linewidth]{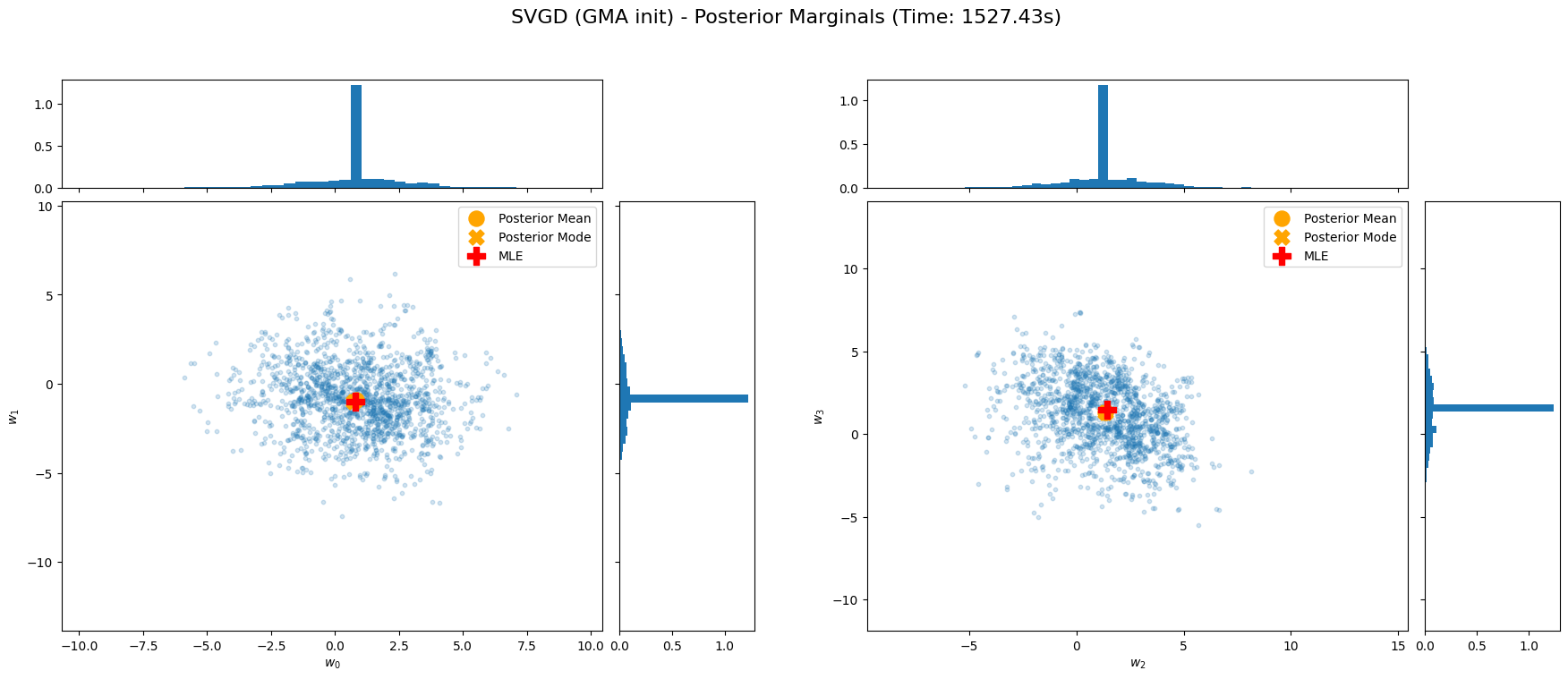}
        \caption{SVGD (GMA init).}
    \end{subfigure}

    % --- Row 4 ---
    \begin{subfigure}{0.48\linewidth}
        \centering
        \includegraphics[width=\linewidth]{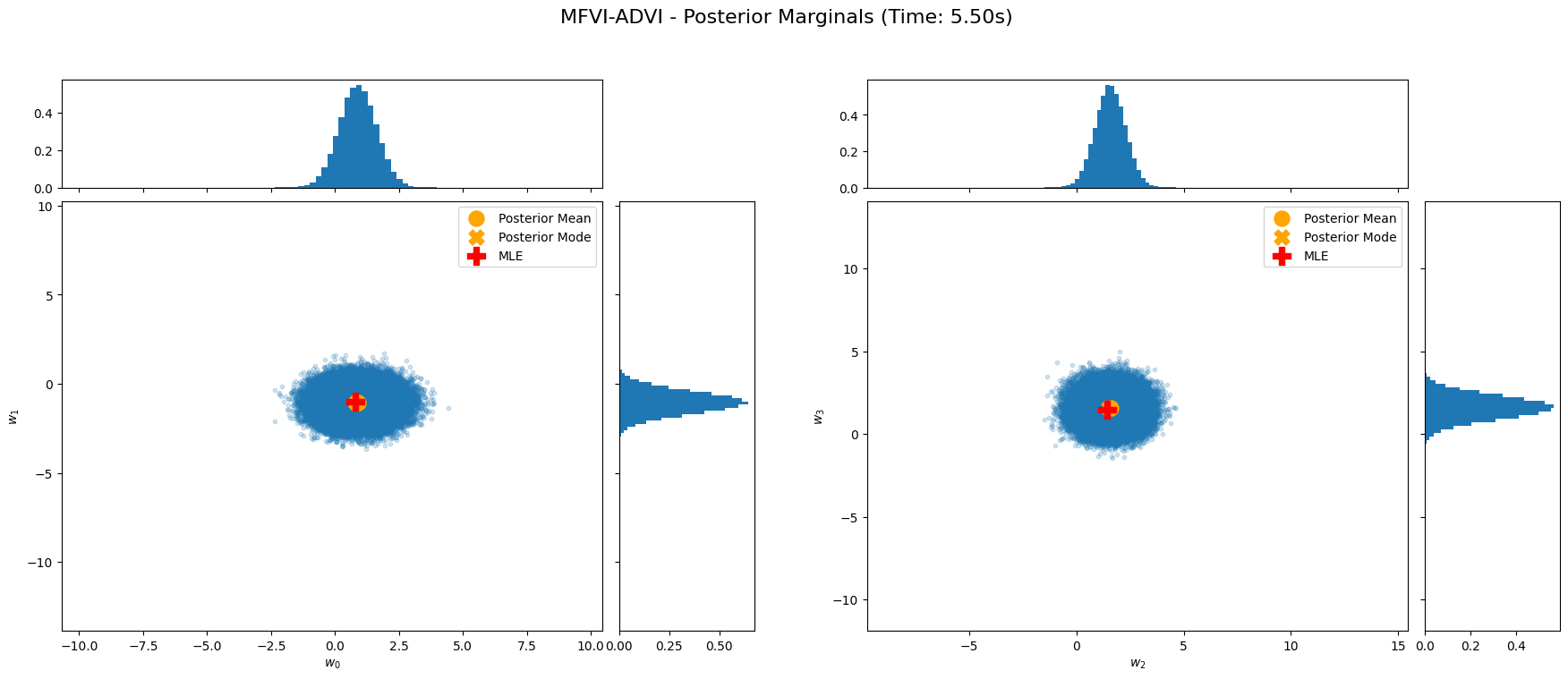}
        \caption{MFVI-ADVI.}
    \end{subfigure}
    \hfill
    \begin{subfigure}{0.48\linewidth}
        \centering
        \includegraphics[width=\linewidth]{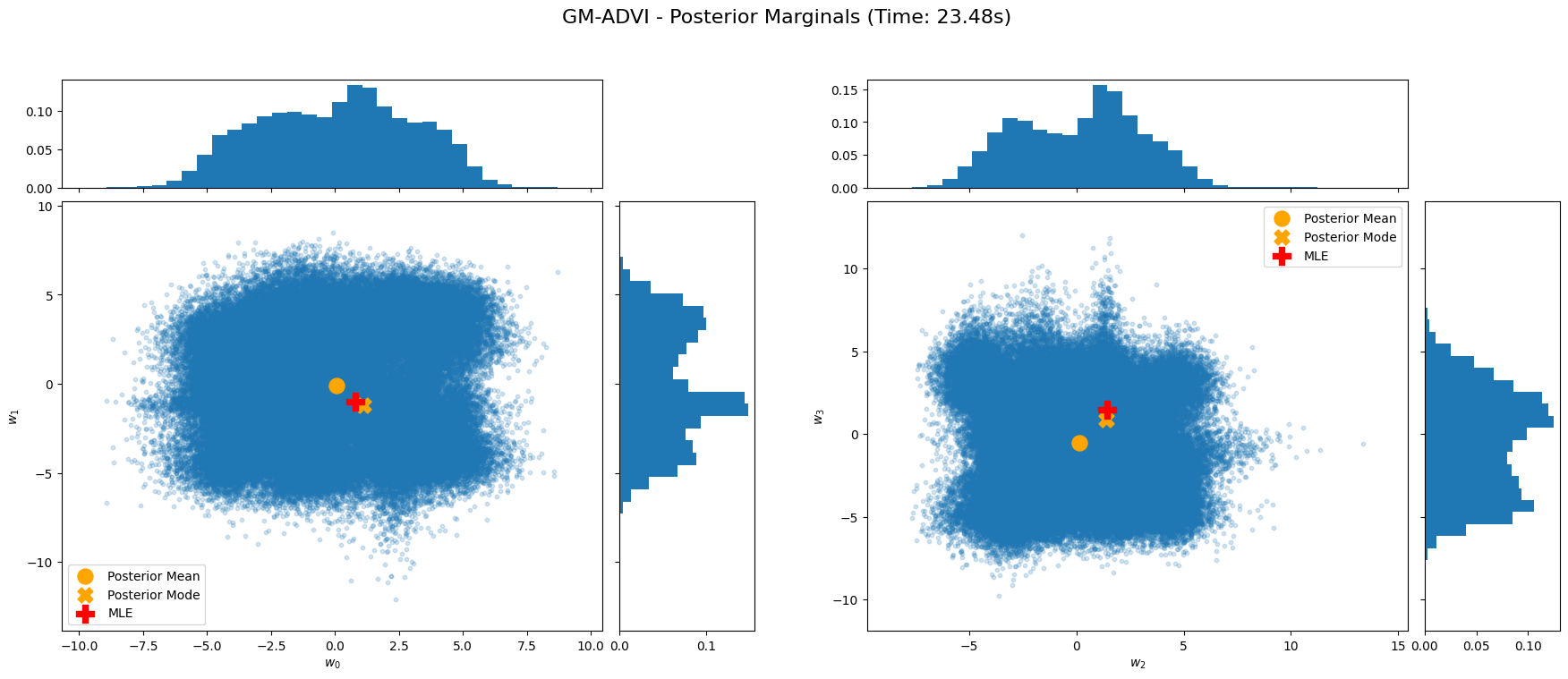}
        \caption{GM-ADVI.}
    \end{subfigure}

    \caption{Comparison of posterior marginals for all methods on the Bayesian logistic regression task. Orange points indicate the posterior mean and mode; the red cross indicates the MLE.}
    \label{fig:GMA_test8_compare}
\end{figure}

\subsubsection{Hierarchical Bayesian linear regression: Minnesota Radon prediction with \textit{WGMA}}

For a more complex, real-world application, we perform hierarchical Bayesian linear regression on the \textit{Radon contamination dataset} \cite{Radon_dataset,Fonnesbeck2021primer}, a well-known benchmark for multi-level modeling \footnote{Python implementations of different variants of the Radon hierarchical model can be found in packages e.g. \textit{PyMC} \cite{Fonnesbeck2021primer} and \textit{Bambi} \cite{Bambi_radon_example}.} \cite{Gelman_Hill_2006,Gelman01082006}. Radon is a radioactive gas and a known carcinogen; its concentration in homes is influenced by factors such as the floor of the house and the local geology.

\paragraph{Data}
The dataset contains radon measurements from 919 households across 85 counties in Minnesota \cite{Gelman01082006,edinburgh_bayesian_2023}. For validation, the data was split into a training set of 735 points (80\%) and a test set of 184 points (20\%). The primary goal is to predict the log-radon level in a house based on two predictors:
\begin{enumerate}
    \item \textit{Household-level predictor $x$}: The floor on which the measurement was taken ($x_i=0$ for basement, $x_i=1$ for first floor).
    \item \textit{County-level predictor $u$}: The log-uranium level of the soil in the county where the house is located ($u_j$).
\end{enumerate}
The hierarchical structure arises from households being nested within counties. This structure suggests that while radon levels might have a general relationship with the floor type, the baseline radon level can vary significantly from one county to another. 

An initial exploratory data analysis (Fig.\ref{fig:Bayesian_hierarchical_LR_EDA}) shows that the response variable, \textit{log-radon}, is approximately normally distributed across all households, with a mean of $1.26$ and a standard deviation of $0.82$. This standard deviation provides a key performance baseline; the predictions from any useful model must achieve a root-mean-square error (RMSE) lower than this value, as it represents the error from a naive model that only predicts the overall mean.

\begin{figure}[H]
    \centering
    \includegraphics[width=0.5\linewidth]{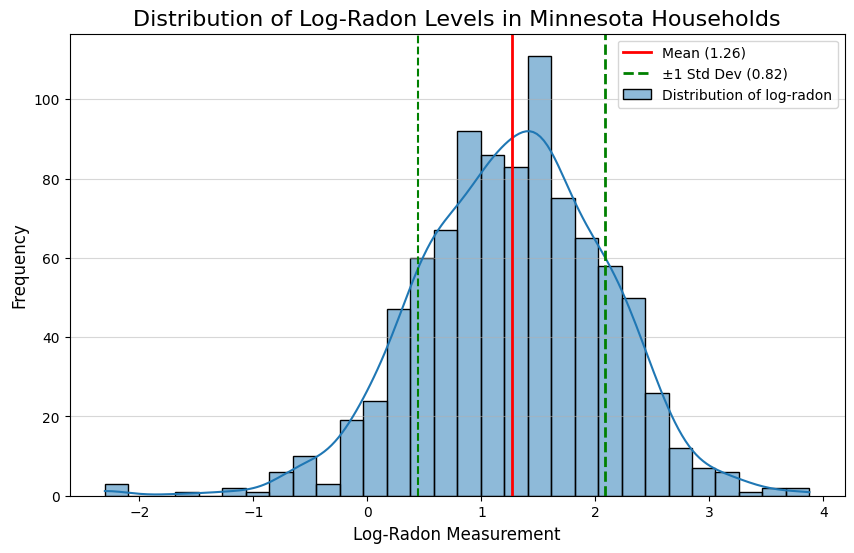}
    \caption{Explanatory data analysis of the log-radon measurements across all 919 households. The distribution is approximately normal, with a mean of $1.26$ and a standard deviation of $0.82$. This standard deviation serves as a baseline error for model evaluation.}
    \label{fig:Bayesian_hierarchical_LR_EDA}
\end{figure}

\paragraph{Hierarchical model with group-level predictors}
We implement one hierarchical model from \cite{Fonnesbeck2021primer}. The model allows the intercept to vary by county, and this variation is itself modeled by a linear regression on the county-level uranium measurements.

The likelihood for the log-radon level $y_i$ in house $i$ located in county $j$ is given by:
\begin{equation}
y_i \sim \mathcal{N}(\alpha_{j[i]} + \beta x_i, \sigma_y^2)
\end{equation}
where $\beta$ is the common slope for the floor effect, and $\alpha_{j[i]}$ is the intercept specific to county $j$.

The key to the hierarchical model is that the county-specific intercepts $\alpha_j$ are not independent but are drawn from a common distribution whose mean depends on the county's uranium level:
\begin{equation} \label{eq:hierarchicalBLR_intercept}
\alpha_j = \underbrace{\gamma_0 + \gamma_1 u_j}_{\text{Predicted Trend, } \mu_j} + \underbrace{\epsilon_j}_{\text{Random Offset}}
\end{equation}
where the random offset $\epsilon_j$ represents the residual variation for county $j$ after accounting for the effect of uranium. These residuals are modeled as:
\begin{equation}
\epsilon_j \sim \mathcal{N}(0, \sigma_\alpha^2)
\end{equation}

In Eq.\ref{eq:hierarchicalBLR_intercept}, the fixed trend $\mu_j = \gamma_0 + \gamma_1 u_j$ allows the intercept to vary between counties. It defines a straight line where the average intercept for a county is predicted by its uranium level ($u_j$). The random offset $\epsilon_j$ is the crucial hierarchical part. Its prior is $\epsilon_j \sim \mathcal{N}(0, \sigma_\alpha^2)$. This term models how much each county's true intercept deviates from the predicted trend line, capturing the variation that is not explained by uranium levels. 
The power of this hierarchical model comes from learning the magnitude of the random offsets by estimating $\sigma_\alpha$. A well-estimated, non-zero $\sigma_\alpha$ allows the model to perform \textit{hierarchical shrinkage} (partial pooling), where individual county estimates can deviate (if a random offset is learned non-zero) from the main trend but are "pulled" back towards it, borrowing statistical strength from the entire dataset. This makes the hierarchical model more flexible.

\paragraph{Priors}
We use weakly informative priors for the model's parameters and hyper-parameters (same as those used in \cite{Fonnesbeck2021primer}):
\begin{itemize}
    \item \textit{Fixed effects}: $\gamma_0 \sim \mathcal{N}(0, 10^2)$, $\gamma_1 \sim \mathcal{N}(0, 10^2)$, $\beta \sim \mathcal{N}(0, 10^2)$
    \item \textit{Variance components}: $\sigma_\alpha \sim \text{HalfCauchy}(5)$, $\sigma_y \sim \text{Uniform}(0, 100)$
\end{itemize}

\paragraph{Posterior}
Combining the likelihood and priors, the full unnormalized posterior distribution for the 90 parameters $\boldsymbol{\theta} = \{\gamma_0, \gamma_1, \beta, \sigma_\alpha, \sigma_y, \{\epsilon_j\}_{j=1}^{85}\}$ is given by:
\begin{equation} \label{eq:radon_posterior}
\begin{split}
p(\boldsymbol{\theta} | \mathbf{y}, \mathbf{X}, \mathbf{u}) \propto & \left[ \prod_{i=1}^{N} \mathcal{N}(y_i | \alpha_{j[i]} + \beta x_i, \sigma_y^2) \right] \cdot \left[ \prod_{j=1}^{J} \mathcal{N}(\epsilon_j | 0, \sigma_\alpha^2) \right] \\
& \cdot \mathcal{N}(\gamma_0|0,10^2) \cdot \mathcal{N}(\gamma_1|0,10^2) \cdot \mathcal{N}(\beta|0,10^2) \\
& \cdot \text{HalfCauchy}(\sigma_\alpha|5) \cdot \text{Uniform}(\sigma_y|0,100)
\end{split}
\end{equation}
The full set of parameters to be inferred from the posterior distribution is $\boldsymbol{\theta}$, which results in a high-dimensional posterior with 90 parameters in total.

\paragraph{Samplers set-up}
After some tuning, the GMA sampler was configured with $N=1500$ components, $M=100$ samples per component, and run for $K=1500$ iterations with $\eta_0=0.1$ and a very small covariance scale of $10^{-4}$. For the benchmarks, MH was run for 150,000 iterations with a proposal covariance of $0.01\mathbf{I}$. HMC \footnote{For previous experiments, we used the Python package \textit{Blackjax} \cite{cabezas2024blackjax} which had dependency on XLA via JAX \cite{jax2018github}; for this hierarchical Bayesian posterior, we used PyMC \cite{abrilpla2023pymc} for implementing HMC (NUTS).} (NUTS \cite{hoffman2014nuts}) was run for a single chain, generating 150,000 posterior samples after 2000 tuning steps. MFVI-ADVI was optimised for 50,000 steps and then used to generate 150,000 samples from the resulting approximation.
We did not employ LMC as we found in our trials that LMC samples in this case are highly sensitive to the small learning rate; also we discard SVGD and GM-ADVI as they are slow and sample qualities are not evidenced to be good. Also, a Gibbs sampler is not a feasible choice for this model as the chosen priors for the variance components (i.e. $\sigma_{\alpha} \sim HalfCauchy$ and $\sigma_y \sim Uniform$) are not conjugate with the Gaussian likelihood. Implementing the necessary Metropolis-within-Gibbs steps for these parameters would be complicated.

\paragraph{Results}
The results for this high-dimensional, real-world inference task are presented in Fig.\ref{fig:GMA_test9_posteriors}, Fig.\ref{fig:GMA_test9_county_effects}, Fig.\ref{fig:GMA_test9_predictive}, and Table.\ref{tab:times_test9}. The posterior predictive check on the held-out test data (Fig.\ref{fig:GMA_test9_predictive}) shows that HMC and MFVI-ADVI achieved the best performance with the lowest RMSE of 0.706. The MH sampler performed nearly well, with an RMSE of 0.708. The GMA sampler was slightly less accurate, with an RMSE of 0.780. All three methods substantially improve upon the baseline error of $0.82$.

An examination of the posterior distributions for the 5 global parameters and 1 county random offset (Fig.\ref{fig:GMA_test9_posteriors}) reveals the reason for this performance gap. The posteriors from HMC, MH, and MFVI-ADVI are all in strong agreement, identifying the same posterior distributions for the global parameters; the GMA posteriors deviate from them in some dimensions such as the slope $\beta$, the noise standard deviation $\sigma_y$ and the first county offset $\epsilon_1$. 
Interestingly, the marginal posteriors from GMA are noticeably sharper, which means it is possibly over-confident than those from the other methods. This posterior sharpness can be misleading, if the posteriors are centered on biased values.

This misalignment between GMA and other methods' estimations explains the behavior seen in the county intercepts plot (Fig.\ref{fig:GMA_test9_county_effects}),  where the GMA-estimated county intercepts are much more scattered around the mean trend compared to the tight clustering of the HMC, MH and MFVI-ADVI estimates. By inferring a different value for the group-level variance $\sigma_\alpha$, the model's predictive accuracy also degrades. In terms of efficiency (Table.\ref{tab:times_test9}), MFVI-ADVI was the fastest method, while other 3 methods (MH, HMC, GMA) show similar speed. The evolution of GMA's weights during optimization is shown in Fig.\ref{fig:GMA_test9_weights}.

This experiment highlights a potential pitfall of GMA sampling in complex, high-dimensional spaces. Due to its mode-seeking $KL(q||p)$ objective, the sampler can be compared to a mountain climber who finds a single, very sharp, and isolated peak. The climber reports back with extreme confidence about this peak's location, ignoring the wider, more probable valley system that other explorers (such as HMC) have found. In this case, GMA may have converged to a sharp, local minimum in the objective, resulting in a not true posterior and a brittle solution that does not generalize well. 

\begin{figure}[H]
    \centering
    \includegraphics[width=0.5\linewidth]{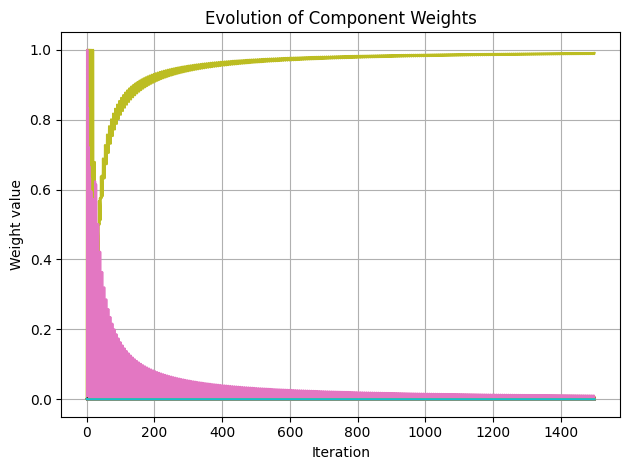}
    \caption{Evolution of component weights during GMA optimization. The weights converge, but the final mixture represents a sharp, local minimum rather than the true posterior.}
    \label{fig:GMA_test9_weights}
\end{figure}

\begin{table}[H]
\centering
\scriptsize
\caption{Execution times and predictive performance for the Radon hierarchical model.}
\label{tab:times_test9}
\begin{tabular}{lcc}
\toprule
\textbf{Method} & Execution time (s) $\downarrow$ & Test Set RMSE $\downarrow$ \\
\midrule
GMA & 443.33 & 0.780 \\
MH & 376.72 & 0.708 \\
HMC (NUTS) & 328.54 & \textbf{0.706} \\
MFVI-ADVI & \textbf{68.42} & \textbf{0.706} \\
\bottomrule
\end{tabular}
\end{table}

\begin{figure}[H]
    \centering
    \includegraphics[width=1.0\linewidth]{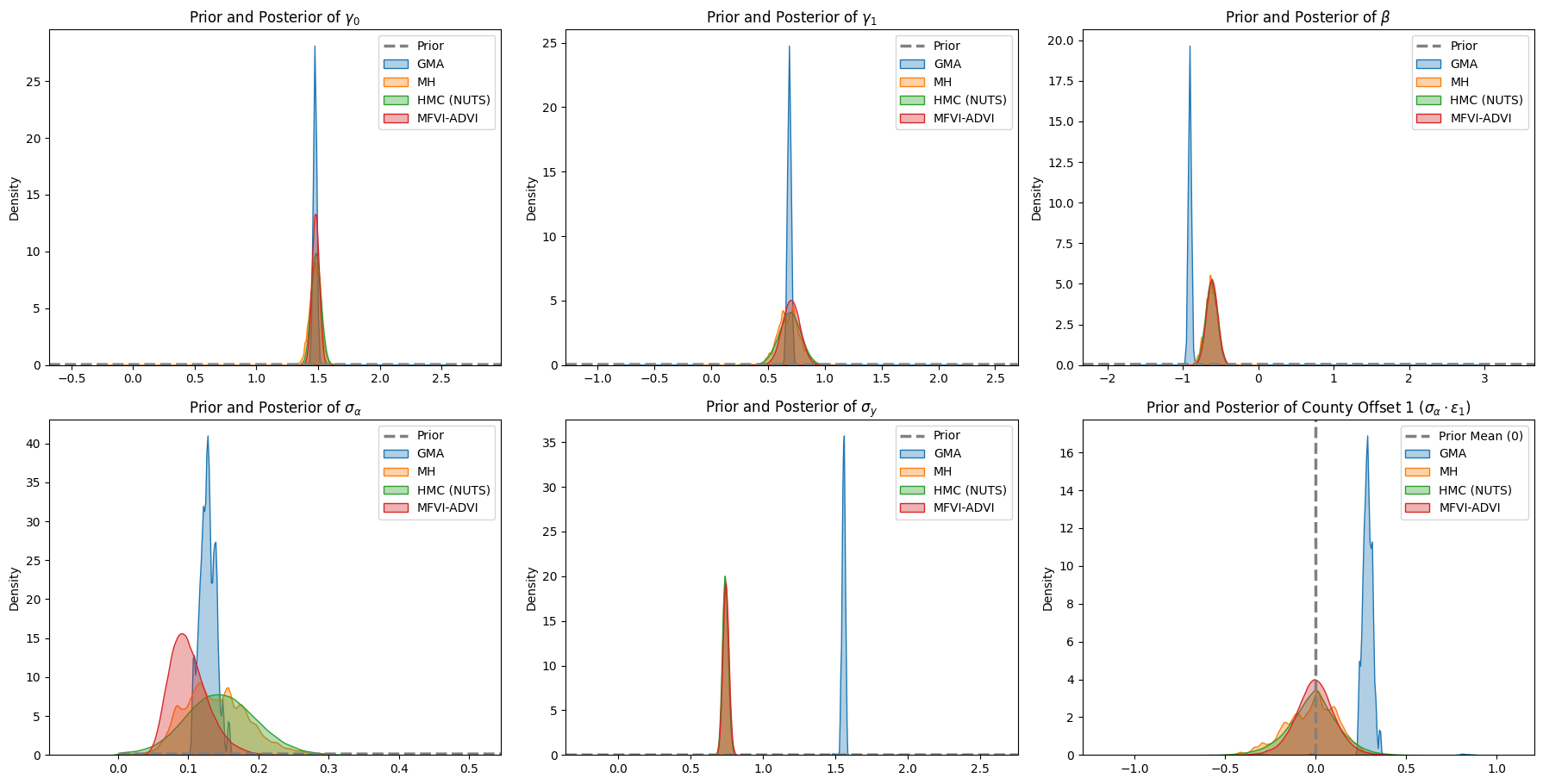}
    \caption{Marginal posterior distributions for the Radon hierarchical model. HMC, MH, and MFVI-ADVI are in strong agreement. GMA produces sharper but different locale posteriors in some dimensions. Note, the last plot is $\epsilon_1$, as we used the re-parameterization trick $\epsilon_1 \sim \sigma_{\alpha}^2 \cdot \mathcal{N}(0,1)$ in our implementation.}
    \label{fig:GMA_test9_posteriors}
\end{figure}

\begin{figure}[H]
    \centering
    \begin{minipage}{0.48\linewidth}
        \centering
        \includegraphics[width=\linewidth]{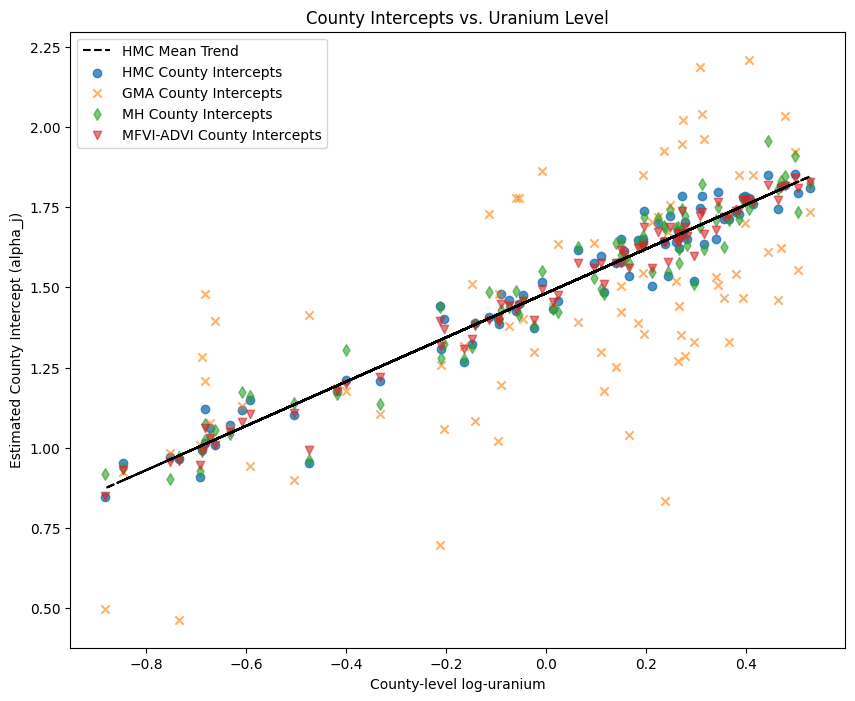}
        \caption{Mean estimated county-level intercepts ($\alpha_j$) versus the county-level log-uranium predictor. The estimates from HMC, MH and MFVI-ADVI cluster tightly around the mean trend, while the GMA estimates show much higher variance.}
        \label{fig:GMA_test9_county_effects}
    \end{minipage}
    \hfill
    \begin{minipage}{0.48\linewidth}
        \centering
        \includegraphics[width=0.8\linewidth]{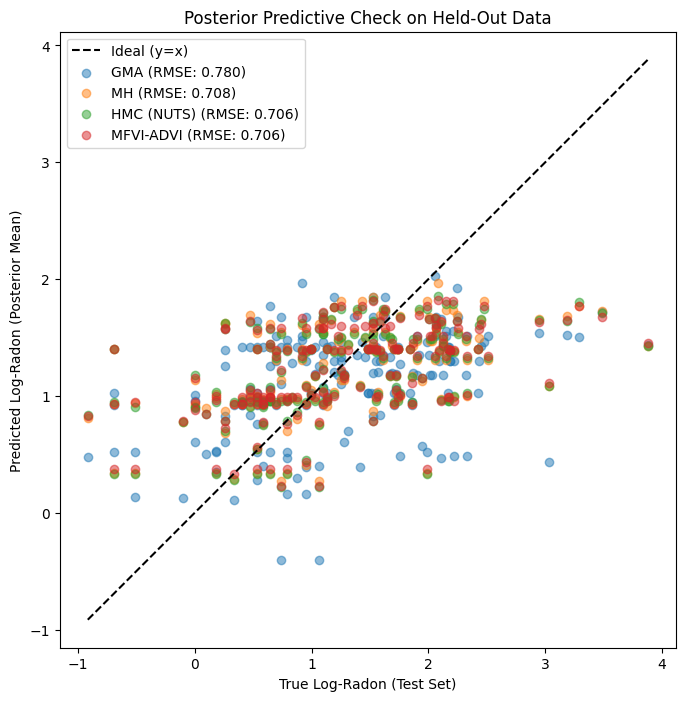}
        \caption{Posterior predictive check on the held-out test set. The predictions from HMC, MH, and MFVI-ADVI show the lowest error (tightly clustered around the ideal $y=x$ line), whereas GMA's predictions are more scattered.}
        \label{fig:GMA_test9_predictive}
    \end{minipage}
\end{figure}

\subsubsection{Hierarchical Bayesian symbolic regression (BSR) with \textit{WGMA} for pendulum physics discovery}
Symbolic regression (SR) is a form of regression analysis that explores possible mathematical expressions to identify a model which balances predictive accuracy with interpretive simplicity for the given dataset.
Here we address the problem of discovering physical laws from noisy observational data through hierarchical Bayesian symbolic regression
\footnote{Traditional symbolic regression does not assume a prior model structure, whereas h-BSR introduces explicit beliefs over both models and parameters. Unlike standard SR, which searches over operators, functions, constants, and variables within a predefined dictionary, our approach makes the candidate space explicit: we specify three potential models with associated parameters, and then update our beliefs using data. This formulation can still be thought of as a simplified symbolic search problem, but it's more naturally aligned with Bayesian regression and model selection.}
(h-BSR). Given a dataset $\mathcal{D} = \{(L_i, T_i)\}_{i=1}^n$, where $L_i \in \mathbb{R}^+$ is the pendulum length and $T_i \in \mathbb{R}^+$ is the observed period, the goal is to infer the most plausible functional form from a set of candidate models $\mathcal{M} = \{\mathcal{M}_1, \mathcal{M}_2, \mathcal{M}_3\}$ while simultaneously estimating the parameters for each model. This approach frames the task as a Bayesian model selection problem, allowing for principled comparison of competing physical hypotheses.

The true dynamics of a simple pendulum are governed by the equation \cite{halliday1997fundamentals,lima_accurate_2006}:
\begin{equation}
T = 2\pi\sqrt{\frac{L}{g}} \approx 2.0061 \cdot L^{0.5}
\end{equation}
where $g = 9.81 \text{m/s}^2$ is the gravitational acceleration. This \textit{length-period relation} is valid when the swing angle is small (e.g. less than 1 radian) \cite{halliday1997fundamentals}. We aim to recover this simple, known relationship from synthetic, noisy data.

\paragraph{Hierarchical Bayesian model.} We define three competing models, each representing a different hypothesis about the pendulum's governing dynamics:
\begin{align}
\mathcal{M}_1 \text{ (Linear):} \quad & T = a + bL \\
\mathcal{M}_2 \text{ (power law):} \quad & T = aL^b + c \\
\mathcal{M}_3 \text{ (Exponential):} \quad & T = a \cdot 10^{L/10} + \gamma
\end{align}
The power law model, $\mathcal{M}_2$, is of particular interest as it contains the true functional form as a special case where $a \approx 2.0061$, $b=0.5$, and $c=0$.

\subparagraph{Model priors.} We assign prior probabilities to each model based on physical intuition. The power law model is given the highest prior probability, while the exponential model is deemed highly unlikely. These discrete prior probabilities reflect our trust-ness in these 3 models before observing any data.
\begin{equation}
p(\mathcal{M}_k) = \begin{cases}
0.3 & \text{for } \mathcal{M}_1 \text{ (linear)} \\
0.6 & \text{for } \mathcal{M}_2 \text{ (power law)} \\
0.1 & \text{for } \mathcal{M}_3 \text{ (exponential)}
\end{cases}
\end{equation}

\subparagraph{Parameter priors.} For each model $\mathcal{M}_k$, we specify prior distributions over its parameter vector $\boldsymbol{\theta}_k$ and the observation noise $\sigma$. A shared noise parameter $\sigma$ is estimated within each model.

\begin{itemize}
    \item \textit{Linear model ($\mathcal{M}_1$):} weakly informative priors are used for the parameters $\boldsymbol{\theta}_1 = \{a, b, \sigma\}$.
    \[
        a \sim \mathcal{N}(0, 2), \quad b \sim \mathcal{N}(0, 2), \quad \sigma \sim \text{HalfNormal}(1)
    \]
    
    \item \textit{power law model ($\mathcal{M}_2$):} physics-informed priors are placed on the parameters $\boldsymbol{\theta}_2 = \{a, b, c, \sigma\}$ to guide the inference towards the theoretical relationship.
    \[
        a \sim \mathcal{N}(2.0, 0.5), \quad b \sim \mathcal{N}(0.5, 0.1), \quad c \sim \mathcal{N}(0.0, 0.2), \quad \sigma \sim \text{HalfNormal}(0.5)
    \]
    To ensure numerical stability during sampling, the parameters are bounded such that $a \in [0.1, 10.0]$ and $b \in [0.1, 1.0]$. These bounds enforce physically plausible constraints (e.g. positive correlation and non-extreme power dependence).

    \item \textit{Exponential model ($\mathcal{M}_3$):} wide, uninformative priors are used for the parameters $\boldsymbol{\theta}_3 = \{a, \gamma, \sigma\}$, as there is no physical basis for this functional form.
    \[
        a \sim \mathcal{N}(0, 5), \quad \gamma \sim \mathcal{N}(0, 5), \quad \sigma \sim \text{HalfNormal}(1)
    \]
    The model is formulated as $T = a \cdot 10^{L/10} + \gamma$, where the length $L$ is scaled to prevent numerical overflow.
\end{itemize}

\subparagraph{Likelihood and posterior inference.} Assuming independent and identically distributed Gaussian noise for the observations, the \textit{likelihood} for each model is given by:
\begin{equation}
p(\mathcal{D} | \boldsymbol{\theta}_k, \mathcal{M}_k) = \prod_{i=1}^n \mathcal{N}(T_i | f_k(L_i; \boldsymbol{\theta}_k), \sigma^2)
\end{equation}
where $f_k(L_i; \boldsymbol{\theta}_k)$ is the prediction from model $\mathcal{M}_k$. The \textit{joint posterior distribution} over the parameters and models is formulated using Bayes' theorem:
\begin{equation}
p(\mathcal{M}_k, \boldsymbol{\theta}_k | \mathcal{D}) \propto p(\mathcal{D} | \boldsymbol{\theta}_k, \mathcal{M}_k) p(\boldsymbol{\theta}_k | \mathcal{M}_k) p(\mathcal{M}_k)
\end{equation}

\subparagraph{Model comparison and selection.} We compare the candidate models based on their expected out-of-sample predictive performance using the Leave-One-Out Cross-Validation (LOO-CV) information criterion. The standard LOO-CV score is the expected log pointwise predictive density (ELPD \cite{spiegelhalter_bayesian_2002}):
\begin{equation}
\text{ELPD}_{\text{loo}} = \sum_{i=1}^n \log p(T_i | \mathcal{D}_{-i}, \mathcal{M}_k)
\end{equation}
where $\mathcal{D}_{-i}$ is the dataset with the $i$-th observation held out. This metric assesses how well a model, trained on $n-1$ data points, predicts the remaining point, averaged over all points. A higher ELPD score indicates better predictive accuracy.

From these scores, we compute model weights using a method known as stacking of predictive distributions \cite{yao_using_2018}. This approach provides a robust measure of each model's contribution to an optimal ensemble prediction. The idea is to find a set of weights $w = (w_1, \dots, w_K)$ that maximizes the predictive performance of a combined, or \textit{stacked}, model.

Specifically, the stacking method defines an ensemble predictive distribution as a weighted average of the individual models' LOO predictive distributions:
\begin{equation}
p_{\text{stack}}(T_i | \mathcal{D}_{-i}) = \sum_{k=1}^K w_k p(T_i | \mathcal{D}_{-i}, \mathcal{M}_k)
\end{equation}
The optimal weights $w_k$ are then found by solving the following constrained optimization problem (i.e. maximizing the log score of this ensemble):
\begin{equation}
\max_{w} \sum_{i=1}^n \log \left( \sum_{k=1}^K w_k p(T_i | \mathcal{D}_{-i}, \mathcal{M}_k) \right) \quad \text{subject to } w_k \ge 0 \text{ and } \sum_{k=1}^K w_k = 1
\end{equation}
The inputs for this optimization are the pointwise predictive densities, which are calculated during the LOO-CV process. The resulting weights $w_k$ represent the utility of each model in forming the best-performing predictive ensemble. These weights can be interpreted as a form of model probability, allowing us to perform both \textit{model selection} (by choosing the model with the highest weight) and \textit{Bayesian model averaging}:
\begin{equation}
p(T_{\text{new}} | L_{\text{new}}, \mathcal{D}) = \sum_{k=1}^K p(T_{\text{new}} | L_{\text{new}}, \mathcal{D}, \mathcal{M}_k) w_k
\end{equation}
where $w_k$ is the computed stacking weight for model $\mathcal{M}_k$. This method is generally more robust than simpler approaches, especially when some models are misspecified or have similar predictive capabilities.

We prefer this predictive-accuracy-based approach over methods relying on the marginal likelihood (model evidence), such as \textit{Bayes factors}, for two primary reasons. First, the marginal likelihood is notoriously difficult and computationally expensive to estimate accurately, often requiring specialized algorithms such as nested sampling that are beyond the scope of this comparative study. Second, Bayes factors can be highly sensitive to the choice of prior distributions, where diffuse (weakly informative) priors can disproportionately penalize a model. In contrast, LOO-CV offers a more robust and practical assessment of a model's ability to generalize to new, unseen data, which is a central goal in scientific modeling.

\paragraph{Experimental setup.}
We generated a high-quality synthetic dataset of 48 observations to validate the framework. The data spans 12 unique pendulum lengths from $0.1$ to $1.2$ meters, with 4 repeated measurements per length $L_i, i=1,2,...,12$. Realistic, length-dependent noise was added according to the model $\sigma(L) = 0.01 + 0.005\sqrt{L}$, resulting in a signal-to-noise ratio of approximately 146:1. This synthetic dataset ensures good coverage of the problem space while maintaining realistic noise characteristics.

Again, posterior inference was conducted using 4 distinct computational methods to compare their accuracy, speed, and overall performance:
\begin{enumerate}
    \item \textbf{HMC (NUTS):} implemented in \texttt{PyMC} \cite{salvatier_probabilistic_2015}, using the \textit{No-U-Turn} sampler \cite{hoffman2014nuts} with 2000 posterior draws per chain across 2 chains (4000 total samples), following 2000 warm-up iterations.
    \item \textbf{MH:} a custom implementation of the random-walk Metropolis-Hastings algorithm, generating 11,000 total iterations, which results in approximately 4125 posterior samples after a 25\% burn-in and thinning by a factor of 2.
    \item \textbf{Mean field ADVI:} implemented in \texttt{NumPyro} \cite{phan_composable_2019}, using Automatic Differentiation Variational Inference (ADVI \cite{kucukelbir_automatic_2016}) with a mean-field Gaussian guide, optimized for 10,000 iterations, from which 4000 samples were drawn.
    \item \textbf{GMA sampling:} our proposed method, using a Gaussian mixture approximation with $N=800$ components, $M=50$ samples per component, and $K=100$ weight-optimization iterations, from which 4000 samples were drawn.
\end{enumerate}
All 4 methods were applied to all three candidate physical models, using the identical priors and likelihoods specified above to ensure a fair and direct comparison of their performance.

\subsubsection*{\textit{Results and discussion}}

\paragraph{Model selection (HMC as an example).} 
Model comparison using Leave-One-Out Cross-Validation (LOO-CV \cite{vehtari_practical_2017}) by all 4 inference methods provide decisive evidence (we shall also see from later posterior predictive plots) in favor of the \textit{power law model} ($\mathcal{M}_2$). Here we use HMC as an example, as shown in Table.\ref{tab:model_comparison} and Fig.\ref{fig:HMC_model_selection}, the power law model achieved a LOO-ELPD score of $136.55$, vastly outperforming the Linear ($53.75$) and Exponential ($44.05$) models. The computed model weights, which represent each model's predictive utility, assigned a weight of $1.000$ to the power law model, effectively ruling out the other candidates. This result strongly confirms that the underlying physical relationship $T \propto \sqrt{L}$ is overwhelmingly supported by the data, reinforcing the physics-informed prior favoring this model.

\begin{table}[H]
\centering
\begin{threeparttable}
\caption{HMC: dynamics comparison using ELPD from LOO-CV}
\label{tab:model_comparison}
\begin{tabular}{lccc}
\hline
\textbf{Model} & \textbf{Rank} & \textbf{LOO-ELPD} $\uparrow$ & \textbf{Weight} $\uparrow$ \\
\hline
\textbf{Power law} & \textbf{1} & $\mathbf{136.55}$ (SE: 5.07) & $\mathbf{1.000}$ \\
Linear & 2 & $53.75$ (SE: 3.53) & $0.000$ \\
Exponential & 3 & $44.05$ (SE: 3.54) & $0.000$ \\
\hline
\end{tabular}
\begin{tablenotes}
\item [1] \footnotesize{These results were obtained using \texttt{arviz.compare} \cite{arviz_2019}, which compares models based on LOO-ELPD \cite{spiegelhalter_bayesian_2002}.} 
\item [2] \footnotesize{\texttt{SE} is the standard error of the ELPD estimate.} 
\item [3] \texttt{weight} is the relative weight for each model, which can be loosely interpreted as the probability of each model (among the 3 compared model) given the data.
\end{tablenotes}
\end{threeparttable}
\end{table}

\begin{figure}
    \centering
    \includegraphics[width=0.7\linewidth]{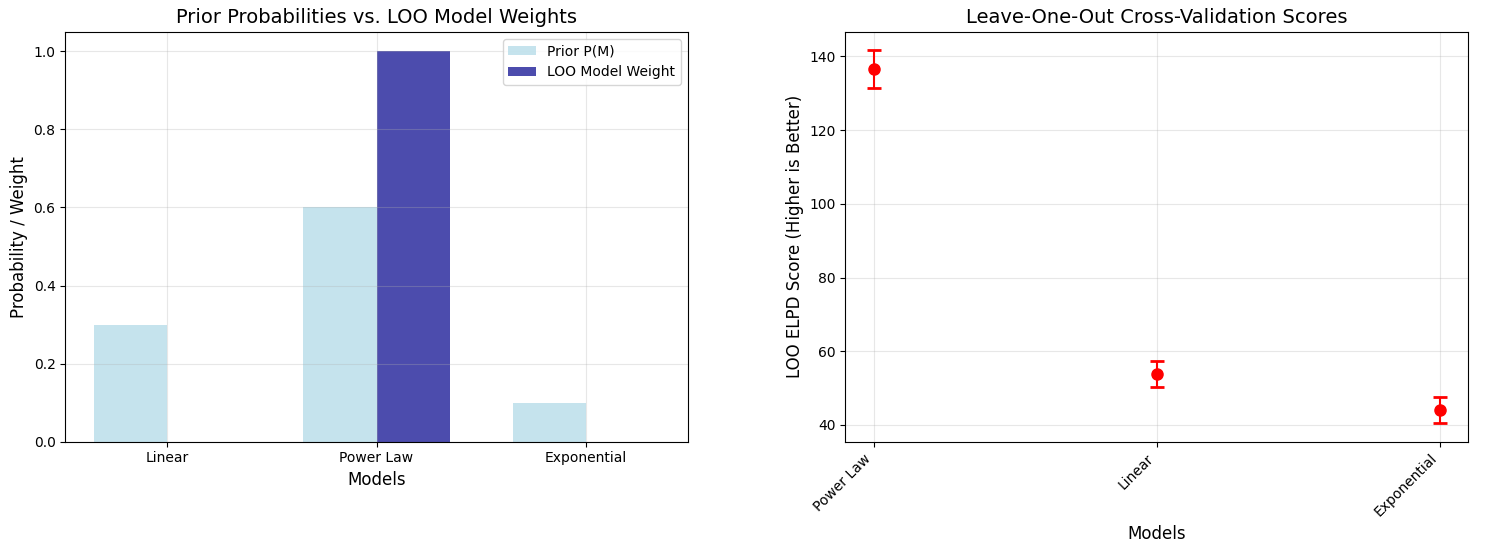}
    \caption{HMC-based physical model selection.}
    \label{fig:HMC_model_selection}
\end{figure}

\paragraph{Parameter recovery.}
Given the conclusive evidence for the \textit{power law model}, we analyze its parameter estimates to assess how well each sampling/inference method recovered the true physics. The ground truth parameters are $a=2.006$, $b=0.5$, and $c=0.0$. Table.\ref{tab:param_recovery} summarizes the posterior means and standard deviations for these parameters from each of the 4 inference methods, from which we observe that, all 4 methods successfully recovered the key physical parameters within a reasonable range. The exponent $b$, which governs the core relationship, was consistently estimated to be near the true value of $0.5$. The coefficient $a$ was also recovered accurately by most methods. The intercept term $c$ was correctly estimated to be near zero, consistent with the theoretical model.

\begin{table}[H]
\centering
\begin{threeparttable}
\caption{Parameter estimates (mean $\pm$ sd) for the power law model ($T = aL^b+c$) across the 4 inference methods. All methods successfully recover the true physical parameters.}
\label{tab:param_recovery}
\begin{tabular}{lccc}
\hline
\textbf{Method} & \textbf{$a$} & \textbf{$b$} & \textbf{$c$} \\
\hline
Ground truth & $2.006$ & $0.500$ & $0.000$ \\
\hline
HMC (NUTS) & $2.000 \pm 0.031$ & $0.506 \pm 0.013$ & $0.008 \pm 0.032$ \\
MH & $2.015 \pm 0.058$ & $0.500 \pm 0.022$ & $-0.008 \pm 0.060$ \\
ADVI & $1.994 \pm 0.003$ & $0.509 \pm 0.002$ & $0.013 \pm 0.002$ \\
GMA sampling & $1.940 \pm 0.010$ & $0.537 \pm 0.009$ & $0.072 \pm 0.009$ \\
\hline
\end{tabular}
\begin{tablenotes}
\item [1] \footnotesize{Estimates of noise level are not presented as not relevant.}
\end{tablenotes}
\end{threeparttable}
\end{table}

Fig.\ref{fig:posterior_comparison} provides a visual comparison of the full posterior distributions for these parameters. HMC, considered the gold standard, produces smooth, well-defined unimodal posteriors. The VI posteriors are similarly well-behaved but exhibit significantly smaller variance - a known characteristic of mean-field variational methods. The MH sampler, with its low acceptance rate ($1.0\%$), manages to explore the correct parameter space, though its posterior approximation is visibly rougher. GMA-sampling also captures the location of the posterior mode but produces a more structured, multi-modal approximation, reflecting its Gaussian mixture-based nature. 

\begin{figure}[H]
    \centering
    \includegraphics[width=\textwidth]{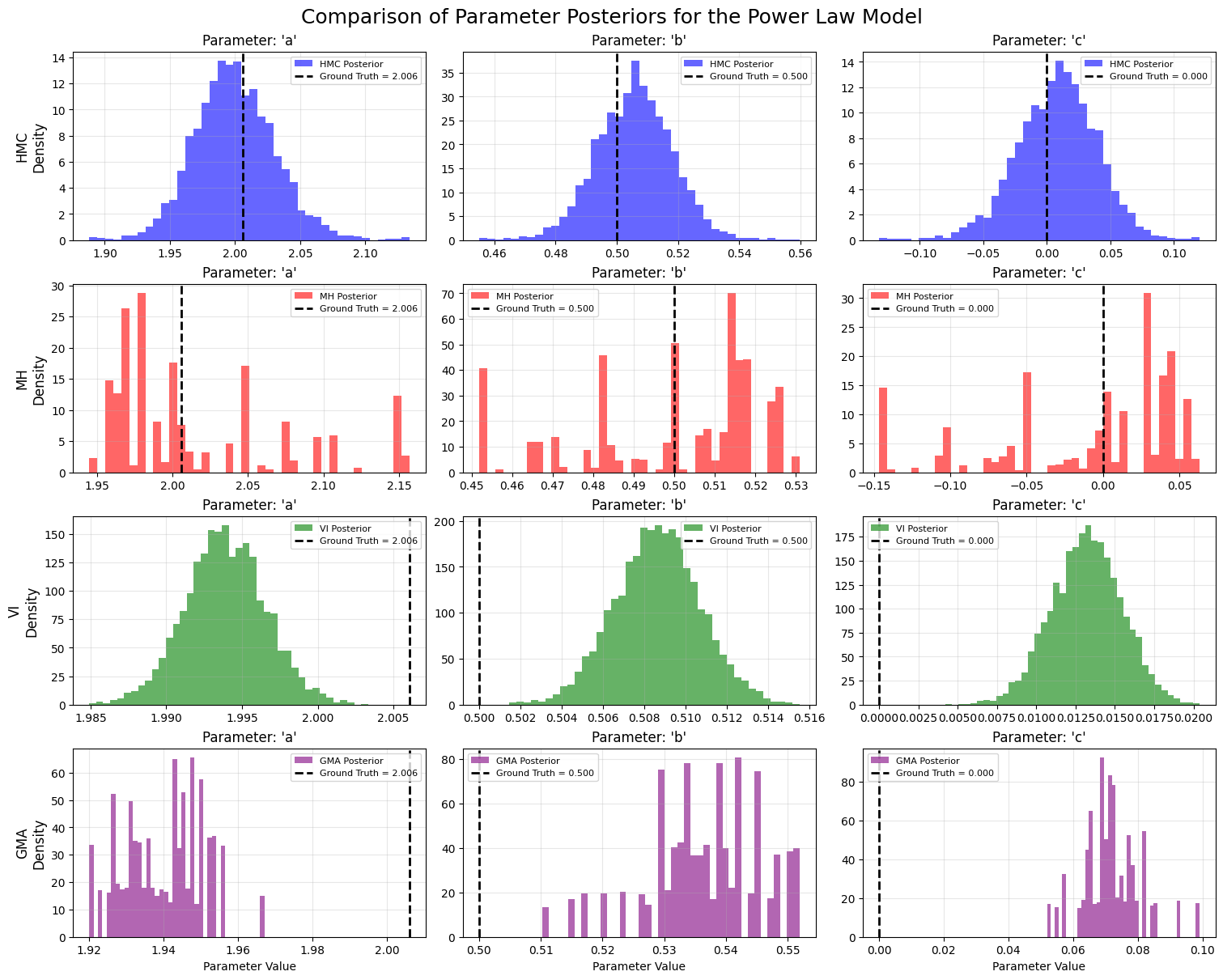} 
    \caption{Comparison of the marginal posterior distributions for the \textit{power law model} parameters from each inference method. The vertical dashed line indicates the ground truth value. All methods correctly put posterior probability mass around the ground truth values.}
    \label{fig:posterior_comparison}
\end{figure}

For GMA sampling, the evolution of the GMA component weights serves as a powerful diagnostic for model fit. As shown in Fig.\ref{fig:GMA_weights_evolution_BSR}, for the well-specified power law model, the weights converge almost immediately to a single dominant component, signaling a simple posterior that is easily captured. In contrast, the weights for the mis-specified linear and exponential models remain chaotically distributed across numerous components.

\begin{figure} [H]
    \centering
    \includegraphics[width=0.9\linewidth]{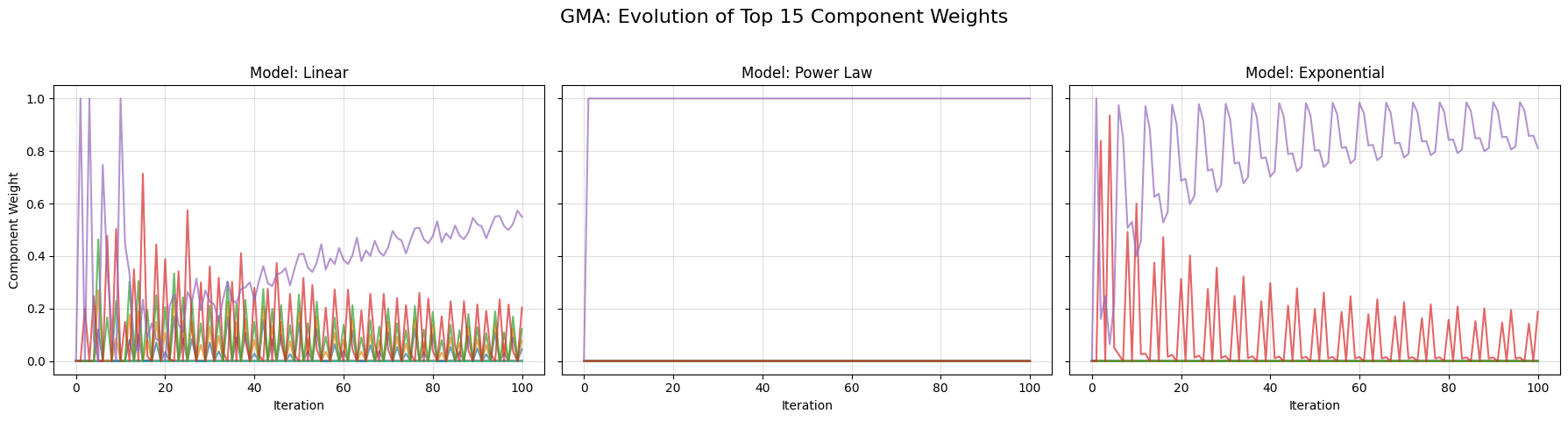}
    \caption{GMA weights evolution over optimisation iterations.}
    \label{fig:GMA_weights_evolution_BSR}
\end{figure}

\paragraph{Posterior predictive performance.}
Uncertainties in parameter estimates are propagated into posterior-based predictions. As shown in Fig.\ref{fig:posterior_predictives}, all 4 inference methods correctly show that the mis-specified linear and exponential models fail to fit the data or extrapolate reasonably; they all produce excellent predictions when using the correctly specified \textit{power law mode}l. The mean predictive curves are virtually indistinguishable from one another and align well with the ground truth, demonstrating strong predictive accuracy both within the range of the training data and during extrapolation. While all methods generate 95\% credible intervals that appropriately capture the observational uncertainty, there are subtle differences in its quantification: the intervals from VI are visibly tighter than those from the MCMC-based HMC and MH, as well as GMA methods, which is consistent with the known tendency of VI approximation methods to underestimate posterior variance.

\begin{figure}[H]
    \centering
    \includegraphics[width=0.85\linewidth]{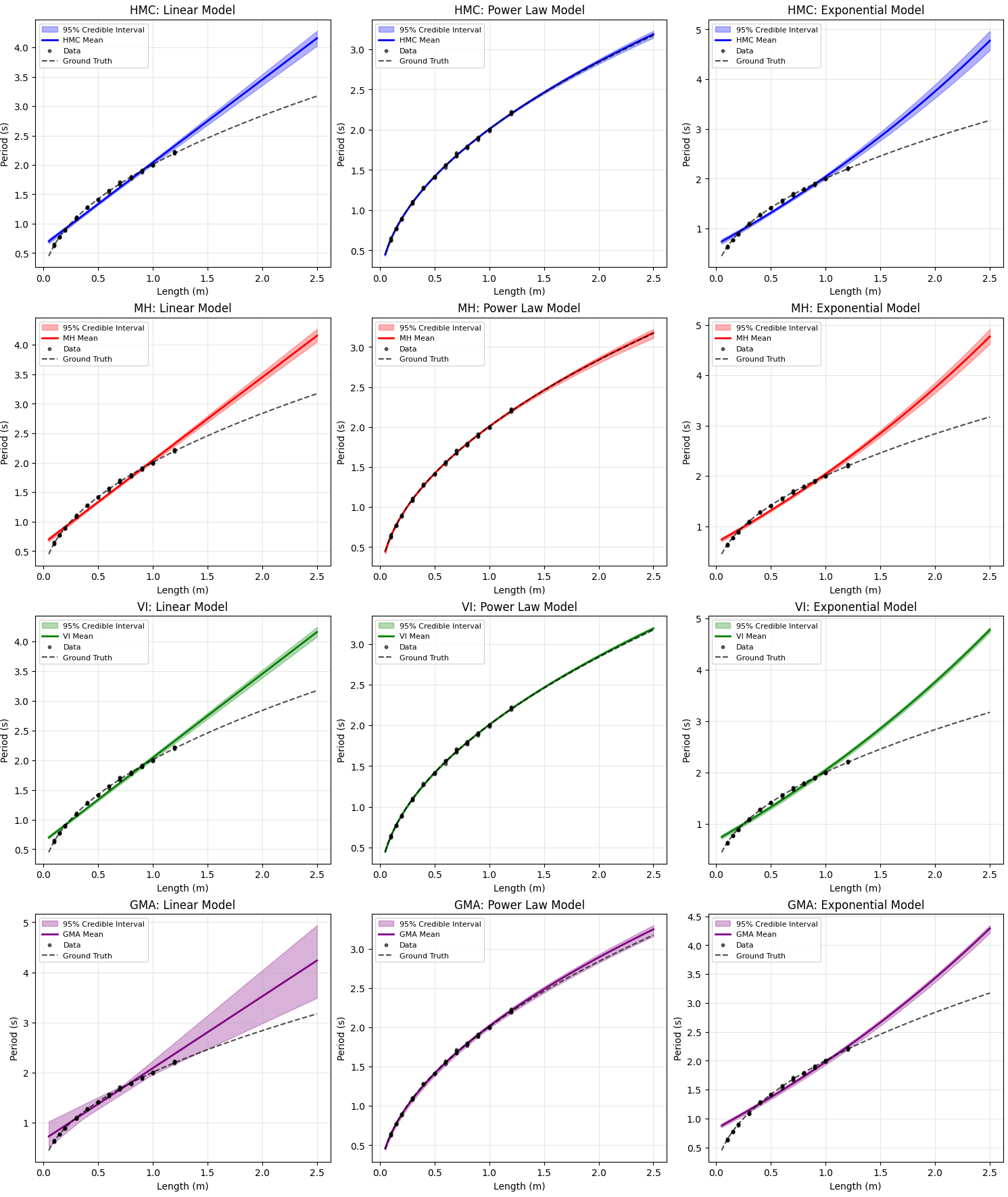}
    \caption{Predictions using posterior samples from the 4 inference methods.}
    \label{fig:posterior_predictives}
\end{figure}

\paragraph{Efficiency comparison.}
Beyond accuracy, the 4 methods exhibited stark differences in computational efficiency. As shown in Fig.\ref{fig:inference_times_BSR}, the MH sampler was the fastest method by a large margin, completing the power law model inference in just $1.57$ seconds, though this speed came at the cost of a very low acceptance rate ($1.0\%$), suggesting inefficient exploration. ADVI was also extremely fast ($8.45$s) and provided a strong balance of speed and accuracy. GMA-sampling demonstrated competitive performance, completing its deterministic optimization in $11.79$s. HMC was by far the slowest, requiring $88.09$s, but provided the most reliable and highest-quality posterior samples. 

\begin{figure}[H]
    \centering
    \includegraphics[width=0.35\linewidth]{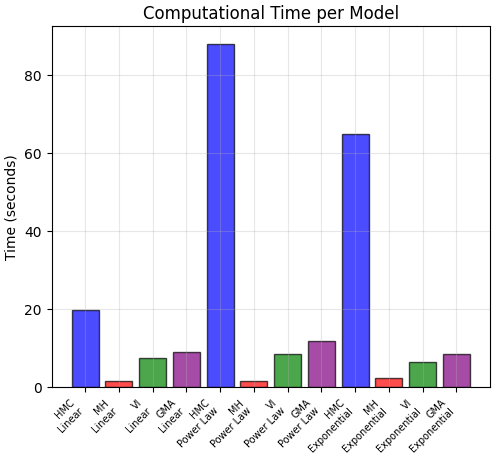}
    \caption{Computational times of the 4 inference methods.}
    \label{fig:inference_times_BSR}
\end{figure}

\paragraph{Discussion.}
These results demonstrate the power of hierarchical Bayesian modeling for automated scientific discovery, and the efficacy and efficiency of GMA sampling. The framework not only identified the correct physical law from a set of candidates but also accurately quantified the uncertainty in its parameters. The use of physics-informed priors for the power law model successfully guided the inference without preventing the data from speaking for itself.

The comparison of inference methods highlights the trade-off between speed, implementation complexity, and sample quality. While HMC provides the most reliable results, its computational cost can be prohibitive. 
MH is simple to implement and runs fast; however, it shows in this case inefficient sampling (high autocorrelation) and requires careful tuning.
VI offers an excellent, scalable alternative for rapid exploration (e.g. prototyping), balancing accuracy and efficiency; however, it can underestimate posterior variance and relies on distributional assumptions.
Our novel GMA-sampling method proves to be a competitive alternative to both MCMC and VI, successfully approximating the posterior in a time comparable to VI. However, GMA IS More complex to implemment with hyper-parameters (e.g. number of components and number of samples per component, as well as learning rate and component variances) that require tuning.
The choice of method should therefore be tailored to the specific goals of the analysis, whether it be rapid prototyping (VI, MH), robust final analysis (HMC), or exploration of novel inference frameworks (GMA).

\subsubsection{Bayesian LSTM with \textit{WGMA} for mortality time series forecasting}

For our third real-world experiment, we fit a Bayesian Long Short-Term Memory (LSTM) network to weekly all-cause mortality data from the World Mortality Database \cite{karlinsky_tracking_2021}. This dataset provides harmonized national-level mortality counts from official vital statistics systems. We focus on the \textit{United States}, where reporting is consistent and coverage spans from early 2015 through the end of 2024, yielding a total of 521 weekly observations.

To prepare the series for modeling, we first aggregated weekly death counts and interpolated minor gaps to obtain a continuous weekly timeline. We then constructed a normalized \textit{mortality index} by scaling each week’s deaths by the long-term mean, thereby capturing relative mortality fluctuations independent of population scale. Taking the logarithm of this index produced a stationary series $\log(\mathcal{M}_t)$, which serves as the input to the forecasting model.

The Bayesian LSTM is trained to use the previous 52 weeks (one year) as context in order to predict the subsequent 52 weeks. Unlike a classical LSTM, which produces a single point forecast, the Bayesian formulation yields a posterior predictive distribution. This allows us to quantify both \textit{aleatoric uncertainty}, arising from natural variability in mortality, and \textit{epistemic uncertainty}, which reflects the limited amount of training data (only 521 points) and model uncertainty \footnote{Bayesian models provide a principled mechanism for capturing these two complementary forms of uncertainty, which is especially valuable when working with short or noisy time series. Aleatoric uncertainty stems from inherent data noise (e.g. sensor errors), epistemic uncertainty reflects the model’s limited knowledge.}.

The final objective is to forecast future weekly log-mortality indices, $\log(\mathcal{M}_t)$, while simultaneously quantifying predictive uncertainty through posterior credible intervals. This probabilistic framework balances the expressive temporal dynamics of deep learning with the uncertainty-awareness of Bayesian inference, making it well-suited to mortality forecasting in small-data regimes.

\paragraph{Data and pre-processing}

The mortality dataset provides weekly all-cause death counts $D_t$ for the United States from 2015 to 2024, yielding a total of 521 consecutive weekly observations. To formulate a standard time-series forecasting problem suitable for neural network modeling, we first aggregate duplicate weeks, interpolate missing values, and ensure a complete weekly calendar. We then construct a normalized mortality index that captures relative fluctuations with respect to the long-term mean:
\begin{equation}
\mathcal{M}_t = \frac{D_t}{\bar{D}}
\end{equation}
where $\bar{D} = \tfrac{1}{T} \sum_{t=1}^{T=521} D_t$ is the average weekly number of deaths over the observed period. This normalization removes secular scale effects while preserving relative mortality variations of forecasting interest. 

Following standard practices in mortality modeling \cite{theactuary2024underinference}, we work with the natural logarithm of the mortality index to stabilize variance, ensure positivity of forecasts, and facilitate modeling of multiplicative effects:
\begin{equation}
y_t = \log(\mathcal{M}_t)
\end{equation}

In preparing the data for the LSTM, we adopt a \textit{sliding window approach}. For a chosen look-back window of length $L = 52$ weeks (i.e. one year), we construct sequential feature-target pairs in which each input sequence captures one year of recent mortality history and the target represents the subsequent week’s log-mortality:
\begin{itemize}
    \item Input sequence (features): $\mathbf{x}_t = (y_{t-L}, y_{t-L+1}, \ldots, y_{t-1})$
    \item Target value (label): $y_t$
\end{itemize}
This yields $T-L = 469$ overlapping sequences from the 521-week dataset, of which the final 52 sequences (corresponding to one forecast year) are held out for evaluation. Both input features and targets are standardized using z-score normalization to have zero mean and unit variance, which accelerates convergence and improves numerical stability during training.

\paragraph{Exploratory Data Analysis (EDA)}

Fig.\ref{fig:EDA_mortality} illustrates the key characteristics of the pre-processed weekly mortality series. The left panel shows the raw weekly death counts, with clear seasonal cycles and pronounced spikes during the COVID-19 pandemic. The middle panel displays the normalized mortality index, which oscillates around 1.0 (the long-term mean), highlighting deviations from expected baseline levels. Right panel shows the log-mortality index, which is the final input to the LSTM. The log transformation stabilizes variance and compresses extreme peaks, making the series more suitable for neural time-series forecasting.

\begin{figure}[h!]
    \centering
    \includegraphics[width=\linewidth]{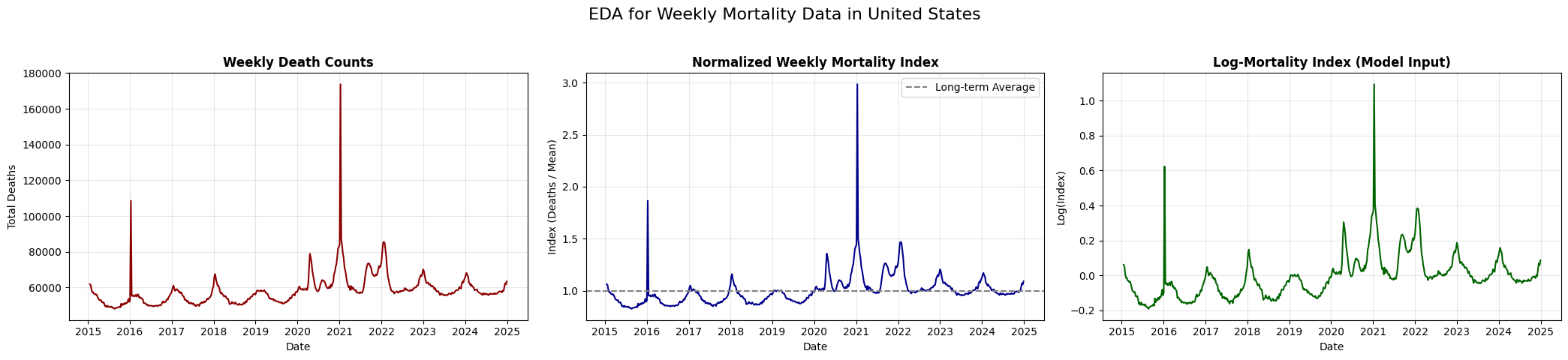}
    \caption{EDA for weekly mortality in the United States (2015-2024). (a) Weekly death counts. (b) Normalized mortality index relative to the long-term mean. (c) Log-mortality index used as model input.}
    \label{fig:EDA_mortality}
\end{figure}

To evaluate out-of-sample performance, we adopt a fixed-horizon train/test split \footnote{No shuffle or mess-up for time series splitting.}. From the $T-L=469$ sequences constructed, the final $N_{\text{forecast}} = 52$ sequences (corresponding to one year of data) are held out as a test set, while the preceding $416$ sequences are used for model training. This ensures that forecast evaluation mimics the realistic setting of predicting future mortality beyond the observed sample. Both training and testing sequences are standardized 
with \textit{z-score} normalization based on the training statistics to prevent information leakage.

\paragraph{Classic LSTM forecasts}

As a benchmark, we trained a standard Long Short-Term Memory (LSTM) network, a type of recurrent neural network (RNN) well-suited for capturing temporal dependencies in sequential data, using backpropagation under a \textit{one-step rolling forecasting} scheme\footnote{In the \textit{one-step} forecasting scheme, the model predicts only the next time point given the true historical inputs, with errors not propagated forward. In contrast, the \textit{multi-step rolling forecast scheme} recursively feeds previous forecasts back as inputs, enabling multi-step prediction but compounding forecast uncertainty. We present the results from \textit{multi-step rolling forecast scheme} in Appendix.\ref{app:Bayesian_LSTM_further_results}.} for both training and testing. 
The network consisted of 3 stacked LSTM layers, each with 16 hidden units, taking one-dimensional input sequences of length $L=52$ (corresponding to one year of past log-mortality values). The final hidden state was passed through a fully connected layer to produce a single-step prediction of the next log-mortality value. In total, the model contained 5,585 trainable parameters and was optimized over 1,000 epochs using the \textit{Adam} \cite{kingma_adam_2017} optimizer with learning rate $0.01$, achieving rapid convergence as shown in the training loss curve (Fig.\ref{fig:classic_lstm_loss}). The fitted model, with point weights estimate $\theta^\star$, provided an excellent in-sample fit with a training RMSE of $0.0072$, and delivered competitive out-of-sample forecasts with a test RMSE of $0.0276$ (Fig.\ref{fig:classic_LSTM_forecast_oneStep}). The one-step forecasts track the observed mortality index closely, although the model exhibits some difficulty in fully capturing the extreme mortality spikes associated with the COVID-19 pandemic. These results highlight the strengths of the classic LSTM in capturing seasonal mortality dynamics, while motivating the Bayesian extension to better account for uncertainty in small-sample forecasting.

\begin{figure}[h!]
    \centering
    \begin{minipage}{0.48\linewidth}
        \centering
        \includegraphics[width=\linewidth]{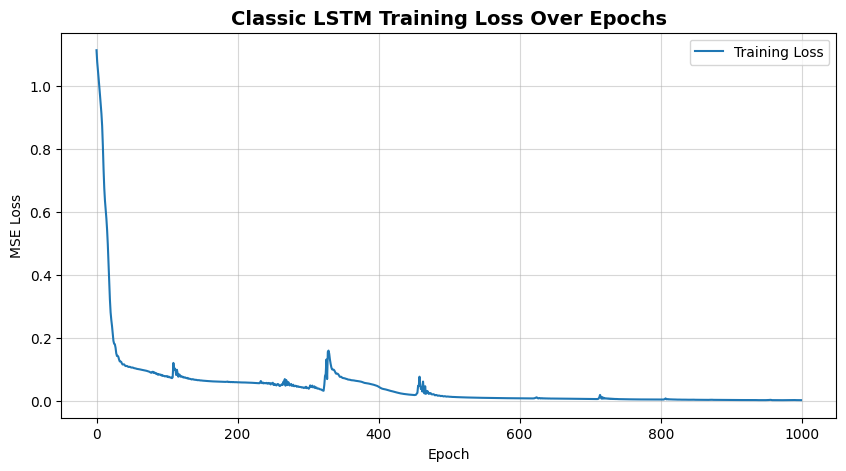}
        \subcaption{Training loss (MSE) over 1,000 epochs.}
        \label{fig:classic_lstm_loss}
    \end{minipage}
    \hfill
    \begin{minipage}{0.48\linewidth}
        \centering
        \includegraphics[width=\linewidth]{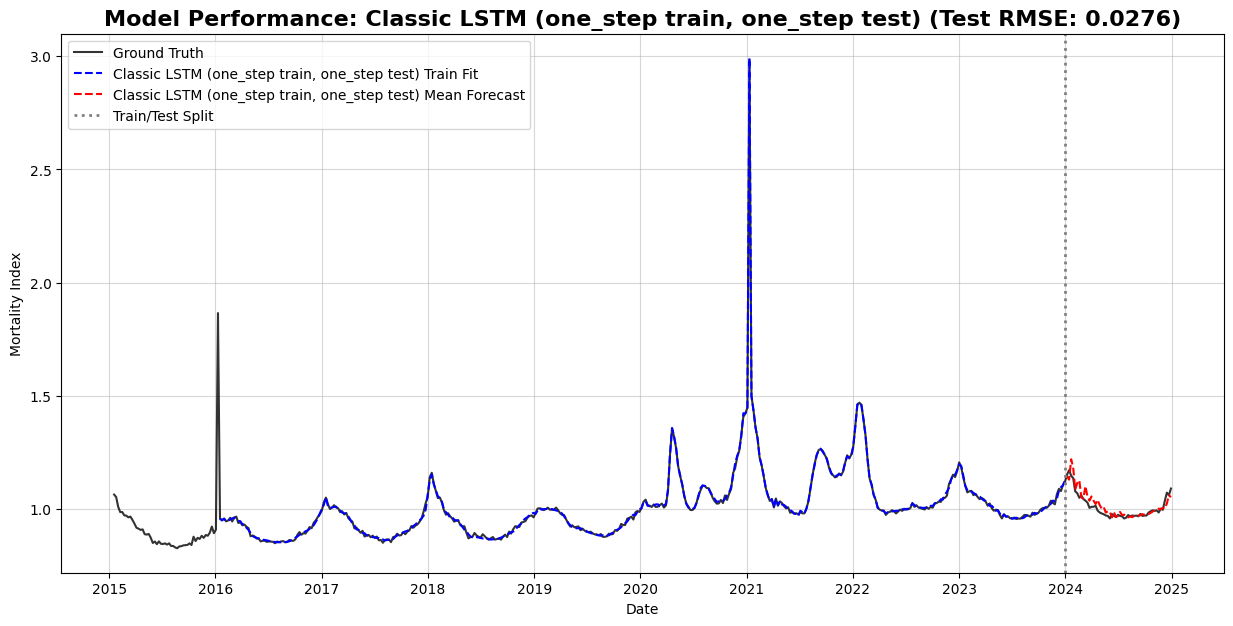}
        \subcaption{Forecast performance (2015-2024) under one-step ahead mode.}
        \label{fig:classic_LSTM_forecast_oneStep}
    \end{minipage}
    \caption{Classic LSTM benchmark results for U.S. weekly mortality. 
    (a) Training loss curve showing rapid convergence. 
    (b) In-sample fit (blue) and out-of-sample forecasts (red) compared against ground truth (black), with train RMSE = 0.0072 and test RMSE = 0.0276.}
    \label{fig:classic_lstm_results}
\end{figure}

\paragraph{Bayesian LSTM model}
We use the same LSTM network architecture as in the classic LSTM benchmark. The network, denoted by a function $f_\theta$ with parameter vector (weights and biases) $\theta$, maps an input sequence $\mathbf{x}_t$ to a prediction of the mean of the next value in the series, $\mu_t$:
\begin{equation}
\mu_t = f_\theta(\mathbf{x}_t)
\end{equation}
This structure allows the model to learn underlying temporal patterns from recent mortality history and exploit them for forecasting. 
Unlike training classic neural networks, where point estimates of parameters are obtained via gradient descent (backpropagation), training a Bayesian LSTM requires inferring a full posterior distribution over the weights $\theta$. This provides a principled way to quantify epistemic uncertainty in addition to aleatoric noise. In practice, exact posterior inference is intractable for high-dimensional recurrent networks, so we employ approximate sampling techniques (here our proposed GMA algorithm) to generate representative posterior samples of the 5585-dimensional $\theta$. These samples are then propagated through the network to produce predictive distributions rather than single deterministic forecasts.

\paragraph{Priors on network parameters}
In Bayesian neural nets (BNNs), we place prior distributions over the network parameters $\theta$. Rather than identifying a single optimal set of weights, we infer a full posterior distribution over them. To regularize the high-dimensional parameter space, we use weakly informative (i.e. \textit{vague}) Gaussian priors:
\[
    \theta \sim \mathcal{N}(0, 1.0)
\]
This choice penalizes extreme weight values deviating from zero, thereby helping to prevent overfitting, which is good when training neural networks with hundreds of parameters on relatively small mortality datasets.

\paragraph{Likelihood}
The LSTM outputs are linked to observed data using a Gaussian likelihood. We assume the observed log-mortality indices are normally distributed around the network predictions, with observation noise $\sigma$:
\begin{equation}
y_t \mid \theta, \sigma \sim \mathcal{N}(\mu_t, \sigma^2), 
\quad \text{where } \mu_t = f_\theta(\mathbf{x}_t)
\end{equation}
To constrain the noise parameter to be positive, we place a weakly informative half-normal prior on the noise level:
\[
    \sigma \sim \text{HalfNormal}(1.0)
\]

\paragraph{Posterior inference and forecasting}
Our objective is to infer the posterior distribution over all unknown parameters \footnote{Now we have 5,585 parameters from the LSTM model and 1 noise parameter, totaling 5586 dimensions to infer.} $\boldsymbol{\phi} = \{\theta, \sigma\}$ given the observed data $\mathbf{y}$. By Bayes’ rule,
\begin{equation}
p(\boldsymbol{\phi} \mid \mathbf{y})  \propto  p(\mathbf{y} \mid \boldsymbol{\phi}) \cdot p(\boldsymbol{\phi})
\end{equation}
Since exact posterior inference is intractable for BNNs, we employ our proposed Gaussian mixture approximation (GMA) sampler to generate samples form it \footnote{Other approximate inference methods such as MH, HMC, and ADVI baselines are not employed here, because during our trials we found that MH and HMC are extremely computational expensive for this task given the 5585 dimensions of $\theta$.}. 

For GMA sampling, we can use \textit{projected gradient descent} (pGD) or \textit{mirror descent} (MD) to optimize the mixture weights on the probability simplex. In pGD-GMA (see Algo.\ref{algo:WGMA-sampling-pgd} and an improved variant Algo.\ref{algo:GMA-sampling-optimal} which is used here), an additive Euclidean gradient step is taken followed by projection back onto the simplex, which guarantees feasibility but may create spiky updates and occasionally collapses weights onto a small subset of components. By contrast, MD-GMA (see Algo.\ref{algo:GMA-sampling-mirror}) uses the geometry induced by negative entropy and performs multiplicative weights updates of the form (Eq.\ref{eq:MD_exponential_update} in Appendix.\ref{app:mirror_descent}):
\[
w_i^{(k+1)} \propto w_i^{(k)} \exp  \big(-\eta_k g_i^{(k)}\big)
\]
followed by normalization. This update automatically preserves positivity and normalization without requiring explicit projection to the probability simplex. From an optimization perspective, MD can be viewed as solving a KL-proximal subproblem \cite{cmu15850notes2020} rather than a Euclidean one, making it more natural for distributions over probability vectors. It also tends to preserve entropy across components, reducing the risk of premature mode collapse (see Appendix.\ref{app:mirror_descent}). 

In terms of computational complexity, both pGD-GMA and MD-GMA share the same dominant pre-computation and gradient estimation costs, $\mathcal{O}(N^2 M d^2) + \mathcal{O}(K \cdot N^2 M)$, but MD avoids the $\mathcal{O}(N \log N)$ projection step of pGD and replaces it with an $\mathcal{O}(N)$ multiplicative update. This makes MD slightly more efficient per iteration. Conceptually, pGD uses Euclidean geometry while MD uses KL (entropy) geometry, which is often more appropriate for simplex-constrained problems. In the present experiments we focus on the pGD variant of GMA, leaving the mirror descent (MD) version for later discussion.

After obtaining samples for LSTM weights, we can propagate the uncertainties in weights to predictions. For forecasting, we again use the \textit{one-step}, autoregressive procedure. For each posterior sample $\boldsymbol{\phi}^{(s)}$, we iteratively generate future log-mortality values by conditioning on the last observed sequence and rolling forward:
\begin{enumerate}
    \item Use the last $L=52$ observed values ($L$ is the lookback period, i.e. the \textit{lag} in autoregressive models) to predict $\mu_{T+1}^{(s)}$.
    \item Draw a predictive sample $\hat{y}_{T+1}^{(s)} \sim \mathcal{N}(\mu_{T+1}^{(s)}, (\sigma^{(s)})^2)$.
    \item Append $\hat{y}_{T+1}^{(s)}$ to the sequence and drop the oldest element to form the input for predicting $\mu_{T+2}^{(s)}$, then repeat.
\end{enumerate}
This recursive sampling procedure produces thousands of possible mortality paths, from which we compute posterior predictive means and credible intervals. Predictions are transformed back to the mortality index scale via $\mathcal{M}_t = \exp(y_t)$.

An alternative forecasting method is the \textit{multi-step rolling forecast scheme} \cite{huang_sequential_2020}, where predictions rather than observed data points are recursively fed back into the model’s $L$-length input sequence. In this setting, each predicted value $\hat{y}_{t}^{(s)}$ is appended to the input window (and oldest value is dropped to maintain a length of $L$) and used to generate the next forecast, effectively propagating (and cumulating) both observation noise and parameter uncertainty forward in time \footnote{Forecasting errors are thus cumulated in this scheme, while in the one-step forecasting scheme, the forecast error is corrected by adding observed data. In the \textit{one-step} scheme, the model always conditions on the true past observations when predicting the next value, avoiding error propagation but limiting forecasts to single-step horizons. In contrast, the \textit{rolling} scheme recursively conditions on its own predictions, enabling multi-step trajectories while compounding forecast uncertainty. Results from \textit{multi-step rolling forecast scheme} are presented in Appendix.\ref{app:Bayesian_LSTM_further_results}.}. A comparison of the two rolling forecast methods is made in Fig.\ref{fig:two_rolling_forecast_schemes}. 

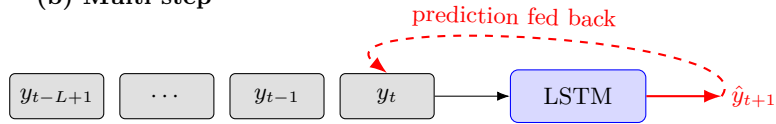
\begin{figure}[h!]
\centering
\begin{tikzpicture}[
    font=\small, >=Latex,
    seq/.style={draw,fill=black!12,rounded corners=2pt,minimum width=1.25cm,minimum height=6mm,inner sep=1pt},
    mdl/.style={draw=blue,fill=blue!15,rounded corners=3pt,minimum width=1.8cm,minimum height=7mm},
    lab/.style={align=center}
]

% ========================= (a) ONE-STEP =========================
\begin{scope}[shift={(0,0)}]
\node[lab,anchor=west] (titleA) at (-0.4,1.3) {\textbf{(a) One-step}};

% input window: y_{t-L}, ..., y_{t-1}, y_t
\node[seq] (a1) at (0,0) {$y_{t-L+1}$};
\node[seq,right=2mm of a1] (a2) {$\cdots$};
\node[seq,right=2mm of a2] (a3) {$y_{t-1}$};
\node[seq,right=2mm of a3] (a4) {$y_{t}$};

% model and output
\node[mdl,right=10mm of a4] (am) {LSTM};
\draw[->,black] (a4.east) -- (am.west);
\draw[->,red,thick] (am.east) -- ++(10mm,0) node[right,red] {$\hat{y}_{t+1}$};

% caption anchored below this row
\node[lab,below=7mm of $(a2)!0.5!(a3)$] (acap) {True past sequence always used};
\end{scope}

% ========================= (b) ROLLING ==========================
\begin{scope}[shift={(0,-2.8)}] % vertical spacing between panels
\node[lab,anchor=west] (titleB) at (-0.4,1.3) {\textbf{(b) Multi-step}};

% input window: same labels
\node[seq] (b1) at (0,0) {$y_{t-L+1}$};
\node[seq,right=2mm of b1] (b2) {$\cdots$};
\node[seq,right=2mm of b2] (b3) {$y_{t-1}$};
\node[seq,right=2mm of b3] (b4) {$y_{t}$};

\node[mdl,right=10mm of b4] (bm) {LSTM};
\draw[->,black] (b4.east) -- (bm.west);
\draw[->,red,thick] (bm.east) -- ++(10mm,0) node[right,red] {$\hat{y}_{t+1}$};

% feedback arrow
\draw[->,red,dashed,thick]
  ($(bm.east)+(10mm,0)$) .. controls +(6mm,8mm) and +(-10mm,8mm) .. ($(b4.north)$)
  node[pos=0.55,above,sloped,red] {prediction fed back};

\node[lab,below=7mm of $(b2)!0.5!(b3)$] {Predictions replace most recent observation};
\end{scope}

\end{tikzpicture}
\caption{Schematic comparison of two rolling forecasting schemes. 
(a) One-step: conditions on observed history (black) sequence $(y_{t-L+1},\ldots,y_{t-1},y_t)$. 
(b) Multi-step: add its own forecasts (red) to the input sequence.}
\label{fig:two_rolling_forecast_schemes}
\end{figure}

\paragraph{Implementation and results}
We evaluated the Bayesian LSTM with GMA sampling (the pGD variant\footnote{Results using the MD optimisation scheme are presented in Appendix.\ref{app:Bayesian_LSTM_further_results}.}) on the pre-processed US mortality index data between years 2015 and 2024. In implementing the pGD-GMA method, we used $N=200$ Gaussian components with $M=30$ local samples each, and optimized the mixture weights over $K=100$ iterations using a decaying learning rate $\eta_k=\eta_0/\sqrt{k+k_0}$ with $\eta_0=0.05$ and $k_0=800$. To initialise the Gaussian mixture, each Gaussian component mean was centred at the point estimate $\theta^\star$ of the trained classic LSTM, with small independent Gaussian jitters applied to each center, ensuring that GMM samples are drawn from a initial proposal cloud tightly concentrated around a well-trained mode.

We also applied stabilization tricks to prevent mode collapse (details in Appendix.\ref{app:combine_precomputing_and_MC_gradient_estimator}): a \textit{tempering parameter} $\beta_k$ that scales the log target density (Eq.\ref{eq:tempering_pGD_GMA} in Appendix.\ref{app:combine_precomputing_and_MC_gradient_estimator}), and a \textit{decaying entropy regularizer} $\lambda_k$ that discourages premature weight collapse (Eq.\ref{eq:entropy_regularisation_pGD_GMA} in Appendix.\ref{app:combine_precomputing_and_MC_gradient_estimator}). Together, these schedules stabilize the projected gradient descent dynamics, ensuring well-behaved entropy trajectories and preventing degeneracy, while maintaining computational efficiency comparable to vanilla pGD updates.

Applying these, the GMA algorithm converged rapidly: by iteration $k=100$ (computational time: 1s), the mixture entropy had stabilized at $H(w)=2.36$ with an effective 
\footnote{At each GMA iteration $k$, the mixture weight vector $w^{(k)}=(w^{(k)}_1,\dots,w^{(k)}_N)$ is used to compute two diagnostics: 
(i) the entropy (in \textit{nats}): $H^{(k)} = -\sum_{i=1}^{N} w^{(k)}_i \log (w^{(k)}_i)$, and (ii) the effective number of components (perplexity): $\text{eff}^{(k)} = \exp  \big(H^{(k)}\big)$. Here $H^{(k)} \in [0,\log N]$, where high entropy (near $\log N$) indicates weights spread across many components, and low entropy (near $0$) indicates collapse onto a few components. The effective number $\text{eff}^{(k)} \in [1,N]$ acts like the 'number of active components': $\text{eff}=N$ for uniform weights and $\text{eff}=1$ for a single dominating weight.
We can observe oscillatory behaviour in the weights evolution curve due to the inverse square root step size: $\eta_k=\eta_0 / \sqrt{k + k_0}$ ($k_0=800$) used in our implementations for both GMAs for this task. This design leads to $\sum_k \eta_k^2 = \infty$, which violates the \textit{Robbins-Monro conditions} for stochastic optimisation \cite{robbins_stochastic_1951}.
}
number of components $\approx 10.6$. The final mixture weights concentrated most strongly on a small subset of Gaussian components, with the leading component (index 127) carrying a normalized weight of $0.42$. This hints that GMA successfully identifies and emphasizes the most informative regions of the posterior while maintaining diversity across components.

\begin{figure}[ht!]
    \centering
    \begin{subfigure}[b]{0.32\textwidth}
        \centering
        \includegraphics[width=\linewidth]{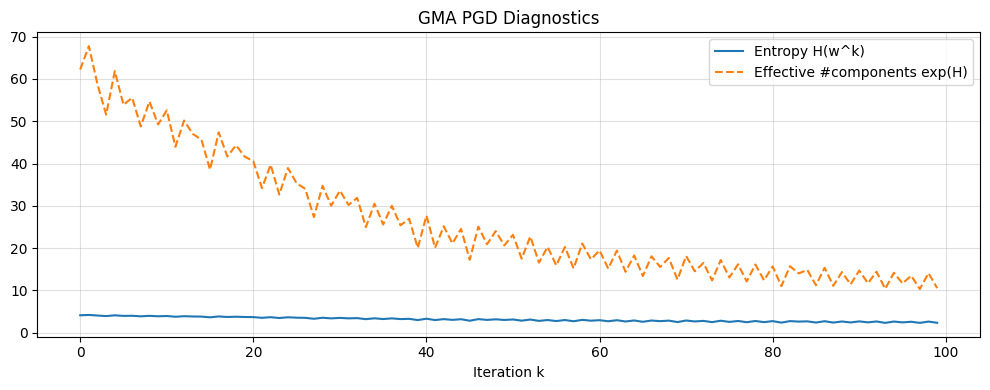}
    \end{subfigure}
    \hfill
    \begin{subfigure}[b]{0.32\textwidth}
        \centering
        \includegraphics[width=\linewidth]{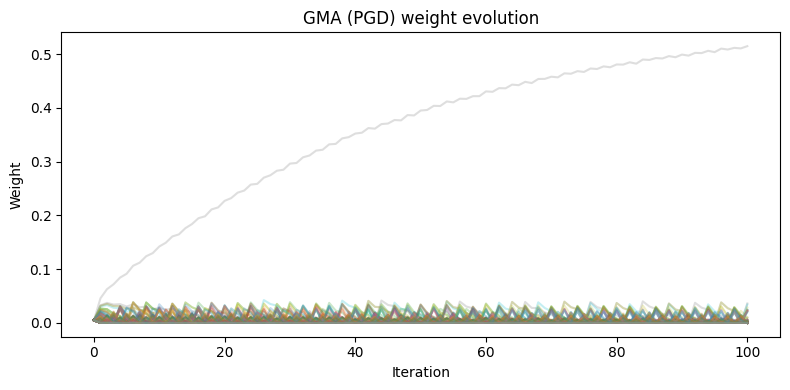}
    \end{subfigure}
    \hfill
    \begin{subfigure}[b]{0.32\textwidth}
        \centering
        \includegraphics[width=\linewidth]{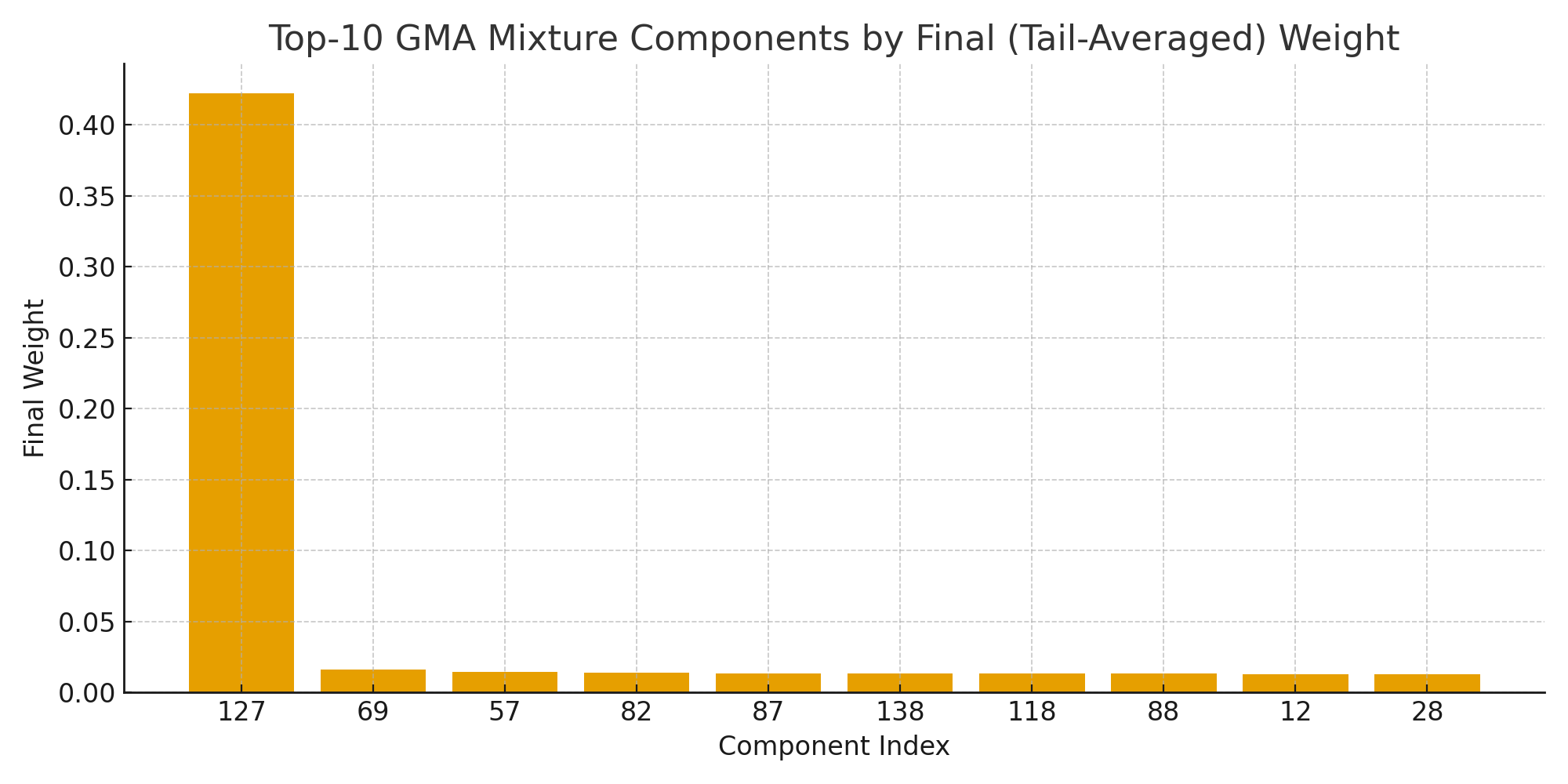}
    \end{subfigure}
    \caption{pGD-GMA sampling diagnosis. Left: entropy \& effective component count over iterations, middle: weights evolution, and right: final top-10 component weights.}
    \label{fig:gma_diagnostics_all}
\end{figure}

Forecasting performance was strong. The root mean squared errors (RMSE) of the posterior mean forecasts are $\text{train RMSE} = 0.0170$ and $\text{test RMSE} = 0.0255$, compared with the classic LSTM benchmark ($\text{train RMSE} = 0.0072$, $\text{test RMSE} = 0.0276$). These results indicate that the Bayesian LSTM achieves comparable predictive accuracy while providing principled uncertainty quantification. Visual inspection of predictive trajectories, as in Fig.\ref{fig:Bayesian_LSTM_forecast_oneStep}, confirms that the model successfully captures both long-term mortality trends and short-term fluctuations, with credible intervals that widen over the forecast horizon to reflect increasing uncertainty. The intervals account for both aleatoric observation noise and epistemic parameter uncertainty, with GMA sampling yielding sharper posterior concentration than standard MCMC or VI baselines. Overall, these results demonstrate the effectiveness of the proposed GMA approach in enabling scalable Bayesian inference for high-dimensional neural time-series models such as LSTMs.

\begin{figure}[H]
    \centering
    \includegraphics[width=0.65\linewidth]{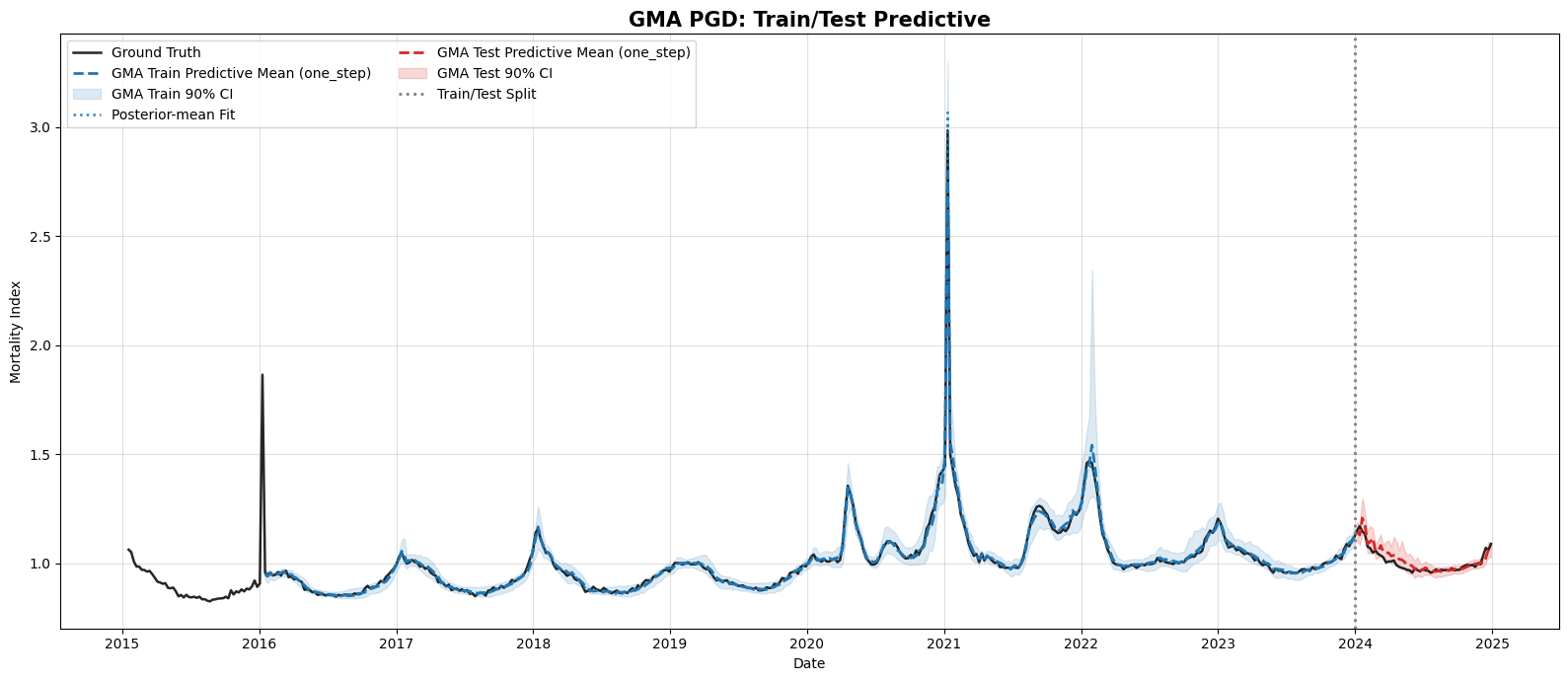}
    \caption{Bayesian LSTM (\textit{one-step rolling}) forecasts of weekly U.S.\ mortality index using pGD-GMA sampling. The posterior mean trajectory (red) closely tracks observed values (black), while the predictive intervals (shaded) capture both short-term fluctuations and long-term uncertainty.}
    \label{fig:Bayesian_LSTM_forecast_oneStep}
\end{figure}

In terms of computational efficiency, training the classic LSTM required approximately 27 seconds for 1,000 epochs of optimization using backpropagation. By contrast, the proposed GMA sampler with projected gradient descent (pGD) completed posterior weight optimization in roughly 1 second for 100 iterations, highlighting its efficiency in tackling the otherwise intractable high-dimensional posterior. This fast way of producing uncertainties in parameters is promising. However, the forecasting stage in the Bayesian setting is substantially more costly: generating predictive trajectories requires rolling out thousands of posterior weight samples through the recurrent network, which can take on the order of minutes to complete. This separation between fast inference of posterior weights and slower uncertainty propagation underscores a key trade-off in Bayesian neural forecasting: efficient sampling of parameters via GMA versus the heavier computational burden of simulating full predictive distributions.

\paragraph{MD-GMA results.}  
Instead of using pGD, which subtracts gradients in Euclidean space and projects back to the probability simplex, we also tried mirror descent (MD, see derivation details in Appendix.\ref{app:mirror_descent}) in GMA to optimise its GMM weights. MD operates in the dual space, applying updates in the log-domain and then exponentiating to yield multiplicative weight rescaling (see e.g. Eq.\ref{eq:MD_exponential_update} and Eq.\ref{eq:temperature_factor_MD_GMA} in Appendix.\ref{app:mirror_descent}). This makes MD particularly well-suited for GMA, since mixture weights must remain positive and sum to unity (i.e. stay in the probability simplex) at every iteration. Similar to pGD-GMA, to initialise the Gaussian mixture, each component mean was centred at the point estimate $\theta^\star$ of the trained classic LSTM and perturbed with small independent Gaussian jitters, so that subsequent GMM samples are drawn from an initial proposal cloud tightly concentrated around a well-trained mode. In addition to tempering ($\beta_k$) and entropy regularization ($\lambda_k$) as used in pGD-GMA, MD-GMA also employs a \textit{temperature schedule} ($\tau_k$) and a \textit{convex mixing coefficient} ($\alpha_k$) to further stabilize the multiplicative updates. Theoretical (see Appendix.\ref{app:mirror_descent_GMA}) and empirical (see Table.\ref{tab:LSTM_for_mortality_summary}) results show that MD-GMA maintains computational efficiency comparable to pGD-GMA, while providing smoother optimization dynamics. The MD-GMA results are presented in Appendix.\ref{app:Bayesian_LSTM_further_results}.

\paragraph{Comparison}
A summary of the results is provided in Table.\ref{tab:LSTM_for_mortality_summary}. 
The classic LSTM achieves very low training RMSE under the one-step foreca scheme, but at the expense of overfitting and poorer generalisation when forecasting on test data. 
In contrast, the Bayesian LSTMs trained with GMA sampling (both pGD and MD) yield slightly higher one-step training error but deliver more balanced performance on held-out test sets. 
The Bayesian models further provide full predictive uncertainty quantification, which the deterministic classic LSTM cannot capture. 
Notably, the computational cost profile is inverted: training the classic LSTM requires tens of seconds of backpropagation, whereas GMA completes posterior weight optimisation in about one second. 
However, Bayesian forecasting incurs a higher cost due to repeated posterior roll-outs, taking several minutes compared to near-instant prediction (forward-passes) in the classic LSTM. 
Overall, these comparisons highlight the trade-off between predictive sharpness and computational efficiency, with Bayesian LSTMs offering principled uncertainty at modest extra forecasting cost.

\begin{table}[H]
\centering
\caption{Forecasting performance (RMSE) and computational cost across methods. 
Training time refers to model fitting (classic LSTM training or GMA weight optimisation), 
while forecasting time measures prediction or posterior predictive roll-out.}
\label{tab:LSTM_for_mortality_summary}
\scriptsize
\setlength{\tabcolsep}{5pt}
\renewcommand{\arraystretch}{0.95}
\begin{tabular}{lcccc}
\toprule
\textbf{Method} & \textbf{Train RMSE} & \textbf{Test RMSE} & \textbf{Training Time} & \textbf{Forecasting Time} \\
\midrule
\multicolumn{5}{l}{\textit{Classic LSTM}} \\
\quad One-step     & $0.0072$ & $0.0276$ & $\sim 27$s (1000 epochs) & $\sim 1$s \\
\quad Multi-step   & $0.1828$ & $0.1030$ & $\sim 27$s (1000 epochs) & $\sim 1$s \\
\midrule
\multicolumn{5}{l}{\textit{Bayesian LSTM (pGD-GMA)}} \\
\quad One-step     & $0.0170$ & $0.0255$ & $\sim 1$s (100 iters) & $\sim 5$min \\
\quad Multi-step   & $0.1681$ & $0.1196$ & $\sim 1$s (100 iters) & $\sim 5$min \\
\midrule
\multicolumn{5}{l}{\textit{Bayesian LSTM (MD-GMA)}} \\
\quad One-step     & $0.0215$ & $0.0256$ & $\sim 1$s (100 iters) & $\sim 5$min \\
\quad Multi-step   & $0.1775$ & $0.1154$ & $\sim 1$s (100 iters) & $\sim 5$min \\
\bottomrule
\end{tabular}
\end{table}

\paragraph{A shortcut for representing uncertainties in large models.} 
Inspired by the warm start strategy used in initialising our GMM centers, we propose a practical approach for uncertainty representation in high-dimensional neural networks, that is, we first train a deterministic model using standard backpropagation, which efficiently provides a maximum a posteriori (MAP)-like, point estimate of the parameters (maybe locally optimal). These trained weights can then serve as a warm start for Bayesian posterior inference methods such as GMA, MH \footnote{MH for example, this substantially improves sampling efficiency, as the chains begin near a plausible mode of the posterior, thereby reducing burn-in and instability.}, or HMC, by perturbing parameters around their trained values. In contrast to classical Laplace approximations, which impose a unimodal Gaussian posterior, our GMM-based GMA sampling approach leverages a mixture of Gaussians to encourage multi-modal local exploration. Because the target and proposal densities are precomputed on a fixed sample cloud, the expensive high-dimensional inference problem is reduced to optimising mixture weights, turning posterior inference into a lightweight, finite-dimensional optimisation task. This yields more expressive posteriors and sharper uncertainty quantification while maintaining computational efficiency. Moreover, one can optionally freeze subsets of the parameters (e.g. lower-layer weights) and only Bayesianise task-specific components (e.g. final layer weights), further reducing cost. This warm-start Bayesianisation strategy scales naturally to very large models, such as transformer-based language models, where full Bayesian inference is infeasible: by combining a deterministic training phase with post-train GMA sampling, we can obtain meaningful uncertainty quantification with limited budget and tractable computation. We give an example of using this post-train sampling strategy to produce uncertainties in LLM in the following section.

\subsubsection{Bayesian language models (BLM): efficient inference and uncertainty representation with \textit{WGMA} in language modelling} \label{sec:BLM}

A quick way to turn a trained, point-estimated LLM into its Bayesian analogue is to \textit{freeze most parameters at their trained values} and \textit{sample only a tractable subset} around the optimized point, e.g. a maximum a posteriori (MAP) estimate. Let $\theta = (\theta_F,\theta_S)$ denote all model parameters split into a frozen block $\theta_F$ and a sampled block $\theta_S$ (e.g. the output head, layer norms, top-$L$ transformer blocks, or low-rank LoRA adapters \cite{seth2025robust}). Given a tokenized dataset $\mathcal{D}=\{(x_t,y_t)\}_{t=1}^T$ and a trained model $\theta^\star=(\theta_F^\star,\theta_S^\star)$, we freeze $\theta_F$ at the trained point $\theta_F^\star$, and target the tempered posterior
\begin{equation} \label{eq:LLM_posterior}
    p_\beta(\theta_S \mid \mathcal{D},\theta_F^\star) \propto 
    \Bigg[\prod_{t=1}^T \mathrm{Cat}  \left(y_t \middle| \mathrm{softmax}(f_{\theta_F^\star,\theta_S}(x_{\le t}))\right)\Bigg]^{\beta}
    \times p(\theta_S),
    \quad \beta\in(0,1]  
\end{equation}
with a Gaussian prior \textit{centered at the MAP estimate}:
\[
p(\theta_S) = \mathcal{N}  \big(\theta_S; \theta_S^\star, \Sigma_0\big), 
\qquad \Sigma_0 = \mathrm{diag}(\sigma_0^2)\ \text{or}\ \Sigma_0 = D^{-1}
\]
where $D$ can be a diagonal preconditioner from Adam's second moment or per-parameter weight-decay scales. Inference here therefore refers to infer the posterior Eq.\ref{eq:LLM_posterior}. In deep learning based language modelling, for example, we aim to infer the distributions of the weights and biases; as exact inference methods are not available due to e.g. intractable integral in high dimensions, approximate inference such as MCMC or VI methods can be employed \footnote{The author wants to distinguish the overloaded usages of 'inference' used in Bayesian inference and LLMs.}.

Conventionally, it is known that inference of a BNN is difficult mainly due to high-dimensional intractability; our objective is to replace the complex transformer posterior with a simple-to-sample GMM density. Specifically, GMA sampling turns a difficult inference problem into a simple GMM weights optimisation problem (similar to VI mechanism), which is practically feasible. GMA uses a $N$-component Gaussian-mixture model (GMM) $q_{\mathbf{w}}(\theta_S)=\sum_{i=1}^N w_i \mathcal{N}(\theta_S;\mu_i,\Sigma_i)$, with initial means $\mu_i=\theta_S^\star+\varepsilon_i$ (small jitter) and isotropic $\Sigma_i=\sigma^2 I$ or block-diagonal covariances (we will talk about key strategies for initialising these components later), to approximate the target $p_\beta$. Good approximation requires minimizing this objective (details see e.g. Eq.\ref{eq:KL_divergence_objective1} in Appendix.\ref{app:gradient_of_exclusive_KL_divergence})
\[
\min_{\mathbf{w}\in\Delta}  KL(q_{\mathbf{w}} \Vert p_\beta)
\mathbb{E}_{q_{\mathbf{w}}}  \big[\log q_{\mathbf{w}}(\mathbf{z})-\log \bar p_\beta(\mathbf{z})\big]
 + \text{const}
\]
which empirically gives (Eq.\ref{eq:KL_divergence_objective2} in Appendix.\ref{app:gradient_of_exclusive_KL_divergence})
\[
    \mathbf{w}^*
    = \arg\min_{\mathbf{w}} KL(q_{\mathbf{w}}(\mathbf{z}) \| p_\beta(\mathbf{z})) 
    = \arg\min_{\mathbf{w}} \left[ \sum_{\mathbf{z}} q_{\mathbf{w}}(\mathbf{z}) \log q_{\mathbf{w}}(\mathbf{z}) - \sum_{\mathbf{z}} q_{\mathbf{w}}(\mathbf{z}) \log \bar{p}_\beta(\mathbf{z}) \right]
\]
The distributional gap in the above can be estimated on a fixed sample \footnote{Note, each sample $s_{i,j}$ id of $d$-dimensional, where $d$ is the total number of weights and bias of the transformer network.} bank $\mathcal{S}=\{s_{i,j}\}_{i=1..N, j=1..M}$ with $s_{i,j}\sim \mathcal{N}(\mu_i,\Sigma_i)$. This sample bank is produced by drawing $M$ samples from each initialised Gaussian components; once drawn, their positions are fixed.

The mixture weights $\mathbf{w}\in\Delta^{N-1}$ are updated by \textit{projected gradient descent} (pGD-GMA). Grouping the pre-drawn samples by their generating component index yields the component-wise gradient for weight $w_i, i=1,2,...,N$ at iteration $k=1,2,...,K$:
\[
g_i^{(k)}
\approx
1 + \frac{1}{M}\sum_{j=1}^M  \Big[\log q_{\mathbf{w}^{(k-1)}}(s_{i,j}) - \log \bar{p}_\beta(s_{i,j})\Big],
\quad
q_{\mathbf{w}^{(k-1)}}(s_{i,j})=\sum_{\ell=1}^N w_\ell^{(k-1)} \mathcal{N}(s_{i,j};\mu_\ell,\Sigma_\ell)
\]
\[
\mathbf{v}^{(k)}=\mathbf{w}^{(k-1)}-\frac{\eta_0}{k} \mathbf{g}^{(k)},
\quad
\mathbf{w}^{(k)}=\Pi_{\Delta}  \big(\mathbf{v}^{(k)}\big)
\]
where a diminishing step size $\eta_k=\eta_0/k$ is used in the gradient update. $\big(\Pi_{\Delta}(\mathbf{v})\big)_i=\max\{v_i-\tau(\mathbf{v}),0\}$ with $\sum_i \max\{v_i-\tau(\mathbf{v}),0\}=1$ projects $\mathbf{v}$ onto the unit weights simplex
\footnote{
Here a standard gradient step $\mathbf{v}^{(k)}=\mathbf{w}^{(k-1)}-\eta_k \mathbf{g}$ is followed by the Euclidean projection onto the probability simplex $\Delta=\{\mathbf{w}\ge 0,\ \sum_i w_i=1\}$. The projection uses the well-known sorter algorithm: find the threshold $\tau$ such that $w_i=\max\{v_i-\tau,0\}$ and $\sum_i w_i=1$, which costs $\mathcal{O}(N\log N)$ per iteration.
}
. $\mathbf{v}$ is the intermediate weight vector after the gradient update step and before projection.
$\bar{p}_\beta=p_\beta \times \texttt{normalising constant}$ is the (tempered) unnormalised LLM posterior (i.e. the target density) in Eq.\ref{eq:LLM_posterior} and $q_{\mathbf{w}}(s_{i,j})=\sum_{\ell=1}^N w_\ell \mathcal{N}(s_{i,j};\mu_\ell,\Sigma_\ell)$ is our approximate GMM density values at the fixed sample positions.
The gradient $g_i$ is an unbiased Monte Carlo estimator with respect to the fixed sample bank.
For efficiency, as in Algo.\ref{algo:GMA-sampling-optimal}, we precompute the Gaussian PDF matrix $P\in\mathbb{R}^{(NM)\times N}$ with $P_{(i-1)M+j,\ell}=\mathcal{N}(s_{i,j};\mu_\ell,\Sigma_\ell)$ and the target log-densities $(p_{\text{tgt}})_{(i-1)M+j}=\log\bar p_\beta(s_{i,j})$, so that per-iteration mixture evaluations reduce to a single matrix-vector product $\mathbf{q}=P \mathbf{w}$, , followed by the component-wise gradient accumulation and a simplex projection.

As in Algo.\ref{algo:GMA-sampling-mirror}, we can also use a more robust method for this constrained, weight optimisation problem, i.e. \textit{mirror descent} (MD-GMA, multiplicative weights), to minimize $\mathrm{KL}(q_{\mathbf{w}}\Vert p_\beta)$ (we ignore the iteration index $k$):
\[
g_i
\approx
1 + \frac{1}{M}\sum_{j=1}^M  \left[\log q_{\mathbf{w}}(s_{i,j}) - \log \bar{p}_\beta(s_{i,j})\right],\ \ 
\tilde{w}_i \propto w_i\exp(-\eta_k g_i),\ \ 
w_i \leftarrow \frac{\tilde{w}_i}{\sum_\ell \tilde{w}_\ell}
\]
where $\eta_k$ is the learning rate in iteration $k$ (e.g. a diminishing step size $\eta_k=\eta_0/k$).

In the weights optimisation procedure, for both pGD-GMA and MD-GMA, we can apply some of the following stabilisation tricks to prevent weight collapse (some were used in our former Bayesian LSTM for mortality modelling task): (i) \textit{tempering} $\beta\uparrow 1$, (ii) \textit{entropy regularization} $\lambda(1+\log w_i)$, (iii) \textit{temperature/convex mixing} in the softmax update, and (iv) \textit{tail averaging} \cite{granziol_iterative_2021} of $\mathbf{w}$ over the last $L$ iterations \footnote{Note $L$ here is the horizon (i.e. length) of the sliding window used in iterative averaging, not the look-back (i.e. lag) length of the sliding window used in LSTM.}. 
The stabilisers (i) to (iii) helps prevent GMM mode collapse, while convex mixing and tail (Polyak) averaging smooth weights evolution trajectories \footnote{Weight trajectories can exhibit oscillating behaviour if the learning rate decaying scheme e.g. $\eta_k=\tfrac{\eta_0}{\sqrt{k}}$ is used as $\sum_k \eta_k^2 = \infty$ violates the Robbins-Monro conditions, as observed in Fig.\ref{fig:gma_diagnostics_all} (Bayesian LSTM inference using pGD-GMA).
When the objective is strongly convex (which is the case of our optimisation objective), using a diminishing step-size $\eta_k \propto 1/k$ can overcome the oscillatory behaviour of SGD \cite{bottou_optimization_2018}.
}.

Also, Gaussian densities are strictly positive, hence $q_{\mathbf{w}}(\mathbf{s}_{i,j})>0$ and $\log q_{\mathbf{w}}$ is well-defined; in implementation, one may clamp by a tiny floor (e.g. $10^{-300}$) for numerical safety. If $\log \bar p$ is evaluated with minibatches, batch sizes should remain consistent across iterations to mitigate estimator drift. Further, preconditioning \footnote{
In high-dimensions, covariance matrices $\Sigma_i$ associated with Gaussian components can become ill-conditioned, leading to unstable numerical behaviour and slow convergence. Preconditioning refers to a reparameterization of the parameter space that rescales or whitens directions of high variance, thereby improving the conditioning of the covariance structure. Specifically, if $D$ denotes a positive diagonal matrix (e.g. constructed from the second-moment statistics of Adam or from per-parameter weight-decay scales), then sampling with covariance $\Sigma_i = \sigma^2 D^{-1}$ ensures that directions of different curvature are appropriately normalized. This yields numerically stable Gaussian densities, prevents degeneracy of $\Sigma_i$ during optimisation, and accelerates convergence of mixture-weight updates. In essence, preconditioning aligns the geometry of the Gaussian proposal with the local curvature of the target posterior, mitigating the effects of anisotropy in high-dimensional parameter subsets.
} helps keep $\Sigma_i$ numerically stable and reduces ill-conditioning in high-dimensional subsets.

\paragraph{Stratified resampling for sample ensemble generation}  
Once the mixture weights $\mathbf{w}$ have been optimized, the final ensemble of approximate posterior samples is generated by re-sampling from the fixed Gaussian sample bank. The re-sampling step converts a weighted particle system into an equally weighted ensemble, which can then be used for posterior prediction without carrying explicit importance weights. Specifically, let $\mathcal{S}_0 = \{s_{i,j}\}_{i=1,\ldots,N; j=1,\ldots,M}$ denote the pre-sampled bank, and let $\mathbf{p} \in \Delta^{NM-1}$ denote the normalized selection probabilities induced by the optimized weights, where $p_{i,j} \propto w_i/M$.  
To generate an ensemble of size $NM$, we draw indices $(i_m,j_m)$ according to
\[
i_m \sim \text{Categorical}(\mathbf{w}), 
\qquad j_m \sim \text{Uniform}(\{1,\ldots,M\}),
\quad m=1,\ldots,NM
\]
and collect the selected particles $s_{i_m,j_m}$ into the final ensemble $\mathcal{S}$. In practice, this step can be efficiently implemented using stratified sampling \cite{neyman_two_1992}, which ensures low-variance estimates and avoid degeneracy due to repeated draws (see Appendix.\ref{app:simple_and_stratified_random_sampling}). The computational cost is linear in the sample bank size, i.e. $\mathcal{O}(NM)$, since each re-sampled index can be obtained by a single categorical draw and a uniform sub-index selection.  

This stratified resampling procedure ensures that the optimized mixture distribution $q_{\mathbf{w}}$ is faithfully represented by an unweighted ensemble of samples, which can then be used directly for posterior predictive Monte Carlo estimation (e.g. averaging logits in language models), providing credible intervals for posterior-based predictiions (i.e. uncertainty propagation). 

\paragraph{Mixture component initialisation and warm start} 
An important trick we propose is to use the point estimate MAP as a warm start, which we call it the 'tickle-around-the-MAP' Bayesianization strategy: we extract the point, (locally) optimal weights $\theta_S^\star$ and use it as a warm start for initialising our Gaussian centers (based which fixed samples are drawn). This makes our approximate density captures at least one mode of the complex posterior, and the following GMM optimisation procedure promotes multi-modal diversity. This 'tickle-around-the-MAP' warm start strategy provides a unified approach for 'fine turning' a deterministic, large model into a Bayesian one (i.e. Bayesianization), and it can be universally used for other chain sampling methods such as MH, LMC and HMC.

Practically, we initialize $N$ component means near $\theta_S^\star$ as $\mu_i=\theta_S^\star+\epsilon_i\odot d$ where $\epsilon_i\sim\mathcal{N}(0,\alpha^2 I)$ and $d$ is a scale vector (e.g. Adam RMS or weight-decay scales). We can employ $\Sigma_i=\sigma^2 I$ or block-diagonal matrices with a small variance $\sigma^2$ that remains inside the local basin but is large enough to expose local multi-modality. This warm start localizes inference and amortizes the cost of evaluating the target density over a fixed sample cloud.

\paragraph{Computational and memory complexity}
With $d=\dim(\theta_S)$, the one-time pre-computation of $P$ costs $\mathcal{O}(N^2 M d^2)$ due to Mahalanobis evaluations, while scoring the bank incurs $\mathcal{O}(L\cdot C_p)$ target evaluations (with $C_p$ the per-sample cost of a forward pass and prior term). Each iteration costs $\mathcal{O}(N^2 M)$ for the matrix-vector multiply, $\mathcal{O}(NM)$ for gradient accumulation, and $\mathcal{O}(N\log N)$ for the projection. The total runtime is therefore $\mathcal{O}  \big(N^2 M d^2\big) + \mathcal{O}  \big(K N^2 M\big)$,
and the memory scales as $\mathcal{O}(NM d + N d^2 + N^2 M)$ for storing the bank, covariances, and the PDF matrix $P$.

\paragraph{Warm start + GMA sampling: a 'shortcut' for uncertainty representation in large models.}
A practical recipe is: train deterministically to $\theta^\star$; select a tractable subset $\theta_S$; place a local Gaussian prior $p(\theta_S)=\mathcal{N}(\theta_S^\star,\Sigma_0)$; build an $N$-component mixture around $\theta_S^\star$; draw a fixed sample bank; and optimize only the mixture weights by pGD using the unbiased MC estimator on the bank. Relative to a Laplace approximation, the mixture captures local multi-modality at negligible marginal cost because (i) target evaluations are amortized (the warm start and local sampling make pGD practical, with the sample bank amortizing the dominant costs) and (ii) the iterative loop reduces to dense matrix-vector multiplies and a simplex projection. Subsetting $\theta_S$ further reduces compute while still enabling well-calibrated predictive intervals for LLMs when full MCMC/VI is infeasible.

\subsubsection*{\textit{Experiments: Bayesianization of a TinyGPT via GMA sampling}}

We study Bayesianized next-token prediction on a small, decoder-only transformer trained with \textit{maximum a posteriori} (MAP) estimation\footnote{
The classic TinyGPT was trained deterministically via backpropagation with maximum-likelihood (cross-entropy) loss and SGD/Adam. 
This is equivalent to a maximum a posteriori point estimate under an implicit flat prior.
} on a byte-level \textit{Byte Pair Encoding} (BPE) corpus of simple English sentences (toy lines). 
We freeze all parameters except the output head (weight tying with the input embedding) and infer a local Bayesian posterior over this \textit{head-only} subset via GMA. 
After optimizing mixture weights on a fixed sample bank, we draw posterior samples to compute token-level Monte Carlo predictions from which we read off probability shifts and entropy changes relative to the MAP model. 
Evaluation is \textit{stream-grounded}: targets are taken from the true encoded token stream to avoid artefacts from whitespace segmentation. 

We report \textit{token-level} \textit{negative log-likelihood} (NLL), \textit{perplexity}, \textit{accuracy}, and \textit{Brier score} (definitions see Appendix.\ref{app:LLM_metrics}), along with qualitative examples; additionally, we examine a 'word-like' subset
\footnote{By 'word-like subset' we mean the set of tokens that consist of contiguous alphabetic sequences (A-Z, a-z), selected using a regular-expression filter. 
This excludes whitespace, punctuation, digits, and byte-level artefacts, and thus approximates next-word prediction behaviour within a byte-level/subword model.}
of tokens (alphabetic sequences identified via a regular-expression filter) to approximate next-word behaviour in this character/subword setting (a byte-level encoding scheme that begins from raw byte values and iteratively merges the most frequent adjacent byte pairs into larger subword units, thereby producing a vocabulary\footnote{
We distinguish between \textit{byte-level} and \textit{token-level}. 
Byte-level refers to raw UTF-8 bytes (0-255), which form the base vocabulary and guarantee full coverage of any input string. 
Token-level refers to the actual units produced by the tokenizer after applying subword merges (e.g. single characters, subwords, or occasionally whole words). 
All reported metrics are computed at the token level, i.e. with respect to these final vocabulary units.
}, i.e. the set of distinct tokens the model can represent and predict).

\paragraph{Experimental designs} 
We design two experiments:  
\textit{(E1) next-character prediction (stream-grounded).} 
Using token pairs from a short toy corpus, we compare the performances of MAP \textit{vs} GMA (ensemble, mean-weight, single-sample) estimated tinyGPTs on token-level accuracy, NLL, and Brier score. 
We also report the same metrics on the 'word-like' subset to approximate next-word behaviour. 
For uncertainty analysis, we tabulate $p(\text{gold})$ and next-token entropy $H=-\sum_v p(v)\log p(v)$ (in nats) for selected contexts to visualize how GMA corrects over/under-confidence relative to MAP.  
\textit{(E2) Long-sentence/paragraph continuation.} 
For a fixed prompt, we decode (using the same temperature/top-$p$ settings) under MAP, GMA ensemble, GMA mean-weight, and several GMA single-sample draws. 
We qualitatively assess diversity and track token-level entropies along the generated tail to illustrate uncertainty propagation under the Bayesianized head.
Our goal is to show that, GMA improves probabilistic calibration and error detection beyond MAP in the setting of language modelling.

\paragraph{Common setup}
A small GPT-style decoder (ctx $=256$, $n_{\text{layer}}=4$, $n_{\text{head}}=8$, $d_{\text{model}}=256$) is trained on a byte-level BPE vocabulary (trained on the toy corpus only). The toy corpus consists of short English tongue twisters and simple sentences; we build train/val splits and evaluate directly from the encoded token stream.

\textit{Subset and prior.} We Bayesianize the \textit{head-only} subset $\theta_S=\{\texttt{lm\_head}\}$ (tied to the embedding). The prior is Gaussian centred at the MAP point, $p(\theta_S)=\mathcal{N}(\theta_S^\star,\Sigma_0)$ with diagonal scale $\Sigma_0=\mathrm{diag}(\sigma_0^2)$; in practice we use a constant base scale per parameter (matching the code).

\textit{GMA configuration.} We initialize $N$ mixture components near $\theta_S^\star$ with small jitter and draw $M$ samples per component to form a fixed bank. We pre-compute the Gaussian PDF matrix $P$ and target log-densities on a fixed minibatch and run projected gradient descent on the mixture weights for $K$ iterations. In default runs we use $(N,M,K,\sigma^2,\eta_0)=(200,8,200,10^{-3},0.2)$, identical to the script.

We compare four prediction modes, all implemented in the code:
\begin{enumerate}
  \item \textbf{MAP}: classic TinyGPT continuation from the point estimate.
  \item \textbf{GMA ensemble (Bayesian)}: probability-space averaging across several GMA samples at each decoding step.
  \item \textbf{GMA mean-weight}: average the sampled parameter vectors into a single head vector and decode once.
  \item \textbf{GMA single-sample}: decode using individual posterior samples (no averaging) to visualize posterior spread.
\end{enumerate}

\textit{Metrics.} On token pairs from the encoded stream we compute accuracy, NLL, and Brier score. We also report results on the word-like subset. For uncertainty we inspect: (i) $p(\text{gold})$ under MAP \textit{vs} GMA; (ii) entropy shifts of the next-token distribution; and (iii) qualitative continuations showing diversification under Bayesian decoding.

\textit{Practical simplifications.} We adopt a byte-level tokenizer (256 base bytes with learned merges) to avoid OOV and spacing artefacts. GMA operates on a \textit{fixed} minibatch when scoring the bank for stability and amortization. We begin with the head-only subset and isotropic covariance in the prior-scaled space; robustness tricks (tempering, entropy regularization, tail averaging) are disabled initially and added only if weight collapse is observed. Evaluation is fully stream-grounded to ensure targets match the model’s tokenization (mitigating ``\texttt{the}'' \textit{vs} `` \texttt{the} " mismatches).

\subsubsection*{\textit{E1: next-character prediction}}

For this experiment we use a small toy corpus of short English tongue twisters and simple declarative sentences, tokenized with the trained byte-level BPE vocabulary. 
The dataset is split into training and held-out evaluation streams, where evaluation is strictly stream-grounded so that targets correspond exactly to the encoded sequence. 
From the evaluation split we extract consecutive token pairs $(x_{\leq t}, y_t)$, which serve as context-target pairs for computing predictive accuracy and calibration metrics. 
In addition to the full evaluation stream, we also construct a 'word-like' subset by selecting only alphabetic tokens via a regular-expression filter, thereby approximating next-word prediction behaviour in this character/subword setting. 
The toy corpus consists of the following sentences:

\begin{quote}
\small
\begin{verbatim}
the quick brown fox jumps over the lazy dog .
a big red cat sits on a mat .
she sells sea shells by the sea shore .
how much wood would a woodchuck chuck .
peter piper picked a peck of pickled peppers .
a good cook could cook good food .
i saw susie sitting in a shoe shine shop .
\end{verbatim}
\end{quote}

In all experiments we compare MAP predictions with 3 posterior-predictive modes derived from the GMA samples:
\begin{enumerate}
    \item \textbf{GMA ensemble:} probability-space averaging across multiple sampled heads at each decoding step. In our experiment we average over $200$ randomly chosen posterior draws \footnote{In our trials, we found that, using more or less samples for averaging the probabilities doesn't matter much, as the posterior is sharply peaked and the posterior variability is small.} from the GMA bank consisting of $1600$ samples.
    \item \textbf{GMA mean-weight:} a single deterministic head formed by averaging all posterior samples (here $1600$ draws) into one mean parameter vector, then decoding once with this head.
    \item \textbf{GMA single-sample:} decoding with individual posterior samples (one of the $1,600$ head vectors) to visualize variability across posterior draws \footnote{Unfortunately, a single sample also yields similar continuations most of the time, unless we deliberately pick a sample from the tails of the posterior.}.
\end{enumerate}
As observed in our later long-form continuation case, since the posterior is sharply peaked, these 3 GMA-based posterior-predictive modes yield largely the same result. Unless otherwise specified, 'GMA' in tables and figures refers to the ensemble variant, which best approximates the Bayesian posterior predictive. 

\paragraph{Results} 
Using token pairs from the toy corpus, we compare the performances of MAP \textit{vs} GMA (ensemble, mean-weight, single-sample) on token-level accuracy, NLL, and Brier score. We also report the same metrics on the word-like subset. For uncertainty analysis, we tabulate $p(\text{gold})$ and next-token entropy $H=-\sum_v p(v)\log p(v)$ (in \textit{nats}) for selected contexts to visualize how GMA corrects over/under-confidence relative to MAP. 

As shown in Table.\ref{tab:e1_results}, on the full token set ($n=137$ stream token pairs), accuracy is identical for MAP and GMA ($0.745$), while GMA achieves a slightly lower NLL ($4.739$ \textit{\textit{vs}} $4.782$), lower perplexity ($114.27$ \textit{\textit{vs}} $119.35$), and lower Brier score ($0.481$ \textit{\textit{vs}} $0.489$), accompanied by a marginally higher entropy ($0.065$ \textit{\textit{vs}} $0.062$ nats). On the word-like subset ($n=96$), both methods again achieve the same accuracy ($0.760$), but GMA outperforms MAP on NLL ($4.347$ \textit{\textit{vs}} $4.388$), perplexity ($77.21$ \textit{\textit{vs}} $80.45$), and Brier score ($0.452$ \textit{\textit{vs}} $0.463$), with entropy increasing slightly ($0.082$ \textit{\textit{vs}} $0.079$ nats). These results indicate that GMA preserves predictive accuracy while offering modest improvements in fit and calibration, with small entropy increases that mitigate MAP’s overconfident spikes. 

Posterior-predictive visualizations, shown in Fig.\ref{fig:posterior_predictive_char}, confirm this pattern: for selected contexts\footnote{By  ``context \texttt{'i'}'' we mean that the model is conditioned on the prefix consisting of the single token \texttt{'i'}. The posterior predictive distribution is then the model’s belief over the next token given this prefix. MAP yields a sharp, single-mode prediction, whereas GMA produces a Bayesian ensemble predictive that redistributes probability mass away from the MAP mode toward plausible alternatives, with credible intervals reflecting posterior variability.} (e.g. \texttt{'i'}), MAP yields a sharp prediction, whereas GMA redistributes probability mass to multiple plausible continuations. Further, as seen in Fig.\ref{fig:entropy_char}, stream-wide entropy shifts $\Delta H$ exhibit mostly positive deviations (up to $\approx 0.262$ \textit{nats} at index 113), highlighting how Bayesianized decoding introduces calibrated uncertainty where the MAP model is overly certain.

\begin{table}[H]
\centering
\footnotesize
\caption{Comparison of MAP \textit{vs} GMA on the toy evaluation set. 
Results are reported for all tokens ($n=137$) and the word-like subset ($n=96$).}
\label{tab:e1_results}
\begin{threeparttable}
\begin{tabular}{lcccccc}
\toprule
\textbf{Subset} & \textbf{Method} & \textbf{Accuracy $\uparrow$} & \textbf{NLL $\downarrow$} & \textbf{PPL $\downarrow$} & \textbf{Brier $\downarrow$} & \textbf{$H$ $\uparrow$} \\
\midrule
\multirow{2}{*}{All tokens (137)} 
  & MAP & 0.745 & 4.782 & 119.35 & 0.489 & 0.062 \\
  & GMA & 0.745 & 4.739 & 114.27 & 0.481 & 0.065 \\
\midrule
\multirow{2}{*}{Word-like (96)} 
  & MAP & 0.760 & 4.388 & 80.45 & 0.463 & 0.079 \\
  & GMA & 0.760 & 4.347 & 77.21 & 0.452 & 0.082 \\
\bottomrule
\end{tabular}
\begin{tablenotes}
\item [1] \footnotesize{NLL = negative log-likelihood, PPL = perplexity, Brier = Brier score, $H$ = entropy (\textit{nats}).}
\end{tablenotes}
\end{threeparttable}
\end{table}

\begin{figure}[H]
    \centering
    % First plot: entropy across stream tokens
    \includegraphics[width=0.85\linewidth]{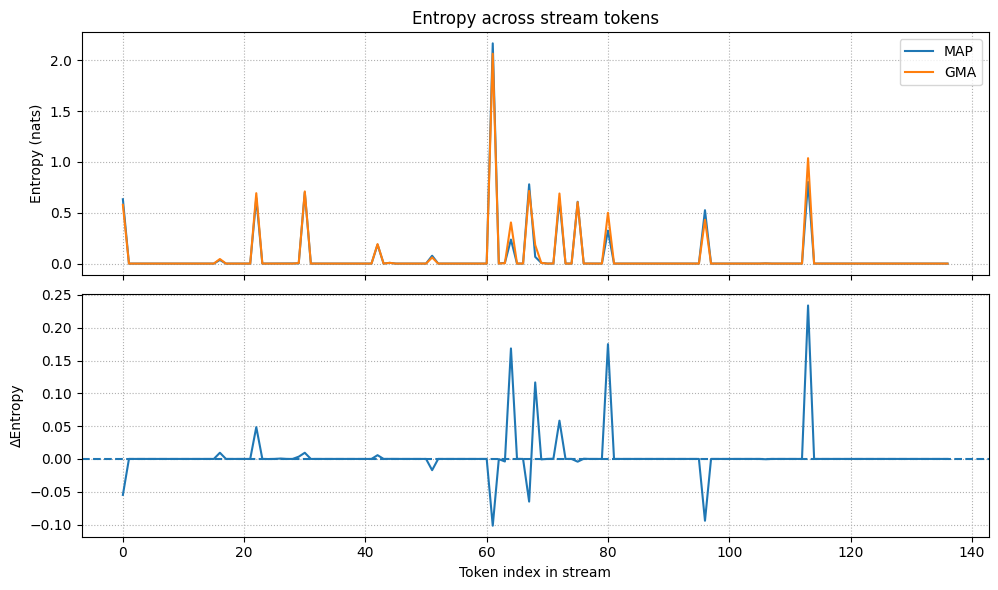}
    \caption{Entropy analysis across the evaluation stream. 
    (Top) Token-level entropy (in \textit{nats}) under MAP \textit{vs} GMA. 
    (Bottom) Entropy difference $\Delta H = H_{\text{GMA}} - H_{\text{MAP}}$, showing mostly positive spikes where GMA softens MAP’s overconfidence.}
    \label{fig:entropy_char}
\end{figure}

\begin{figure}[H]
    \centering
    % Second plot: posterior predictive distribution for one context
    \includegraphics[width=0.9\linewidth]{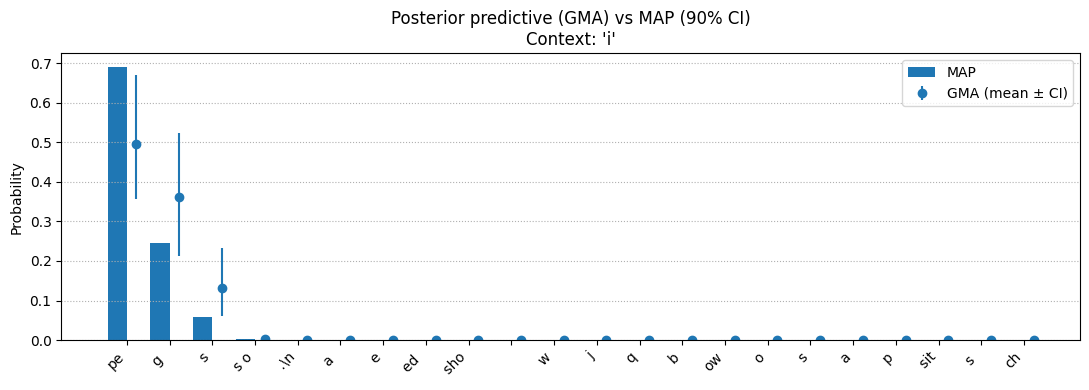}
    \caption{Posterior predictive distribution under MAP \textit{vs} GMA for the context \texttt{'i'}. 
    MAP shows a peaked prediction, whereas GMA redistributes probability mass to plausible alternatives with 90\% credible intervals, reflecting posterior uncertainty.}
    \label{fig:posterior_predictive_char}
\end{figure}

To illustrate this at the token level, Table.\ref{tab:e1_examples_token} shows predictions on the prefix \texttt{ ``the quick ...''}. Both MAP and GMA assign high confidence to the correct continuations, but GMA adjusts probability mass slightly, yielding marginally higher $p(\text{gold})$ values and entropy softening in ambiguous contexts.

\begin{table}[h!]
\centering
\footnotesize
\caption{Example token-level predictions under MAP \textit{vs} GMA on the prefix \texttt{``the quick ...''}. 
$p(\text{gold})$ denotes the probability assigned to the gold token; $H$ is entropy (in \textit{nats}).}
\label{tab:e1_examples_token}
\begin{tabular}{lccccc}
\toprule
\textbf{Context} & \textbf{Gold token} & \textbf{Method} & \textbf{Predicted token} & $p(\text{gold})$ & $H$ \\
\midrule
\multirow{2}{*}{\texttt{the }} & \multirow{2}{*}{q} & MAP & q   & 0.689 & 0.633 \\
 &  & GMA & q   & 0.725 & 0.601 \\
\multirow{2}{*}{\texttt{the q}} & \multirow{2}{*}{u} & MAP & u   & 1.000 & 0.000 \\
 &  & GMA & u   & 1.000 & 0.000 \\
\multirow{2}{*}{\texttt{the qu}} & \multirow{2}{*}{ick} & MAP & ick & 1.000 & 0.000 \\
 &  & GMA & ick & 1.000 & 0.000 \\
\multirow{2}{*}{\texttt{the quick}} & \multirow{2}{*}{\textvisiblespace} & MAP & \textvisiblespace & 1.000 & 0.000 \\
 &  & GMA & \textvisiblespace & 1.000 & 0.000 \\
\multirow{2}{*}{\texttt{the quick }} & \multirow{2}{*}{b} & MAP & b   & 1.000 & 0.000 \\
 &  & GMA & b   & 1.000 & 0.000 \\
\bottomrule
\end{tabular}
\end{table}

Finally, long-form continuation examples in Table.\ref{tab:e1_examples_continuation} demonstrate how MAP and GMA (ensemble/mean-weight.single sample) decode the same prompt. In this case, as the posterior is sharply peaked around the MAP head, different GMA modes (ensemble, mean-weight, or even single-sample draws) converge to nearly identical continuations; but GMA ensemble maintains calibrated probability distributions internally, leading to slightly softer predictions at the token level.

\begin{table}[H]
\centering
\footnotesize
\caption{Example long-form continuations given the prompt \texttt{``the quick brown ...''}. 
Whitespace is shown as `$\cdot$'.}
\label{tab:e1_examples_continuation}
\begin{threeparttable}
\begin{tabular}{lp{0.8\linewidth}}
\toprule
\textbf{Method} & \textbf{Continuation} \\
\midrule
MAP & fox$\cdot$jumps$\cdot$over$\cdot$the$\cdot$lazy$\cdot$dog$\cdot$.\ 
a$\cdot$big$\cdot$red$\cdot$cat$\cdot$sits$\cdot$on$\cdot$a$\cdot$mat$\cdot$.\ 
she$\cdot$sells$\cdot$sea$\cdot$shells$\cdot$by$\cdot$the$\cdot$sea$\cdot$shore$\cdot$. \\
\addlinespace
GMA ensemble & fox$\cdot$jumps$\cdot$over$\cdot$the$\cdot$lazy$\cdot$dog$\cdot$.\ 
a$\cdot$big$\cdot$red$\cdot$cat$\cdot$sits$\cdot$on$\cdot$a$\cdot$mat$\cdot$.\ 
she$\cdot$sells$\cdot$sea$\cdot$shells$\cdot$by$\cdot$the$\cdot$sea$\cdot$shore$\cdot$. \\
\addlinespace
GMA mean-weight & fox$\cdot$jumps$\cdot$over$\cdot$the$\cdot$lazy$\cdot$dog$\cdot$.\ 
a$\cdot$big$\cdot$red$\cdot$cat$\cdot$sits$\cdot$on$\cdot$a$\cdot$mat$\cdot$.\ 
she$\cdot$sells$\cdot$sea$\cdot$shells$\cdot$by$\cdot$the$\cdot$sea$\cdot$shore$\cdot$. \\
\addlinespace
GMA single-sample & fox$\cdot$jumps$\cdot$over$\cdot$the$\cdot$lazy$\cdot$dog$\cdot$.\ 
a$\cdot$big$\cdot$red$\cdot$cat$\cdot$sits$\cdot$on$\cdot$a$\cdot$mat$\cdot$.\ 
she$\cdot$sells$\cdot$sea$\cdot$shells$\cdot$by$\cdot$the$\cdot$sea$\cdot$shore$\cdot$. \\
\bottomrule
\end{tabular}
\begin{tablenotes}
\item[1] \footnotesize{On this toy continuation task, all four methods (MAP, GMA ensemble, GMA mean-weight, and a single GMA draw) produced the same deterministic sequence.}
\end{tablenotes}
\end{threeparttable}
\end{table}

\subsubsection*{\textit{E2: Long-form continuation and calibration on a larger corpus (WGMA)}}

\paragraph{Setup.}
We train a larger TinyGPT (context $=512$, $n_{\text{layer}}=6$, $n_{\text{head}}=8$, $d_{\text{model}}=384$) on a byte-level BPE vocabulary of size $4096$ learned over a Sherlock Holmes bundle from \textit{Project Gutenberg} (${\sim}1.88$M characters; boilerplate removed; whitespace lightly normalized) \cite{projectgutenberg}. 
After \texttt{30,000} optimizer steps, the best held-out validation loss is $\text{NLL}=0.0915$ ($\text{PPL}=1.10$).
We then Bayesianize the \textit{head-only} subset via \textit{WGMA} with $(N,M,K,\sigma^2,\eta_0)=(200,30,1000,10^{-3},0.2)$, producing $NM=6000$ posterior head samples. 
The learned mixture has weight entropy $H(\mathbf{w})=\mathbf{1.988}$ \textit{nats} (max \footnote{For mixture weights $\mathbf{w}=(w_1,\dots,w_N)$, the entropy in \textit{nats} is $H(\mathbf{w})=-\sum_i w_i\log w_i$, maximized by the uniform distribution $w_i=\frac{1}{N}$, giving $H_{\max}=\log N$. With $N=200$, $\log 200 \approx 5.30$ \textit{nats}.} $\log N = \log 200 \approx 5.30$), i.e.\ a moderately concentrated weight distribution. 
Pre-computation of target log-densities for the fixed sample bank took ${\sim}7$ minutes (6000 evaluations), while the $K=1000$ pGD steps completed in ${\sim}3$ seconds on GPU.

\textit{Prediction modes.}
“WGMA” refers to the \textit{ensemble} posterior predictive with $S=128$ samples per decoding step.
We also report \textit{WGMA mean-weight}, which averages all $6000$ heads into a single deterministic head, and \textit{WGMA single-sample}, which is a representative draw (we use sample \#123).
Further, we include an \textit{adaptive} mode that decodes with MAP by default and switches to the WGMA ensemble for $k=5$ steps whenever the MAP next-token entropy exceeds $H_{\text{thresh}}=0.6$.

\subsubsection*{\textit{Results}}

\textit{Mixture diagnostics.}
WGMA learns a \textit{moderately concentrated} mixture: the weight entropy is $H(\mathbf{w})=\mathbf{1.988}$ \textit{nats} (max $\log 200\approx5.30$), indicating that a small subset of components dominates the mass while a long tail remains. 
Empirically, the heaviest few components together account for a sizeable fraction of the probability (Fig.\ref{fig:e2_weights}), with the remainder spread thinly across many low-weight components. This structure is desirable for head-only Bayesianization: the leading components anchor a compact, high-posterior region near the MAP, whereas the tail supplies gentle diversity for posterior averaging without inducing collapse or excessive variance in the ensemble predictive.  

\begin{figure}[H]
    \centering
    \includegraphics[width=0.6\linewidth]{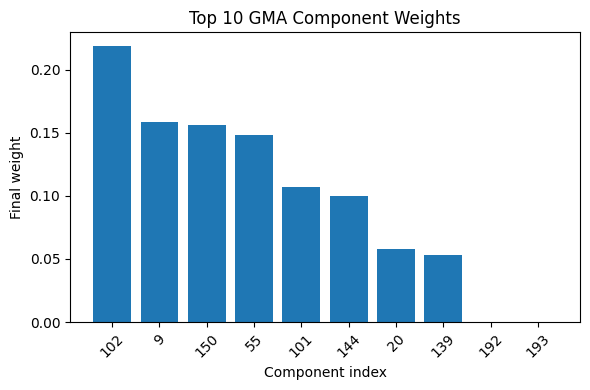}
    \caption{Top-10 WGMA component weights after optimization ($N=200$, $M=30$, $K=1000$).}
    \label{fig:e2_weights}
\end{figure}

\paragraph{Calibration.}
We evaluate calibration\footnote{Calibration means the predicted probabilities should match empirical frequencies: if the model’s top token has confidence $\hat p\approx0.70$ across many cases, it should be correct about $70\%$ of the time (perfect calibration: $\Pr(Y=\hat y \mid \hat p=p)=p$). 
\textit{ECE} buckets confidences and averages the absolute gap between binwise mean confidence and accuracy (lower better). 
\textit{Risk-Coverage/AURC}: sort by confidence; at coverage $c$ keep the top-$c$ fraction and measure the error rate among retained items; the \textit{area} under this curve summarizes selective risk (lower better).} on $5,000$ stream-grounded token pairs (same tokenizer as training). 
As shown in Fig.\ref{fig:e2_conf}-\ref{fig:e2_AURC} and Table.\ref{tab:e2_calibration}, WGMA improves calibration: ECE drops from $0.1101$ to $0.1013$ and AURC from $0.5795$ to $0.5743$.

\begin{figure}[H]
  \centering
  \begin{minipage}{0.48\linewidth}
    \centering
    \includegraphics[width=\linewidth]{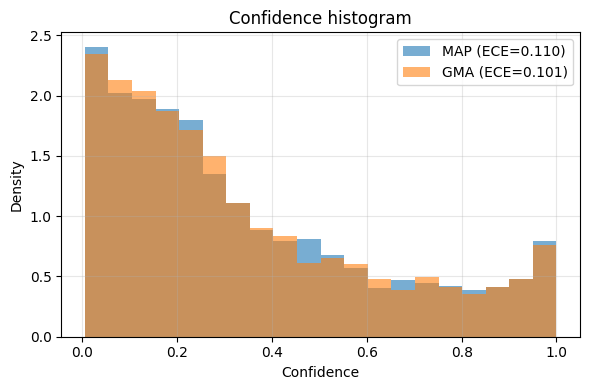}
    \caption{Confidence histograms (20 bins). WGMA shifts mass slightly leftwards and aligns confidence with accuracy more closely (ECE $0.1101 \to 0.1013$).}
    \label{fig:e2_conf}
  \end{minipage}\hfill
  \begin{minipage}{0.48\linewidth}
    \centering
    \includegraphics[width=\linewidth]{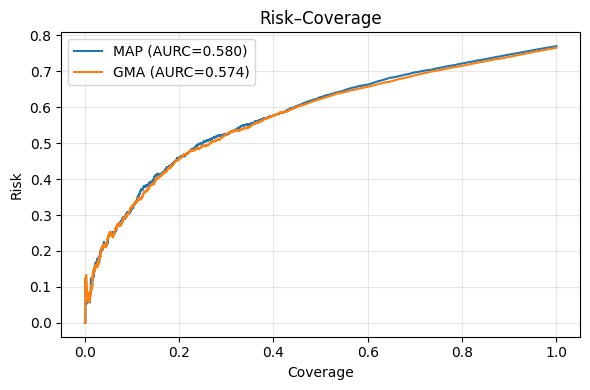}
    \caption{Risk-Coverage curves on held-out tokens (AURC: MAP $0.5795$ \textit{vs} WGMA $0.5743$). WGMA is slightly lower over most coverages.}
    \label{fig:e2_AURC}
  \end{minipage}
\end{figure}

\begin{table}[H]
\centering
\footnotesize
\caption{Calibration on 5,000 held-out token pairs. For calibration (ECE/AURC), we used 5,000 held-out token pairs from the corpus. The target token is the actual next token in the text, so calibration metrics use real ground truth.}
\label{tab:e2_calibration}
\begin{tabular}{lcc}
\toprule
\textbf{Method} & \textbf{ECE $\downarrow$} & \textbf{AURC $\downarrow$} \\
\midrule
MAP & 0.1101 & 0.5795 \\
WGMA ensemble ($S=128$) & \textbf{0.1013} & \textbf{0.5743} \\
\bottomrule
\end{tabular}
\end{table}

\paragraph{Posterior-mean reranking of MAP candidates.}
This is a \textit{single-step} next-token exercise: at a fixed context we take the MAP model’s top-$k$ candidates and \textit{re-score} that same set using WGMA’s posterior-mean predictive to obtain $p_{\text{WGMA}}$. 
The goal is to \textit{rerank}, i.e.\ compare $p_{\text{MAP}}$ \textit{vs} $p_{\text{WGMA}}$ on the same $k$ tokens, to reveal how Bayesian averaging \textit{shifts local preferences} without running a full decode. 
As observed in Table.\ref{tab:e2_rerank}, on \texttt{``We arrived at Baker Street where we found''}, WGMA lowers the top MAP probability while nudging several plausible alternatives upward, yielding a flatter next-token distribution.

\begin{table}[H]
\centering
\footnotesize
\caption{Reranking MAP top-$k$ candidates with WGMA posterior mean on context \texttt{``We arrived at Baker Street where we found''}. Arrows in $p_{\text{WGMA}}$ indicate change relative to $p_{\text{MAP}}$.}
\label{tab:e2_rerank}
\begin{tabular}{lp{0.33\linewidth}cccc}
\toprule
\# & \textbf{Candidate (token)} & $p_{\text{MAP}}$ & $p_{\text{WGMA}}$ & Rank$_{\text{MAP}}$ & Rank$_{\text{WGMA}}$ \\
\midrule
1  & \texttt{. At }                    & 0.6687 & 0.6238 \dec & 1  & 1 \\
2  & \texttt{,\textbackslash n}        & 0.0753 & 0.1082 \inc & 2  & 2 \\
3  & \texttt{."\textbackslash n\textbackslash n "} & 0.0598 & 0.0730 \inc & 3  & 3 \\
4  & \texttt{ed }                      & 0.0389 & 0.0290 \dec & 4  & 4 \\
5  & \texttt{\textbackslash n the }    & 0.0164 & 0.0206 \inc & 6  & 5 \\
6  & \texttt{, and that }              & 0.0164 & 0.0198 \inc & 5  & 6 \\
7  & \texttt{."\textbackslash n\textbackslash n "} & 0.0125 & 0.0158 \inc & 8  & 7 \\
8  & \texttt{. He }                    & 0.0130 & 0.0118 \dec & 7  & 8 \\
9  & \texttt{. This }                  & 0.0091 & 0.0108 \inc & 9  & 9 \\
10 & \texttt{, when}                   & 0.0075 & 0.0104 \inc & 10 & 10 \\
\bottomrule
\end{tabular}
\end{table}

\paragraph{Long-form continuations.}
Unlike the single-step reranking exercise, here we run \textit{multi-step} autoregressive decoding from a fixed prompt with identical hygiene across methods (temperature $0.7$, top-$k=50$, top-$p=0.90$, repetition-penalty $1.1$, no-repeat 3-grams). 
This probes how WGMA’s posterior averaging influences the evolving next-token distribution over many steps (uncertainty propagation) while preserving fluency. 
Examples are shown in Table.\ref{tab:e2_cont}. With strong MAP training (PPL $\approx 1.10$) and head-only Bayesianization, MAP and WGMA yield broadly similar byte-level continuations; WGMA nevertheless maintains better internal calibration (Table.\ref{tab:e2_calibration}) as it averages probabilities at each step \footnote{Ensemble uses $S=128$ samples per step; mean-weight averages all $6000$ heads once; the single-sample draw is \#123; the adaptive mode switches to WGMA for $k=5$ steps whenever MAP entropy exceeds $H_{\text{thresh}}=0.6$.}.
\textit{Qualitatively}, in this prompt, the WGMA mean-weight, single-sample, and adaptive variants exhibit more natural clause structure and fewer brittle artifacts than the MAP and ensemble traces (Table.\ref{tab:e2_cont}); however, this would require systematic human evaluation (or automatic repetition/distinct-$n$ metrics) to generalize beyond the shown example.

\begin{table}[H]
\centering
\footnotesize
\caption{Continuations from the prompt \texttt{``Holmes looked at me and smiled. ''} (first $\sim$180 characters, whitespace normalized; line breaks removed).}
\label{tab:e2_cont}
\begin{tabularx}{\linewidth}{l X}
\toprule
\textbf{Method} & \textbf{Continuation (trimmed)} \\
\midrule
MAP &
Under stretched out upon his finger\hbox{-}tiple the very ates in front of the coarse tobacco with Station now\hbox{-}class mantelpiece, like the box in his arms of a cab's all linguous\ldots \\
\addlinespace
WGMA ensemble ($S=128$) &
Not one upon the light increasoteal draught. And then, sudden legs on the top of the other at lunks by a man who chink and same feeling veloafer; but thern surprise were the other's energy\ldots \\
\addlinespace
WGMA mean\hbox{-}weight &
Not at all. “You've done me if you don't know now. I believe that I won't let yourself in a cabman steps before long as I took a London and gave him still and again and paced some time with here\ldots \\
\addlinespace
WGMA single\hbox{-}sample (\#123) &
Not at all. “Well, I was, as I am very pale and come for my painfully. “You have neither of you. I will leave your time, Watson,” said he, “but I must have a little problem to effect upon your own\ldots \\
\addlinespace
Adaptive ($H{>}0.6$, $k=5$) &
Not seven o'clock. As kept the tradmiserable I felt that my advantage may rest upon me. “You have important matter, Miss Stapleton\ldots,” said Dr.\ Mortimer, “for is the escaped from the supposition\ldots \\
\bottomrule
\end{tabularx}
\end{table}

\textit{Summary.} 
On the larger Holmes corpus, WGMA preserves MAP-level fluency and improves calibration (ECE $0.1101 \to 0.1013$, AURC $0.5795 \to 0.5743$) while reshaping top-$k$ token preferences without architectural changes or heavy compute. 
Gains come from averaging over plausible perturbations of the softmax head around the MAP, which dampens overconfident peaks and reallocates probability mass toward linguistically consistent alternatives. 
Because WGMA operates with a fixed sample bank, expensive target evaluations are amortized; each optimization step then reduces to dense matrix--vector multiplies plus a simplex projection, yielding a calibration lift at negligible marginal cost.

We expect the \textit{ensemble} posterior predictive to remain close to MAP on NLL/PPL while improving calibration (lower Brier/ECE and better selective risk); the \textit{mean-weight} head to interpolate between MAP and full ensemble; and \textit{single-sample} decoding to reflect local posterior variability. 
In E2, reranking illustrates how Bayesian averaging shifts local one-step preferences (Table.\ref{tab:e2_rerank}), and long-form decoding shows that WGMA maintains fluency while propagating better-calibrated uncertainty (Table.\ref{tab:e2_cont}, Table.\ref{tab:e2_calibration}). 
An \textit{adaptive} variant retains MAP’s stylistic continuity yet switches to the ensemble exactly when uncertainty spikes, improving selective prediction and reducing repetition. 
Intuitively, once the backbone is well learned, uncertainty concentrates in the softmax head; WGMA samples plausible head weights near the MAP and averages them, reducing variance and mitigating brittle peaks.

\subsubsection{\textit{WGMA} inference of the Lotka-Volterra (LV) dynamical system} \label{sec:LV_system}

\begin{figure}[ht]
    \centering
    \includegraphics[width=0.35\columnwidth]{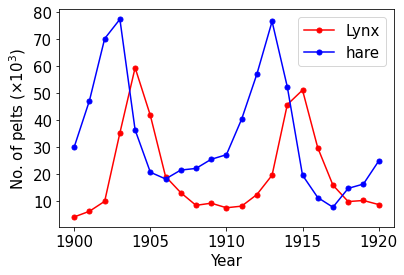}
    \caption{Hudson Bay lynx-hare dataset showing the temporal evolution of prey (hares) and predator (lynx) populations between 1900 and 1920.}
    \label{fig:LV_data}
\end{figure}

We consider the classical LV predator-prey system \cite{alfred_jlotka_elements_1925,Wildlife1924}, a nonlinear ecological model capturing the interactions between hare (the prey) and lynx (the predator) populations. Its dynamics are governed by the coupled ordinary differential equations \cite{blackmore_dynamical_2001,bullett_lotkavolterra_2017,lemos-silva_lotka-volterra_2023,huang_electrostatics-based_2024}:
\begin{equation} \label{Eq:LV_dynamics}
\begin{aligned}
    \frac{dX}{dt} &= aX - bXY \\
    \frac{dY}{dt} &= cXY - dY
\end{aligned}
\end{equation}
where $X(t)$ and $Y(t)$ denote the prey (hare) and predator (lynx) populations, and $\theta=(a,b,c,d)$ are the interaction parameters: prey growth rate $a$, predation rate $b$, predator reproduction rate $c$, and predator mortality rate $d$. The goal is to perform Bayesian inference of these parameters, together with the initial states $(X_0,Y_0)$ and observation noise levels, given 21 annual observations of hare and lynx pelts \footnote{The observed data can be accessed at  \cite{stan2025lotkavolterra}: \url{https://github.com/stan-dev/example-models/blob/master/knitr/lotka-volterra/hudson-bay-lynx-hare.csv}.} (Fig.\ref{fig:LV_data}).

Formally, the Bayesian posterior can be written as:
\begin{equation} \label{Eq:LV_posterior}
    p(\theta,X_0,Y_0,\sigma_1,\sigma_2 \mid X',Y')  \propto  
    p(X',Y'\mid \theta,X_0,Y_0,\sigma_1,\sigma_2) 
    p(\theta) p(X_0) p(Y_0) p(\sigma_1) p(\sigma_2)
\end{equation}
where $(X',Y')=\{(x'_i,y'_i)\}_{i=1}^{21}$ are the observed hare and lynx counts, and $\sigma_1,\sigma_2$ are separate observation noise scales for prey and predator respectively \footnote{One can also use a shared noise level $\sigma_1 = \sigma_2$, e.g. \cite{huang_electrostatics-based_2024}.}.

\subsubsection*{\textit{Experimental design}}

We follow the priors and likelihood of the \textit{Stan} Lotka-Volterra example \cite{LV_Carpenter} and later use its estimated values as reference values; this setup also allows us to compare the GMA-based posterior inference results directly with the reference values, under consistent priors and likelihoods. The 4 ODE parameters are given independent Gaussian priors on the \textit{natural} scale \footnote{In our implementation, we work on log scale and include the Jacobian.}:
\[
a,d \sim \mathcal{N}(1, 0.5^2), \qquad
b,c \sim \mathcal{N}(0.05, 0.05^2)
\]
Initial states use log-normal priors:
\[
\log X_0,\ \log Y_0 \sim \mathcal{N}(\log 10, 1)
\]
and the two observation-noise scales are also log-normal:
\[
\log \sigma_1,\ \log \sigma_2 \sim \mathcal{N}(-1, 1)
\]

The likelihood assumes \textit{independent} \footnote{Independence here refers to the observation errors across time and species; the latent dynamics are deterministic given parameters and initial conditions. This state independence assumption, while common in other literature \cite{carpenter_stan_2017, huang_electrostatics-based_2024} as well, is strong - it assumes the two trajectories as well as the sequential states are independent.} log-normal observation errors for hare and lynx counts:
\[
\log x'_i \sim \mathcal{N}(\log x_i,\ \sigma_1^2),\qquad
\log y'_i \sim \mathcal{N}(\log y_i,\ \sigma_2^2)
\]
where $(x_i,y_i)$ are the ODE solutions at observation times $t_i$, given the ODE parameters (a,b,c,d). 
To obtain the states $(x_i,y_i)$, we integrate the LV system with a fourth-order \textit{Runge-Kutta} (RK4) method using step size $\mathrm{d}t=0.005$ and enforce positivity by truncating states below $10^{-12}$.

For posterior inference we use a two-stage \textit{refining GMA} procedure, as described in Section.\ref{sec:refined_GMA}. Stage.1 performs a coarse exploration of the posterior landscape using a wide Gaussian mixture bank, while Stage.2 re-centres and shrinks the mixture around the most informative components (i.e. those with top-ranked weights), thereby refining the approximation. To account for the heterogeneous scales of the parameters, we initialise an anisotropic diagonal covariance on the log-parameters with standard deviations
\[
\bigl[
\overset{\log a}{0.28}, 
\overset{\log b}{0.22}, 
\overset{\log c}{0.24}, 
\overset{\log d}{0.28}, 
\overset{\log X_0}{0.55}, 
\overset{\log Y_0}{0.55}, 
\overset{\log \sigma_1}{0.22}, 
\overset{\log \sigma_2}{0.22}
\bigr]
\]
with tighter scales for interaction coefficients $(b,c)$, i.e. the interaction coefficients which are sensitive to state perturbations, and broader scales for initial states $(X_0,Y_0)$. 

\textbf{Stage.2 re-centering and shrinkage.} Let $w^{(1)}  \in  \Delta^{N-1}$ be the Stage.1 mixture weights and $\{\mu^{(1)}_\ell\}_{\ell=1}^N$ the Stage.1 means with corresponding diagonal std vector $\sigma^{(1)}$ given above. We select the top $20\%$ components by $w^{(1)}$ and resample Stage.2 centres from this subset with probabilities proportional to $w^{(1)}$, then add a small Gaussian \textit{jitter}:
\[
\mu^{(2)}_i  =  \mu^{(1)}_{J_i}  +  \varepsilon_i,\qquad 
J_i \sim \mathrm{Cat}  \left(\frac{w^{(1)}_{\text{top}}}{\sum w^{(1)}_{\text{top}}}\right),\quad
\varepsilon_i \sim \mathcal{N}  \Bigl(0,\ \mathrm{diag}\bigl((\rho \zeta \sigma^{(1)})^2\bigr)\Bigr)
\]
with shrink factor $\rho=0.35$ and jitter scale $\zeta=0.15$. The Stage.2 covariances are set deterministically to
\[
\Sigma^{(2)}  =  \mathrm{diag}  \bigl((\rho \sigma^{(1)})^2\bigr)
\]
i.e. each log-parameter’s standard deviation is reduced to $0.35$ times its Stage.1 value (variances scaled by $0.35^2$), preserving the same anisotropy while focussing the bank around the high-weight region.

To improve efficiency and stability we (i) evaluate mixture densities in log-space via a \textit{log-sum-exp} \footnote{Details about the  \textit{log-sum-exp} operation see Appendix.\ref{app:combine_precomputing_and_MC_gradient_estimator}.} operation which mitigates numerical underflow in high dimensions, (ii) use the small RK4 step above to reduce discretisation bias in the likelihood computation, and (iii) precompute the matrix of component log-densities for all banked samples. In each stage the unnormalised target (prior $+$ likelihood) is evaluated \textit{once per banked sample} ($NM$ ODE solves per stage; $2NM$ total). Subsequent weight updates reuse cached log-densities; per-iteration work is then dominated by the log-sum-exp over $N$ components for each of the $NM$ samples, yielding $\mathcal{O}(K N^2M d^2)$ arithmetic after a one-off precomputation cost of $\mathcal{O}(N^2 Md)$ per stage.

%==========================
% pGD-GMA \textit{vs} MD-GMA (LV)
%==========================
\subsubsection*{\textit{LV inference results: pGD-GMA vs MD-GMA}}

We ran the two-stage refined GMA pipeline with identical hyperparameters for both optimizers (\textit{pGD-GMA} and \textit{MD-GMA}): $N=200$ mixture components, $M=30$ bank samples per component, $K=500$ weight updates per stage, RK4 with $\mathrm{d}t=0.005$. Stage-1 used anisotropic log-scale standard deviations stated before; Stage-2 applied a refine factor $0.35$ with jitter $0.15$. Each stage evaluates the unnormalised target exactly once per banked sample (thus $NM=6000$ ODE solves per stage; $2NM=12,000$ total). All subsequent weight updates reuse cached log-densities via a numerically stable \textit{log-sum-exp} mixture \footnote{Details about the \textit{log-sum-exp} operation to improve numerical stability can be found in Appendix.\ref{app:combine_precomputing_and_MC_gradient_estimator}.}.

MD-GMA and pGD-GMA update achieve comparable accuracy and runtime, while MD-GMA exhibits markedly improved optimisation stability (less mode collapse).
Runtime was \textbf{886.23 s} for pGD-GMA and \textbf{873.52 s} for MD-GMA. The weight evolution diagnostics in Fig.\ref{fig:pGD_GMA_LV}a-b and Fig.\ref{fig:MD_GMA_LV}a-b reveal a clear qualitative gap: pGD-GMA concentrates nearly all mass onto a single component after a few iterations (most others fall below $10^{-15}$ on log-scale), whereas MD-GMA keeps a small cohort of non-negligible components throughout Stage.2 (several weights persist between $10^{-3}$ and $10^{-1}$ on log-scale). This mitigates premature mode collapse and preserves diversity for Stage.2 re-centering. 

Despite weight differences, the inference results from both optimizers, however, are similar. Approximate prior and posterior plots in Fig.\ref{fig:pGD_GMA_LV}c-d and Fig.\ref{fig:MD_GMA_LV}c-d show that Stage-2 re-centers tightly around the posterior bulk for both optimizers. The two-stage, progressive GMA method is able to narrow down the search region; the reference values, however, lie on the fringe of the posterior clouds. Quantitatively, the two optimizers (pGD and MD) deliver virtually identical posterior means (Table.\ref{tab:param-means-LV}), indicating that, given the same sample banks (prior samples), they converge to the same posterior mode. Nonetheless, the modes deviate from the reference values (vertical red lines in the histograms in Fig.\ref{fig:pGD_GMA_LV}b and Fig.\ref{fig:MD_GMA_LV}b).
Posterior predictive medians and $90\%$ credible predictive intervals in Fig.\ref{fig:pGD_GMA_LV}d and Fig.\ref{fig:MD_GMA_LV}d are visually similar across optimizers, and both reproduce the oscillatory dynamics. Phases are captured well while peak amplitudes are somewhat damped compared to observed data. 

\begin{table}[ht!]
\centering
\caption{Final posterior means (from Stage.2).}
\label{tab:param-means-LV}
\setlength{\tabcolsep}{6pt}
\begin{tabular}{lrrrr}
\toprule
Param & Reference \cite{LV_Carpenter} & pGD-GMA & MD-GMA & Rel.\ error (pGD / MD) \\
\midrule
$\alpha$ & 0.55   & 0.8098 & 0.8120 & +47.2\% / +47.6\% \\
$\beta$  & 0.028  & 0.04582 & 0.04585 & +63.6\% / +63.8\% \\
$\delta$ & 0.024  & 0.06495 & 0.06516 & +170.6\% / +171.5\% \\
$\gamma$ & 0.80   & 0.8957 & 0.8918 & +12.0\% / +11.5\% \\
$X_0$    & 33.956 & 5.61 & 5.67 & $-83.5\%$ / $-83.3\%$ \\
$Y_0$    & 5.933  & 10.34 & 10.34 & +74.3\% / +74.3\% \\
$\sigma_1$ & 0.25 & 0.549 & 0.546 & +119.6\% / +118.4\% \\
$\sigma_2$ & 0.25 & 0.489 & 0.489 & +95.6\% / +95.6\% \\
\bottomrule
\end{tabular}
\end{table}

\begin{table}[H]
\centering
\caption{Runtime and posterior predictive RMSEs (median trajectory \textit{vs} observations).} 
\label{tab:rmse-runtime}
\setlength{\tabcolsep}{10pt}
\begin{tabular}{lccc}
\toprule
Method & Runtime (s) & RMSE (Hare) & RMSE (Lynx) \\
\midrule
pGD-GMA & 886.23 & 28.26 & 15.37 \\
MD-GMA  & 873.52 & 28.40 & 15.21 \\
\bottomrule
\end{tabular}
\end{table}

%--------------------------
% Figures: pGD
%--------------------------
\begin{figure}[H]
\centering

\begin{subfigure}[b]{0.48\linewidth}
  \centering
  \includegraphics[width=\linewidth]{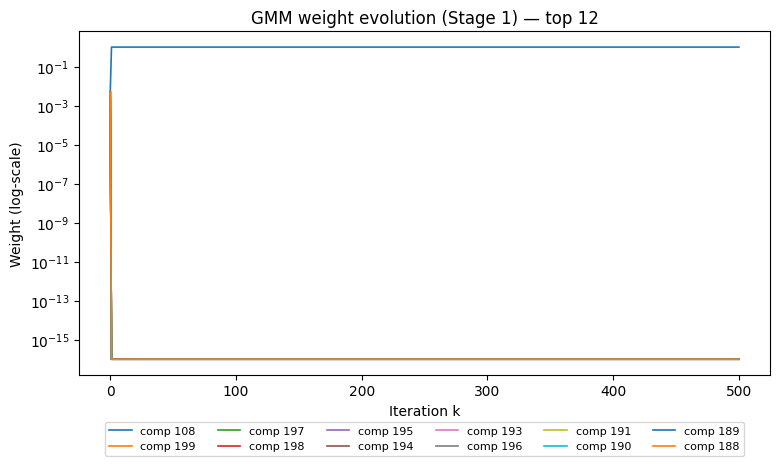}
  \caption{Weights evolution, Stage.1 (pGD).}
\end{subfigure}\hfill
\begin{subfigure}[b]{0.48\linewidth}
  \centering
  \includegraphics[width=\linewidth]{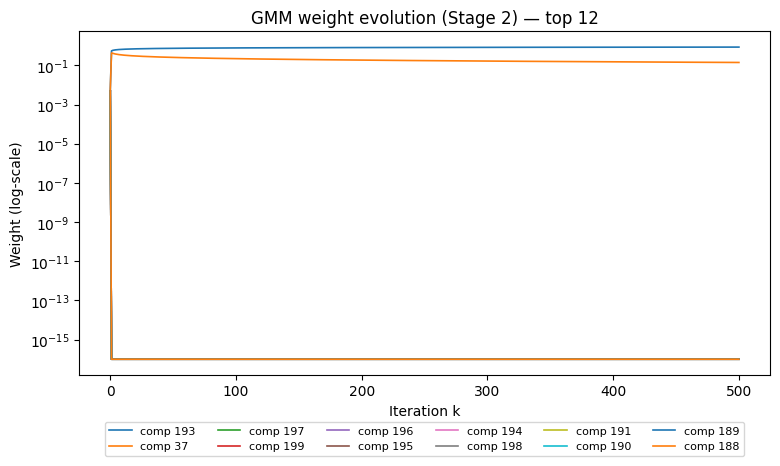}
  \caption{Weights evolution, Stage.2 (pGD).}
\end{subfigure}

\begin{subfigure}[b]{0.98\linewidth}
  \centering
  \includegraphics[width=\linewidth]{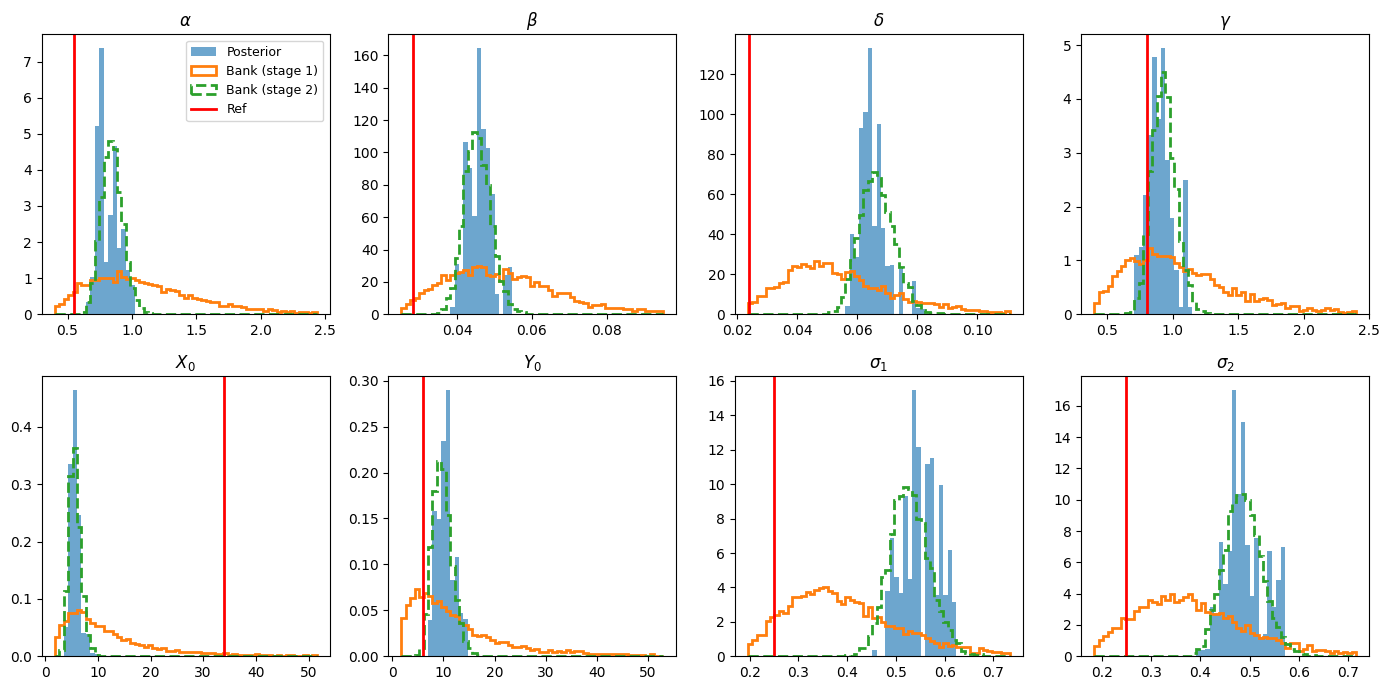}
  \caption{Approx. prior, posterior, and reference values (marginal, pGD).}
\end{subfigure}

\begin{subfigure}[b]{0.98\linewidth}
  \centering
  \includegraphics[width=\linewidth]{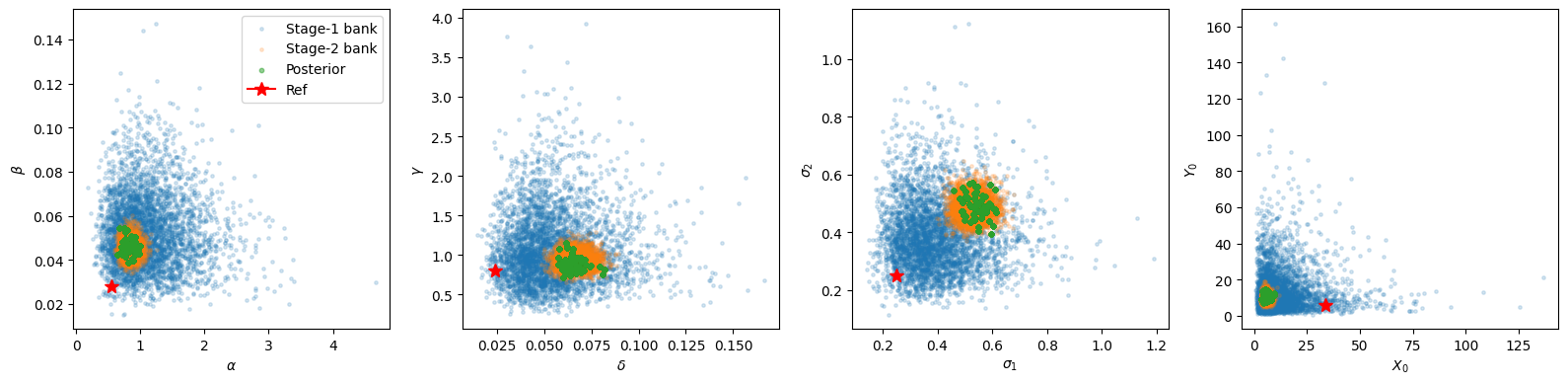}
  \caption{Approx. prior, posterior, and reference values (2D, pGD).}
\end{subfigure}

\begin{subfigure}[b]{0.95\linewidth}
  \centering
  \includegraphics[width=\linewidth]{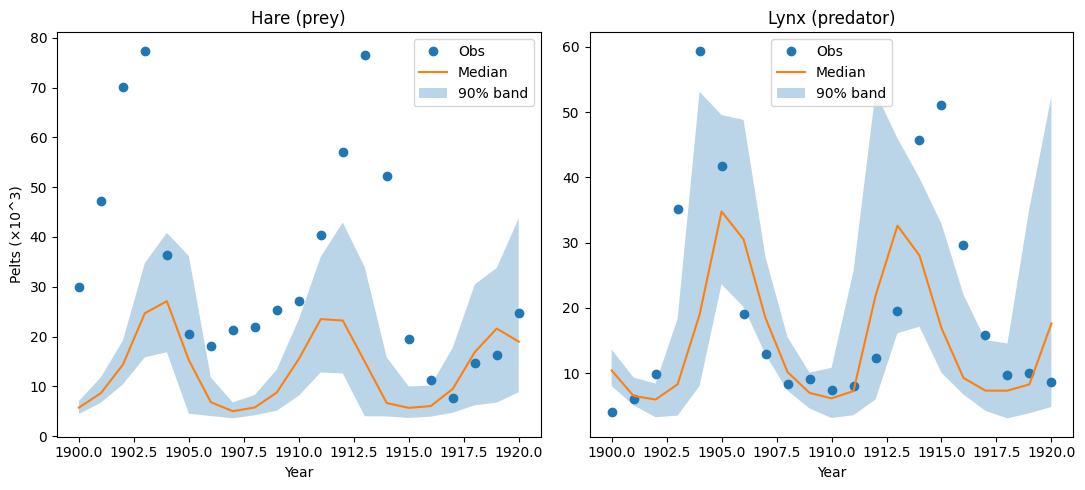}
  \caption{Posterior predictive (90\% credible interval, pGD).}
\end{subfigure}

\caption{pGD-GMA: weights evolution, posteriors and posterior predictives.}
\label{fig:pGD_GMA_LV}
\end{figure}

%--------------------------
% Figures: MD
%--------------------------
\begin{figure}[H]
\centering

\begin{subfigure}[b]{0.48\linewidth}
  \centering
  \includegraphics[width=\linewidth]{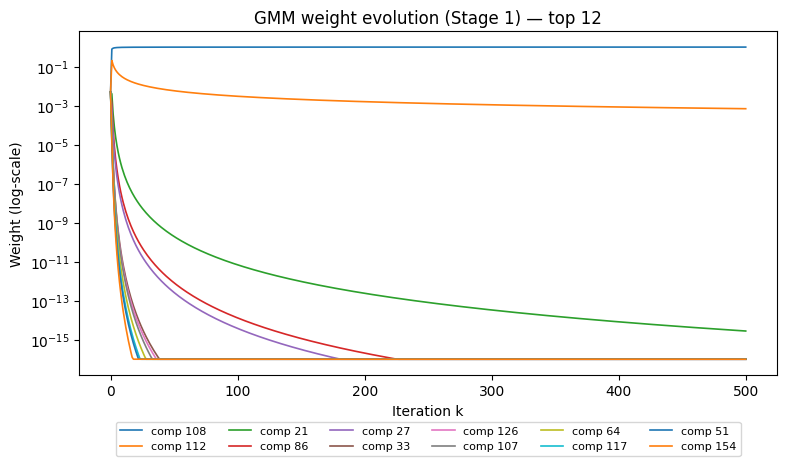}
  \caption{Weights evolution, Stage.1 (MD).}
\end{subfigure}\hfill
\begin{subfigure}[b]{0.48\linewidth}
  \centering
  \includegraphics[width=\linewidth]{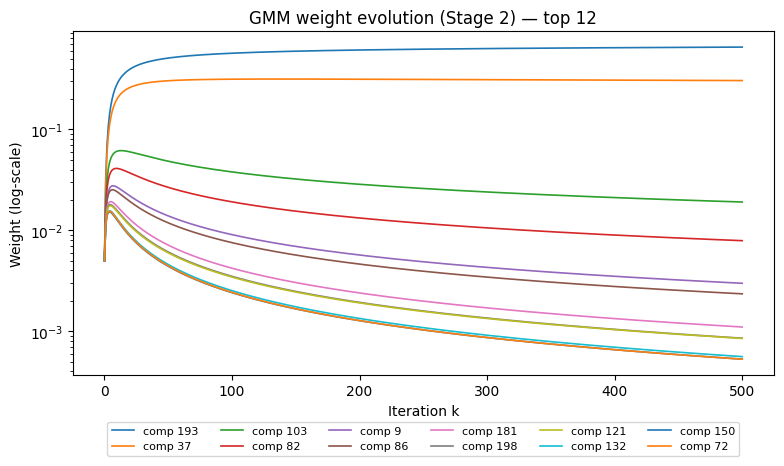}
  \caption{Weights evolution, Stage.2 (MD).}
\end{subfigure}

\begin{subfigure}[b]{0.98\linewidth}
  \centering
  \includegraphics[width=\linewidth]{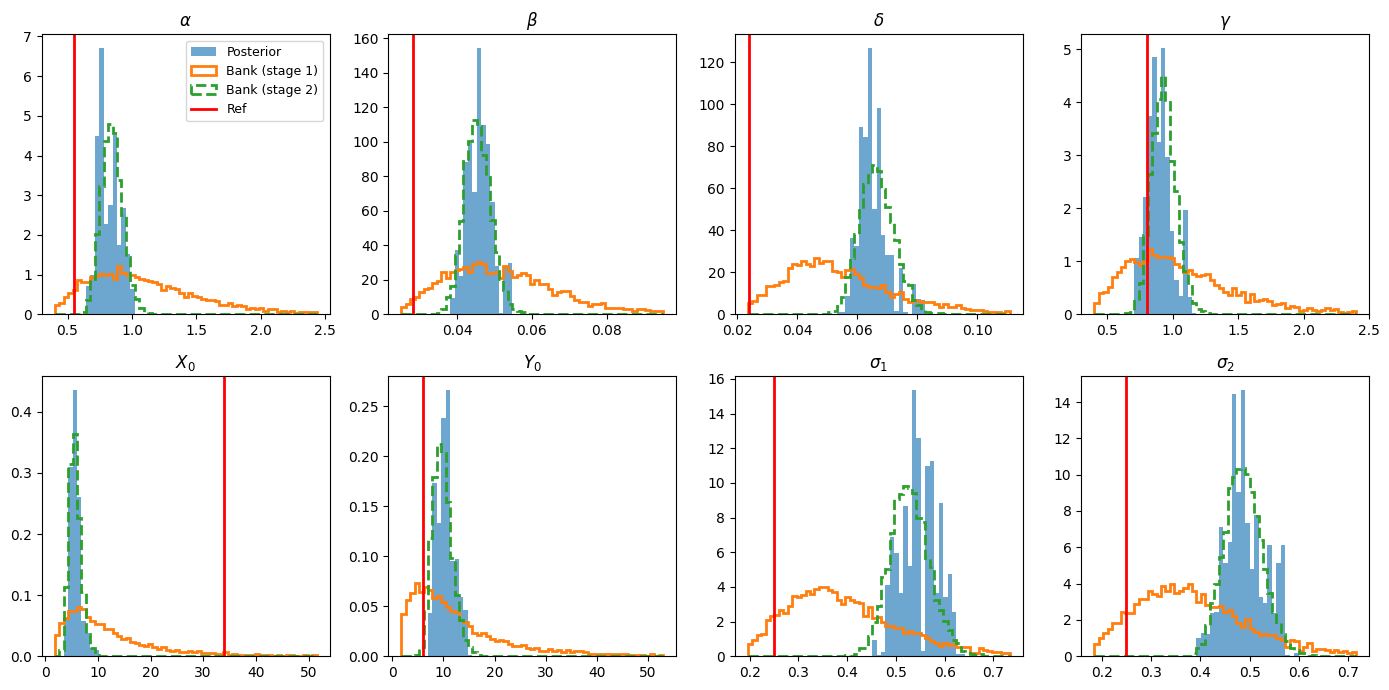}
  \caption{Approx. prior, posterior, and reference values (marginal, MD).}
\end{subfigure}

\begin{subfigure}[b]{0.98\linewidth}
  \centering
  \includegraphics[width=\linewidth]{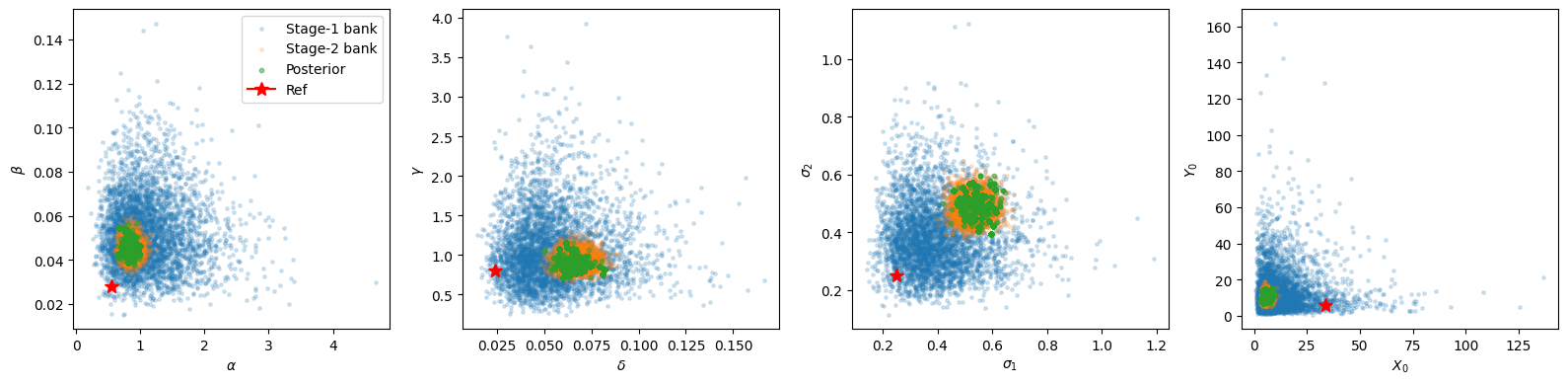}
  \caption{Approx. prior, posterior, and reference values (2D, MD).}
\end{subfigure}

\begin{subfigure}[b]{0.95\linewidth}
  \centering
  \includegraphics[width=\linewidth]{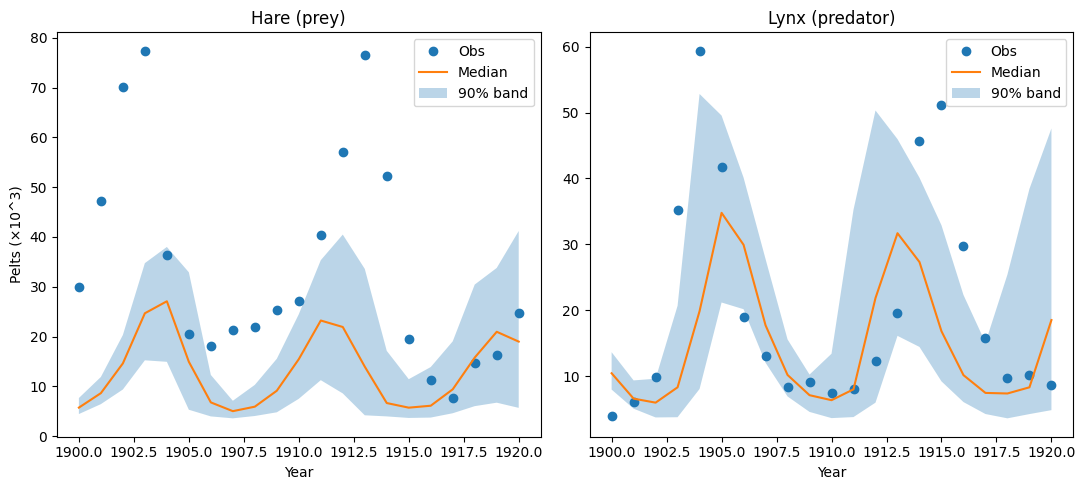}
  \caption{Posterior predictive (90\% credible interval, MD).}
\end{subfigure}

\caption{MD-GMA: weights evolution, posteriors and posterior predictives.}
\label{fig:MD_GMA_LV}
\end{figure}

\subsubsection{SIR model inference using \textit{LMA}} \label{subsec:SIR_model_inference}

\paragraph{Experimental setup.}
We fit a simple susceptible-infectious-removed (SIR) model to daily new COVID-19 cases in England, using data retrieved using the UKHSA dashboard API (specimen-date series \footnote{UKHSA COVID-19 dashboard: \url{https://coronavirus.data.gov.uk/}. JSON API for daily cases by specimen date (England): \url{https://api.ukhsa-dashboard.data.gov.uk/themes/infectious_disease/sub_themes/respiratory/topics/COVID-19/geography_types/Nation/geographies/England/metrics/COVID-19_cases_casesByDay}.
A \textit{specimen-date series} means the time series is indexed by the date the patient’s sample was taken (the specimen collection date), not the date the test was processed or reported.
})
To keep runtime short while illustrating the methodology, we use a tiny early epidemic window: from 2020-03-01 and keep the first 21 days, then thin by a stride of 2 (yielding 11 observations). We apply a centered 7-day moving average to reduce day-of-week effects. The population size is fixed to $N=56,550,138$. 
As described in Section.\ref{subsec:LMA}, we use the multi-start optimisation strategy to find local modes of the posterior; in our implementation, optimisation uses \textbf{6} random prior initialisations. Posterior approximation flattens local covariances by $\kappa=1.5$. We draw \textbf{1000} posterior samples from the fitted Laplace mixture for summaries and predictive curves.

\paragraph{The SIR model.}
The latent dynamics follow the continuous-time SIR ODEs \cite{Hethcote2000}:
\begin{equation} \label{eq:SIR_model}
    \dot S(t)=-\beta \frac{S(t)I(t)}{N},\qquad
    \dot I(t)=\beta \frac{S(t)I(t)}{N}-\gamma I(t),\qquad
    \dot R(t)=\gamma I(t)
\end{equation}
with initial condition $S_0=N-I_0$, $I_0>0$, $R_0(0)=0$.
Here $S(t),I(t),R(t)$ denote the numbers of susceptible, infectious, and removed (recovered/isolated) individuals; $N$ is the total population size (assumed constant). The parameters are:
$\beta>0$, the effective transmission rate (per day) that scales the mass-action term $\beta S(t)I(t)/N$; 
$\gamma>0$, the removal/recovery rate (per day), so the mean infectious period is $1/\gamma$; 
$I_0$ is the initial number of infectious individuals at $t=0$. 
A key derived quantity is the basic reproduction number $R_0=\beta/\gamma $ in a fully susceptible population.
Daily infection incidence on the observation grid is approximated by $-\Delta S_t$ between consecutive solution times.

\paragraph{Bayesian model for parameters.}
To derive the Bayesian posterior we combine the deterministic SIR dynamics with a Poisson observation model for daily cases.
Let $y_{1:T}$ denote observed new cases (specimen date) on a daily grid and let
\[
\lambda_t(\theta) = \rho \big(-\Delta S_t(\theta)\big),
\quad \text{ with } \theta=(\beta,\gamma, I_0,\rho),\ \rho\in(0,1], \quad t=1,\dots,T
\]
where $-\Delta S_t(\theta)$ is the model-implied daily incidence obtained by numerically solving the SIR ODEs in Eq.\ref{eq:SIR_model} under parameters $\theta$.
We assume conditional independence across days
\[
y_t \mid \theta \ \sim\ \mathrm{Poisson} \big(\lambda_t(\theta)\big),
\qquad t=1,\dots,T
\]
so the likelihood is $p(y_{1:T}\mid\theta)=\prod_{t=1}^T \mathrm{Poisson}\big(y_t;\lambda_t(\theta)\big)$.
With an independent prior $p(\theta)=p(\beta) p(\gamma) p(I_0) p(\rho)$, Bayes’ rule yields:
\[
p(\theta \mid y_{1:T}) \ \propto\ p(\theta) \prod_{t=1}^T \mathrm{Poisson}\big(y_t;\lambda_t(\theta)\big)
\]
or in log form \footnote{
The Poisson \textit{pmf} is $p(y_t\mid \theta)=\frac{e^{-\lambda_t(\theta)}\lambda_t(\theta)^{y_t}}{y_t!}$, taking log of the unnormalised posterior$p(\theta\mid y_{1:T})  \propto  p(\theta) \prod_{t=1}^T p(y_t\mid\theta)$ gives the log-likelihood:
$$
\log p(y_{1:T}\mid\theta)
=\sum_{t=1}^T \Big\{ y_t\log \lambda_t(\theta) - \lambda_t(\theta) - \log(y_t!)\Big\}
$$
As $y_t! = \Gamma(y_t+1)$, we can write $\log(y_t!)=\log\Gamma(y_t+1)$.
}:
\[
\log p(\theta \mid y_{1:T})   =  
\underbrace{\sum_{t=1}^T \big[ y_t \log \lambda_t(\theta) - \lambda_t(\theta) - \log\Gamma(y_t+1)\big]}_{\text{log-likelihood}}
 +  \underbrace{\log p(\theta)}_{\text{log-prior}}
 +  \text{const}
\]
maximizing which is justified by: 
the $y_t\log\lambda_t(\theta)$ term rewards parameter values that place mass near the observed counts;
the $-\lambda_t(\theta)$ term penalizes overly large means;
the $-\log\Gamma(y_t+1)$ term depends only on the data (not on $\theta$), it’s part of the likelihood but is a constant \textit{w.r.t.} $\theta$;
the $\log p(\theta)$ term incorporates prior information (regularization).
The normalising constant \textit{const} collects any terms independent of $\theta$ (e.g. normalizing constants); it can be dropped for optimization and Laplace mixture construction, unless an absolute marginal likelihood is needed.
Also note that, $\lambda_t(\theta)$ is not a free parameter: it is the output of the SIR ODE solution under $\theta=(\beta,\gamma,I_0,\rho)$, making the log-likelihood is nonlinear and skewed in $\theta$, i.e. a small change in $(\beta,\gamma)$ can alter the entire trajectory $-\Delta S_t(\theta)$ across all $t$.

We use weakly informative priors and respect support constraints: 
\[
\beta \sim \mathrm{LogNormal} \big(\log 0.35, 0.5^2\big),\quad
\gamma \sim \mathrm{LogNormal} \big(\log (1/7), 0.5^2\big)
\]
\[
I_0 \sim \mathrm{LogNormal} \big(\log 100, 1^2\big),\quad
\rho \sim \mathrm{Beta}(2,8)
\]
where the LogNormal distribution is parameterised by the \textit{log-mean} and \textit{log-standard deviation}, so the prior medians are $0.35$ day$^{-1}$, $1/7$ day$^{-1}$, and $100$ for $(\beta,\gamma,I_0)$, respectively.
The basic reproduction number is treated as a derived quantity,
\[
\mathcal R_0  =  \frac{\beta}{\gamma}
\]
with an induced (heavy-tailed) prior under the independent LogNormal marginals above.

\paragraph{Inference via Laplace-mixture approximation (LMA).}
We maximise the log posterior from \textit{six} random prior starts using L-BFGS-B in an unconstrained reparameterisation space (log for positive variables, logit for $\rho$). Let $\ell(\theta)=\log \bar p(\theta)$ be the log unnormalised posterior. For each mode mean $\hat\theta_j$ we compute a finite-difference Hessian $H_j=\nabla^2 \ell(\hat\theta_j)$ in $\theta$-space, set the local covariance (see Section.\ref{subsec:LMA})
\[
\Sigma_j=\big(-H_j\big)^{-1},\qquad \Sigma_j\leftarrow \kappa^2 \Sigma_j \ \ (\kappa=1.5)
\]
and initialise evidence-based weights
\[
\tilde w_j \propto \exp\{\ell(\hat\theta_j)\} (2\pi)^{d/2} |\Sigma_j|^{1/2},
\qquad w_j=\tilde w_j\big/\sum_r \tilde w_r
\]
yielding the Gaussian mixture
\(
q(\theta)=\sum_{j=1}^J w_j \mathcal N(\theta;\hat\theta_j,\Sigma_j)
\)
, where $J$ is the total number of modes discovered by the multi-start optimisation.
Posterior draws are obtained by direct (stratified) sampling from $q(\theta)$; each draw $\theta^{(k)}$ produces a \textit{model-mean trajectory} \footnote{These trajectories are the conditional means of the observation model; drawing full observation-level posterior predictives would additionally sample $Y_t^{(k)}\sim\mathrm{Poisson}(\lambda_t^{(k)})$.} $\lambda^{(k)}_t=\rho^{(k)}(-\Delta S^{(k)}_t)$ by solving the SIR ODE with parameters $\theta^{(k)}$.

\paragraph{Results.}
Multi-start optimisation (six initialisations) recovered \textit{four} unique posterior modes with near-identical log-evidence, yielding almost uniform mixture weights $\textbf{w}=[0.247, 0.247, 0.253, 0.252]$; the four modes identified are reported in Table.\ref{tab:SIR_posterior_modes} and all imply $R_0 \approx 4$. 
Fig.\ref{fig:SIR_posteriors} shows the marginal posteriors for $\beta$, $\gamma$, the derived $R_0=\beta/\gamma$, $\rho$, and $I_0$.
Summary statistics of these posterior distributions are presented in Table.\ref{tab:SIR_posterior_summary}.
Drawing $1000$ samples from the fitted Laplace mixture produced a heavy-tailed posterior for $R_0$ with median $\mathbf{3.95}$ and a $95\%$ credible interval of $[2.27, 32.12]$, where the long upper tail reflects weak identifiability of $\gamma$ in an early, near-exponential phase.
Fig.\ref{fig:SIR_predictives} shows the predictive mean trajectories $\lambda_t^{(k)}$, together with the $95\%$ credible interval, produced using the posterior mean, median, and the dominating-mode mean. The credible band is visibly asymmetric about the median, consistent with a positively skewed, non-negative predictive distribution. 
Computational time in Table.\ref{tab:SIR_timing} is dominated by optimisation and by forward simulations; mixture construction and sampling are negligible. Overall, LMA is fast and effective, particularly in the presence of small data.

\begin{table}[H]
\centering
\small
\begin{tabular}{rccccccc}
\toprule
Mode & $\beta$ (day$^{-1}$) & $\gamma$ (day$^{-1}$) & $I_0$ & $\rho$ & $R_0=\beta/\gamma$ & $\log p(\hat\theta)$ & Weight \\
\midrule
\#1 & 0.461 & 0.113 & 341.5 & 0.314 & 4.08 & $-1003.4$ & 0.247 \\
\#2 & 0.460 & 0.111 & 339.4 & 0.317 & 4.14 & $-1003.4$ & 0.247 \\
\#3 & 0.462 & 0.114 & 340.3 & 0.315 & 4.05 & $-1003.4$ & 0.253 \\
\#4 & 0.459 & 0.110 & 342.4 & 0.315 & 4.17 & $-1003.4$ & 0.252 \\
\bottomrule
\end{tabular}
\caption{Posterior modes and weights identified by multi-start optimisation.}
\label{tab:SIR_posterior_modes}
\end{table}

\begin{figure}[H]
  \centering
  \includegraphics[width=1.0\linewidth]{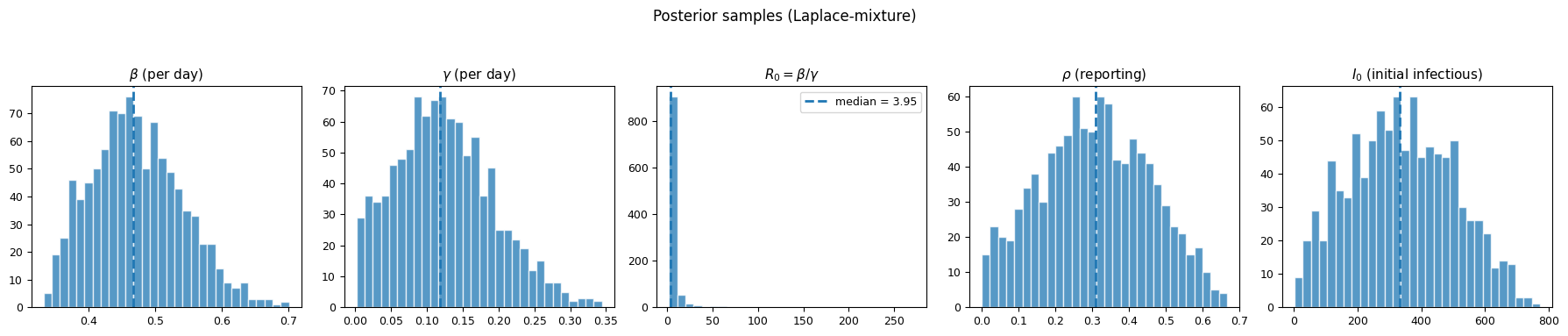}
  \caption{Posterior histograms for $\beta$, $\gamma$, $R_0=\beta/\gamma$, $\rho$, and $I_0$.}
  \label{fig:SIR_posteriors}
\end{figure}

\begin{table}[H]
\centering
\small
\begin{tabular}{lccccc}
\toprule
 & $\beta$ (day$^{-1}$) & $\gamma$ (day$^{-1}$) & $\mathcal R_0=\beta/\gamma$ & $\rho$ & $I_0$ \\
\midrule
Mean   & 0.473 & 0.125 & 7.121 & 0.308 & 338.758 \\
Median & 0.467 & 0.118 & 3.946 & 0.308 & 332.624 \\
\bottomrule
\end{tabular}
\caption{Posterior means and medians from LMA.}
\label{tab:SIR_posterior_summary}
\end{table}

\begin{figure}[H]
  \centering
  % Left: R0
  \begin{subfigure}{0.42\linewidth}
    \centering
    \includegraphics[width=\linewidth]{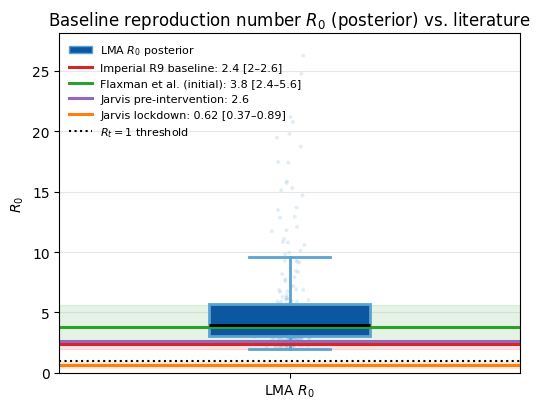}
    \caption{Baseline reproduction number $R_0$ posterior (ours, blue) with literature references (other colors).}
    \label{fig:SIR_r0}
  \end{subfigure}
  \hfill
  % Right: posterior predictives
  \begin{subfigure}{0.56\linewidth}
    \centering
    \includegraphics[width=\linewidth]{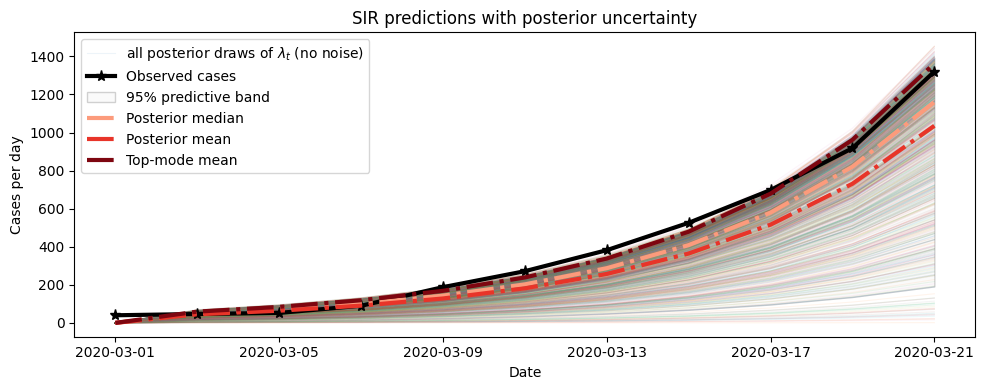}
    \caption{Predictive $\lambda_t$ trajectories with 95\% credible band (all posterior samples), plus posterior median/mean and top-mode mean.}
    \label{fig:SIR_predictives}
  \end{subfigure}
  \caption{Uncertainty summaries: (a) constant-$R_0$ posterior \textit{vs} published benchmarks; (b) posterior predictive mean trajectories and credible band.}
  \label{fig:SIR_r0_vs_predictive}
\end{figure}

\begin{table}[H]
\centering
\small
\begin{tabular}{lcccc}
\toprule
Optimisation & Laplace-mixture & Sampling & Predictive sims & Total \\
\midrule
2.737 & 0.133 & 0.048 & 1.393 & 4.311 \\
\bottomrule
\end{tabular}
\caption{Computational times (seconds) for LMA-SIR inference experiment.}
\label{tab:SIR_timing}
\end{table}

We compare with contemporaneous public estimates of the reproduction number $R_0$. \textit{Imperial College Report 9} (16 Mar 2020) \cite{Ferguson2020report9} adopted a baseline $R_0=2.4$ with sensitivity $2.0-2.6$, calibrated to the early exponential growth in Wuhan under an assumed mean generation time of $\sim 6.5$ days derived from incubation and infectiousness profiles (symptomatic infectious from 12h before symptom onset; asymptomatic from 4.6 days post-infection). \textit{Flaxman et al.} (Jun 2020) \cite{flaxman2020estimating} inferred time-varying transmission from mortality data across 11 European countries using a semi-mechanistic model; with their chosen generation-interval distribution and the observed initial growth of deaths, they reported an initial $R_t$ averaged across countries of $3.8\ (2.4 - 5.6)$. \textit{Jarvis et al.} (May 2020) \cite{jarvis2020quantifying} used UK contact-survey data to translate contact reductions into transmission: they placed a prior $R_0 \sim \mathcal{N}(2.6, 0.54^2)$ pre-intervention and, given a 74\% fall in mean daily contacts (10.8$\rightarrow$2.8), estimated a lockdown reproduction number of $0.62$ with 95\% confidence interval (CI) 0.37-0.89 for all contact types (and $0.37$ with CI 0.22-0.53 for physical contacts only). These estimates differ in data streams (growth in Wuhan cases, country-level deaths, UK contact patterns), generation-interval assumptions, and modelling frameworks; nevertheless, they provide a consistent early-pandemic scale ($R_0\approx 2.4-4$) against which our constant $R_0$ SIR/LMA results can be contextualised.

Notably, the highest-weight mode (the 'Top-mode mean' predictives in Fig.\ref{fig:SIR_r0_vs_predictive}) tracks the early growth of reported cases (the 'Observed cases') well, indicating that a constant-$\beta$ SIR with a simple \textit{Poisson} observation model can reproduce the short-horizon trend after smoothing. However, uncertainty grows rapidly forward in time: the mapping $\theta \mapsto \lambda_t(\theta)$ is highly non-linear, and with limited information in the early exponential regime, $\beta$ and $\gamma$ remain weakly identified; this manifests as a heavy upper tail for $R_0$ and wide predictive bands. The Laplace-mixture approximation is particularly appropriate here: the posterior exhibits a flat ridge with several nearby modes of comparable evidence (as observed from Table.\ref{tab:SIR_posterior_modes}); the mixture captures both local curvature around each mode and between-mode variability, which a single, unimodal Laplace approximation would miss.

As a toy example to demonstrate the use of LMA for real-world problems, there are limitations and refinement recipes for the data and model used, including:
(1) The constant-$\beta$ assumption ignores interventions around mid-March 2020; a piecewise-constant or spline-based $R_t$ would be more realistic. (2) Treating $\rho$ as constant conflates under-reporting with $I_0$; allowing time-varying $\rho_t$ or anchoring to serology would reduce confounding. (3) The SIR infectiousness profile is exponential; SEIR or renewal models better match COVID-19 serial intervals and can stabilise $\gamma$. (4) Placing a tighter prior on the infectious period $1/\gamma$ (e.g. 5-7 days) or directly on $R_0$ can curb implausible upper tails. (5) Using a longer window (with awareness of changing $\beta$ and $\rho$) or additional data types (e.g. hospitalisations) would improve identifiability.

\subsubsection{Bayesian optimal experimental design for logistic dose-response via an \textit{LMA} prior} \label{subsec:BOED}

In a dose-response experiment we test a single drug at a small set of concentrations (``doses'') and observe how a biological system reacts (e.g. whether cells are killed). Because running wells is costly, we face a budgeting question: \textit{at which doses should we spend our limited replicates so that the data most improve our understanding of the dose-response curve and quantities such as EC50 \footnote{The half-maximal effective concentration \textit{$EC_{50}$} is the dose $d_{50}$ at which the effect is $50\%$. In the logistic model $\pi(d;\theta)=\sigma \big(\alpha+\beta\log d\big)$ with $\sigma(t)=1/(1+e^{-t})$, $d_{50}$ solves $\pi(d_{50};\theta)=0.5$, hence $d_{50}=\exp(-\alpha/\beta)$. In the centered predictor used later, $\pi(d;\theta)=\sigma \big(\alpha+\beta(\log d-x_{\text{offset}})\big)$, giving $d_{50}=\exp \big(x_{\text{offset}}-\alpha/\beta\big)$. \textit{$EC_{50}$ vs $IC_{50}$ vs $ED_{50}$:} $EC_{50}$ is the concentration giving 50\% effect; $IC_{50}$ is the concentration giving 50\% inhibition (numerically identical to $EC_{50}$ if ``effect'' is inhibition); $ED_{50}$ is the dose giving a 50\% response rate in a population (often in vivo). \textit{Uncertainty:} with posterior $\theta=(\alpha,\beta)\sim \mathcal{N}(\hat\theta,\Sigma)$, the delta method for $g(\alpha,\beta)=\exp \big(x_{\text{offset}}-\alpha/\beta\big)$ uses $\nabla g=\big(-g/\beta, (\alpha/\beta^{2})g\big)$ and
$\operatorname{Var} \big[g(\theta)\big]\approx \nabla g^{\top}\Sigma \nabla g$; a 95\% CI is $g(\hat\theta)\pm 1.96 \sqrt{\operatorname{Var} \big[g(\theta)\big]}$.}}? 
Optimal experimental design (OED) \cite{Fedorov1972OED} and Bayesian optimal experimental design (BOED) \cite{Pukelsheim2006OED} decides where to spend the limited measurements so we can learn the most that matters. OED helps to decide which data to collect in the first place, which can then used to e.g. fitting a mechanistic model (e.g. a logistic \textit{dose-response model}) \footnote{OED/BOED makes data collection strategic, so that the model to be fitted afterwards is as informative and decision-ready as possible given the constraints.}. BOED \cite{Pukelsheim2006OED} answers the question by choosing an allocation across the candidate doses that maximizes the \textit{expected reduction in uncertainty} \textit{before} any new outcomes are seen. Intuitively, BOED places more replicates where an extra measurement is predicted to be most informative about the curve’s location and steepness, while still reserving a few measurements elsewhere to avoid blind spots. It uses historical information to focus new measurements where they will actually reduce uncertainty, rather than re-measuring what’s already known.

We model the probability of response with a logistic ($S$-shaped) curve as a function of log-dose. For such models, the BOED objective, i.e. mutual information or expected information gain (EIG), has a convenient form: it can be computed from \textit{prior} (or current posterior) averages of the Bernoulli response probabilities at each candidate dose, without simulating outcomes inside the objective. A practical challenge is that our prior uncertainty over the curve parameters, learned from historical cell lines, is often \textit{not} well captured by a single Gaussian (it may be skewed or multi-modal). Relying on a single-mode Laplace approximation can then distort those prior averages and bias the design.

To preserve realistic uncertainty while keeping computations light, we use a Laplace Mixture Approximation (LMA; Section.\ref{subsec:LMA}): instead of one local Gaussian around a mode, we build a small, weighted mixture of local Gaussians, each fitted by a Laplace (mode-curvature) approximation in a different region of the parameter space. Sampling from this LMA prior lets us evaluate EIG faithfully and rank doses by their \textit{per-replicate information gain}. The experiment then proceeds mechanically: fix a dose grid and a replicate budget, compute per-dose gains under the LMA prior, allocate replicates to maximize EIG (with a minimal spread to cover the grid), run the wells, and finally update the model with the collected data. In short, BOED uses prior knowledge to guide \textit{where} to measure, so the limited experiment yields the most learning about the \textit{dose-response relationship}.

\paragraph{BOED principles}
We consider a logistic dose-response setting. Let $d\in\mathcal{D}\subset\mathbb{R}_+$ denote dose and $y\in\{0,1\}$ a binary response. With parameters $\theta=(\alpha,\beta)^\top$, the observation model is
\[
y_i \mid d_i,\theta  \sim  \mathrm{Bernoulli} \left(\pi(d_i;\theta)\right), \qquad
\pi(d;\theta)  =  \sigma \big(\alpha + \beta \log d\big), \quad
\sigma(t)=\tfrac{1}{1+e^{-t}}
\]
To choose where to spend a fixed experimental budget, we place a candidate grid $\{d_k\}_{k=1}^K\subset\mathcal{D}$ and represent a design $\xi$ by allocation weights $\gamma=(\gamma_1,\dots,\gamma_K)\in\Delta^{K-1}$, so that $n_k = N \gamma_k$ replicates are assigned to $d_k$, with $\Delta^{K-1}=\{\gamma\in\mathbb{R}_+^K:\sum_{k=1}^K \gamma_k=1\}$. Uncertainty about $\theta$ is encoded by a prior density $p(\theta)$ on $\mathbb{R}^2$.

The design quality is measured by the expected information gain
\footnote{By definition of conditional mutual information,
\[
I(\theta;y\mid\xi)
=\iint p(\theta,y\mid\xi) \log\frac{p(\theta,y\mid\xi)}{p(\theta\mid\xi) p(y\mid\xi)}  \mathrm{d}\theta  \mathrm{d}y
\]
As the design $\xi$ is chosen \textit{before} observing $\theta$ and is independent of it, $p(\theta\mid\xi)=p(\theta)$ and
$p(\theta,y\mid\xi)=p(\theta) p(y\mid\theta,\xi)$. Therefore:
\[
I(\theta;y\mid\xi)=\iint p(\theta) p(y\mid\theta,\xi) \Big[\log p(y\mid\theta,\xi)-\log p(y\mid\xi)\Big] \mathrm{d}\theta \mathrm{d}y
\]
Applying Fubini/Tonelli to exchange integrals:
\[
I(\theta;y\mid\xi)
=\int p(\theta) \left[\int p(y\mid\theta,\xi)\log p(y\mid\theta,\xi)  \mathrm{d}y\right]\mathrm{d}\theta
-\int p(y\mid\xi)\log p(y\mid\xi)  \mathrm{d}y
\]
Recognizing (negative) entropies, we arrive at \cite{Lindley1956measure}:
\[
I(\theta;y\mid\xi)
= - \mathbb{E}_{\theta\sim p(\theta)} \big[H[y\mid\theta,\xi]\big] + H[y\mid\xi]
\]
which is the stated identity. (For discrete $y$, integrals are replaced by sums.)}
(mutual information, MI) between parameters and future outcomes under the design $\xi$:
\begin{equation} \label{eq:BOED_EIG}
    \mathrm{EIG}(\xi)  =  I(\theta; y \mid \xi)  =  H \big[y\mid \xi\big]  -  \mathbb{E}_{\theta\sim p(\theta)} \big[  H \big[y \mid \theta,\xi\big]  \big]
\end{equation}
The first term \(H \big[y\mid \xi\big]\) measures the uncertainty \textit{before} knowing \(\theta\); the second term \( \mathbb{E}_{\theta\sim p(\theta)} \big[  H \big[y \mid \theta,\xi\big]  \big]\) measures the expected uncertainty \textit{after} we learn \(\theta\) (so we use that model’s prediction).

The EIG objective thus rewards designs whose \textit{marginal} predictions are uncertain (first term large) while penalizing designs that would still be ambiguous even if $\theta$ were known (second term large). That is, EIG is high exactly where \textit{plausible parameter values disagree about what will happen}, so one more measurement is especially discriminative. For Bernoulli outcomes, this appears per dose as
\(
\Delta_k = h(\bar\pi_k) - \mathbb{E}[h(P_{s,k})]
\) (see later Eq.\ref{eq:BOED_EIG3}): it is \textit{large} when the prior-averaged response $\bar\pi_k$ is near $0.5$ \textit{and} the induced probabilities $P_{s,k}$ vary \footnote{
To see why the “variation across $\theta$” helps, note that for fixed $\bar\pi_k$ the Bernoulli entropy $h$ is strictly concave with $h''(p)=-1/[p(1-p)]<0$. A second-order expansion gives
\[
\mathbb{E}[h(P_{s,k})] \approx h(\bar\pi_k) + \tfrac12 h''(\bar\pi_k) \mathrm{Var}(P_{s,k})
\quad\Rightarrow\quad
\Delta_k \approx \frac{\mathrm{Var}(P_{s,k})}{2 \bar\pi_k(1-\bar\pi_k)}
\]
Thus, greater dispersion of $P_{s,k}$ (i.e. \textit{confident disagreement} across plausible $\theta$, pushing probabilities toward $0$ or $1$) reduces the conditional-entropy term and increases the information gain.
} across $\theta$, and it is \textit{small} when predictions saturate near $0$ or $1$ or when all plausible $\theta$ make similar predictions.
Maximizing EIG therefore steers budget toward doses that most reduce posterior uncertainty (e.g. in $EC_{50}$), rather than spreading effort over low-information regions.

As replicates are conditionally independent given $\theta$, the conditional entropy decomposes as
\begin{equation} \label{eq:BOED_EIG1}
    H \big[y \mid \theta,\xi\big]  =  \sum_{k=1}^K n_k  h \big(\pi(d_k;\theta)\big),
    \qquad
    h(p)  =  -p\log p - (1-p)\log(1-p)
\end{equation}
By contrast, the marginal entropy $H[y\mid\xi]$ does \textit{not} in general decompose additively after marginalising $\theta$ (replicates become dependent). A convenient \textit{additive surrogate} is obtained by replacing the joint entropy with the sum of marginal entropies (an upper bound):
\begin{equation} \label{eq:BOED_EIG2sur}
H_{\mathrm{sur}} \big[y\mid \xi\big]  =  \sum_{k=1}^K n_k  h \big(\bar{\pi}_k\big),
\qquad
\bar{\pi}_k  =  \mathbb{E}_{\theta\sim p(\theta)} \left[\pi(d_k;\theta)\right]
\end{equation}

Using Eq.\ref{eq:BOED_EIG}, Eq.\ref{eq:BOED_EIG1}, and the surrogate Eq.\ref{eq:BOED_EIG2sur}, we optimise the following Monte Carlo \textit{surrogate} objective under our LMA prior:
\begin{equation} \label{eq:BOED_EIG3}
    \widehat{\mathrm{EIG}}_{\mathrm{sur}}(\gamma)
     = 
    \sum_{k=1}^K n_k  h \big(\widehat{\bar{\pi}}_k\big)
     - 
    \sum_{k=1}^K n_k \left(\frac{1}{S}\sum_{s=1}^S h(P_{s,k})\right),
    \qquad n_k = N \gamma_k
\end{equation}
where $\theta^{(s)}\sim q_{\mathrm{LMA}}$, $P_{s,k}=\sigma(\alpha^{(s)}+\beta^{(s)}\log d_k)$, and $\widehat{\bar{\pi}}_k=\tfrac{1}{S}\sum_{s=1}^S P_{s,k}$. This estimator avoids inner sampling over outcomes and depends only on samples from the prior surrogate. Note that, later in our implementation, we used the centered predictor $P_{s,k}=\sigma \big(\alpha^{(s)}+\beta^{(s)} (\log d_k - x_{\text{offset}})\big)$, where $x_{\text{offset}}=\frac{1}{K}\sum_{k=1}^K \log(d_k+\varepsilon)$ is a pure reparameterization to improve conditioning and reduces posterior correlation between $(\alpha,\beta)$.

\paragraph{Laplace mixture approximation (LMA) of the prior}
As the MI objective involves \textit{prior averages of nonlinear functionals} of the parameters, i.e. $\theta\mapsto \pi(d_k;\theta)$ and $h \big(\pi(d_k;\theta)\big)$, a single-Gaussian (Laplace) surrogate can distort $\bar{\pi}_k$ and hence $\mathrm{EIG}(\xi)$ when $p(\theta)$ is multi-modal, skewed, or exhibits ridge-like dependence (e.g. weak identifiability between $\alpha$ and $\beta$). To avoid these biases, we replace the single-mode Laplace approximation with a \textit{mixture of local Laplace components} (Section~\ref{subsec:LMA}), which preserves multiple modes and local curvature while remaining cheap to sample.

We approximate the prior $p(\theta)$ by a $J$-component Gaussian mixture
\[
q_{\mathrm{LMA}}(\theta)
=
\sum_{j=1}^J  w_j  \mathcal{N} \big(\theta;\mu_j,\Sigma_j\big),
\qquad
\sum_{j=1}^J  w_j = 1,    w_j \ge 0
\]
with component means at local modes and covariances from local curvature:
\[
\mu_j \in \arg\max_{\theta} \log p(\theta),
\qquad
\Sigma_j  =  \Big[-\nabla^2_{\theta}\log p(\theta)\Big]^{-1}_{\theta=\mu_j}
\]
To calibrate (or refine) the mixture weights, we solve a \textit{weights-only} refitting problem - what we call in Section.\ref{sec:methodology} the \textbf{WGMA}: starting from fixed components $(\mu_j,\Sigma_j)$ (e.g. from random or LMA initialisation), optimise only the weights $\mathbf{w}$ on the simplex by minimising the exclusive-KL divergence from $p$ to $q_{\mathrm{LMA}}$ using the \textit{unnormalised} target $\bar{p}$ (e.g. Algo.\ref{algo:WGMA-sampling-pgd}):
\[
w^\star 
\in 
\arg\min_{w\in\Delta^{J-1}}
\mathrm{KL} \left(\bar{p} \Big\|  \sum_{j=1}^J w_j \mathcal{N}(\cdot;\mu_j,\Sigma_j)\right)
\equiv
\arg\min_{w\in\Delta^{J-1}}
\Big\{
- \int \bar{p}(\theta) \log\Big[\sum_{j=1}^J w_j \mathcal{N}(\theta;\mu_j,\Sigma_j)\Big]\mathrm{d}\theta
\Big\}
\]
This yields a light-weight, multi-modal surrogate $q_{\mathrm{LMA}}$ that preserves the dominant modes and local curvature of the prior, capturing multi-modality and skew in $p(\theta)$ and yielding more reliable prior expectations for the BOED objective.

In practice, the prior $p(\theta)$, if not human-specified, is not analytically available, therefore we are unable to use multi-start MAP to find the $J$ modes (as used in our SIR experiment in Section.\ref{subsec:SIR_model_inference}); instead, we have samples 
$\hat\theta_i=(\hat\alpha_i,\hat\beta_i)$ by fitting multiple logistic regressors to the data. We therefore introduce a clustering-based LMA fitting and GMA-refining procedure, as detailed in Appendix.\ref{sec:constructing_LMA_from_empirical_samples}, to find the Laplace mixture.

\paragraph{EIG estimation with LMA}
Given $q_{\mathrm{LMA}}$, we estimate the expected information gain by Monte Carlo without inner sampling over outcomes. Draw $\theta^{(s)}\sim q_{\mathrm{LMA}}$ for $s=1,\dots,S$ and, for each dose grid point $d_k$, compute
\[
P_{s,k} = \sigma \big(\alpha^{(s)} + \beta^{(s)} (\log d_k - x_{\text{offset}})\big) \text{ (centered predictor), }
\quad
\widehat{\bar{\pi}}_k  =  \frac{1}{S}\sum_{s=1}^S P_{s,k}
\]

With $h(p)=-p\log p-(1-p)\log(1-p)$, the resulting estimator is precisely Eq.\ref{eq:BOED_EIG3}.

\paragraph{Design optimization on a fixed grid}
We optimize allocations $\gamma\in\Delta^{K-1}$ (with $n_k=N\gamma_k$) to maximize $\widehat{\mathrm{EIG}}_{\mathrm{sur}}(\gamma)$. A simple and effective strategy is a greedy allocator that, at each step, adds one replicate to the dose with the largest per-replicate marginal gain:
\[
\Delta_k  =  h \big(\widehat{\bar{\pi}}_k\big)  -  \frac{1}{S}\sum_{s=1}^S h(P_{s,k}),
\qquad
k^\star \in \arg\max_k \Delta_k
\]
then updates $n_{k^\star} \leftarrow n_{k^\star}+1$ (equivalently $\gamma_{k^\star} \leftarrow \gamma_{k^\star}+\tfrac{1}{N}$) and repeats until $\sum_k n_k=N$. The complete LMA-assisted, EIG-oriented BOED procedure is summarized in Algo.\ref{algo:LMA-EIG}.

\begin{algorithm}[H]
\caption{LMA-EIG greedy BOED design on a fixed grid (one-shot)}
\label{algo:LMA-EIG}
\begin{algorithmic}[1]
\STATE \textbf{Input:} grid $\{d_k\}_{k=1}^K$, budget $N$, prior $p(\theta)$, components $J$, samples $S$.
\STATE \textbf{LMA build:} from historical estimates $\{\hat\theta_i\}$, run $k$-means clustering to get component means $\{\mu_j\}$, set local covariances $\{\Sigma_j\}$ with a small ridge, and refine the mixture weights $\mathbf{w}$ via the reverse-KL objective (WGMA).
\STATE \textbf{Sampling:} draw $\theta^{(s)}\sim q_{\mathrm{LMA}}$ for $s=1,\dots,S$ and compute $P_{s,k}=\sigma(\alpha^{(s)}+\beta^{(s)}\log d_k)$.
\STATE Initialize $n_k\leftarrow 0$ for all $k$.
\FOR{$t=1$ \TO $N$}
  \STATE For each $k$, compute $\Delta_k = h \big(\tfrac{1}{S}\sum_s P_{s,k}\big) - \tfrac{1}{S}\sum_s h(P_{s,k})$.
  \STATE Pick $k^\star \in \arg\max_k \Delta_k$ and set $n_{k^\star}\leftarrow n_{k^\star}+1$.
\ENDFOR
\STATE \textbf{Output:} design $\xi=\{(d_k,n_k)\}_{k=1}^K$ and $\gamma_k=n_k/N$.
\end{algorithmic}
\end{algorithm}

This pipeline builds a multi-modal prior surrogate via LMA, computes an inner-sampling-free surrogate MI estimator, and greedily allocates replicates, yielding data-efficient, information-rich dose schedules while faithfully propagating prior uncertainty beyond single-Gaussian approximations. Note that, the Bayesian prior is fixed through out this one-shot experiment; after the measurements are collected, the prior can be updated to a posterior for downstream inference or for a subsequent design round. Two practical options for updating the Bayesian prior are: (i) a fast single-Gaussian Laplace update that first moment-matches $q_{\mathrm{LMA}}$ to a Gaussian $\mathcal N(\mu_{\text{prior}},\Sigma_{\text{prior}})$ and then refines $(\mu,\Sigma)$ around the posterior mode via Newton/Hessian updates; and (ii) a mixture-preserving update that applies a Laplace step to each component $(\mu_j,\Sigma_j)$ and reweights the components by their (Laplace-approximated) marginal likelihood. If an adaptive design is desired, one can iterate this update after each observed replicate, recomputing EIG under the current posterior and allocating the next replicate accordingly.

\subsection*{\textit{Experiment}}

\paragraph{Setup and data.}
We simulate a GDSC‐style single-drug screen on a seven-point geometric dose grid
$\{10^{-3}, 3{\times}10^{-3}, 10^{-2}, 3{\times}10^{-2}, 10^{-1}, 3{\times}10^{-1}, 1\}$.
For each (cell line, dose) we generate a viability in $[0,1]$ and then binarize to a quantal response
($y=1$ if viability $\le 0.5$, else $y=0$; $y=1$ means the well “responded” at that dose, i.e. viability fell past the 50\% threshold).
The observation model is logistic in \textit{centered} log-dose, $\pi(d;\theta)=\sigma \big(\alpha+\beta(\log d - x_{\text{offset}})\big)$, where $x_{\text{offset}}$ is the drug’s mean log-dose, which stabilizes the $(\alpha,\beta)$ numerics.
An example of the synthetic data is shown in Table.\ref{tab:BOED}.
We split cell lines \footnote{In the dataset, a cell line corresponds to many rows; one row per dose.} into a historical panel $\mathcal{H}$ and a held-out target $c^\star$.
There are $60$ total cell lines; thus $|\mathcal{H}|=59$ historical lines and one held-out target ($c^\star=\texttt{CL\_0001}$).

\begin{table}[H]
  \centering
  \small
  \caption{Examples of the synthetic data (drug = \textit{DrugA}, target = \texttt{CL\_0001}). 
  \textit{Viability} is the normalized live-cell fraction in $[0,1]$; the binarized response is $y=1$ if viability $\le 0.5$ (i.e. $\ge 50\%$ inhibition), else $0$. 
  \textbf{Target = \texttt{CL\_0001}} denotes the held-out cell line $c^\star$ for which we design and evaluate; the remaining lines belong to the historical panel $\mathcal{H}$ used to build the empirical prior.}
  \vspace{0.25em}
  \begin{tabular}{lcccccc}
    \toprule
    \multirow{2}{*}{\textbf{dose}} & \multicolumn{2}{c}{\textbf{CL\_0001}} & \multicolumn{2}{c}{\textbf{CL\_0002}} & \multicolumn{2}{c}{\textbf{CL\_0003}} \\
    \cmidrule(lr){2-3}\cmidrule(lr){4-5}\cmidrule(lr){6-7}
    & \textbf{viability} & \textbf{$y$} & \textbf{viability} & \textbf{$y$} & \textbf{viability} & \textbf{$y$} \\
    \midrule
    0.001 & 1.000 & 0 & 1.000 & 0 & 1.000 & 0 \\
    0.003 & 1.000 & 0 & 1.000 & 0 & 1.000 & 0 \\
    0.010 & 1.000 & 0 & 1.000 & 0 & 1.000 & 0 \\
    0.030 & 1.000 & 0 & 1.000 & 0 & 1.000 & 0 \\
    0.100 & 1.000 & 0 & 1.000 & 0 & 1.000 & 0 \\
    0.300 & 1.000 & 0 & 1.000 & 0 & 1.000 & 0 \\
    1.000 & 0.667 & 0 & 0.333 & 1 & 0.000 & 1 \\
    \bottomrule
  \end{tabular}
  \label{tab:BOED}
\end{table}

The exploratory plots in Fig. \ref{fig:eda_BOED} show where signal lives on the (log-scaled) dose axis.
Representative dose-viability curves are nearly flat near $1.0$ from $10^{-3}$ through $3{\times}10^{-1}$, with most lines dropping only at the top dose ($d=1$).
The cross-line mean viability stays close to $1$ up to $3{\times}10^{-1}$ and then falls to $\approx 0.55$ at $d=1$, with small standard errors at low doses that widen slightly at the top dose.
After binarization, the empirical response rate $P(y=1)$ is essentially zero across lower doses, rises modestly by $d=0.3$, and reaches $\approx 0.38$ at $d=1$, confirming that most information concentrates at the highest doses.

\begin{figure}[H]
  \centering
  \begin{subfigure}[H]{0.32\linewidth}
    \centering
    \includegraphics[width=\linewidth]{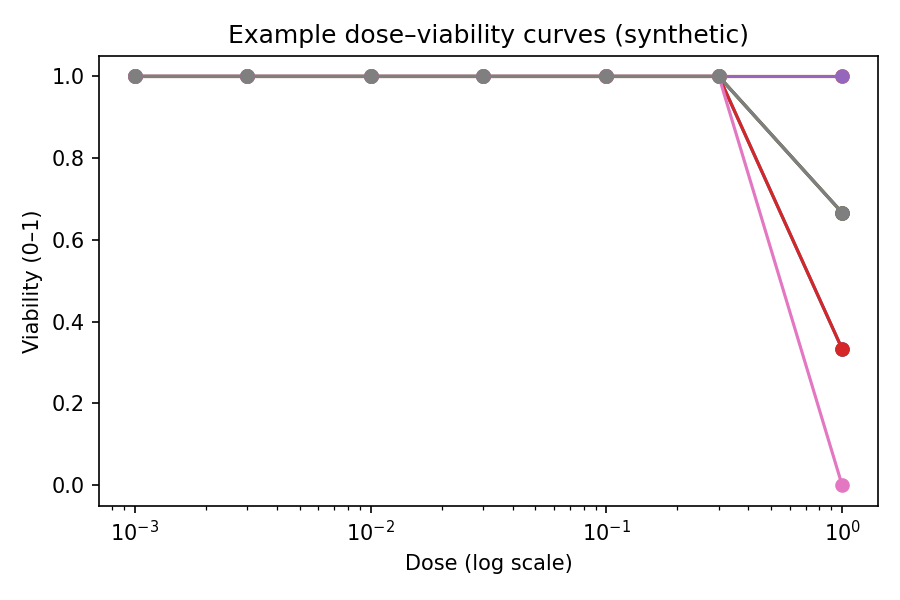}
    \caption{Example dose-viability curves (synthetic).}
    \label{fig:eda_curves}
  \end{subfigure}\hfill
  \begin{subfigure}[H]{0.32\linewidth}
    \centering
    \includegraphics[width=\linewidth]{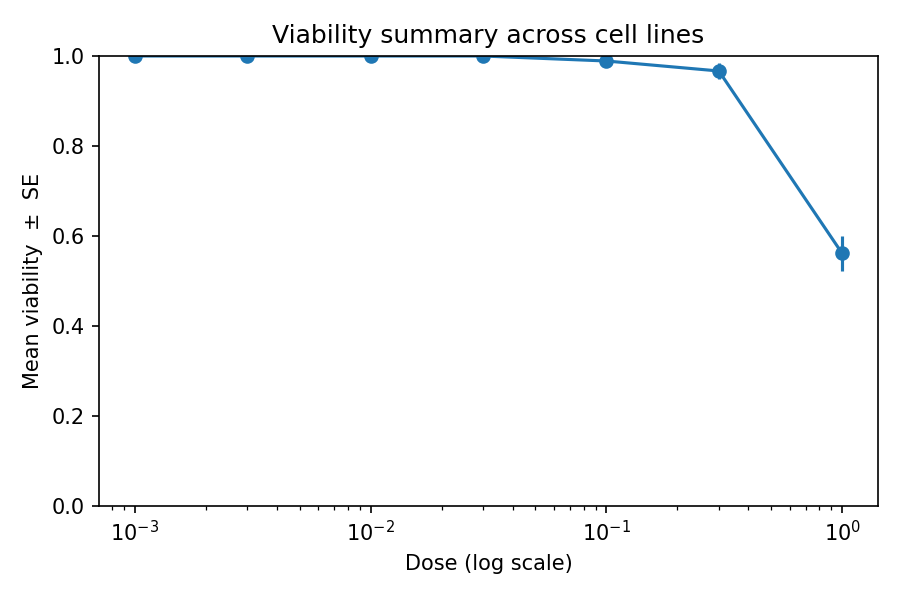}
    \caption{Mean viability $\pm$ SE across cell lines.}
    \label{fig:eda_mean}
  \end{subfigure}\hfill
  \begin{subfigure}[H]{0.32\linewidth}
    \centering
    \includegraphics[width=\linewidth]{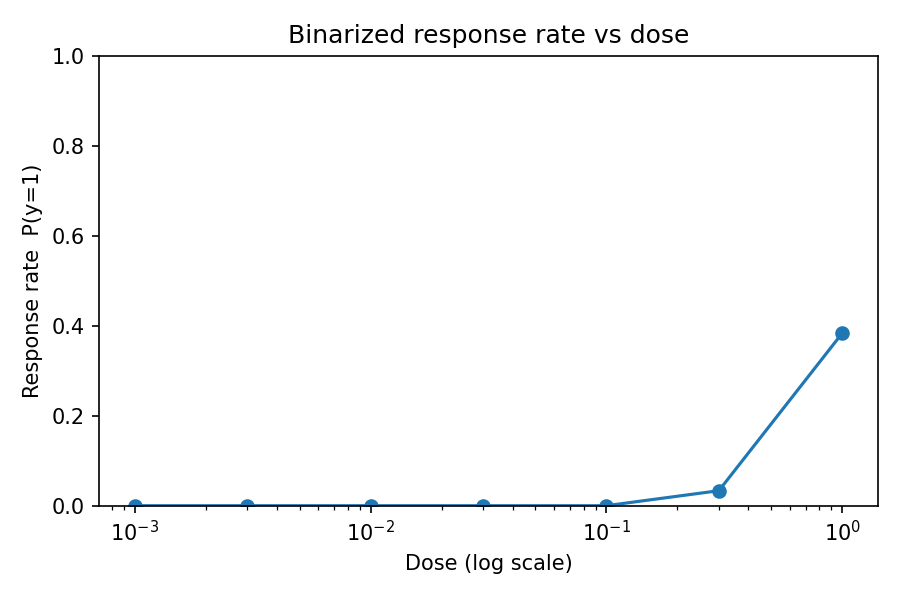}
    \caption{Binarized response rate $P(y=1)$ \textit{vs} dose.}
    \label{fig:eda_rate}
  \end{subfigure}
  \caption{EDA of the synthetic data.
  \textit{Viability} is the normalized live-cell fraction in a well ($1=$no inhibition, $0=$complete kill).
  We define a binary per-well response by thresholding viability: $y=1$ if viability $\le 0.5$ (i.e.\ $\ge 50\%$ inhibition), else $y=0$.
  The \textit{response rate} at dose $d$ is the empirical frequency
  $\widehat{P}(y=1\mid d)=\frac{\#\{\text{wells at dose } d \text{ with } y=1\}}{\#\{\text{wells at dose } d\}}$ shown in panel~\ref{fig:eda_rate}.
  In contrast, the \textit{model-based} response probability used in BOED averages the logistic probability over the LMA prior/posterior uncertainty,
  $\bar{\pi}(d)=\mathbb{E}_{\theta\sim q_{\mathrm{LMA}}} \big[\sigma \big(\alpha+\beta(\log d - x_{\text{offset}})\big)\big]$,
  so it can differ from the observed frequency due to pooling across cell lines and uncertainty in $\theta=(\alpha,\beta)^\top$.}
  \label{fig:eda_BOED}
\end{figure}

\paragraph{Empirical prior, LMA surrogate, and posterior update.}
Because the target density $p(\theta)$ is unavailable, we estimate it from empirical samples. Using only the historical panel $\mathcal{H}$, we fit a separate (penalized) logistic model for each cell line $i$ to obtain one estimate $\hat\theta_i=(\hat\alpha_i,\hat\beta_i)$, then the cloud $\{\hat\theta_i\}_{i\in\mathcal{H}}$ is treated as \textit{empirical samples} from the prior over $\theta$. From these points we build a Laplace-mixture surrogate $q_{\mathrm{LMA}}(\theta)=\sum_{j=1}^J w_j \mathcal{N} \big(\theta;\mu_j,\Sigma_j\big)$, where component means $\mu_j$ are obtained by $k$-means clustering on $\{\hat\theta_i\}$, covariances $\Sigma_j$ are local empirical covariances with a small ridge floor, and the weights $\boldsymbol{w}$ are refined by a \textit{weights-only} reverse-KL fit (WGMA) on the simplex. The detailed procedure of constructing and refining this LMA from empirical samples is presented in Appendix.\ref{sec:constructing_LMA_from_empirical_samples}.

\noindent\textit{Using the LMA in BOED.}
To evaluate the EIG objective on a fixed dose grid, we draw $\theta^{(s)}\sim q_{\mathrm{LMA}}$ and compute, for each grid point $d_k$ with centered predictor:
\[
P_{s,k}=\sigma \big(\alpha^{(s)}+\beta^{(s)} (\log d_k - x_{\text{offset}})\big),\qquad
\bar\pi_k=\frac{1}{S}\sum_{s=1}^S P_{s,k},\qquad
\Delta_k=h(\bar\pi_k)-\frac{1}{S}\sum_{s=1}^S h(P_{s,k})
\]
where $h$ is the Bernoulli entropy. These $\Delta_k$ terms drive the greedy allocation.

\noindent\textit{Moment matching for the posterior Laplace step.}
For the single target line $c^\star$, once outcomes are observed (or simulated) under a candidate design, we compress the mixture prior to a single Gaussian by moment matching:
\[
\mu_{\text{prior}}=\sum_{j=1}^J w_j \mu_j,\qquad
\Sigma_{\text{prior}}=\sum_{j=1}^J w_j \left(\Sigma_j+(\mu_j-\mu_{\text{prior}})(\mu_j-\mu_{\text{prior}})^\top\right)
\]
and perform one standard Laplace posterior update with this Gaussian prior. From that posterior we report summaries such as $\mathrm{EC}_{50}$ and its delta-method confidence interval.

\paragraph{Design objective and optimization.}
With total budget $N$ and fixed grid $\{d_k\}_{k=1}^K$, a design is
$\gamma\in\Delta^{K-1}$ with $n_k=N\gamma_k$ replicates at $d_k$.
We estimate per-dose mutual-information gains
$\Delta_k=h(\bar\pi_k)-\tfrac{1}{S}\sum_{s=1}^S h(P_{s,k})$ via LMA sampling, where
$\theta^{(s)} \sim q_{\mathrm{LMA}}$, $P_{s,k}=\sigma(\alpha^{(s)}+\beta^{(s)} (\log d_k - x_{\text{offset}}))$,
$\bar\pi_k=\tfrac{1}{S}\sum_s P_{s,k}$, and $h$ is Bernoulli entropy.
A greedy allocator adds one replicate at a time to $\arg\max_k$ of this (optionally
entropy-regularized) score; we enforce a one-per-dose floor and include a small entropy bonus
$\tau N H(\gamma)$ to softly spread mass.
Baselines are (i) \textit{Uniform} ($\gamma_k=1/K$) and
% (ii) \textit{local Fisher $D$-opt} \cite{Wynn1970} (greedy log-det on the grid evaluated at $\mu_{\text{prior}}$).
(ii) a \textit{local} Fisher $D$-optimal design \cite{Wynn1970,Fedorov1972OED,Pukelsheim2006OED}, i.e. maximize $\det M(\theta_0,\xi)$ for the Fisher information $M$ at a nominal value $\theta_0$ (we take $\theta_0=\mu_{\text{prior}}$), implemented via sequential augmentation in the spirit of \cite{Wynn1970} or a Fedorov-exchange step \cite{Fedorov1994}.
For evaluation on $c^\star$, we simulate prospective outcomes using its empirical fit $\hat\theta_{c^\star}$,
reusing common random numbers across designs, update to a Gaussian posterior via Laplace
(with the moment-matched prior), and report: estimated EIG, posterior entropy for $(\alpha,\beta)$, and
$\mathrm{EC}_{50}$ with a delta-method 95\% CI.

\paragraph{Settings.}
We use a seven-point dose grid (\(K=7\)) and a total budget of \(N=21\) wells; we estimate EIG with \(S=2000\) draws from the LMA prior and build that prior with \(J=5\) Gaussian components. Viability is binarized at 0.5 (i.e. \(\ge 50\%\) inhibition counts as a response), local mixture covariances include a small ridge of \(10^{-2}\) for stability, and random seeds are fixed for reproducibility.

\paragraph{Results.}

The empirical prior $p(\theta)$ for the logistic parameters $\theta=(\alpha,\beta)^\top$ was approximated with a $J=5$-component Laplace mixture built from 59 historical cell lines. Table.\ref{tab:BOED_lma_components} lists the component weights, means, and (diagonal) covariances. Two tight components near $(\alpha,\beta)\approx(-4.95, 0)$ together carry $\sim 56.7\%$ weight and represent near-flat dose-response in centered log-dose, while a third component at $(-7.18, 2.7)$ (39.0\%) captures steeper responders; the remaining components are small/negligible. This multi-modality supports our use of a mixture surrogate rather than a single Gaussian when estimating prior-averaged quantities in BOED.

\begin{table}[H]
  \centering
  \small
  \caption{Laplace Mixture Approximation (LMA) of the empirical prior for \textit{DrugA}, built from 59 historical cell lines ($J=5$ components). Each row shows the mixture weight $\omega_j$, component mean $\mu_j=(\mu_{j,\alpha},\mu_{j,\beta})$, and the diagonal of the local covariance $\mathrm{diag}(\Sigma_j)$. Duplicate means reflect $k$-means cluster centers that landed at the same location.}
  \vspace{0.25em}
  \begin{tabular}{r r r r r r}
    \toprule
    \textbf{Comp $j$} & $w_j$ & $\boldsymbol{\mu_{j,\alpha}}$ & $\boldsymbol{\mu_{j,\beta}}$ & $(\Sigma_j)_{11}$ & $(\Sigma_j)_{22}$ \\
    \midrule
    0 & 0.044 & -4.946 & 0.000 & 0.010 & 0.010 \\
    1 & 0.523 & -4.946 & 0.000 & 0.010 & 0.010 \\
    2 & 0.390 & -7.184 & 2.700 & 0.148 & 0.100 \\
    3 & 0.044 & -4.946 & 0.000 & 0.010 & 0.010 \\
    4 & 0.000 & -5.729 & 0.870 & 1.852 & 2.279 \\
    \bottomrule
  \end{tabular}
  \label{tab:BOED_lma_components}
  \vspace{0.25em}
\end{table}

As seen in Fig.\ref{fig:delta_per_dose} (and quantified in Table.\ref{tab:delta_per_dose}), the per-replicate mutual-information gains $\Delta_k$ are essentially negligible at the five lowest doses ($\le 10^{-1}$; $\Delta_k\approx 0.001$-$0.005$), jump sharply at $3{\times}10^{-1}$ ($\Delta_k\approx 0.137$), and peak at the top dose $1$ ($\Delta_k\approx 0.468$). This strongly \textit{left-skewed} profile shows that most information is concentrated at the highest doses, so information-seeking designs naturally allocate the bulk of replicates to the top of the grid while keeping a few exploratory measurements at lower doses. 
This decision is consistent with our prior knowledge from historical lines in Fig.\ref{fig:eda_BOED}: responses are almost surely flat at low doses and separate at the top. The per-dose EIG we computed is therefore tiny at low doses and peaks at the highest dose, so the BOED design allocates most replicates up high while keeping a few exploratory points elsewhere.

\begin{figure}[H]
  \centering
  \includegraphics[width=0.65\linewidth]{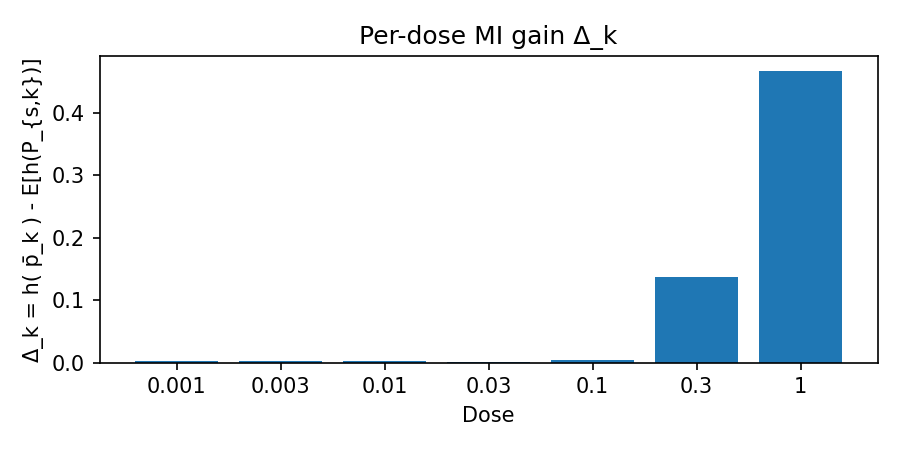}
  \caption{Per-dose mutual-information gain $\Delta_k$.}
  \label{fig:delta_per_dose}
\end{figure}

\begin{table}[H]
  \centering
  \caption{Per-dose MI gains $\Delta_k$.}
  \vspace{0.25em}
  \begin{tabular}{lccccccc}
    \toprule
    Dose & $10^{-3}$ & $3{\times}10^{-3}$ & $10^{-2}$ & $3{\times}10^{-2}$ & $10^{-1}$ & $3{\times}10^{-1}$ & $1$ \\
    \midrule
    $\Delta_k$ & 0.0025 & 0.0023 & 0.0021 & 0.0014 & 0.0045 & 0.1371 & 0.4675 \\
    \bottomrule
  \end{tabular}
\label{tab:delta_per_dose}
\end{table}

Given these gains, LMA-EIG spreads one replicate across all doses for exploration and assigns the remaining budget to the highest dose, yielding
$\{n_k\}=(1,1,1,1,1,1,15)$, i.e.\ $\gamma=(0.048,\ldots,0.048,0.714)$.
The Uniform baseline gives $(3,\ldots,3)$, while local $D$-opt concentrates on the two upper doses, $\{n_k\}=(0,0,0,0,10,0,11)$, as recorded in Table.\ref{tab:BOED_design_allocations}.

\begin{figure}[H]
  \centering
  \includegraphics[width=0.7\linewidth]{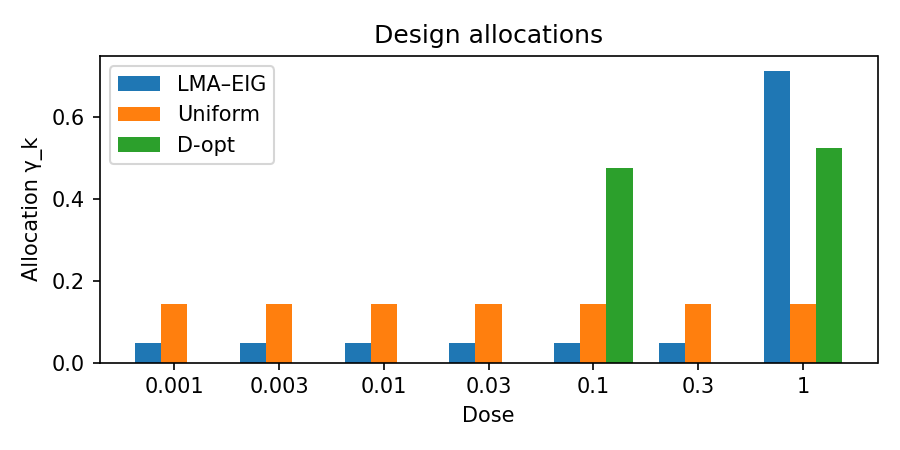}
  \caption{Design allocations $\gamma_k$ for LMA-EIG, Uniform, and local $D$-opt.}
\end{figure}

\begin{table}[H]
  \centering
  \caption{Allocations by design: counts $n_k$ and weights $\gamma_k$ (rounded).}
  \vspace{0.25em}
  \begin{tabular}{l l l}
    \toprule
    Design & $n_k$ & $\gamma_k$ \\
    \midrule
    LMA-EIG      & $[1,1,1,1,1,1,15]$ & $[0.048,0.048,0.048,0.048,0.048,0.048,0.714]$ \\
    Uniform      & $[3,3,3,3,3,3,3]$  & $[0.143,0.143,0.143,0.143,0.143,0.143,0.143]$ \\
    Local $D$-opt & $[0,0,0,0,10,0,11]$ & $[0.000,0.000,0.000,0.000,0.476,0.000,0.524]$ \\
    \bottomrule
  \end{tabular}
\label{tab:BOED_design_allocations}
\end{table}

This allocation strategy results in higher EIG and tighter, more accurate $EC_{50}$ than uniform sampling, using the same total budget. As in Table.\ref{tab:BOED_design_performance}, LMA-EIG attains the largest estimated EIG ($\widehat{\mathrm{EIG}}=7.16$), versus $5.19$ for local $D$-opt and $1.85$ for Uniform. 
Posteriors in $(\alpha,\beta)$ are similarly tight across designs, see the overlapping $1\sigma/2\sigma$ ellipses in Fig.\ref{fig:posterior_contours} and the similar Gaussian entropies $H\approx 2.22$ in Table.\ref{tab:BOED_design_performance}. But the induced uncertainty in $EC_{50}$ differs: LMA-EIG and local $D$-opt are accurate and precise near $0.0695$-$0.0697$, whereas Uniform is biased high and less precise (Table.\ref{tab:BOED_design_performance}).

\begin{table}[H]
  \centering
  \caption{Design performance: EIG, posterior entropy $H$, and $EC_{50}$ (95\% CI).}
  \label{tab:BOED_design_performance}
  \vspace{0.25em}
  \begin{tabular}{l c c c}
    \toprule
    Design & $\widehat{\mathrm{EIG}}$ & $H$ & $EC_{50}$ (95\% CI) \\
    \midrule
    LMA-EIG        & 7.16 & 2.2187 & $0.06972\ [0.06899, 0.07044]$ \\
    Local $D$-opt  & 5.19 & 2.2187 & $0.06949\ [0.06858, 0.07040]$ \\
    Uniform        & 1.85 & 2.2177 & $0.07778\ [0.07512, 0.08044]$ \\
    \bottomrule
  \end{tabular}
\end{table}

\begin{figure}[H]
  \centering
  \includegraphics[width=0.5\linewidth]{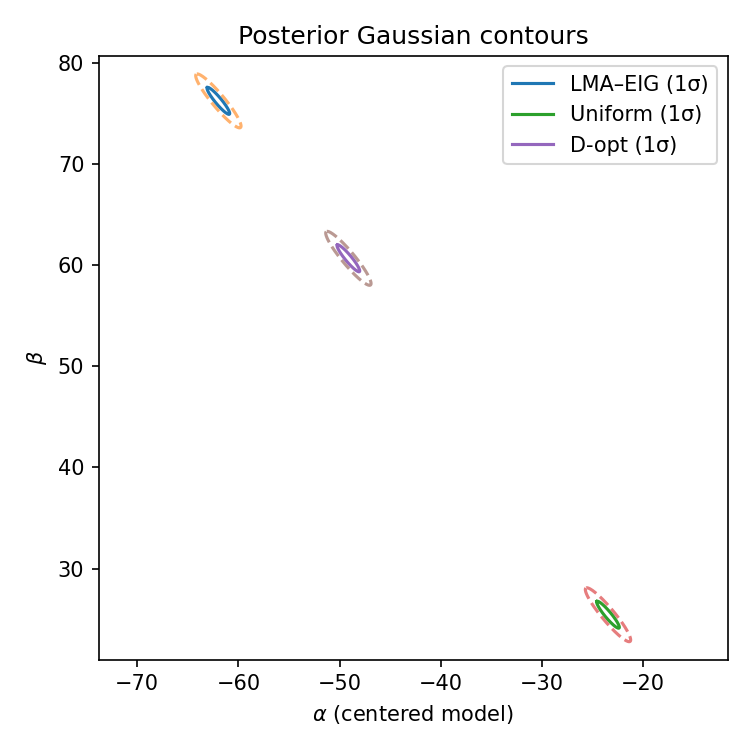}
  \caption{Posterior Gaussian $1\sigma$ (solid) and $2\sigma$ (dashed) contours in $(\alpha,\beta)$ for each design.}
  \label{fig:posterior_contours}
\end{figure}

Overall, the LMA-EIG rule leverages a multi-modal prior surrogate to place measurements where they most reduce predictive uncertainty, yielding the highest EIG and the most accurate $EC_{50}$ estimates; it clearly outperforms a Uniform allocation and matches (or slightly improves upon) local $D$-optimal performance. More broadly, BOED produces more robust design decisions than purely OED criteria by explicitly averaging over parameter uncertainty.

\subsubsection{Sensor network localization with \textit{LMA} and \textit{EM-GMA} posteriors} \label{subsec:sensor_network_localization}

A widely adopted benchmark for evaluating approximate Bayesian inference methods is the \textit{sensor network localization} (SNL) problem \cite{ihler2005nonparametric,ahn2013rdmc,lan2014wormhole,guo_boosting_2017}.
In this task, a collection of $N$ sensors is deployed in the unit square, with positions $\mathbf{s}_i \in \mathbb{R}^2$. To resolve global translation/rotation/reflection symmetries, we fix $N_{\text{ANCH}}$ sensors (\textit{anchors}) at predetermined coordinates. The remaining $N_{\text{UNK}} = N - N_{\text{ANCH}}$ sensors are treated as unknown random variables. Consequently, the posterior is defined over $2 N_{\text{UNK}}$ real-valued parameters (as each sensor coordinate is 2D). In the canonical literature benchmark with $N=11$ and $N_{\text{ANCH}}=3$, we have $N_{\text{UNK}}=8$ and thus a $2\times 8=16$-dimensional posterior distribution.

The observation model is defined in terms of noisy inter-sensor distances. For each unordered pair $(i,j)$, let $d_{ij} = \|\mathbf{s}_i - \mathbf{s}_j\|_2$ denote the Euclidean distance. A latent binary variable $Z_{ij}$ indicates whether a communication
link between sensors $i$ and $j$ is present, with probability
\[
  \Pr(Z_{ij}=1 \mid \mathbf{s}_i,\mathbf{s}_j)
  = \exp \left(-\tfrac{d_{ij}^2}{2R^2}\right)
\]
where $R>0$ controls the expected communication radius. Conditional on link presence ($Z_{ij}=1$), the observed measurement is a noisy range $Y_{ij} \sim \mathcal{N}(d_{ij},\sigma^2)$; if $Z_{ij}=0$, then $Y_{ij}=0$ deterministically.
The joint likelihood of all observations $(Y,Z)$ given sensor locations $\theta$ factorizes as:
\begin{align}
  p(Y,Z \mid \theta)
  = \prod_{i<j} \Big[
     p_{\text{link},ij}(\theta)  
       \mathcal{N}(Y_{ij}\mid d_{ij},\sigma^2)
  \Big]^{Z_{ij}}
  \Big[
     1-p_{\text{link},ij}(\theta)
  \Big]^{1-Z_{ij}}
\end{align}
with $p_{\text{link},ij}(\theta) = \exp(-d_{ij}^2 / (2R^2))$. A uniform prior over the
bounding box $[0,1]^2$ is typically assumed for each unknown sensor position.

As per Bayes' theorem, the posterior distribution is therefore:
\[
  p(\theta \mid Y,Z) \;\propto\; p(\theta) p(Y,Z \mid \theta)
\]
where $p(\theta)$ encodes the prior over positions. This posterior is highly non-Gaussian and often multi-modal due to ambiguous distance constraints and symmetry considerations. Marginals of individual sensor locations are typically crescent-shaped or multi-modal, making SNL a challenging stress test for approximate inference. In the literature, performance on this task is commonly reported in terms of the \textit{relative error of the posterior mean} (REM), computational time, alongside visual comparisons of the inferred marginal distributions.

Several representative inference approaches have been tested on this task. Early work by Ihler et al. \cite{ihler2005nonparametric} applied \textit{nonparametric belief propagation} (NBP) to capture multi-modal marginals in a distributed message-passing framework. Ahn et al. \cite{ahn2013rdmc} later introduced \textit{Regeneration Darting Monte Carlo} (RDMC), which combines local Hamiltonian dynamics with regeneration-based global jumps to improve mode exploration. Lan et al. \cite{lan2014wormhole} proposed \textit{Wormhole HMC} (WHMC), modifying the Riemannian metric to create ``wormholes'' between posterior modes, thereby reducing mixing times. More recently, Guo et al. \cite{guo_boosting_2017} demonstrated that single-Gaussian variational methods such as ADVI underfit the complex geometry of the posterior, and advocated \textit{boosting-based variational inference} (BVI) that constructs a Gaussian mixture sequentially to approximate the multi-modality. Together, these works have established SNL as a canonical and demanding testbed for Bayesian inference methods.

\subsubsection*{\textit{Experiment}}
\paragraph{Experimental setup and data}
We consider the planar SNL with $N=11$ sensors in a square region $[0,L]^2$ with $L=1.2$. Among them, $N_{\text{ANCH}}=5$ anchors have known coordinates (placed near the four corners and centre), and $N_{\text{UNK}}=N-N_{\text{ANCH}}=6$ sensors have unknown positions $\{\mathbf x_i\}_{i=1}^{N_{\text{UNK}}}\subset(0,L)^2$ to be inferred. Links are generated probabilistically as in prior work: for any pair $(i,j)$ with true distance $d_{ij}$, a binary edge $Z_{ij}\sim\mathrm{Bernoulli}(e^{-d_{ij}^2/(2R^2)})$ with $R=0.3$; for realised links ($Z_{ij}=1$), we observe a noisy range $Y_{ij}\sim\mathcal N(d_{ij},\sigma^2)$ with $\sigma=0.02$. As is common in the literature, we fix the number of observed links to $|{\cal E}|=14$ by selecting the $14$ pairs with the largest $e^{-d_{ij}^2/(2R^2)}$. The posterior is defined by these likelihoods together with a uniform prior over $(0,L)^2$ for each unknown coordinate. For computation we re-parameterise each unknown coordinate by $\mathbf x = L\cdot\sigma(\mathbf u)$ with $\sigma(\cdot)$ the elementwise sigmoid; the induced prior on $\mathbf u$ contributes the log-Jacobian $\sum_\ell[\log\sigma(u_\ell)+\log(1-\sigma(u_\ell))]$.
All methods use the same synthetic dataset, identical seeds, and independent initialisations. We report execution time and a \textit{relative error of the mean} (REM):
$$\mathrm{REM}=\|\bar{\mathbf x}-\mathbf x^\star\|_1/ \|\mathbf x^\star\|_1$$
where $\bar{\mathbf x}$ is the posterior mean in $x$-space and $\mathbf x^\star$ the ground truth for the $N_{\text{UNK}}$ sensors.

\paragraph{Approaches}
We compare 7 posterior approximation schemes in the unconstrained $u$-space: \textit{HMC/NUTS} (4,000 draws after 1,000 warmup) \cite{hoffman2014nuts}, \textit{MFVI-ADVI} (mean-field Gaussian) \cite{kucukelbir_automatic_2016}, \textit{S-ADVI} (per-dimension monotone warp of a Gaussian) \cite{shao_nonparametric_2024}, \textit{GM-ADVI} (mixture of $K = 24$ diagonal Gaussians trained with a stabilised SIWAE objective) \cite{morningstar_automatic_2020}, \textit{BVI} (boosting VI that iteratively adds Laplace components to the residual; $T=12$ rounds) \cite{guo_boosting_2017}; \textit{LMA} (Laplace mixture built from modes discovered by multi-start ascent with full-Hessian covariances); and a fast \textit{EM-GMA} (population EM for an inclusive-KL GMM, $K=24$, $M_{\text{bank}}=4096$, 40 iterations, full covariances).
All methods draw $4,000$ posterior samples for analysis.

\begin{table}[H]
\centering
\small
\begin{tabular}{lcc}
\toprule
Method & Time (s) $\downarrow$ & REM $\downarrow$ \\
\midrule
HMC/NUTS      & 18.6 & 0.181 \\
MFVI-ADVI     &  \underline{\textbf{7.1}} & 0.183 \\
S-ADVI        &  7.8 & 0.178 \\
GM-ADVI       & 14.3 & 0.510 \\
BVI           & 75.0 & 0.458 \\
LMA           & 16.9 & \underline{\textbf{0.018}} \\
EM-GMA (fast) & 58.1 & 0.235 \\
\bottomrule
\end{tabular}
\caption{Runtime and accuracy on the SNL task
($N = 11$, $N_{\text{ANCH}} = 5$, $|{\cal E}| = 14$). REM is the relative $\ell_1$ error of the posterior mean in $x$-space. }
\label{tab:snl-results}
\end{table}

\begin{figure}[H]
\centering
\includegraphics[width=0.9\textwidth]{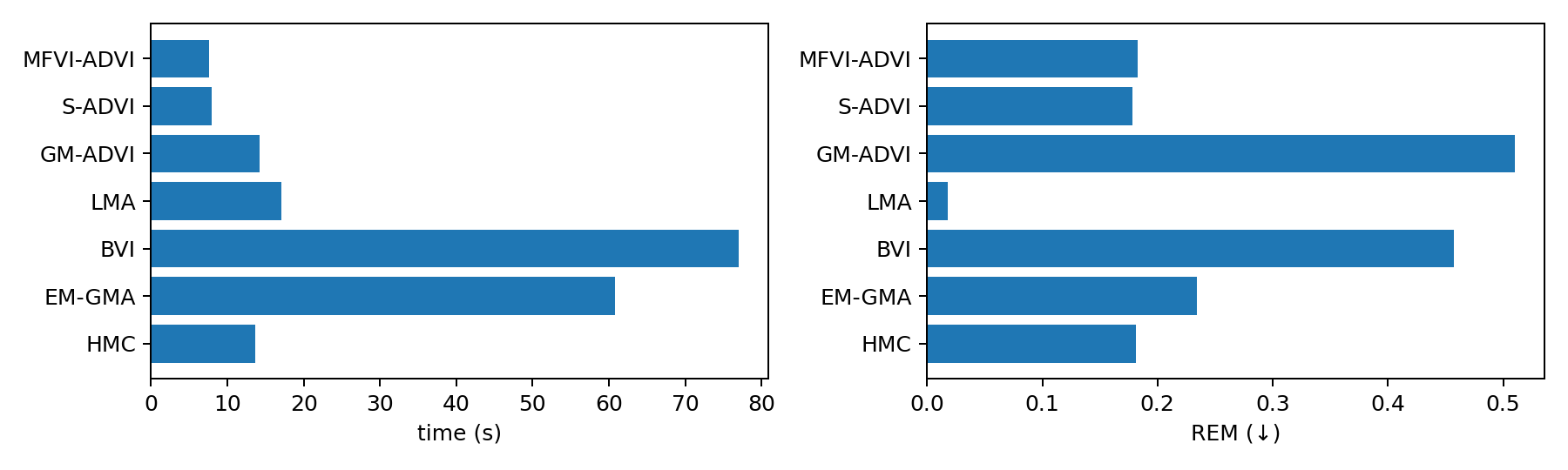}
\caption{Execution time (left) and REM (right). LMA is the most accurate; S/MFVI-ADVI are competitive with HMC on this instance; GM-ADVI and BVI underperform.}
\label{fig:snl_performances}
\end{figure}

\begin{figure}[H]
\centering
\includegraphics[width=\textwidth]{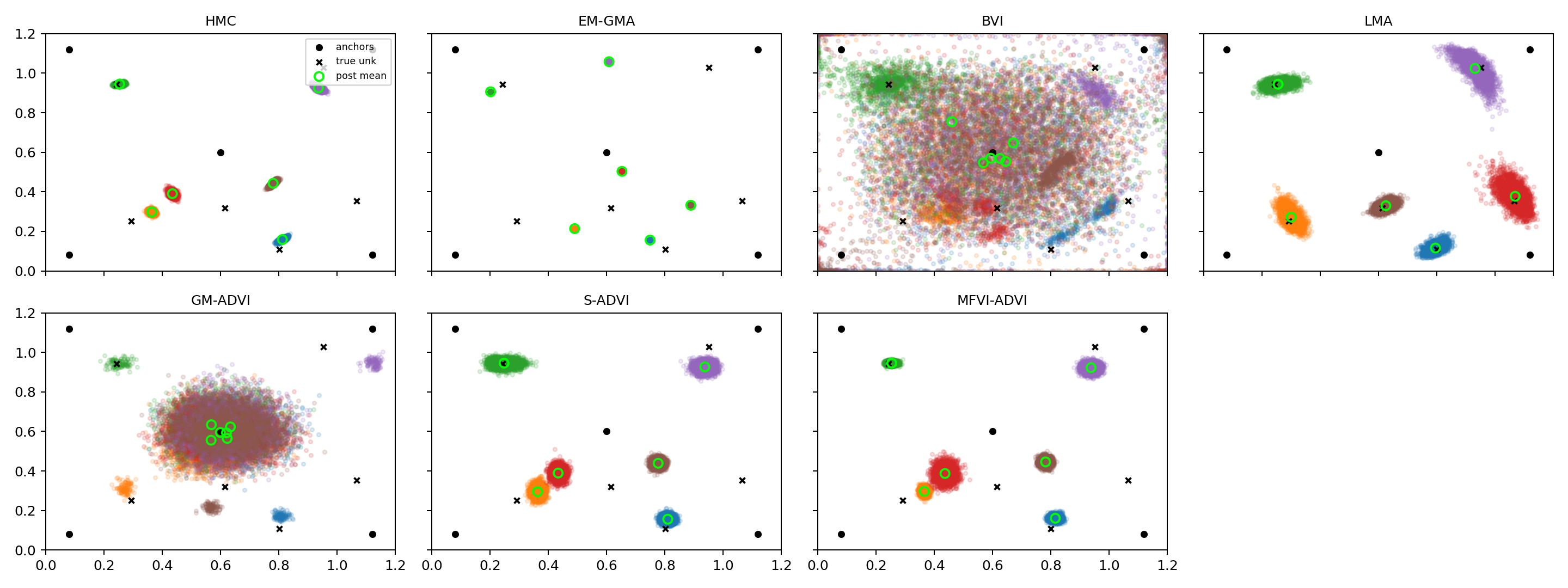}
\caption{Posterior point clouds in the coordinate plane for each method ($4 000$ samples each). Black dots: anchors. Black crosses: true unknown locations. Green circles: posterior means.}
\label{fig:snl_posterior}
\end{figure}

\paragraph{Results}
Results are presented in Table.\ref{tab:snl-results}, Fig.\ref{fig:snl_performances} and Fig.\ref{fig:snl_posterior}.
We observe that, LMA attains the best accuracy by a large margin, with REM $0.018$ and tight, well-oriented ellipses around the ground truth in Fig.\ref{fig:snl_posterior}.
In this case, the posterior is sufficiently concentrated around a few modes that second-order local Gaussianisation (using exact Hessians) captures both location and anisotropic uncertainty very well. The multi-start search finds distinct
modes (when they exist), and evidence-weighted mixing places mass appropriately.

The two mean-field baselines, i.e. S-ADVI and MFVI-ADVI, yield similar REM ($\approx 0.18$), on par with HMC. S-ADVI’s monotone per-coordinate warp slightly alleviates skewness, making it the strongest of the two. Their posterior sample plots show nearly spherical clusters, implying under-dispersion along elongated directions (which is a standard MFVI bias).
HMC (NUTS) serves as a reference but not a gold standard here: with a single chain and $\sim 4$k post-warmup draws, the means are close to S/MFVI-ADVI. With longer runs or more chains, HMC can possibly nudge the mean further towards LMA’s estimate, but at a higher computational cost.
The GM-ADVI mixture (with 24 diagonal components) underperforms (REM $0.51$) and produces a large central blob (Fig.\ref{fig:snl_posterior}). In this geometry, diagonal components struggle to align with narrow, rotated covariance structure; the mixture tends to average modes and inflate variance rather than resolve anisotropy.
BVI is both slow and inaccurate here (75s; REM $0.46$). Boosting the residual with a small number of components can be brittle in higher dimension: the residual objective may focus on dispersed “missing mass”, yielding overly diffuse components and a mean pulled away from the truth (as the very broad point clouds indicate in Fig.\ref{fig:snl_posterior}). More rounds or stronger curvature control could help, but would further increase runtime.
EM-GMA (population EM with SNIS and a fixed bank) lands in the middle of the pack (REM $0.235$). SNIS can be somewhat degenerate early on, and full-covariance GMM updates then chase the current bank’s coverage, leading to low effective sample sizes, noisy responsibilities, and over-tight covariances around accidental samples, so the mixture “locks onto” a biased subset of the posterior and under-covers the true modes.

Overall, in this synthetic SNL case with low range noise and well-spread anchors, a \textit{curvature-aware multi-modal Gaussianisation} (LMA) offers the best accuracy-cost trade-off: it discovers posterior modes and captures local anisotropy via full-Hessian covariances, so each component places its mean and spread correctly. Lightweight mean-field VI is competitive with short HMC runs, whereas mixture-based approximations require either full covariances, careful responsibility/weight stabilisation, or more components to be reliable. For EM-GMA specifically, when accuracy matters, annealing the responsibilities and enforcing ESS checks (and/or enlarging the proposal bank) can close part of the gap to LMA. More broadly, LMA’s ability to locate Gaussian means (modes) and spreads (local curvature) simultaneously makes it a cheap, mode-aware approximation usable for direct sampling and as a robust initialisation for subsequent GMA refinements.

\section{Towards theoretical guarantees} \label{sec:theory}

Here we present a primer on theoretical justifications for the GMA-sampling method. In particular, we provide informal convergence analysis, addressing 3 aspects: (1) an arbitrary distribution can be approximated by a GMM with a finite number of components within certain error bounds, and (2) the two-stage sampling method (GMA) is consistent with the approximated GMM, and (3) providing a rough error bound for the two-stage sampling method.

\subsection{Approximation of arbitrary distributions by a GMM with finite components}

\begin{theorem}[Universal approximation property of GMMs] \label{thm:univ_approx_GMMs}
Any sufficiently smooth probability density function $p(\mathbf{z})$ on \(\mathbb{R}^d\) can be approximated arbitrarily closely in $L^1$ distance by a Gaussian Mixture Model (GMM) with a finite, sufficiently large number of components covering whole support.
\end{theorem}

\begin{proof} (sketch)
    A GMM with $N$ components is defined as \cite{reynolds_gaussian_2009}:
    \[
    q_{\boldsymbol{\theta}}(\mathbf{z}) = \sum_{i=1}^N w_i \mathcal{N}(\mathbf{z} | \boldsymbol{\mu}_i, \boldsymbol{\Sigma}_i)
    \]
    where the parameters \(\boldsymbol{\theta} = \{w_i, \boldsymbol{\mu}_i, \boldsymbol{\Sigma}_i\}_{i=1}^N\) consist of weights ($w_i \geq 0$, \(\sum_{i=1}^N w_i = 1\)), means (\(\boldsymbol{\mu}_i \in \mathbb{R}^d\)), and positive definite covariance matrices (\(\boldsymbol{\Sigma}_i \in \mathbb{R}^{d \times d}\)).

    According to results in mixture density estimation theory (e.g. Li \& Barron \cite{li_mixture_1999}, Norets \cite{norets_approximation_2010}), a GMM can approximate any continuous density $p(\mathbf{z})$ on a compact support $\mathcal{Z} \subset \mathbb{R}^d$ with arbitrary precision, provided $N$ is sufficiently large. This is based on the fact that the set of mixtures of Gaussian densities is dense in the space of continuous densities with respect to the $L^1$ norm, under the condition that $p(\mathbf{z})$ has bounded support or decays sufficiently fast at infinity.
    
    Let $\epsilon > 0$ be the desired approximation error. The $L^1$ distance 
    \footnote{
    In density estimation, the KL divergence between two densities, however, is more natural to use than the $L^2$ distance often seen in function fitting literature \cite{li_mixture_1999} as it has intrinsic connection with maximum likelihood which is one of most useful method in mixture density estimation.
    The $L^1$ norm was used in the analysis for two primary reasons: its practical probabilistic meaning and its convenient mathematical properties for error decomposition. 
    (1) Probabilistic interpretation. The $L^1$ norm has a direct and intuitive connection to how distinguishable two probability distributions are - it is equal to twice the \textit{total variation} (TV) distance (see Appendix.\ref{app:distributional_distance_metrics} for the definition of TV): $\|p - q\|_{L^1} = 2 \cdot D_{TV}(p, q)$. The TV distance represents the largest possible difference in probability that the two distributions can assign to any single event. For example, if $\|p - q\|_{L^1} = 0.1$, then the TV distance is $0.05$. This means for any possible event (e.g. a sample falling into a certain range), the probability calculated by $p$ and the probability calculated by $q$ will differ by at most 5\%. This provides a worst-case guarantee on how similar the sampling outcomes from the two distributions will be, making it a very practical measure of error.
    (2) Mathematical properties for analysis.
    From a mathematical standpoint, the $L^1$ norm is a true \textit{metric}, which makes it ideal for the kind of error analysis performed. First, it satisfies the triangle inequality. Using a true distance metric is important as, we shall see later, it allows us to decompose the total error into manageable parts using the triangle inequality of norms:
        $$\|p - p_{\text{samples}}\|_{L^1} \le \|p - q_{w_{opt}}\|_{L^1} + \|q_{w_{opt}} - q_{\mathbf{w}_K}\|_{L^1} + \|q_{\mathbf{w}_K} - p_{\text{samples}}\|_{L^1}$$
    other measures, e.g. KL divergence, are not metrics and do not satisfy symmetry and the triangle inequality, making such a direct decomposition impossible.
    Further, key theoretical results used in the derivation are often stated in terms of the $L^1$ norm. For example, the universal approximation property of GMMs is proven in terms of $L^1$ convergence, and standard bounds for Monte Carlo sampling error are also readily available for the $L^1$ distance \cite{taylor_nonparametric_1985,nussbaum_devroye_1988}. Using the $L^1$ norm allows us to plug these established results directly into our analysis.
} 
between $p(\mathbf{z})$ and $q_{\boldsymbol{\theta}}(\mathbf{z})$ is defined as 
\footnote{
The $L^1$ norm of a function $f(\mathbf{z})$ is defined as $\|\mathbf{f}\|_{L^1} = \int |f(\mathbf{z})| d\mathbf{z}$; for any vector $\mathbf{v}=[v_1, v_2, \ldots, v_N]$, the $L^1$ norm is the sum of the absolute values of its components: $\|\mathbf{v}\|_{L^1} = \sum_{i=1}^N |v_i|$.
}:
      \[
      \| p - q \|_{L^1} = \int_{\mathcal{Z}} |p(\mathbf{z}) - q_{\boldsymbol{\theta}}(\mathbf{z})| d\mathbf{z}
      \]
      By the universal approximation theorem, there exists a finite $N$ and parameters \(\boldsymbol{\theta}\) such that $\| p - q_{\boldsymbol{\theta}} \|_{L^1} < \epsilon$, provided $p(\mathbf{z})$ is continuous.
      For densities on unbounded domains, the approximation still holds if $p(\mathbf{z})$ has a sufficient rate of decay (e.g. $p(\mathbf{z})$ decays faster than any Gaussian tail), a condition met by most distributions of practical interest. The number of components $N$ required depends on the complexity of $p(\mathbf{z})$, e.g. its number of modes and curvature. For a density with $k$ modes and bounded support, $N \geq k$ components can often suffice with appropriate placement of \(\boldsymbol{\mu}_i\), and the error decreases as $N$ increases due to the flexibility of Gaussian mixtures. Therefore, the GMM with finite $N$ can approximate any continuous target density within an $\epsilon$-error bound in $L^1$ norm, justifying the use of GMA as a flexible, universal representation for the target distribution.
\end{proof}

\subsection{Consistency and convergence of the two-stage sampling method}

\begin{theorem}[Consistency of GMA two-stage sampling] \label{thm:consistency_GMA}
    The two-stage sampling method in GMA, which involves sampling each Gaussian component, optimizing component weights via gradient descent to minimize the KL divergence, and then performing stratified sampling proportional to the optimised weights $\mathbf{w}^*$, produces an empirical sample distribution that converges to the approximate GMM $q_{\mathbf{w}^*}(\mathbf{z})$ as the number of samples $M$ per component goes to infinity.
\end{theorem}
\begin{proof} (sketch) The proof involves analyzing the two stages of the algorithm. \\
    Stage 1: \textit{weight optimisation}.
    The algorithm aims to find weights $\mathbf{w}$ that make the GMM $q_{\mathbf{w}}(\mathbf{z}) = \sum_{i=1}^N w_i \mathcal{N}(\mathbf{z} | \boldsymbol{\mu}_i, \boldsymbol{\Sigma}_i)$ a good approximation of the target density $p(\mathbf{z})$.  
    To achieve this, it performs gradient descent on the weights $\mathbf{w}^{(k)}$ to minimize the KL divergence objective $KL(q_{\mathbf{w}}(\mathbf{z}) \| p(\mathbf{z})) = \int q_{\mathbf{w}}(\mathbf{z}) \log \frac{q_{\mathbf{w}}(\mathbf{z})}{p(\mathbf{z})} d\mathbf{z}$.
    The gradient update \(\mathbf{w}^{(k)} = \mathbf{w}^{(k-1)} - \eta \cdot \nabla_{\mathbf{w}} KL\) uses (Eq.\ref{eq:KL_divergence_objective4_gradient_GMM0b}):
    \[
    \nabla_{w_i} KL = 1 + \mathbb{E}_{z \sim q_{\mathbf{w}}} \left[ \log q_{\mathbf{w}}(z) - \log \bar{p}(z) \right]
    \]
    which is estimated using $M$ samples per component (Eq.\ref{eq:KL_divergence_objective4_gradient_GMM1}):
    \[
    g_i = 1 +  \sum_{j=1}^M \mathcal{N}(\mathbf{s}_{i,j} | \boldsymbol{\mu}_i, \boldsymbol{\Sigma}_i) \left[ \log q_{\mathbf{w}}(\mathbf{s}_{i,j}) - \log \bar{p}(\mathbf{s}_{i,j}) \right]
    \]
    As $M \to \infty$, the sample average converges to the expectation $1 + \mathbb{E}_{\mathbf{s} \sim \mathcal{N}_{\boldsymbol{\mu}_i,\Sigma_i}} [ \log \frac{q_{\mathbf{w}}(\mathbf{s})}{\bar{p}(\mathbf{s})} ]$ (Eq.\ref{eq:KL_divergence_objective4_gradient_GMM1}) by the \textit{law of large numbers}, and with many $K$ iterations, \(\mathbf{w}^{(K)} \to \mathbf{w}^*\) that minimizes $KL(q_{\mathbf{w}} \| p)$, assuming \(\eta\) is sufficiently small. If the optimisation landscape is convex, it has a unique global minimum; otherwise a local minimizer is sought. Evidence of guarantee can be found in e.g. foundational theory for stochastic optimisation \cite{robbins_stochastic_1951}, constrained optimisation (e.g. projected gradient descent) \cite{goldstein_convex_1964,levitin_constrained_1966,bertsekas1999nonlinear,nocedal2006numerical}, convex optimisation \cite{goldstein_convex_1965,boyd2004convex,bubeck_convex_2015,nesterov_lectures_2018}, optimisation for machine learning \cite{bottou_optimization_2018}, etc.

    Stage 2: \textit{stratified sampling}.
    The ensemble samples are drawn by selecting component $i_m \sim \text{Categorical}(\mathbf{p} = \mathbf{w}^*)$ where $\mathbf{w}^*$ is the normalised optimiser, and sample $j_m \sim \text{Uniform}(\{1, \ldots, M\})$ within each Gaussian component, producing $\mathbf{s}_{i_m, j_m}$.
    The resulting \textit{sample} distribution of the finally selected samples is:
    \[
    p_{\text{samples}}(\mathbf{z}) = \sum_{i=1}^N w_i^* \cdot \frac{1}{M} \sum_{j=1}^M \delta(\mathbf{z} - \mathbf{s}_{i,j})
    \]
    where $\delta$ is the \textit{Dirac delta function}. As the number of samples per component $M \to \infty$, the empirical measure of each component, $\frac{1}{M} \sum_{j=1}^M \delta(\mathbf{z} - \mathbf{s}_{i,j})$, converges in distribution to the true component density $\mathcal{N}(\mathbf{z} | \boldsymbol{\mu}_i, \boldsymbol{\Sigma}_i)$.
    Therefore, the overall sample distribution $p_{\text{samples}}(\mathbf{z})$ converges \footnote{In general, stratified sampling produces samples consistent with the population sample space or the target, underlying distribution, but with faster convergence rate than MC methods and with smaller variance, see Appendix.\ref{app:simple_and_stratified_random_sampling} an analysis.} in distribution to $q_{\mathbf{w}^*}(\mathbf{z}) = \sum_{i=1}^N w_i^* \mathcal{N}(\mathbf{z} | \boldsymbol{\mu}_i, \boldsymbol{\Sigma}_i)$. This proves that the sampling procedure is consistent with the GMM defined by the optimised weights. The rational of drawing from each mixture component and taking weighted average is equivalent to sampling from the mixture density is justified in e.g. \cite{morningstar_automatic_2020}.

    The final convergence of the entire process to the true target $p(\mathbf{z})$ relies on the combination of these two stages. The universal approximation theorem ensures that a good GMM approximation $q_{\mathbf{w}^*}$ exists for a sufficiently large $N$ (and well placed components), while the sampling stage ensures that the samples produced are faithful to and consistent with that GMM. Since $q_{\mathbf{w}^*}$ approximates $p(\mathbf{z})$ within $\epsilon$ (from the universal approximation property), and the sampling is consistent with $q_{\mathbf{w}^*}$, $p_{\text{samples}}$ converges to $p(\mathbf{z})$ in distribution as $N, M \to \infty$. The quality of the approximation of the final selected samples, however, depends on a good approximation $q_{\mathbf{w}^*} (\mathbf{z}) \approx p (\mathbf{z})$ with large $N$ and good placement \footnote{Our experiments find that, sampling and inference results are less sensitive to placement; in general uniform (linear interpolation) placement of Gaussian centers in the support space would generate reasonably well results.} satisfying e.g. separation conditions in \cite{dasgupta_learning_nodate}, the optimisation finding a good set of weights $\mathbf{w}^*$ and $M$ being large enough to make the sampling error negligible.
     
\end{proof}

\subsection{Error bounds with finite component GMM approximation}

The two-stage GMA sampling method, whose use of fixed Gaussian components and optimised weights aligns with aforementioned theoretical framework, therefore produces samples consistent with $p(\mathbf{z})$ as $N, M \to \infty$, with convergence guaranteed \footnote{The proofs assume continuous densities and sufficient $N$; for discrete or highly irregular densities, additional conditions may apply.} by the gradient descent optimisation \cite{bottou_optimization_2018} and the stratified sampling process \cite{morningstar_automatic_2020}, provided the initial means and variances are well-chosen over the support (i.e. reasonable placement). Convergence of the weight optimisation process depends on the gradient descent’s step size $\eta$ and the number of iterations $K$. Under the \textit{Robbins-Monro conditions} \cite{robbins1951stochastic} (\(\sum \eta_k = \infty\), \(\sum \eta_k^2 < \infty\)), we used a diminishing learning rate $\frac{\eta_0}{k}$, which guarantees the weights converge to a local minimum of the $KL$ divergence.
The stratified sampling error decreases as \cite{cochran_sampling_1977} $1/\sqrt{M}$ (same rate as Monte Carlo methods \cite{kennedy2016montecarlo} but with smaller constant and reduced variance, see Appendix.\ref{app:simple_and_stratified_random_sampling}), ensuring consistency with $q_{\mathbf{w}^*}$.
Overall convergence to $p(\mathbf{z})$ requires $N$ to be sufficiently large to capture the target density’s complexity, with $M$ and $K$ ensuring sampling and optimisation accuracy.

We can formally bound the total error between the true target density $p(\mathbf{z})$ and the empirical distribution of the final samples $p_{\text{samples}}(\mathbf{z})$ generated by the two-stage GMA sampling procedure. The total error, measured in $L^1$ norm, can be decomposed into three distinct sources: approximation error $\mathcal{E}_{\text{approx}}$, optimisation error $\mathcal{E}_{\text{opt}}$, and sampling error $\mathcal{E}_{\text{sample}}$. Let $q_{w_{opt}}(\mathbf{z})$ be the best possible GMM approximation to $p(\mathbf{z})$ using the fixed set of $N$ Gaussian components, and let $q_{\mathbf{w}_K}(\mathbf{z})$ be the GMM produced by the pGD algorithm after $K$ iterations. Using the triangle inequality, we can decompose the total error as (enabled by the triangle inequality of the $L^1$ norm):

$$\|p - p_{\text{samples}}\|_{L^1} \le \underbrace{\|p - q_{w_{opt}}\|_{L^1}}_{\text{Approximation error $\mathcal{E}_{\text{approx}}$}} + \underbrace{\|q_{w_{opt}} - q_{\mathbf{w}_K}\|_{L^1}}_{\text{optimisation error $\mathcal{E}_{\text{opt}}$}} + \underbrace{\|q_{\mathbf{w}_K} - p_{\text{samples}}\|_{L^1}}_{\text{Sampling error $\mathcal{E}_{\text{sample}}$}}$$

The total error is therefore bounded by the sum of the three errors, and each error is manageable. In the following, we analyse each term from the above RHS. 

\paragraph{Approximation error $\mathcal{E}_{\text{approx}}$} This is the inherent error from using a (optimal) GMM with $N$ fixed Gaussian components to represent $p(\mathbf{z})$. The universal approximation property (\textit{Theorem}.\ref{thm:univ_approx_GMMs}) guarantees this error can be made arbitrarily small by increasing $N$ and choosing the component means and covariances appropriately to cover the support of $p$. For a target density $p$ with certain smoothness properties, theoretical results \cite{li_estimation_1999,li_mixture_1999,genovese_rates_2000,Devore2010constructive,kruijer_adaptive_2010,norets_approximation_2010,cathy_adaptive_2011,nguyen_convergence_2013} show this error decreases as $N$ grows, often polynomially, such that $\mathcal{E}_{\text{approx}} = \|p - q_{w_{opt}}\|_{L^1} \le \mathcal{O}(N^{-c})$ for some constant $c > 0$ that depends on the smoothness of $p$ and the data dimension $d$. With certain conditions, e.g. the target density is within the convex hull of a Gaussian family, the approximation error measured by KL divergence can decrease at the rate $\mathcal{O}(1/N)$ \cite{li_mixture_1999}.
See Section.\ref{subsec:GMM_fitting} some related work on this polynomial convergence rate.

\paragraph{Optimisation error $\mathcal{E}_{\text{opt}}$} This error arises because the gradient descent algorithm runs for a finite number of iterations $K$, and may not find the optimal weights \footnote{Further, the objective function, $KL(q_{\mathbf{w}} \| p)$, is generally non-convex with respect to the mixture weights $\mathbf{w}$ (in our setting, it is convex though, see Appendix.\ref{app:gradient_of_exclusive_KL_divergence}), meaning the algorithm is only guaranteed to find a local minimum, not necessarily the global optimum.} $\mathbf{w}_{opt}$. For stochastic gradient descent on a general convex objective (as in our case), standard convergence theory states that the error in the objective function decreases as \footnote{While a convergence rate of $\mathcal{O}(K^{-1/2})$ is standard for stochastic gradient descent on convex problems (which is the case in our setting), the analysis for non-convex objectives is more complex. However, under certain smoothness conditions, SGD-based methods can find a stationary point (where the gradient is close to zero) at a similar rate.} $\mathcal{O}(1/\sqrt{K})$ \cite{nemirovski_robust_2009} or $\mathcal{O}(1/K)$ \cite{bottou_optimization_2018} (strongly convex objective, using a diminishing step-size).
We can relate the $L^1$ distance between the GMMs to the distance between their weight vectors:
    $$
    \|q_{\mathbf{w}_{opt}} - q_{\mathbf{w}_K}\|_{L^1} = \left\| \sum_{i=1}^N (w_{opt,i} - w_{K,i}) \mathcal{N}_i \right\|_{L^1} \le \sum_{i=1}^N |w_{opt,i} - w_{K,i}| \|\mathcal{N}_i\|_{L^1} = \|\mathbf{w}_{opt} - \mathbf{w}_K\|_{L^1}
    $$
The middle step arises from the triangle inequality: for any sum of functions $f_1, f_2, \ldots, f_N$, the norm of the sum is less than or equal to the sum of the norms, i.e. $\|\sum_{i=1}^N f_i \| \le \sum_{i=1}^N \| f_i \|$. And the property of norm with scalar multiplication: $\|c \cdot f\| = |c| \cdot \|f\|$.
The last step arises from the unit integral of PDFs: $\|\mathcal{N}_i\|_{L^1} = \int |\mathcal{N}_i(\mathbf{z})| d\mathbf{z} = \int \mathcal{N}_i(\mathbf{z}) d\mathbf{z} = 1$.
Thus, the optimisation error is bounded by the error in the weight vector, which diminishes as the number of iterations $K$ increases: $\mathcal{E}_{\text{opt}} \le \mathcal{O}(K^{-1/2})$, as in standard SGD.

These rates also apply to projected gradient descent (pGD) \cite{bottou_optimization_2018} and mirror descent (MD) \cite{blair_problem_1985,nemirovski_robust_2009}.
The above optimisation-error analysis extends to constrained updates on the probability simplex $\Delta^{N-1} = \Big\{ \mathbf w \in \mathbb R^N : w_i \ge 0,  \sum_{i=1}^N w_i = 1 \Big\}$. For a convex, Lipschitz-smooth objective $F(\mathbf w)=\mathrm{KL} \left(q_{\mathbf w} \| p\right)$ over $\Delta^{N-1}$, pGD with a diminishing stepsize achieves the same rates as unconstrained GD: with full (deterministic) gradients one has $f(\mathbf w_K)-f(\mathbf{w}_{opt})=\mathcal{O}(1/K)$, while with unbiased stochastic gradients and Robbins-Monro stepsizes the expected suboptimality satisfies $\mathbb{E}  \left[f(\bar{\mathbf w}_K)-f(\mathbf{w}_{opt})\right]=\mathcal{O}(K^{-1/2})$ \cite{bottou_optimization_2018,nemirovski_robust_2009}. The same orders hold for MD with multiplicative-weights update on $\Delta^{N-1}$ (see Appendix.\ref{app:mirror_descent_GMA}), where constants depend on the Bregman diameter (scaling like $\sqrt{\log N}$ for the entropy mirror map) but the $K$-rates remain $\mathcal{O}(1/K)$ (deterministic) and $\mathcal{O}(K^{-1/2})$ (stochastic) \cite{nemirovski_robust_2009}. Combining these optimisation guarantees with the $L^1$-stability of mixtures,
\[
\|q_{\mathbf{w}_{opt}}-q_{\mathbf w_K}\|_{L^1}
\le \|\mathbf{w}_{opt}-\mathbf w_K\|_{L^1}
\]
implies that the optimisation component of the density approximation error decays at the same order: $\mathcal{E}_{\mathrm{opt}}=\mathcal{O}(K^{-1/2})$ under stochastic updates, and $\mathcal{O}(K^{-1})$ in the deterministic case.

\paragraph{Sampling error $\mathcal{E}_{\text{sample}}$} This is the Monte Carlo error from using a finite set of $S = N \times M$ total samples to represent the continuous, fitted GMM density $q_{\mathbf{w}_K}(\mathbf{z})$. Classical Monte Carlo sampling, as used in our first stage Gaussian sampling, holds an convergence rate $\mathcal{O}(S^{-1/2})$ \cite{mcbook,Veach1997robust}. Stratified random sampling, as used in our second sampling stage, guarantees that the re-sampled samples are consistent \footnote{See e.g. Eq.\ref{eq:stratified_sampling_morningstar} from \cite{morningstar_automatic_2020} or a consistency analysis of stratified sampling in Appendix.\ref{app:simple_and_stratified_random_sampling}.} with $q_{\mathbf{w}_K}(\mathbf{z})$, and has a convergence rate $\mathcal{O}(S^{-1/2})$ with constant improvement (see Appendix.\ref{app:simple_and_stratified_random_sampling}), and the sampling error is bounded as $\mathcal{E}_{\text{sample}} = \mathbb{E}\left[\|q_{\mathbf{w}_K} - p_{\text{samples}}\|_{L^1}\right] \le \mathcal{O}((NM)^{-1/2})$.

\paragraph{Overall error bound}
Combining these 3 sources of error, we arrive at an overall bound for the expected error of the GMA-sampling algorithm:

\begin{equation}
    \mathbb{E}\left[\|p - p_{\text{samples}}\|_{L^1}\right] \le C_{\text{approx}}N^{-c} + \frac{C_{\text{opt}}}{\sqrt{K}} + \frac{C_{\text{sample}}}{\sqrt{NM}}
\end{equation}
where $c$, $C_{\text{approx}}$, $C_{\text{opt}}$, and $C_{\text{sample}}$ are constants that depend on the properties of the target density $p$ (e.g. smoothness and dimension) and the specific GMM components (finer analysis is demanded though). This inequality shows the trade-offs involved in the algorithm: increasing $N$ reduces approximation and sampling errors, while increasing $K$ and $M$ reduces optimisation and sampling errors, respectively, all at the cost of increased computation.

\section{Related work} \label{sec:related_work}

We review methods relevant to the tasks of density approximation and sampling addressed in this paper. The following discussion is non-exhaustive but briefly lists key approaches and results.

\subsection{Mixture density estimation and approximation} \label{subsec:GMM_fitting}

Density estimation has been long researched. It refers to the process of constructing an estimate of the PDF of a random variable based on observed data \cite{fix_discriminatory_1989,silverman1986density,silverman_e_1989}. Methods for density estimation can be parametric or non-parametric \cite{silverman1986density}. Parametric density estimation assumes a specific functional form for the density (e.g. normal, exponential, etc) with a finite number of parameters to be estimated. For example, estimating the parameters $\boldsymbol{\theta} = \{w_i, \boldsymbol{\mu}_i, \boldsymbol{\Sigma}_i\}_{i=1}^N$ in a GMM $q_{\boldsymbol{\theta}}(\mathbf{z}) = \sum_{i=1}^N w_i \cdot \mathcal{N}(\mathbf{z} | \boldsymbol{\mu}_i, \boldsymbol{\Sigma}_i)$, using maximum likelihood estimation (MLE). This compact, parametric representation is efficient if the assumed PDF model is correct. Non-parametric density estimation does not assume any specific form of the density function
\footnote{Non-parametric methods make less rigid assumptions about the distribution of the observed data. Although it may assume an underlying density, the data will be allowed to speak for themselves in determining the estimate of the PDF more
than would be the case if the density were constrained to fall in a given parametric family \cite{silverman1986density}.}
; it is thus flexible and adaptive to complex distributions such as multi-modal and irregular ones. An intuitive and straightforward non-parametric density estimation approach is to construct a normalized sample \textit{histogram}, where data are divided into bins, frequencies are tallied, and heights are scaled to sum to unity, yielding a piecewise constant approximation of the underlying distribution. However, histograms are sensitive to bin choice and placement, which can lead to misleading interpretations. A more principled non-parametric method is \textit{kernel density estimation} (KDE), which smooths data using kernel functions to mitigate these issues. KDE approximates PDFs directly from samples without assuming a parametric form, using the estimator $\hat{p}_h(x) = \frac{1}{n h^d} \sum_{i=1}^n K\left( \frac{x - x_i}{h} \right)$, where $K$ is a kernel function (e.g. RBF), $h$ is the bandwidth controlling smoothness, and $d$ is the dimensionality \cite{silverman_density_2018}. KDE is simple to implement, asymptotically consistent, and effective in low dimensions but suffers from the curse of dimensionality in higher dimensions. Nearest neighbor density estimation \cite{silverman1986density} estimates the density value at point based on the distance between the point and its nearest neighbors. These non-parametric methods typically require more data to perform well and more computationally intensive. 

Approximating PDFs using weighted mixtures of simpler basis functions is a powerful approach in statistics and machine learning. These approximations typically take the form $p(\mathbf{z}) \approx q_{\boldsymbol{\theta}}(\mathbf{z}) = \sum_{i=1}^N w_i \cdot \phi_i(\mathbf{z}; \theta_i)$, where $\phi_i(\mathbf{z}; \theta_i)$ are basis functions (e.g. Gaussians, sigmoids, splines, trigonometric functions, etc), $\theta_i$ are parameters for basis $\phi_i$, and $w_i$ are weights satisfying non-negativity and partition of unity ($w_i \geq 0$, $\sum_{i=1}^N w_i = 1$) to ensure a valid PDF \cite{verbeek_mixture_2004}. Sigmoidal basis functions, commonly used in neural networks, enable universal approximation of arbitrary functions \cite{barron1993universal}. B-splines, piecewise polynomials with local support, are ideal for smooth, numerically stable function approximation and are also used in non-parametric regression and density estimation \cite{Deboor1978splines}. They offer flexibility by controlling smoothness through spline order and knot placement. Mixture distributions have been used in Bayesian modelling to represent prior or to approximate posterior distributions, due to their expressiveness and flexibility in representing complex probabilistic structures \cite{bishop1998approximating,west1993approximating}, although posterior collapse can happen as well \cite{gershman_nonparametric_2012}. For example, learnable Gaussian mixtures are used in variational auto-encoders to capture multi-modal priors \cite{dilokthanakul2016deep,johnson2016composing,jiang2017variational}. Flexibility and performance of VAEs can be enhanced with a new type of prior called 'variational mixture of posteriors' prior (\textit{VampPrior}) which uses a mixture distribution (e.g. Gaussians) with components given by variational posteriors conditioned on learnable pseudo-inputs \cite{tomczak_vae_2018}. Mixture distributions have also been used as variational posteriors, including fixed-weight mixtures \cite{oh2019modeling}, continuous relaxations ( e.g. concrete relaxation of the categorical distribution \cite{poduval2020mixture}), and advanced sampling schemes that improve approximation coverage \cite{domke2019divide}.
Graves\cite{graves_stochastic_2016} introduced a method that allows gradient backpropagation through mixtures with diagonal covariance structure using recursive sampling. Morningstar et al. \cite{morningstar_automatic_2020} explored using mixtures, whose component distributions are reparameterizable, for approximating the variational posterior in ADVI. They used stratified sampling for sampling each component, and derived the stratified ELBO (SELBO) specialised for mixtures. Roeder et al. \cite{roeder2017sticking} further proposed a pathwise gradient extension to the SELBO objective that reduces gradient variance while preserving mode-seeking behaviour.
Overall, although mixture distributions can enhance expressiveness and modeling flexibility in variational inference, they often require careful handling of reparameterization, component weighting, and sampling schemes to ensure efficient and stable optimisation. Although with a mixture posterior approximator, traditional ELBO objective can still fail to capture the multi-modality.

Gaussian mixture models (GMMs), which represent a target density as a weighted sum of Gaussian components $q_{\boldsymbol{\theta}}(\mathbf{z}) = \sum_{i=1}^N w_i \cdot \mathcal{N}(\mathbf{z} | \boldsymbol{\mu}_i, \boldsymbol{\Sigma}_i)$, are among the most widely used mixture distributions due to its theoretical expressiveness and practical flexibility \cite{bishop_pattern_2006}. GMMs, also known as radial basis function (RBF) networks in approximation literature \cite{zeevi_density_1997}, are flexible and under mild regularity conditions, they can approximate any continuous density arbitrarily well in the $L_1$ or $L_2$ sense \cite{silverman1986density,li_mixture_1999,norets_approximation_2010}, given sufficient components. As a semi-parametric method, GMMs combine the structure of parametric models with the flexibility of non-parametric approaches \cite{genovese_rates_2000}. Finite mixtures of (multivariate) Gaussian distributions have broad utility in e.g. model-based clustering \cite{verbeek_mixture_2004,vempala_spectral_2004,maugis_non_2011,cathy_adaptive_2011,young_finite_2019} and image segmentation \cite{alfo_finite_2008,nguyen_extension_2010}. 

GMMs have been used as a flexible candidate family in VI. As VI turns an inference problem into optimisation \footnote{One therefore notes the interlinks among inference, optimisation and sampling.}, VI with GMMs essentially translates into estimating the GMM parameters $\boldsymbol{\theta} = \{w_i, \boldsymbol{\mu}_i, \boldsymbol{\Sigma}_i\}_{i=1}^N$. Methods for GMM parameter estimation include maximum likelihood (e.g. the EM algorithm \cite{dempster_maximum_1977, reynolds_gaussian_2009, segol_improved_2021}), maximizing a posteriori (MAP) \cite{reynolds_gaussian_2009}, classic Bayesian methods such as the Gibbs sampling \cite{rasmussen_infinite_1999,blei2015bayesian,tosh_lower_bounds,bouchard-cote_particle_nodate,murphy_machine_2012,kamper2013gibbs} which quickly focuses on one of the modes \cite{blei_variational_2017}, MCMC \cite{neal1993probabilistic}, and modern Bayesian methods such as mean-field variational inference (MFVI \cite{Jordan1999introduction}), automatic differentiation variation inference (ADVI \cite{kucukelbir_automatic_2015,kucukelbir_automatic_2016}) method, as well as diffusion-based sampling methods \cite{song_generative_2019,chen_sampling_2023,gatmiry_learning_2025}. MLE and MAP give point estimates while Bayesian methods yield distributions for parameters. MLE, for example, finds the model parameters which maximize the likelihood of the GMM given the training data; however, the likelihood function is a nonlinear function of the parameters $\boldsymbol{\theta}$, which disables direct maximization of this objective \cite{reynolds_gaussian_2009}. MFVI and ADVI, on the other hand, minimize the discrepancy (e.g. KL divergence or ELBO) between the target $p(\mathbf{z})$ and the candidate $q_{\boldsymbol{\theta}}(\mathbf{z})$. Although with different objectives, both methods, in general, employ first-order, gradient-based optimisation strategies to search for the optimal parameters. Sampling routines, e.g. MH \cite{metropolis_equation_1953}, HMC \cite{duane_hybrid_1987,homan_no-u-turn_nodate}, LMC \cite{Roberts1996}, or ParVI \cite{Liu2016SVGD}, generate samples which resemble the target shape (i.e. their stationary distribution converges to the target distribution), and the variation of samples gives uncertainty quantification.

Fitting a semi-parametric GMM to data was computationally challenging. Local search heuristics, for example, have weak performance guarantees \cite{dasgupta_learning_nodate}. MLE, one of the most useful methods for mixture density estimation \cite{li_mixture_1999}, was shown to be superior in terms of computational complexity and sampling properties \cite{day_estimating_1969}. Early work by Day \cite{day_estimating_1969} addressed estimating two-component GMMs with common covariance matrices using MLE, moment estimators, minimum $\chi^2$, and Bayes estimators. This was extended to multiple components and heterogeneous covariances for clustering and density estimation. MLE for GMMs is efficiently computed using the expectation-maximization (EM) algorithm \cite{dempster_maximum_1977}, which iteratively optimises \footnote{It does this by repeating two steps, i.e. the expectation (E) step and the maximization (M) step, until converge. Providing a GMM that fits the underlying structure of the data, this iterative process is guaranteed to find a local maximum of the log-likelihood function.} means, covariances, and mixing weights, converging to a local maximum of the log-likelihood \cite{jordan_hierarchical_1993, jordan_convergence_1995,balakrishnan_statistical_2017}. Efficiently estimating GMMs unlocks the potential of GMMs for a vast range of applications, notably in density estimation and clustering \cite{cathy_adaptive_2011}. 
In high-dimensional settings, GMM approximation remains theoretically valid, but practical challenges such as overfitting, computational cost, and the curse of dimensionality become more pronounced \cite{bishop_pattern_2006, scott1992multivariate}. Techniques such as regularization (penalisation) \cite{maugis_non_2011,cathy_adaptive_2011}, dimensionality reduction \cite{verbeek_mixture_2004}, and sparsity priors \cite{zhang_information-theoretic_2006,alfo_finite_2008,kruijer_adaptive_2010,tosh_lower_bounds,newman_fast_2025} have been proposed to mitigate these issues and improve density estimation performance in complex, high-dimensional spaces.

Barron and Cover \cite{barron_minimum_1991} proposed a minimum complexity density estimator based on minimum description length criterion developed by Kolmogorov \cite{kolmogorov_three_1968}, achieving near-optimal minimax rates (e.g. $\mathcal{O}\left( \log n / n \right)^{\frac{2r}{2r+1}}$ where $n$ is the number of data and $r$ is the degree of smoothness) for parametric and non-parametric models. Zeevi and Meir \cite{zeevi_density_1997} developed a convex combination of basis densities with MLE, providing error bounds decomposed into approximation and estimation errors. Dasgupta \cite{dasgupta_learning_nodate} introduced a provably correct algorithm for learning GMMs under separation conditions, projecting data to a low-dimensional subspace for clustering. Li and Barron \cite{li_mixture_1999} provided theoretical guarantees for GMM approximation and convergence, achieving log-likelihood bounds of $\mathcal{O}(1/N)$ where $N$ is the number of mixture components\footnote{In many literature, e.g., \cite{li_mixture_1999} and \cite{lindberg_estimating_2024}, the number of mixture components is denoted by $k$ or $K$; in this work, we used $K$ to denote the total number of gradient descent iterations, and $k$ denoting each iteration number. To avoid overloading of symbols, we use $N$ to denote the number of Gaussian components. In some literature, e.g. \cite{bubeck_convex_2015}, the number of iteration is denoted by $t$.}. Genovese and Wasserman \cite{genovese_rates_2000} analyzed GMM sieve estimators for a true density which is itself a GMM without finite support, deriving minimax rates in Hellinger distance for Gaussian mixture sieve estimators. Lindberg et al.\ \cite{lindberg_estimating_2024} develop a method for estimating parameters of GMMs using sparse polynomial moment systems, proving that a generic $N$-component GMM in $\mathbb{R}^n$ is identifiable from its first $3N+2$ moments and introducing a homotopy continuation algorithm that outperforms the EM algorithm in high-dimensional settings.

Finite mixtures of asymmetric distributions such as gamma distributions are flexible alternatives to classic Gaussian mixtures. Young et al. \cite{young_finite_2019} developed an expectation-conditional-maximization (ECM) algorithm for estimating mixtures-of-gamma distributions.

\paragraph{Convergence rates of GMM approximation and estimation} 
Many literature have been focused on the convergence analysis of GMM approximation. 
Norets \cite{norets_approximation_2010} established that finite mixtures of Gaussian regressions can approximate any conditional density with bounded support under weak regularity conditions, achieving convergence in KL distance as the number of mixture components increases, with bounds dependent on the partition fineness and standard deviation decay. An approximation error bound for a mixture model $\mathcal{M}_0$ with $N$ components approximating $F$ is derived \footnote{This formula can be found in \cite{norets_approximation_2010} Page 9, Corollary 2.1, Part (iii).} \cite{norets_approximation_2010}: $KL(F,\mathcal{M}_{0})\le c\cdot(\frac{1}{N})^{1/(d\cdot[2+1/(q-2)+\varepsilon])}$, where $d$ is the dimension, $\varepsilon > 0$ can be arbitrarily close to zero, constant $c$ doesn't rely on $N$, $q>2$ is related to the moment conditions of the target density (reflecting its smoothness and tail behaviour). This error bound is polynomial in the number of component $N$, i.e. the error bound decreases as a power of $N$.

Li and Barron analyzed the convergence rates of GMMs by framing the problem in terms of approximating a target density with a convex combination of components \cite{li_mixture_1999,li_estimation_1999}. They demonstrated that for a target density within the convex hull of the Gaussian family, the approximation error, measured by KL divergence, decreases at a rate of $\mathcal{O}(1/N)$, with $N$ being the number of components. For targets outside this class, the error is the best achievable approximation error plus a term of the same $1/N$ order. They also proposed a greedy estimation algorithm that iteratively adds one component at a time, showing its log-likelihood also converges to the optimal achievable log-likelihood with a difference of order $\mathcal{O}(1/N)$ \cite{li_mixture_1999}. Their risk analysis formalizes the trade-off, where the number of components $N$ balances the approximation error ($\sim 1/N$) and the estimation error ($\sim N \log(M)/M$, where $M$ is the total number of samples).

Genovese and Wasserman \cite{genovese_rates_2000} investigated the convergence rate of density estimation using a Gaussian mixture sieve, where the number of components $N$ increases with the sample size $M$. Assuming the true density is itself a Gaussian mixture without a compactly supported mixing measure, they established a convergence rate in Hellinger distance of order $(\log M)^{(1+\eta)/6}/M^{1/6}$ (with $\eta > 0$), when using $N \sim M^{2/3}/(\log M)^{1/3}$ components (note the rate relies on the number of samples $M$ while the the number of components $N$ is linked to $M$). They further showed that \cite{genovese_rates_2000} by using a robust sieve which includes a long-tailed component in the mixture, the approximation error can be improved to $\mathcal{O}(\log N / N)$, leading to an overall convergence rate of $(\log M / M)^{1/4}$ using $N \sim \sqrt{M/\log M}$. For non-compactly supported measures, the rates depend heavily on the tail behaviour of the true density. These results complement the independent work of Li and Barron \cite{li_mixture_1999} which also achieves similar rate in Kullback-Leibler distance (corresponding to the $(\log M / M)^{1/4}$ rate in Hellinger).

Classical approximation theory \cite{Devore2010constructive} shows that a compactly supported $\beta$-Hölder function can be approximated at rate $N^{-\beta}$ by a suitable linear combination of shifted kernels, provided $\sigma \sim 1/N$ and the weights are carefully chosen \cite{kruijer_adaptive_2010}.
Kruijer et al. \cite{kruijer_adaptive_2010} investigated approximating a probability density of any Hölder-smoothness by location-scale mixtures, and showed that, under local Hölder-smoothness and mild tail conditions, there exists a finite mixture with non-negative weights that approximates the target density in KL distance and $L^2$ distance at order\footnote{This formula can be found in \cite{kruijer_adaptive_2010} Page 5, Theorem 1.} $\mathcal{O}(\sigma^{2\beta})$, and hence, choosing $\sigma \sim N^{-1}$, an approximation error of order $N^{-2\beta}$. This result covers a broad class of kernels and applies to densities with exponential or polynomial tails.

Maugis and Michel \cite{cathy_adaptive_2011} proposed a penalized maximum likelihood estimator for GMMs which adapts to the smoothness of the target density, achieving optimal convergence rates, dependent on the smoothness parameter $\beta$, for $\beta$-Hölder smooth densities. It establishes a polynomial approximation rate \footnote{This rate can be found in \cite{cathy_adaptive_2011} Page 6, Theorem 2.} in terms of the component variance $\sigma$ rather than the number of components $N$: $KL(f,\rho_{\sigma})\le c_{\beta}\sigma^{2\beta}$, with the number of components $N< G_{\beta} \sigma^{-1}|ln~\sigma|^{\frac{3}{2}}$ required to achieve this, where $G_{\beta}$ is a positive constant depending on the smoothness $\beta$ parameter of $f$, $\sigma$ is the variance of each Gaussian component. We can infer a polynomial convergence rate with respect to $N$, as the number of components $N$ is roughly proportional to $\sigma^{-1}$ \cite{kruijer_adaptive_2010}, we can say $\sigma \approx N^{-1}$, which hints the polynomial rate $\mathcal{O}((N^{-1})^{2\beta}) = \mathcal{O}(N^{-2\beta})$.

Using Wasserstein distances, Nguyen \cite{nguyen_convergence_2013} studied convergence behaviour of latent mixing measures that arise in finite and infinite mixture models, and established the convergence rates of posterior distributions for latent mixing measures for both finite mixtures of multivariate distributions and infinite mixtures based on the Dirichlet process.
The posterior contraction rate\footnote{This rate can be found in \cite{nguyen_convergence_2013} Page 14, Theorem 5.} of the mixing measure under the $L_2$ Wasserstein metric behaves as $(\log n)^{1/2}n^{-1/4}$ for finite mixtures, where $n$ denotes the number of observed data points. For Dirichlet process mixtures, the rate depends on both the dimension $d$ of the parameter space and the regularity of the likelihood: for ordinary smooth likelihood functions with $\beta$ smoothness (e.g. for Laplace kernels), the convergence rate\footnote{This rate can be found in \cite{nguyen_convergence_2013} Page 17, Theorem 6.} is $\left( \log n/n \right)^{\frac{2}{((d+2)(4 + (2\beta + 1)d'))}}$ with exponent dependent on the smoothness parameter $\beta$, the dimension $d$ and a constant $d' > d$; for super-smooth likelihood functions (e.g. Gaussian densities), the rate becomes $(\log n)^{-1/\beta}$. These results show that the speed at which the posterior over the mixing measure concentrates around the true mixing measure is governed by the sample size $n$, the dimension $d$, and the smoothness $\beta$ of the likelihood, but is not directly related to the number of mixture components.

For GMMs estimation, the EM algorithm is efficient and widely used \cite{balakrishnan_statistical_2017}. 
Balakrishnan et al. \cite{balakrishnan_statistical_2017} analysed both the standard and first-order EM algorithms at both the population and finite sample levels, and developed a theoretical framework for quantifying when and how quickly EM-type iterates converge to a small neighborhood of a given global optimum of the population likelihood. Their analysis guarantees good behaviour of the EM and gradient EM algorithms when suitable initialisation is given \footnote{For example, one may assume the EM algorithm is initialized in a neighborhood of the true center \cite{zhao_statistical_2020}.}.
The local convergence rate of using EM and its variant (e.g. gradient EM) generally depends on the initialisation configuration, mixing coefficients, minimum $R_{\min}$ and maximum $R_{\max}$ pairwise distances between the true centers, dimensionality $d$ and number of components $N$ \cite{yan_convergence_2017,zhao_statistical_2020}.
Yan et al. \cite{yan_convergence_2017} established linear convergence of the gradient EM algorithm given certain minimum sample size.
Zhao et al. \cite{zhao_statistical_2020} studied the convergence behaviour of using EM algorithm to estimate GMMs with an arbitrary number of mixture components and mixing weights, and showed that, as long as the means of the components are separated by at least \footnote{$\Omega(\cdot)$ denotes asymptotic lower bound up to constant factors, i.e. at least on the order of ($\cdot$).} $\Omega(\sqrt{\min\{N, d\}})$, the EM algorithm converges locally to the global optimum of the log-likelihood, and the convergence rate is linear.
To effectively estimate GMM parameters, minimum separation between the Gaussian components and minimum sample size are required depending on the assumptions, method and applications. The minimum distance between the $N$ Gaussian centers is required to be e.g. $\Omega(\sqrt{(N \log d})^{1/4})$ in \cite{vempala_spectral_2004}, $\Omega(\sqrt{\min\{N, d\}})$ in \cite{zhao_statistical_2020}, and $\Omega(\sqrt{\log N})$ in \cite{segol_improved_2021}. The minimum sample size $S$ required typically relies on $N$, $d$, $R_{\min}$ and $R_{\max}$, e.g. $\frac{S}{\log S} \geq C \frac{N d R_{\max}^2}{\kappa^2 R_{\min}^2}$ from \cite{zhao_statistical_2020} using isotropic Gaussians (where $\kappa$ is the smallest mixing weight), $S \geq C \frac{N^6 R_{\max}^6 d}{R_{\min}^2}$ from \cite{yan_convergence_2017},  $\frac{S}{\log S} \propto \mathcal{O}(Nd R_{\max}^2/R_{\min}^2)$ from \cite{segol_improved_2021}, etc. The variations in the minimum sample size are due to differences in the concentration results that appear in their proofs \cite{zhao_statistical_2020}. A good summary of the separation distance and sample size requirements can be found in e.g. \cite{segol_improved_2021}.

\paragraph{Convergence rate of gradient-based optimisation} 
We consider the setting where a function $f(w)$ is optimised over $k=1,2,...,K$ iterations using first-order (gradient) information, aiming to minimize the absolute or expected gap $\epsilon$ between the current and minimum function values. For a convex objective function $f$, the convergence of stochastic gradient descent (SGD) depends on the step-size sequence $\{\eta_k\}$ where $\eta_k$ is the step size (learning rate) at iteration $k$. Robbins and Monro’s stochastic approximation theory \footnote{ Robbins and Monro provided the mathematical proof that an iterative process using noisy estimates of a gradient (e.g. sample average) can converge to a true minimizer under certain conditions on the step size (learning rate) $\eta$.} \cite{robbins_stochastic_1951} requires that the learning rates satisfy $\sum_{k=1}^\infty \eta_k = \infty$ but $\sum_{k=1}^\infty \eta_k^2<\infty$. When the objective is strongly convex, Bottou et al. \cite{bottou_optimization_2018} show that choosing a diminishing step-size $\eta_k \propto 1/k$, to overcome any oscillatory behaviour of SGD, yields an error decay rate \footnote{This bound relies on the variance of the stochastic gradients, the Lipschitz constant, the strong convexity parameter, and initialisation, etc, see Theorem 4.7, Page 28 in \cite{bottou_optimization_2018}.} $\mathcal{O}(1/k)$. When strong convexity is absent, the classical stochastic approximation can still attain an $\mathcal{O}(K^{-1/2})$ rate \cite{nemirovski_robust_2009} for general convex objectives, albeit worse than the rate $\mathcal{O}(1/k)$ as in the strongly convex case. Further, the rate $\mathcal{O}(K^{-1/2})$ is guaranteed for constant step size of the form $\eta=c/\sqrt{K}$ where $c>0$ is a constant.

In deterministic optimisation, the error is often measured by $\epsilon_k = f(w_k)-f(w^*)$ where $w^*$ is the minimiser. For smooth convex functions with Lipschitz-continuous gradients, the steepest-descent (standard gradient descent) method exhibits a sublinear
\footnote{Linear convergence: $f(x_k) - f^* \leq C \rho^k$, with $0 < \rho < 1$, which implies the error shrinks by a constant proportion at each step. 
Sublinear convergence: $f(x_k) - f^* \leq \frac{C}{k^p}$, with $p > 0$, which means error decays polynomially rather than exponentially.
In general, for vanilla gradient descent with step size $\eta$ applied to a convex, $\beta$-smooth function: 
if the function is strongly convex, the convergence is linear \cite{bottou_optimization_2018}, i.e. $\mathbb{E}[f(x_k) - f^*] = \mathcal{O}(\rho^k)$ for some $\rho < 1$;
if the function is convex but not strongly convex, the convergence is sublinear, i.e. $\mathbb{E}[f(x_k) - f^*] = \mathcal{O}(1/k)$ or $\mathcal{O}(1/\sqrt{k})$. 
Sublinear convergence implies early iterations make relatively large improvements, and progress slows as it gets closer to the optimum. More iterations are needed to achieve high-accuracy solutions compared to linear or superlinear methods.}
$\mathcal{O}(1/k)$ rate \cite{bottou_optimization_2018}, i.e. the distance to the optimum under gradient descent decays as $1/k$. This rate is provably non-optimal, no first-order method using only gradient information can converge faster than $\mathcal{O}(1/k^2)$ \cite{nesterov1983method}. Nesterov’s accelerated gradient method, which uses a momentum-type extrapolation step $\tilde{w}_k = w_k + \beta_k(w_k - w_{k-1})$ followed by a gradient update at $\tilde{w}_k$, achieves the optimal $\mathcal{O}(1/k^2)$ rate \cite{nesterov1983method,nesterov_lectures_2018,bottou_optimization_2018}, while standard gradient descent converges with a distance to the optimal value decaying with a rate $1/k$. This accelerated rate cannot be improved in general \cite{nesterov1983method,bottou_optimization_2018}, and in stochastic settings the acceleration can at best improve constants but not the $\mathcal{O}(1/k)$ rate of SGD.

For constrained convex problems $\min_{w\in C} f(w)$ with a closed convex set $C$, projected gradient descent (pGD) performs a gradient step \footnote{If the gradient does not exist, we can replace it by a subgradient. Strictly speaking, the projected (sub)gradient descent (pGD) method is not a descent method, as we don't necessarily have $f(x_{t+1}) \leq f(x_t)$ in each iteration \cite{bubeck_convex_2015}.} followed by an orthogonal projection onto $C$. Because the projection only enforces feasibility, pGD inherits the same complexity as unconstrained gradient methods. Bubeck \cite{bubeck_convex_2015} shows that, when $f$ is convex with Lipschitz-continuous gradients ($\beta$-smooth), pGD with a constant step size $\eta = 1/\beta$ achieves a sublinear decrease rate \footnote{This rate can be found in \cite{bubeck_convex_2015} Page 43, Theorem 3.7, and Page 43, Theorem 3.7.} $\mathcal{O}(1/k)$, with the constant depending on the smoothness $\beta$ and starting point; if the objective is only Lipschitz (possibly non-smooth), the projected subgradient method with an appropriate diminishing step size ($\eta \propto k^{-1/2}$) reduces the gap $f(w_k)-f(w^*)$ at the slower rate \footnote{This rate can be found in \cite{bubeck_convex_2015} Page 37, Theorem 3.2.} $\mathcal{O}(k^{-1/2})$. Finally, when $f$ is both $\alpha$-strongly convex and $\beta$-smooth, pGD with a fixed step size of $1/\beta$ enjoys a convergence rate \footnote{This rate can be found in \cite{bubeck_convex_2015} Page 51, Theorem 3.10.} $\mathcal{O}(\exp{-k/\kappa})$, i.e. the distance between the iterates and the minimiser contracts at an exponential rate determined by the condition number $\kappa=\beta/\alpha$ (multiplied by a factor dependent on initialisation); if the function $f$ is $\alpha$-strongly convex and $L$-Lipschitz, then pGD with a diminishing step size $\eta \propto 1/k$ yields the rate \footnote{This rate can be found in \cite{bubeck_convex_2015} Page 50, Theorem 3.9.} $\mathcal{O}(1/k)$.  These results show that pGD matches the known $\mathcal{O}(1/k)$ and $\mathcal{O}(1/\sqrt{k})$ rates of gradient descent and subgradient methods for smooth and non-smooth convex objectives, while retaining linear convergence in the strongly convex case.

\subsection{Sampling methods}
Sampling from a prescribed distribution $p(\mathbf{z})$, either fully or partially known, is not a easy task \cite{Bayesian_signal_processing_Joseph}. When the full PDF is known, methods such as inverse transform sampling \cite{Devroye1986}, rejection sampling \cite{vonNeumann1951}, importance sampling \cite{Kahn1949}, and Quasi-Monte Carlo methods \cite{Niederreiter1992} can be used to generate samples. Approximate inference methods provide a pathway to produce (quasi) samples, these include Markov chain Monte Carlo (MCMC) and variational inference (VI) approaches. MCMC methods such as Metropolis-Hastings (MH) sampling \cite{Metropolis1953}, Gibbs sampling \cite{Geman1984}, Langevin Monte Carlo (LMC) \cite{Roberts1996}, Hamiltonian Monte Carlo (HMC) \cite{Duane1987}, slice sampling \cite{Neal2003}, nested sampling \cite{skilling_nested_2004}, etc, are commonly used in sampling complex geometries. Variational inference (VI) methods such as mean-field VI \cite{Jordan1999introduction}, stochastic VI \cite{Hoffman2013}, black-box VI \cite{Ranganath2014}, variational autoencoders (VAEs) \cite{kingma_auto-encoding_2022}, etc, are widely used in many applications \cite{Jordan1999introduction,Roberts2002AR}. Particle-based VI (ParVI) methods, such as Stein variational gradient descent (SVGD) \cite{Liu2016SVGD}, Sequential Monte Carlo (SMC) \cite{Doucet2001}, particle-based energetic VI (EVI) \cite{Wang2021EVI}, electrostatics-based ParVI (EParVI \cite{huang2024EParVI}), material point method based ParVI (MPM-ParVI \cite{huang_variational_2024-1}), smoothed particle hydrodynamics based ParVI (SPH-ParVI \cite{huang_variational_2024}), are efficient and less affected by the curse of dimensionality. Some methods (e.g. slice sampling, Metropolis-Hastings) are gradient-free, while others, e.g. HMC, LMC and most VI methods, utilise the gradient information $\nabla_{\mathbf{z}} p(\mathbf{z})$ to guide sampling, accommodating intractable densities in Bayesian inference at higher computational cost. Deterministic methods \footnote{Physics-based ParVI methods such as EParVI, SPH-ParVI and MPM-ParVI are also deterministic.}, such as integrated nested Laplace approximations (INLA) \cite{Rue2009INLA}, offer fast inference for latent Gaussian models, achieving high accuracy with reduced computational cost compared to MCMC \footnote{INLA has proven particularly effective for models with structured additive predictors and latent Gaussian random fields, where traditional MCMC methods may be too slow or impractical.}.

\textit{Stratified sampling}, as a special Monte Carlo random sampling, reduces variance, compared to simple random sampling, via sub-region sampling. It divides the sample space into non-overlapping strata, samples from each (proportionally to the strata probability), and combines the weighted results, preserving unbiasedness while reducing variance by ensuring all regions are represented. For example, we want to estimate $\mathbb{E}[Y]$ for a real-valued random variable $Y$, we partition the sample space into $N$ disjoint strata $A_1, \dots, A_N$ such that $P\left( Y \in \bigcup_{i=1}^N A_i \right) = 1$. If we denote $p_i=P(Y\in A_i)$ as the weight of strata $i$, $\mu_i=\mathbb{E}[Y\mid Y\in A_i]$, $\sigma_i^2=\mathrm{Var}(Y\mid Y\in A_i)$ be the population mean and variance of strata $i$, and overall variance $\sigma^2$, it is shown that (see Appendix.\ref{app:simple_and_stratified_random_sampling}), $E(\bar{Y}) = \sum_{i=1}^N p_i \mu_i = \mu$ and $\mathrm{Var}(\bar{Y}) \approx \sum_{i=1}^N \frac{p_i^2 \sigma_i^2}{n_i}= \frac{1}{n}\sum_{i=1}^N \frac{p_i^2 \sigma_i^2}{\alpha_i}$ where $\alpha_i=\frac{n_i}{n}$, $n_i$ being the number of samples drawn from strata $i$, and $n=n_1+n_2+...+n_N$ is the totoal number of samples being drawn.
Compared to simple random sampling which draws $Y_1, \dots, Y_n$ i.i.d. from the sample space, stratified sampling pre-specifies the fraction of samples to draw from each $A_i$,
then generate them from the conditional distribution $p(Y \mid Y \in A_i)$.
This guarantees that each stratum is represented according to $p_i$, preserving unbiasedness while typically reducing variance \cite{glasserman_monte_2003}.

\paragraph{Convergence rate of Monte Carlo sampling}
To quantify the Monte Carlo sampling error, we can view finite sample approximation of a target density as an empirical measure obtained from $S$ \textit{i.i.d.} samples. Classical Monte Carlo analysis shows \cite{mcbook} that, if we estimate an expectation $\mu=\mathbb{E}[Y]$ by the average $\hat{\mu}_S = \frac{1}{S}\sum_{i=1}^S Y_i$, where $\{Y_i\}_{i=1}^S$ are $S$ \textit{i.i.d.} samples from a distribution with mean $\mu$ and variance $\sigma^{2}$, then its root-mean-squared error obeys $\sqrt{\mathbb{E}[(\hat{\mu}_S-\mu)^2]}=\sigma/\sqrt{S}$, where $\sigma^2$ is the variance of $Y$. Equivalently, the sample mean $\hat{\mu}_S$ converges to $\mu$ at an $\mathcal{O}(S^{-1/2})$ rate. This $S^{-1/2}$ rate is very slow but independent of dimensionality \cite{mcbook}. It balances accuracy and efficiency: improving precision by a factor of ten requires a hundred-fold increase in sample size $S$. A similar conclusion holds in the variance analysis of Monte Carlo integration \cite{Veach1997robust}: for an integrable function $f$, the variance of the estimator $F_S=\frac{1}{S}\sum_{i=1}^S f(X_i)$ scales as $\sigma[F_S] = \sigma[f(X_1)]/\sqrt{S}$, i.e. the standard deviation (RMS error) decays as $S^{-1/2}$. \textit{Chebyshev’s inequality} and the \textit{central limit theorem} imply \cite{Veach1997robust} probabilistic bounds on the absolute error $|F_S - \mathbb{E}[f(X_1)]|$, showing that for any fixed threshold $\epsilon$, the probability that the absolute error exceeds $\epsilon$ decreases at a rate of order $S^{-1/2}$.

Convergence of stratified sampling is $\mathcal{O}(S^{-1/2})$, which is observed from the variance $\mathrm{Var}(\bar{Y})$ of a stratified sampler. This asymptotic convergence rate is the same as standard Monte Carlo \footnote{Also same rate as simple random sampling, but typically with a smaller constant. See Appendix.\ref{app:simple_and_stratified_random_sampling}.}, but with a reduced variance constant (see Appendix.\ref{app:simple_and_stratified_random_sampling}). In particular, stratified sampling with Neyman optimal allocation minimizes variance by distributing samples in proportion to both the stratum probability and its standard deviation, i.e. $n_i \propto p_i\sigma_i$, yielding potentially substantial efficiency gains over naive random sampling.

These general results extend to density estimation: the empirical distribution obtained from $S$ \textit{i.i.d.} samples converges to the true density at the same $S^{-1/2}$ rate in integrated distances such as the $L^1$ or total-variation norm, because the mean of any integrable function under the empirical distribution differs from its true mean by $\mathcal{O}(S^{-1/2})$. Specifically, when approximating $p(\mathbf{z})$ by a Gaussian mixture with $N$ components and $M$ samples per component (total number of samples $S=NM$), the sampling error (in expectation) scales as $\mathcal{O}\big((NM)^{-1/2}\big)$.  This reflects the fundamental Monte Carlo convergence limit: variance reduction techniques or quasi-Monte Carlo sequences \cite{practicalqmc} may improve the constant but cannot improve the $\mathcal{O}(S^{-1/2})$ asymptotic rate.

\subsection{Sampling with GMMs}
Sampling from GMMs is pivotal in generative modeling \cite{bishop_pattern_2006} and clustering \cite{reynolds_gaussian_2009}, with diverse methods leveraging GMMs' multi-modal nature to generate samples. 
For a mixture distribution, it has the discrete categorical weights. To sample from a GMM, traditionally one would first choose a Gaussian component with probability proportionally to its weight (i.e. the assignment), and then draw samples from the chosen component. Morningstar et al. \cite{morningstar_automatic_2020} used the idea of \textit{stratified sampling} to sample GMMs, they draw one sample from each component individually
and compute a weighted average over the sample, producing a sample drawn from the GMM. This is justified by the principle often used in mixture-based variational inference that, the expectation under a mixture distribution is the weighted sum of expectations under the individual components \cite{morningstar_automatic_2020}:
\begin{equation} \label{eq:stratified_sampling_morningstar}
\mathbb{E}_{q(z)}[f(z)] = \int f(z) \left( \sum_{k=1}^K \alpha_k q_k(z) \right) dz
= \sum_{k=1}^K \alpha_k \int f(z) q_k(z) dz
= \sum_{k=1}^K \alpha_k \mathbb{E}_{q_k(z)}[f(z)]
\end{equation}
The expectation can then be conveniently done using the reparameterization trick, if the component distributions (e.g. Gaussians) are reparameterizable. 

Classic sampling methods, e.g.  MH \cite{metropolis_equation_1953}, HMC \cite{duane_hybrid_1987,homan_no-u-turn_nodate}, LMC \cite{Roberts1996} and ParVI \cite{Liu2016SVGD}, can be used to generate samples from GMMs. However, they need to evaluate the GMM density or its gradient, which can be expensive.
Variational inference (VI) methods provide an efficient tool for approximating GMMs, they can outperform traditional MCMC technique (e.g. HMC) even for small datasets \cite{kucukelbir_automatic_2015,blei_variational_2017}. MFVI \cite{jordan_introduction_1999,xing_generalized_2002,blei_variational_2017}, for example, simplifies complex dependencies among variables to enable efficient inference and learning. MFVI can be framed as a tractable alternative to exact or sampling-based inference in models such as Boltzmann machines or belief networks \cite{jordan_introduction_1999}. 
Xing et al. \cite{xing_generalized_2002} propose a generalized mean-field (GMF) algorithm for exponential family graphical models, extending the traditional MFVI by clustering variables into disjoint subsets, which allows for more structured approximations while preserving convergence guarantees and lower bounds on likelihood. These MFVI frameworks are applicable to GMMs.
Blei et al. \cite{blei_variational_2017} provide a detailed statistical exposition of MFVI with a specific worked example on Bayesian Gaussian mixture models. They demonstrated how the posterior over mixture component means and assignments can be approximated by fully factorized variational distributions, and explained how coordinate ascent can be used to optimise the evidence lower bound (ELBO). This work emphasizes practical implementation, scalability, and the accuracy-efficiency tradeoff compared to MCMC. 

Automatic differentiation variational inference (ADVI \cite{kucukelbir_automatic_2016,kucukelbir_automatic_2015}) provides another efficient tool for learning probabilistic models, particularly for fitting large data. For GMMs, ADVI approximates the posterior distribution of the GMM parameters $p(\boldsymbol{\theta} | \boldsymbol{z})$ using a variational posterior $q(\boldsymbol{\theta})$ (typically a mean-field Gaussian in the transformed real coordinate space), which minimizes the Kullback-Leibler (KL) divergence or maximizes the evidence lower bound (ELBO). ADVI automatically derives the efficient variational inference algorithm, and doesn't require conjugacy. It does not directly sample from the individual GMM components $\mathcal{N}(\mu_k, \sigma_k)$ as the component parameters are unknown; instead, it samples from the variational posterior $q(\boldsymbol{\theta})$. It uses Monte Carlo gradient estimates and mini-batch sub-sampling, which can introduce high ELBO variance, particularly in high-dimensional settings with many components or high-dimensional data. ADVI is well integrated into modern probabilistic programming languages such as Stan \cite{carpenter_stan_2017} and PyMC \cite{salvatier_probabilistic_2015} for convenient use. However, like traditional VI methods, ADVI are forced to be unimodal in order to facilitate use of the reparameterization trick \footnote{VI methods are restricted to 'reparameterizable' distributions for which such a transformation from the base distribution to the target exists. The choice of posterior in ADVI thus is limited \cite{morningstar_automatic_2020}.} \cite{morningstar_automatic_2020}.

Diffusion models, particularly score-based generative models (SGMs \cite{song_generative_2019,song_maximum_2021,song_score-based_2021,huang_classification_2022}), have gained attention as a powerful approach for sampling from complex distributions such as GMMs, especially when mixture components are not well separated. These models learn the score function $\nabla_{\mathbf{z}} \log p(\mathbf{z})$ from data, allowing them to sample from multi-modal and high-dimensional distributions without relying on strong separation assumptions \cite{chen_sampling_2023,gatmiry_learning_2025}. Once the score function is learned, samples can be generated using either Langevin dynamics \cite{song_generative_2019} or stochastic differential equations (SDEs) \cite{song_score-based_2021}, both of which exploit the score estimates to guide the sampling process. Under certain regularity conditions, such as bounded log-Sobolev constants, SGMs can sample with polynomial complexity \cite{chen_sampling_2023}. However, learning an accurate score function from data is itself a computationally intensive task and remains a key challenge in practice. Despite this, diffusion-based sampling provides a robust and flexible alternative to traditional methods for sampling GMMs, particularly for overlapping modes or high dimensions.

The GMA method proposed in this work integrates the sampling and optimisation procedures - they are done under the same hood. We first assign Gaussian components, then sample each component, and use these fixed samples in the optimisation procedure to evaluate the KL divergence objective, finally we re-sample these samples as per the optimised component weights.
Perhaps the closest methodologies to our method are the \textit{boosting variational inference} (BVI \cite{guo_boosting_2017}) method, the \textit{Gaussian mixture ADVI} (GM-ADVI \cite{morningstar_automatic_2020}) method, and the \textit{spline ADVI} (S-ADVI \cite{shao_nonparametric_2024}) method.
Guo et al. \cite{guo_boosting_2017} proposed a Laplacian gradient boosting variational inference (BVI) method which, inspired by the idea of gradient boosting, incrementally constructs a posterior approximation as a mixture of parametric distributions (e.g. Gaussians). Unlike MFVI, BVI can flexibly capture multi-modality, general covariance, and complex posterior shapes by adding mixture components iteratively to minimize the KL divergence from the true posterior. BVI starts with single base Gaussian distribution and incrementally adds new base Gaussian components to it in a greedy manner, i.e. in each iteration it searches for a new Gaussian component which most steeply (approximately) decreases the KL divergence between the GMM and the target, which involves a two stage search: optimising weights using SGD (a convex optimisation problem) and searching for the optimal component using functional gradient descent (non-convex). This iterative, boosting or greedy optimisation procedure is guaranteed \cite{zhang_sequential_2003} to converge at rate $\mathcal{O}(1/k)$. To find the new, optimal weight and base distribution, BVI alternately optimises both: the new base distribution is sought via functional gradient boosting \footnote{Seeking optimal $\boldsymbol{\mu}, \boldsymbol{\Sigma}$ for a new Gaussian basis distribution is a non-convex optimisation problem. \cite{guo_boosting_2017} optimises a variational objective, i.e. the Taylor expansion of the KL divergence at fixed weights, following its functional gradient - essentially the expected log gap between the GMM and the unnormalised target (the \textit{residual log-density}), plus a regularisation term preventing degeneracy. See Eq.(17) on Page 8 in \cite{guo_boosting_2017}.}, which additively perturbs a current solution to minimize the exclusive KL divergence objective; optimal weights are found by optimising the convex, exclusive KL divergence objective given the chosen distribution. During weights optimisation, to satisfy the \textit{Robbins-Monro conditions} \cite{robbins_stochastic_1951}, the step size is chosen to be $c/k$.

BVI aims to find a flexible GMM which expressively approximates the Bayesian posterior, as conventional VI candidate family, e.g. Gaussian, is limited by their approximation power e.g. enforcing uni-mode when approximating multi-modal densities \footnote{For example, Laplace’s approximation analytically approximates the posterior distribution by fitting a Gaussian centered at the MAP estimate, with its precision given by the observed Fisher information at that point. Depending on the fitting objective (e.g. forward or inverse KL divergence), there are mode-seeking and mode-averaging behaviours.}. The candidate GMM is obtained by incrementally adding new Gaussian components to it with minimum KL divergence criteria (this iterative approach for constructing GMM was also used by Li \& Barron \cite{li_mixture_1999}). However, there are major issues with BVI: first, high computational cost involved. Alternately optimise the weights $\{w_i\}$ and Gaussian component parameters (means $\boldsymbol{\mu}_i$ and variances $\Sigma_i$) is computationally expensive: the weight optimisation step is a convex optimisation problem, and can be efficiently solved by SGD; optimizing the means $\boldsymbol{\mu}_i$ and variances $\Sigma_i$, however, is a non-convex optimisation problem, and heuristic, approximate optimisation approach which involves computing the Hessian of the log residual-density for Laplace approximation, using numerical approximation such as finite difference, is adopted in \cite{guo_boosting_2017}. 
When augmenting the GMM, BVI iteratively samples from the previous GMM, evaluates the residual log-density at these sample positions, finds the optimal mean and variance for the new Gaussian component, fix the new component and optimise weights, and repeat this alternating optimisation procedure till convergence. In each iteration, BVI has to draw samples from the previous GMM as the GMM is dynamically evolving, which is expensive, and these samples are discarded after each iteration. Also, optimising the approximate, Taylor expanded KL divergence objective with heuristics (local Laplace approximation) may require extra efforts (e.g. manually stabilising the tails of the log residual-density) and more iterations to converge. Moreover, The computational cost increases polynomially with the dimension. There are shortcuts for Hessian, e.g. using a diagonal approximation, which reduces the cost to linear dimension dependence, however, it sacrifices for slow convergence. 
Second, BVI has some hyper-parameters such as the regularisation parameter $\lambda$ and the initial base Gaussian to be chosen.
Therefore, while BVI offers a flexible mixture-based approximation, it also poses implementation and interpretation challenges \cite{kobyzev_normalizing_2021,shao_nonparametric_2024}. Theoretical properties of BVI such as convergence were studied by Locatello et al. \cite{locatello_boosting_2018}.

Morningstar et al. \cite{morningstar_automatic_2020} proposed the GM-ADVI method, which approximates a posterior distribution using a GMM. To train this GMM, they use a stratified ELBO (SELBO) objective. The ELBO, and by extension SELBO, is equivalent to maximizing a lower bound on the log-likelihood of the data and can be expressed as:
\begin{equation} \label{eq:SELBO}
    \mathcal{L}_{\text{SELBO}} (\theta) := \sum_{i=1}^{N} w_{i} \cdot \mathbb{E}_{\mathbf{z}_i \sim q_{i,\theta}(\mathbf{z})} \left[ \log \left( \frac{\bar{p}(\mathbf{z})}{q_{i,\theta}(\mathbf{z})} \right) \right]
\end{equation}
which allows fitting a mixture posterior using ADVI. 
However, because maximizing the ELBO is equivalent to minimizing the exclusive KL divergence, it exhibits mode-seeking behavior, and low-density regions between modes can block exploration. To avoid this component collapse and encourage exploration, Morningstar et al. applied the importance weighted estimate of the evidence (the importance weighted autoencoder \textit{IWAE}, originally proposed by Burda et al. \cite{burda_importance_2016}), which weight the unlikely samples less and produces higher-variance posterior to better cover the space, to the case of mixture posteriors, and derived the stratified-IWAE (SIWAE) objective:
\begin{equation} \label{eq:SIWAE}
\mathcal{L}_{\text{SIWAE}} (\theta) := 
\mathbb{E}_{\{z_{i,j} \sim q_{i,\theta}(\mathbf{z})\}_{i=1,j=1}^{N,M}} 
\left[ 
\log \left( \frac{1}{M} \sum_{j=1}^{M} \sum_{i=1}^{N} 
w_i \cdot \frac{\bar{p}(\mathbf{z}_{i,j})}
{q_{\theta}(\mathbf{z}_{i,j})} 
\right) 
\right]
\end{equation}
$\mathcal{L}_{\text{SIWAE}}$ is a tighter lower bound on the evidence than the IWAE and SELBO objectives when the number of components $N>1$. Importance sampling with the SIWAE objective allows the learned posterior to have higher variance, which helps it better cover the space and capture multiple modes. Both SELBO and SIWAE are designed to be simple augmentations of existing variational inference code \cite{morningstar_automatic_2020}.
Once the optimization is complete, they use the learned GMM to generate final samples. This is done by randomly choosing a Gaussian component based on its learned weight (i.e. choosing a component index), then drawing a single sample from the chosen Gaussian component, and repeating this process to generate the desired number of samples. 
One may note that, our exclusive KL divergence objective Eq.\ref{eq:KL_divergence_objective4} is the same as the negative SELBO in Eq.\ref{eq:SELBO}. Therefore, following the (projected) gradient descent (\ref{eq:KL_divergence_objective4_gradient_GMM1_MC_cc}) dynamics, GMA sampling can also lead to mode collapse, meaning that any accessible, largest discrepancy mode will be fitted, if high density regions are separated by low-density regions and initial Gaussian components do not fully cover all modes.

S-ADVI \cite{shao_nonparametric_2024} approximates a posterior using a mixture of spline basis functions, whose parameters are estimated using stochastic backpropagation \cite{graves_stochastic_2016} with a different objective (IWAE) to ours. To generate samples from the estimated spline mixture, it first uses Metropolis-Hastings (MH \cite{metropolis_equation_1953}) to sample each spline density, then randomly picks a subset samples and builds a concrete distribution based on the optimised weights, which approximates the target. By this, it gets around the two-stage sampling procedure as it is not straightforward in their case.

These GMM sampling and variational inference methods collectively offer trade-offs between computational efficiency, flexibility, and sample quality, enabling a wide spectrum of applications, from Bayesian inference to tasks such as image classification and reconstruction \cite{shao_nonparametric_2024}. 
Our GMA method differs itself from these 3 methods in that, 
first, as a sampling method, GMA features a two-stage sampling scheme with a focus on efficiently generating samples from a GMM which approximates the target, balancing sampling efficiency and sample quality (accuracy) is of priority in our design. To achieve this, we only optimise the GMM weights for fixed components, and only sample each component once. These samples are re-used during the optimisation process to evaluate the KL divergence between the current GMM configuration and the target. As a result, the one-time Gaussian sampling (Step 3 in Algo.\ref{algo:WGMA-sampling-pgd}) is linear in the number of samples $NM$ and has quadratic dependence on the dimension $d$; the stratified sampling cost at the end is linearly dependent on the number of samples $NM$. BVI, GM-ADVI, and S-ADVI all focus on approximating a Bayesian posterior following the VI paradigm, samples generated from GMM is only used to evaluate their objectives for the purpose of optimisation; they are discarded and not re-used (e.g. BVI samples the GMM in each iteration, S-ADVI uses MH to generate samples from each spline, GM-ADVI draws samples from each Gaussian component in each iteration).
Second, our optimisation objective is different from, and simpler than, that of BVI, GM-ADVI, and S-ADVI. We optimise the exclusive KL divergence (\ref{eq:KL_divergence_objective4_cc}) between fixed Gaussian components and target, evaluated at pre-sampled sample positions, to obtain optimal weights, while BVI optimises a regularised, exclusive KL divergence between a current GMM configuration and a new GMM with added component (similar to the iterative GMM optimisation method used in \cite{li_mixture_1999}), GM-ADVI optimises the SIWAE objective which is a tighter bound than both IWAE and the traditional ELBO to find the optimal GMM approximation, S-ADVI optimises the IWAE objective, a tighter log-likelihood lower bound than ELBO, to shape the spline approximation. The IWAE objective \cite{burda_importance_2016} is similar but different to our weights optimisation objective. To apply first-order opsitmisation method (e.g. projected gradient descent in our case), one needs to compute the gradient. Computing the gradient of our objective in each optimisation iteration involves only calculating the expected log gap between current configuration and target (\ref{eq:KL_divergence_objective4_gradient_GMM1_cc}), which can be conveniently approximated by an MC gradient estimator (\ref{eq:KL_divergence_objective4_gradient_GMM1_MC_cc}) evaluated at fixed sample positions. The whole optimisation procedure in our case is simpler than those in the other three methods (e.g. BVI issues an alternating optimisation scheme which involves computing Hessians of the log gap via numerical approximation). 
Third, we address that fitting a GMM to a target PDF is a constrained optimisation problem, for which projected gradient descent is used. The constrain $0 \leq w_i \leq 1$ is principally enforced in our optimisation procedure without inducing much extra cost, while other methods use heuristics to enforce this (e.g. BVI enforces summation of weights to unity without non-negativity guarantee).

While WGMA is simple and efficient, it has several limitations:
(i) Initialization sensitivity. Naive “uniform” placement of centers (e.g. linear grids or interpolants) is infeasible in high dimensions due to the curse of dimensionality: the number of components required to cover a hyper-rectangle grows exponentially with the dimension $d$. A more informed strategy is to place centers near high-score or high-curvature regions of the target, e.g. using MAP points and local precision from
$\nabla_{\mathbf{z}}\log p(\mathbf{z})$ and $-\nabla_{\mathbf{z}}^2\log p(\mathbf{z})$ (Laplace-style). Practical alternatives include: running a few short optimizations from random starts to harvest candidate modes; $k$-means on a small pilot sample (Appendix.\ref{sec:constructing_LMA_from_empirical_samples}); or low-discrepancy designs (e.g. Sobol sequence).
(ii) Limited expressiveness with fixed components.
With means/covariances fixed, WGMA can only reweight the existing components; if centers miss important regions or shapes are badly mis-specified, the best achievable fit under $\mathrm{KL}(q_{\mathbf{w}} \| p)$ may be biased (mode under-coverage or overspread mass). Good, anisotropic initialization mitigates this; further remedies include tempering, entropy regularization, component splitting, or a brief EM search to adjust locations/shapes.
(iii) High-dimensional targets.
In large $d$, distance concentration and heterogeneous scales make isotropic proposals brittle and exacerbate (i)-(ii). Although mixture-sieve theory provides rates for density estimation as $N \to \infty$, practical accuracy for complex, multi-modal posteriors may reply on e.g. multi-stage refinement.
LMA addresses (i) by placing components at multiple modes with curvature-informed (often anisotropic) covariances. EM-GMA relaxes (ii) by optimizing means and covariances jointly with weights, improving coverage when the initial bank is imperfect (the sample bank is refreshed in each iteration, see Appendix.\ref{app:em_GMA}). In practice, a Laplace-initialised start followed by weight-only WGMA, and, if needed, a brief EM polish, balances robustness and cost.

\section{Discussion}\label{sec:discussion}

\subsection{Geometry-aware preconditioning and constraints}
Preconditioning (e.g. Adam second moments or weight-decay scales) effectively whitens the parameter space: steep directions are down-scaled and flat directions up-scaled. This makes Gaussian proposals closer to isotropic in transformed coordinates, stabilizes density evaluations, and lowers MC-gradient variance under a fixed sample bank, improving convergence at negligible overhead. Box constraints (e.g. positivity of scale parameters in hierarchical BLR) are enforced naturally in the projected-gradient or mirror-descent update on the simplex.

\subsection{Placing and learning components}
\textbf{Laplace-mixture initialization.} Beyond random/uniform placements or mode perturbations, seeding components around one or two MAP modes with covariance from the observed Fisher (Laplace approximation) injects curvature information and reduces both overspreading and premature collapse. LMA can thus serve as an initialier for WGMA.

\noindent\textbf{EM refinement.} After a weight-only fit, a few EM steps can jointly adjust means, covariances, and weights to correct residual misalignment; this improves fit but adds extra cost, so we keep it optional to preserve simplicity. Our main pipeline WGMA (Algo.\ref{algo:WGMA-sampling-pgd}) optimizes weights with fixed components.
\smallskip

\noindent\textbf{Weight initialisation.} Uniform $w^{(0)}_i=1/N$ is a stable default on the simplex. Random simplex draws add diversity but may slow early progress or induce premature sparsity. We therefore use uniform weights in Step 2 of Algo.\ref{algo:WGMA-sampling-pgd}.

\subsection{Choosing \texorpdfstring{$(N,M,K)$}{(N,M,K)} and step-size schedules}
In GMA, user-defined hyperparameters include: number of components $N$, samples per component $M$, iterations $K$, centres $\{\mu_i\}$, covariances $\{\Sigma_i\}$, and a learning-rate schedule $\{\eta_k\}$.

\noindent\textbf{Learning rates $\eta_k$.} Although $\eta_k=\eta/k$ satisfies Robbins-Monro, in practice it can decay too fast for noisy, non-strictly-convex objectives and freeze weights early. Robust choices are
\[
\eta_k=\frac{\eta}{k+k_0}\ (k_0 \in [50,200]),\qquad
\eta_k=\frac{\eta}{\sqrt{k+k_0}},\qquad
\text{or a two-phase constant }\eta\ \to\ \text{decay}
\]
Mirror descent (multiplicative weights)
$\tilde w_i^{(k)} \propto w_i^{(k-1)} \exp(-\eta_k g_i^{(k)})$, $w^{(k)}=\tilde w^{(k)}/\|\tilde w^{(k)}\|_1$
removes Euclidean projection and is typically smoother. Practically, Polyak averaging smoothes oscillations.
\smallskip

\noindent\textbf{Number of components $N$.} Greedy constructions add components by minimizing $KL(q_{\theta} \| p)$ \cite{li_mixture_1999,guo_boosting_2017}; penalized-likelihood criteria with slope heuristics offer data-driven $N$ \cite{maugis_non_2011,maugis2011data_driven}. Bayesian nonparametrics, e.g. \textit{Dirichlet Process Mixture Model} (DPMM), let the effective number of components adapt to data \cite{gorur_dirichlet_2010,rinaldi_hdpgmm_2022}, while finite-mixture priors yield consistent $N$ under conditions \cite{miller_mixture_2018,ohn_optimal_2023}.

\noindent\textbf{Number of samples per component $M$.} For fixed $N$, increasing $M$ lowers MC-gradient variance (\ref{eq:KL_divergence_objective4_gradient_GMM1_MC_cc}) and often stabilizes optimisation more than increasing $N$. From a density-estimation viewpoint, allowing $N$ to grow with $n$, where $n$ is the dataset size that defines the posterior $p(\theta\mid x_{1:n})$, yields mixture-sieve rates under Hellinger loss \cite{genovese_rates_2000}.
\smallskip

\noindent\textbf{Iterations $K$.} Convergence can require nontrivial $K$, especially with noisy gradients; it also relies on the learning rate used.

\subsection{Preventing mode collapse}
\textbf{Small variances induce mode collapse.}
With fixed components $q_i$ and weights $\mathbf{w}$, the mixture is $q_{\mathbf{w}}(z)=\sum_i w_i q_i(z)$ and the objective is
$\mathrm{KL}(q_{\mathbf{w}} \| p)=\int q_{\mathbf{w}}(z) [\log q_{\mathbf{w}}(z)-\log p(z)] \mathrm{d}z$.
The weight gradient satisfies (\ref{eq:KL_divergence_objective4_gradient_GMM1_cc})
\begin{align*}
g_i
&:= \partial_{w_i} \mathrm{KL}(q_{\mathbf{w}} \| p)
= \int q_i(z) \big[\log q_{\mathbf{w}}(z)-\log p(z)+1\big] \mathrm{d}z \\
&= \mathbb{E}_{z\sim q_i}\!\big[\log q_{\mathbf{w}}(z)-\log p(z)\big] + \text{const}
\end{align*}
where the additive constant is identical across $i$ and cancels in multiplicative updates.
For isotropic Gaussians $q_i(z)=\mathcal N(z;\mu_i,\sigma^2 I)$,
\[
\log q_i(z)=-\tfrac{d}{2}\log(2\pi\sigma^2)-\tfrac{\|z-\mu_i\|^2}{2\sigma^2},
\quad
\Delta_{ab}(z):=\log q_a(z)-\log q_b(z)
= -\frac{\|z-\mu_a\|^2-\|z-\mu_b\|^2}{2\sigma^2}
\]
when $\sigma^2$ is tiny, even small distance gaps yield huge log-likelihood contrasts; consequently $g_a-g_b$ becomes large.
With multiplicative (mirror descent) updates
$w_i^{(k+1)}\propto w_i^{(k)}\exp\{-\eta  g_i^{(k)}\}$, the ratio evolves as
\[
\frac{w_a^{(k+1)}}{w_b^{(k+1)}}
=\frac{w_a^{(k)}}{w_b^{(k)}} \exp\!\big\{-\eta [g_a^{(k)}-g_b^{(k)}]\big\}
\]
so large gradient gaps drive near-binary weights in a few steps (collapse). Intuitively, overly sharp components make mass concentrated on a single narrow Gaussian and under-covering the posterior (biased, over-confident forecasts).
\textit{Practical fixes:} enlarge or anneal $\sigma^2$; temper the target early ($\beta<1$); add an entropy penalty on $\mathbf{w}$; use anisotropic (curvature-aware) covariances; and prefer increasing samples per component $M$ over shrinking the step size excessively.

\subsection{Objective choice and diagnostics}

\noindent\textbf{Inclusive \textit{vs} exclusive KL.} Minimizing the exclusive $KL(q \| p)$ is mode-seeking and may under-cover isolated modes. Approximations to the inclusive $KL(p \| q)$ are mode-covering but require sampling from $p$. In the weight-only setting (i.e. WGMA) we target the exclusive KL and mitigate under-coverage via tempering, entropy regularization (Section.\ref{subsec:stabilizers}), and LMA initialization; inclusive objectives are used in EM-GMA (Section.\ref{subsec:em_GMA}).

\noindent\textbf{Indicators to monitor.}
\textit{(a) Weight entropy.} $H(\mathbf{w})=-\sum_{i=1}^N w_i\log w_i$ detects degeneracy: $H(\mathbf{w})\approx\log N$ indicates a well-spread (uniform) bank; $H(\mathbf{w})\ll \tfrac{1}{2}\log N$ flags collapse (most mass on a few components).
\textit{(b) Top-$k$ mass.} $M_k=\sum_{i=1}^k w_{(i)}$ with $w_{(1)}\ge\cdots\ge w_{(N)}$ tracks concentration; a rapid rise signals over-concentration even if $H(\mathbf{w})$ looks acceptable.
\textit{(c) Iterative movement.} The step-to-step $\ell_1$ change $\Delta_k=\|\mathbf{w}^{(k)}-\mathbf{w}^{(k-1)}\|_1$ is a robust convergence proxy. Target $\Delta_k\in[10^{-3},10^{-2}]$ early (healthy progress); declare practical stationarity once $\Delta_k<10^{-4}$ for a stable window.

\subsection{Memory efficient implementation}
The naive code that materialized $N$ full $d\times d$ covariances (e.g. $\sim 500\times 6000\times 6000$) is infeasible. A safe implementation, as used in our experiments, never builds dense covariances and computes log-pdfs by dot products under isotropic or structured (diag/low-rank) forms. Largest arrays are the flattened sample matrix $(NM)\times d$ and the log-pdf matrix $(NM)\times N$; both fit in typical high-RAM sessions. 

\subsection{Extensions in high dimensions}
High-dimensional GMM learning has statistical query (SQ) lower bounds \cite{diakonikolas_statistical_2017}, so some hardness is inherent. Practically, we mitigate it via target-informed placements (LMA initialiser; anisotropic covariances), variance-aware schedules (tempering, annealing, entropy), and the multi-stage coarse$\to$refine runs (Section.\ref{sec:refined_GMA}). A Bayesian alternative places priors on means, covariances, and weights; hierarchical or non-conjugate priors with MCMC are viable but more costly \cite{gorur_dirichlet_2010,rinaldi_hdpgmm_2022}. Choosing covariance priors in high dimensions is delicate \cite{jing_variance_2024}.

\section{Conclusion}\label{sec:conclusion}

We developed the Gaussian Mixture Approximation (GMA) method to approximate an unnormalised target $\bar p(z)$ with a mixture $q_{\theta}(z)=\sum_{i=1}^N w_i \mathcal N(z;\mu_i,\Sigma_i)$ and optimise a chosen subset of the mixture parameters $\theta=(\mathbf{w},\{\mu_i,\Sigma_i\}_{i=1}^N)$ by minimizing the reverse KL on a \textit{bank} of samples drawn from the current components. The optimisation is VI-like in spirit, i.e. fitting within the GMM family, but the KL is evaluated at fixed sample locations from the bank, decoupling expensive re-drawing from the inner loop and keeping gradients stable. After optimisation, we form the final particles by \textit{stratified resampling} from the bank according to the learned weights; this is a sampling step that yields an empirical ensemble for posterior prediction and uncertainty quantification. Stratified resampling reduces variance relative to simple Monte Carlo by enforcing per-component allocation; see Appendix.\ref{app:simple_and_stratified_random_sampling}.

Within this blueprint, \textit{weights-only GMA (WGMA)} fixes the component locations and shapes $\{(\mu_i,\Sigma_i)\}_{i=1}^N$ and optimises only the mixture weights $\mathbf{w}\in\Delta^{N-1}$ using projected or mirror descent on the reverse KL estimated on the fixed bank. Because means and covariances are not updated, the inner loop is extremely lightweight and scales well to large models or structured parameter subsets (e.g. head-only updates in LLMs). After convergence, stratified resampling converts the fitted mixture into particles that directly support posterior summaries.

To improve the placement and anisotropy of components from the outset, we use a \textit{Laplace Mixture Approximation (LMA)} to initialise GMA near high-density regions. Components are seeded at one or more posterior modes with covariances informed by local curvature (often anisotropic, from the Hessian or Fisher information), which provides strong initial coverage and mitigates overspreading or premature collapse. \textit{It computes them in closed form from the Laplace formula at each mode}: means/covariances come from curvature, and weights come from Laplace evidence contributions (Appendix.\ref{app:LMA}). LMA provides a competent standalone approximation using Laplace-calibrated weights and curvature-informed covariances around each mode; it can also serves as a robust initializer for WGMA - a short WGMA pass can then reweight the fixed Laplace components for additional accuracy at negligible cost. LMA is useful in multi-modal or heterogeneous-scale settings.

When weights-only optimisation is too restrictive, \textit{EM-GMA} relaxes this constraint and \textit{dynamically optimises} means, covariances, and weights. A population-EM procedure computes responsibilities on a fixed or periodically refreshed bank (e.g. via self-normalised importance sampling under the current mixture) and performs M-step updates of $(\mathbf{w},\{\mu_i,\Sigma_i\})$. Periodic bank refresh realigns sample support with the evolving mixture, improves mass coverage when the initial bank is imperfect, and reduces bias from stale proposals. In practice, EM-GMA is most effective when warm-started from LMA, followed by a brief EM polish that adjusts locations and shapes before handing back to a fast weights-only refinement if desired.

All variants benefit from a common set of stabilisers and implementation choices. Geometry-aware preconditioning (e.g. using second-moment or Fisher-scaled coordinates) improves conditioning and lowers gradient variance; tempering and entropy regularisation prevent early weight collapse; Polyak tail iterate-averaging stabilises noisy gradients without biasing early exploration. For WGMA, mirror-descent updates on the simplex avoid explicit projection; for EM-GMA, covariance regularisation and eigenvalue floors ensure numerically stable M-steps. Memory-safe logpdf computations using isotropic, diagonal, or low-rank structures avoid dense $d\times d$ materialisation and make the approach practical at scale. Together, these elements make GMA a cohesive framework that trades off speed and coverage via WGMA, LMA, and EM-GMA while preserving a unified optimise-then-resample workflow.

Across a comprehensive suite of 1D and 2D multi-modal geometries, i.e. connected and isolated tri-modal bumps, four-modal Gaussians, moon, double-banana, wave, and star-shaped densities, and Neal’s funnel, WGMA matched or exceeded the accuracy of established baselines, at substantially lower time cost. We compared against 7 competitors: Metropolis-Hastings (MH), Hamiltonian Monte Carlo (HMC), Langevin Monte Carlo (LMC), Stein Variational Gradient Descent (SVGD), Mean-Field ADVI (MFVI-ADVI), and Gaussian-Mixture ADVI (GM-ADVI), in addition to our GMA methods. With appropriate hyperparameter tuning, GMA reliably captured all modes in the 1D/2D benchmarks; on the funnel and the star-shaped density it matched gradient-based samplers and outperformed variational baselines under the same sample budget.

On real-world applications, WGMA produced calibrated posteriors for \textit{Bayesian logistic regression} and \textit{hierarchical BLR} (Radon), and yielded accurate, low-overhead ensembles for \textit{Bayesian language models} when restricted to feasible parameter subsets (e.g. head-only), once bank scores were precomputed. WGMA also delivered strong fits in \textit{hierarchical Bayesian symbolic regression} for pendulum discovery and in a \textit{Bayesian LSTM} for mortality forecasting. For dynamical systems, WGMA handled the \textit{Lotka-Volterra} model effectively, while \textit{LMA} enabled efficient \textit{SIR} inference and powered \textit{Bayesian optimal experimental design} for logistic dose-response via an LMA prior. Finally, \textit{sensor-network localization} benefited from mass-covering mixtures produced by \textit{EM-GMA}, which were effective both for direct sampling and as warm starts for fast WGMA refinement in higher-dimensional settings.

On the theoretical side, we formalized this optimize-then-resample blueprint: the learned mixture \(q_{\mathbf{w}^\star}\) inherits the universal approximation power of GMMs, while the stratified resampling stage provides a law-of-large-numbers estimator of expectations under \(q_{\mathbf{w}^\star}\). An error decomposition clarifies how mixture approximation error, optimization error, and sampling error contribute to downstream risk, and guides practical diagnostics. 

\textbf{Limitations and outlook.} High-dimensional Gaussian-mixture learning is intrinsically hard without additional structure; poor bank placement or overly sharp components can induce mode under-coverage and weight collapse in WGMA. LMA and EM-GMA mitigate these issues at additional cost by learning locations and shapes. A practical recipe is to initialize with LMA, run WGMA for fast weight fitting, and optionally apply a brief EM-GMA polish when coverage matters most. Overall, GMA offers a flexible bridge between MCMC and VI, combining the simplicity and speed of variational optimization, and provides a practical solution for posterior inference, model comparison, and decision-making.

\section*{Code Availability}
All codes used in this work is available at: \url{https://github.com/YongchaoHuang/GMA}

\newpage
\bibliographystyle{plain}
\bibliography{references}

\newpage
\appendix

\section{Gradient of exclusive KL divergence} \label{app:gradient_of_exclusive_KL_divergence}

\subsection{$q_{\mathbf{w}}(\mathbf{z})$ and un-normalised $p(\mathbf{z}) = \frac{\bar{p}(\mathbf{z})}{Z_p}$}
We follow \cite{vieira2014kldivergence} to derive the gradient of the KL divergence $\nabla_{\mathbf{w}} KL(q_{\mathbf{w}}(\mathbf{z}) \| p(\mathbf{z}))$.

\begin{align*}
KL(q_{\mathbf{w}}(\mathbf{z}) \| p(\mathbf{z})) 
&= \sum_{\mathbf{z}} q_{\mathbf{w}}(\mathbf{z}) \log \left( \frac{q_{\mathbf{w}}(\mathbf{z})}{p(\mathbf{z})} \right) \\
&= \sum_{\mathbf{z}} q_{\mathbf{w}}(\mathbf{z}) [\log q_{\mathbf{w}}(\mathbf{z}) - \log p(\mathbf{z})] \\
&= \sum_{\mathbf{z}} q_{\mathbf{w}}(\mathbf{z}) \log q_{\mathbf{w}}(\mathbf{z}) - \sum_{\mathbf{z}} q_{\mathbf{w}}(\mathbf{z}) \log p(\mathbf{z})
\end{align*}
where the first term is entropy of $q_{\mathbf{w}}$, and the second term is the cross-entropy between $q_{\mathbf{w}}$ and $p(\mathbf{z})$.

Let's look at the second term. If we have $p(\mathbf{z}) = \frac{\bar{p}(\mathbf{z})}{Z_p}$, we can further derive: 
\begin{align*}
\sum_{\mathbf{z}} q_{\mathbf{w}}(\mathbf{z}) \log p(\mathbf{z}) 
&= \sum_{\mathbf{z}} q_{\mathbf{w}}(\mathbf{z}) \log \left( \frac{\bar{p}(\mathbf{z})}{Z_p} \right) \\
&= \sum_{\mathbf{z}} q_{\mathbf{w}}(\mathbf{z}) [\log \bar{p}(\mathbf{z}) - \log Z_p ] \\
&= \sum_{\mathbf{z}} q_{\mathbf{w}}(\mathbf{z}) \log \bar{p}(\mathbf{z}) - \sum_{\mathbf{z}} q_{\mathbf{w}}(\mathbf{z}) \log Z_p \\
&= \sum_{\mathbf{z}} q_{\mathbf{w}}(\mathbf{z}) \log \bar{p}(\mathbf{z}) - \log Z_p
\end{align*}

Therefore, the KL divergence objective becomes:

\begin{equation} \label{eq:KL_divergence_objective1}
KL(q_{\mathbf{w}}(\mathbf{z}) \| p(\mathbf{z})) 
= \sum_{\mathbf{z}} q_{\mathbf{w}}(\mathbf{z}) \log q_{\mathbf{w}}(\mathbf{z}) - \sum_{\mathbf{z}} q_{\mathbf{w}}(\mathbf{z}) \log \bar{p}(\mathbf{z}) + \log Z_p
\end{equation}

Since $\log Z_p$ is an additive constant, and we can drop it because we're optimizing KL divergence w.r.t $\mathbf{w}$, which gives:

\begin{equation} \label{eq:KL_divergence_objective2}
    \mathbf{w}^*
    = \arg\min_{\mathbf{w}} KL(q_{\mathbf{w}}(\mathbf{z}) \| p(\mathbf{z})) 
    = \arg\min_{\mathbf{w}} \left[ \sum_{\mathbf{z}} q_{\mathbf{w}}(\mathbf{z}) \log q_{\mathbf{w}}(\mathbf{z}) - \sum_{\mathbf{z}} q_{\mathbf{w}}(\mathbf{z}) \log \bar{p}(\mathbf{z}) \right]
\end{equation}

\paragraph{Gradient} The gradient of the objective in Eq.\ref{eq:KL_divergence_objective2} w.r.t $\mathbf{w}$ is:

\begin{equation} \label{eq:KL_divergence_objective2_gradient}
\begin{aligned}
\nabla_{\mathbf{w}} KL(q_{\mathbf{w}}(\mathbf{z}) \| p(\mathbf{z}))
& = \nabla_{\mathbf{w}} \left[ \sum_{\mathbf{z}} q_{\mathbf{w}}(\mathbf{z}) \log q_{\mathbf{w}}(\mathbf{z}) - \sum_{\mathbf{z}} q_{\mathbf{w}}(\mathbf{z}) \log \bar{p}(\mathbf{z}) \right] \\
&= \sum_{\mathbf{z}} \nabla_{\mathbf{w}} \left[ q_{\mathbf{w}}(\mathbf{z}) \log q_{\mathbf{w}}(\mathbf{z}) \right] - \sum_{\mathbf{z}} \nabla_{\mathbf{w}} \left[ q_{\mathbf{w}}(\mathbf{z}) \log \bar{p}(\mathbf{z}) \right] \\
&= \sum_{\mathbf{z}} \nabla_{\mathbf{w}} q_{\mathbf{w}}(\mathbf{z}) (1 + \log q_{\mathbf{w}}(\mathbf{z})) - \sum_{\mathbf{z}} \nabla_{\mathbf{w}} q_{\mathbf{w}}(\mathbf{z}) \log \bar{p}(\mathbf{z}) \\
&= \sum_{\mathbf{z}} \nabla_{\mathbf{w}} q_{\mathbf{w}}(\mathbf{z}) [1 + \log q_{\mathbf{w}}(\mathbf{z}) - \log \bar{p}(\mathbf{z})] \\
&= \mathbf{1} + \sum_{\mathbf{z}} \nabla_{\mathbf{w}} q_{\mathbf{w}}(\mathbf{z}) [\log q_{\mathbf{w}}(\mathbf{z}) - \log \bar{p}(\mathbf{z}) ]
\end{aligned}
\end{equation}
The first term in the second last line \footnote{Note, one may mistakenly derive $\sum_{\mathbf{z}} \nabla q_{\mathbf{w}}(\mathbf{z}) = \nabla \left[ \sum_{\mathbf{z}} q_{\mathbf{w}}(\mathbf{z}) \right] = \nabla [1] = 0$ for any $q$ which is a probability distribution. However, we cannot apply this logic here because this is a constrained optimisation problem.}: $\sum_{\mathbf{z}} \nabla_{\mathbf{w}} q_{\mathbf{w}}(\mathbf{z}) = \nabla_{\mathbf{w}} \sum_{\mathbf{z}}  q_{\mathbf{w}}(\mathbf{z}) = \nabla_{\mathbf{w}} (\sum_{i-1}^N w_i) = [1,1,...,1]^{\top}$.

\subsection{Un-normalised $q_{\mathbf{w}} (\mathbf{z}) = \frac{\bar{q}_{\mathbf{w}} (\mathbf{z})}{Z_q}$ and $p(\mathbf{z}) = \frac{\bar{p}(\mathbf{z})}{Z_p}$}

If we have un-normalised $\bar{q}_{\mathbf{w}} (\mathbf{z})$, plugging $q_{\mathbf{w}} (\mathbf{z})$, i.e. $q_{\mathbf{w}} (\mathbf{z}) = \frac{1}{Z_q} \bar{q}_{\mathbf{w}} (\mathbf{z})$ into the definition of the sample-based KL divergence \footnote{One can also directly plug  $q_{\mathbf{w}} (\mathbf{z})$, i.e. $q_{\mathbf{w}} (\mathbf{z}) = \frac{1}{Z_q} \bar{q}_{\mathbf{w}} (\mathbf{z})$ into Eq.\ref{eq:KL_divergence_objective2} as a shortcut.}, obtaining:

\begin{align*}
    KL(q_{\mathbf{w}}(\mathbf{z}) \| p(\mathbf{z})) 
&= \sum_{\mathbf{z}} q_{\mathbf{w}}(\mathbf{z}) \log \left( \frac{q_{\mathbf{w}}(\mathbf{z})}{p(\mathbf{z})} \right) \\
&= \sum_{\mathbf{z}} \frac{\bar{q}_{\mathbf{w}}(\mathbf{z})}{Z_q} \log \left( \frac{\bar{q}_{\mathbf{w}}(\mathbf{z}) / Z_q}{p(\mathbf{z})} \right) \\
&= \frac{1}{Z_q} \sum_{\mathbf{z}} \bar{q}_{\mathbf{w}}(\mathbf{z}) \left[ \log \bar{q}_{\mathbf{w}}(\mathbf{z}) - \log Z_q - \log p(\mathbf{z}) \right] \\
&= \frac{1}{Z_q} \left[ \sum_{\mathbf{z}} \bar{q}_{\mathbf{w}}(\mathbf{z}) \log \bar{q}_{\mathbf{w}}(\mathbf{z}) - \sum_{\mathbf{z}} \bar{q}_{\mathbf{w}}(\mathbf{z}) \log Z_q - \sum_{\mathbf{z}} \bar{q}_{\mathbf{w}}(\mathbf{z}) \log p(\mathbf{z}) \right]
\end{align*}

The second term is:
\[
\sum_{\mathbf{z}} \bar{q}_{\mathbf{w}}(\mathbf{z}) \log Z_q = \log Z_q \cdot \int \bar{q}_{\mathbf{w}}(\mathbf{z})   d\mathbf{z} = \log Z_q
\]
so:
\[
KL = \frac{1}{Z_q} \left[ \sum_{\mathbf{z}} \bar{q}_{\mathbf{w}}(\mathbf{z}) \log \bar{q}_{\mathbf{w}}(\mathbf{z}) - \log Z_q - \sum_{\mathbf{z}} \bar{q}_{\mathbf{w}}(\mathbf{z}) \log p(\mathbf{z}) \right]
\]
Further, we have $p(\mathbf{z}) = \frac{\bar{p}(\mathbf{z})}{Z_p}$, which gives:
\begin{equation} \label{eq:KL_divergence_objective3} \tag{\ref{eq:KL_divergence_objective1}b}
\begin{aligned}
    KL &= \frac{1}{Z_q} \left[ \sum_{\mathbf{z}} \bar{q}_{\mathbf{w}}(\mathbf{z}) \log \bar{q}_{\mathbf{w}}(\mathbf{z}) - \log Z_q - \sum_{\mathbf{z}} \bar{q}_{\mathbf{w}}(\mathbf{z}) \log \frac{\bar{p}(\mathbf{z})}{Z_p} \right] \\
    &= \frac{1}{Z_q} \left[ \sum_{\mathbf{z}} \bar{q}_{\mathbf{w}}(\mathbf{z}) \log \bar{q}_{\mathbf{w}}(\mathbf{z}) - \sum_{\mathbf{z}} \bar{q}_{\mathbf{w}}(\mathbf{z}) \log \bar{p}(\mathbf{z}) + \log Z_p - \log Z_q \right] \\
\end{aligned}
\end{equation}

Since $\log Z_p$ in independent of the weights $\mathbf{w}$, it can be dropped during from the optiomisation objective; $\log Z_q = \int \bar{q}_{\mathbf{w}}(\mathbf{z}) d \mathbf{z}$, however, relies on $\mathbf{w}$, which cannot be dropped. We therefore arrive at:

\begin{equation} \label{eq:KL_divergence_objective4_1}
    \mathbf{w}^*
    = \arg\min_{\mathbf{w}} KL(q_{\mathbf{w}}(\mathbf{z}) \| p(\mathbf{z})) 
    = \arg\min_{\mathbf{w}} \frac{1}{Z_q} \left[  \sum_{\mathbf{z}} \bar{q}_{\mathbf{w}}(\mathbf{z}) \log \bar{q}_{\mathbf{w}}(\mathbf{z}) - \sum_{\mathbf{z}} \bar{q}_{\mathbf{w}}(\mathbf{z}) \log \bar{p}(\mathbf{z}) - \log Z_q \right] 
\end{equation}
which is in general not a convex objective \footnote{This non-convexity is why finding the global minimum is challenging, and algorithms often find a local minimum instead.} in $\mathbf{w}$.

\subsection{GMMs with $Z_q=\sum_{i=1}^N w_i = 1$}
For GMMs, specifically, we have $Z_q=\int \sum_{i=1}^N w_i \cdot \mathcal{N}(\mathbf{z};\boldsymbol{\mu}_i,\Sigma_i) d\mathbf{z} = \sum_{i=1}^N w_i \int \mathcal{N}(\mathbf{z};\boldsymbol{\mu}_i,\Sigma_i) d\mathbf{z} = \sum_{i=1}^N w_i = 1$. As a result, we have $\bar{q}_{\mathbf{w}}(\mathbf{z})=q_{\mathbf{w}}(\mathbf{z})=\sum_{i=1}^N w_i \cdot \mathcal{N}(\mathbf{z};\boldsymbol{\mu}_i,\Sigma_i)$, which simplifies Eq.\ref{eq:KL_divergence_objective4_1} to:

\begin{equation} \label{eq:KL_divergence_objective4} \tag{\ref{eq:KL_divergence_objective2}b}
    \arg\min_{\mathbf{w}} KL(q_{\mathbf{w}}(\mathbf{z}) \| p(\mathbf{z})) 
    = \arg\min_{\mathbf{w}} \left[ \sum_{\mathbf{z}} q_{\mathbf{w}}(\mathbf{z}) \log q_{\mathbf{w}}(\mathbf{z}) - \sum_{\mathbf{z}} q_{\mathbf{w}}(\mathbf{z}) \log \bar{p}(\mathbf{z}) \right] 
\end{equation}
This objective now becomes \textbf{convex} in $\mathbf{w}$ because: the first term (negative entropy of $q_{\mathbf{w}}(\mathbf{z})$) is a summation (or integral) of a convex function\footnote{$q_{\mathbf{w}}(\mathbf{z}) \log q_{\mathbf{w}}(\mathbf{z})$ is convex because $x \log x$ is convex and $q_{\mathbf{w}}(\mathbf{z}) = \sum_{i=1}^N w_i \cdot \mathcal{N}(\mathbf{z};\boldsymbol{\mu}_i,\Sigma_i)$ is linear in $\mathbf{w}$, the composition of a convex function with a linear function is convex.} $q_{\mathbf{w}}(\mathbf{z}) \log q_{\mathbf{w}}(\mathbf{z})$, which is again convex; the second term is a linear function of $\mathbf{w}$. Combing both yields a convex objective. Jointly optimising $\boldsymbol{\theta} = \{w_i, \boldsymbol{\mu}_i, \boldsymbol{\Sigma}_i\}_{i=1}^N$ with the KL divergence objective is non-convex and difficult in general \cite{guo_boosting_2017}; however, choosing the optimal weight alone is a convex optimisation problem \cite{guo_boosting_2017}.

\paragraph{Gradient} 
Similar to Eq.\ref{eq:KL_divergence_objective2_gradient}, the gradient of the KL divergence objective in Eq.\ref{eq:KL_divergence_objective4} \textit{w.r.t.} $\mathbf{w}$ is:

\begin{equation} \label{eq:KL_divergence_objective4_gradient}
\begin{aligned}
\nabla_{\mathbf{w}} KL(q_{\mathbf{w}}(\mathbf{z}) \| p(\mathbf{z}))
& = \nabla_{\mathbf{w}} \left[  \sum_{\mathbf{z}} q_{\mathbf{w}}(\mathbf{z}) \log q_{\mathbf{w}}(\mathbf{z}) - \sum_{\mathbf{z}} q_{\mathbf{w}}(\mathbf{z}) \log \bar{p}(\mathbf{z}) \right] \\
&= \sum_{\mathbf{z}} \nabla_{\mathbf{w}} \left[ q_{\mathbf{w}}(\mathbf{z}) \log q_{\mathbf{w}}(\mathbf{z}) \right] - \sum_{\mathbf{z}} \nabla_{\mathbf{w}} \left[ q_{\mathbf{w}}(\mathbf{z}) \log \bar{p}(\mathbf{z}) \right] \\
&= \sum_{\mathbf{z}} \nabla_{\mathbf{w}} q_{\mathbf{w}}(\mathbf{z}) (1 + \log q_{\mathbf{w}}(\mathbf{z})) - \sum_{\mathbf{z}} \nabla_{\mathbf{w}} q_{\mathbf{w}}(\mathbf{z}) \log \bar{p}(\mathbf{z}) \\
&= \sum_{\mathbf{z}} \nabla_{\mathbf{w}} q_{\mathbf{w}}(\mathbf{z}) [1 + \log q_{\mathbf{w}}(\mathbf{z}) - \log \bar{p}(\mathbf{z})] \\
&= \mathbf{1} + \sum_{\mathbf{z}} \nabla_{\mathbf{w}} q_{\mathbf{w}}(\mathbf{z}) [\log q_{\mathbf{w}}(\mathbf{z}) - \log \bar{p}(\mathbf{z}) ]
\end{aligned}
\end{equation}
Again, the unity term cannot be eliminated from the last second line as it is constrained optimsiation
Eq.\ref{eq:KL_divergence_objective4_gradient} shows that, the gradient of the discrepancy between the target $p(\mathbf{z})$ and the approximator $q_{\mathbf{w}}(\mathbf{z})$, in the case of GMMs, is the sum of the weighted gap at all sample positions plus unity. It is thus an averaged direction which, as one can imagine, can be affected by noise or outlier samples.

We further develop Eq.\ref{eq:KL_divergence_objective4_gradient}. 
As $q_{\mathbf{w}}(\mathbf{z}) = \sum_{i=1}^N w_i \cdot \mathcal{N}(\mathbf{z}; \boldsymbol{\mu}_i, \Sigma_i)$, we have:
\[
   \nabla_{w_i} q_{\mathbf{w}}(\mathbf{z}) = \frac{\partial q_{\mathbf{w}}(\mathbf{z})}{\partial w_j} = \mathcal{N}(\mathbf{z}; \boldsymbol{\mu}_j, \Sigma_j)
\]
as only the $j$-th term is relevant when differentiation w.r.t. $w_j$. Therefore, 

\[
\nabla_{\mathbf{w}} q_{\mathbf{w}}(\mathbf{z}) = 
\begin{bmatrix}
\mathcal{N}(\mathbf{z}; \boldsymbol{\mu}_1, \Sigma_1) \\
\mathcal{N}(\mathbf{z}; \boldsymbol{\mu}_2, \Sigma_2) \\
\vdots \\
\mathcal{N}(\mathbf{z}; \boldsymbol{\mu}_N, \Sigma_N)
\end{bmatrix}
\]

The gradient of the KL divergence objective in Eq.\ref{eq:KL_divergence_objective4_gradient} then becomes:
\begin{equation} \label{eq:KL_divergence_objective4_gradient_GMM0}
\nabla_{\mathbf{w}} \mathrm{KL}(q_{\mathbf{w}}(\mathbf{z}) \| p(\mathbf{z}))
    = \mathbf{1} + \sum_{\mathbf{z}} 
    \left( 
    \begin{bmatrix}
    \mathcal{N}(\mathbf{z}; \boldsymbol{\mu}_1, \Sigma_1) \\
    \vdots \\
    \mathcal{N}(\mathbf{z}; \boldsymbol{\mu}_N, \Sigma_N)
    \end{bmatrix}
    \cdot 
    [\log q_{\mathbf{w}}(\mathbf{z}) - \log \bar{p}(\mathbf{z})]
    \right)
\end{equation}
The gradient $\nabla_{\mathbf{w}} \mathrm{KL}(q_{\mathbf{w}}(\mathbf{z}) \| p(\mathbf{z}))$ is a vector of dimension $N$ (of the same length as $\mathbf{w}$). Each element of the vector corresponds to $\nabla_{w_i} KL$ (in continuous form):
\begin{equation} \tag{\ref{eq:KL_divergence_objective4_gradient_GMM0}b} \label{eq:KL_divergence_objective4_gradient_GMM0b}
    \nabla_{w_i} KL = 1 + \int \mathcal{N}(\mathbf{z}; \boldsymbol{\mu}_i, \Sigma_i) \left[ \log q_{\mathbf{w}}(\mathbf{z}) - \log \bar{p}(\mathbf{z}) \right] d\mathbf{z} = 1 + \mathbb{E}_{\mathbf{z} \sim \mathcal{N}_i} \left[ \log q_{\mathbf{w}}(\mathbf{z}) - \log \bar{p}(\mathbf{z}) \right]
\end{equation}
which can be estimated using $M$ samples $\{\mathbf{s}_{i,j}\}$ drawn from $\mathcal{N}_i$:
\begin{equation} \label{eq:KL_divergence_objective4_gradient_GMM1}
\begin{aligned}
    \nabla_{w_i} KL 
    &\approx 1 + \sum_{j=1}^M \mathcal{N}(\mathbf{s}_{i,j}; \boldsymbol{\mu}_i, \Sigma_i) [\log q_{\mathbf{w}}(\mathbf{s}_{i,j}) - \log \bar{p}(\mathbf{s}_{i,j})] \\
    &= 1 + \mathbb{E}_{\mathbf{s} \sim \mathcal{N}_{\boldsymbol{\mu}_i,\Sigma_i}} [ \log \frac{q_{\mathbf{w}}(\mathbf{s})}{\bar{p}(\mathbf{s})} ]
\end{aligned}
\end{equation}
where $q_{\mathbf{w}}(\mathbf{z}=\mathbf{s}_{i,j}) = \sum_{i=1}^N w_i \cdot \mathcal{N}(\mathbf{s}_{i,j}; \boldsymbol{\mu}_i, \Sigma_i)$. 

We can use a Monte Carlo estimator for the second term, which avoids multiplying the normal PDFs $\mathcal{N}(\mathbf{s}_{i,j}; \boldsymbol{\mu}_i, \Sigma_i)$:
\begin{equation} \label{eq:KL_divergence_objective4_gradient_GMM1_MC}
    \nabla_{w_i} KL \approx 1 + \frac{1}{M} \sum_{j=1}^M \left[ \log q_{\mathbf{w}}(\mathbf{s}_{i,j}) - \log \bar{p}(\mathbf{s}_{i,j}) \right]    
\end{equation}
with $\mathbf{s}_{i,j} \sim \mathcal{N} (\mathbf{z};\boldsymbol{\mu}_i, \Sigma_i)$. 

Finally, the $i$-th weight $w_i$ can be updated as per projected gradient descent (pGD):
\begin{equation}
\begin{aligned}
    v_i^{(k)} &= w_i^{(k-1)} - \eta \cdot \left[ \nabla_{\mathbf{w}} KL \right]_i^{(k-1)} \\
    w_i^{(k)} &= \text{Proj}_{\Delta}(v_i^{(k)})     \\
\end{aligned}
\end{equation}
where $\Delta$ is the probability simplex $\Delta = \{ \mathbf{w} \in \mathbb{R}^N \mid w_i \ge 0 \text{ and } \sum_{i=1}^N w_i = 1 \}$.

We can check the quality of the approximation via \footnote{Note, the heuristic $Z_q \approx \tfrac{1}{NM}\sum_{i,j} q_{\mathbf{w}}(\mathbf{s}_{i,j}) \approx 1$ is not a valid normalisation check unless $s_{i,j} \sim  q_{\mathbf{w}}$.} $Z_q \approx \frac{1}{N M} \sum_{i=1}^N \sum_{j=1}^M q_{\mathbf{w}}(\mathbf{s}_{i,j}) \approx 1$. This condition can be used to fine tune $N$ and $M$ in practical implementation.

Note the second term in Eq.\ref{eq:KL_divergence_objective4_gradient_GMM1} represents the expected discrepancy between the approximating GMM and the unnormalised target \footnote{It is termed the \textit{residual log-density} in \cite{guo_boosting_2017}.}. 
Therefore, the intuition being, this gradient direction drives weights towards regions that helps fix large gaps \cite{guo_boosting_2017}.
If we have $q_{\mathbf{w}}(\mathbf{z}) \propto \bar{p}(\mathbf{z})$ satisfied everywhere (similar to importance sampling), then we would have equal, constant $\nabla_{w_i} KL$ for all weights, which meets the Karush-Kuhn-Tucker (KKT) optimality conditions for a constrained optimisation problem \footnote{The reasoning behind this KKT optimality condition comes from the geometry of constrained optimisation: when all the gradient components are equal, a proper constrained optimisation method (e.g. projected gradient descent) will make no further changes to the weights, because any update would either be zero or move the solution out of the feasible set in a way that the projection step cancels out - when all gradient components are equal, the update step was entirely orthogonal to the feasible set, the projection step cancels it out completely, and the weights do not change.}: \textit{at the optimal point, the gradients for all weights that are greater than zero (i.e. $w_i>0$) must be equal to the same constant value, and the gradients for weights that are zero ($w_i=0$) can be larger}.

\section{Robbins-Monro step-size conditions}  \label{app:Robbins_Monro}
The stochastic approximation framework of Robbins and Monro \cite{robbins1951stochastic} provides conditions under which iterative updates with noisy gradients converge to a stationary point. In particular, if the sequence of learning rates $\{\eta_k\}_{k\geq 1}$ satisfies
\[
\sum_{k=1}^\infty \eta_k = \infty
\qquad \text{and} \qquad
\sum_{k=1}^\infty \eta_k^2 < \infty
\]
then the updates converge almost surely to a solution of the underlying fixed-point equation, provided the noise variance is bounded.

The learning rate schedule $\eta_k = \eta_0/k$ satisfies these conditions. Indeed, the harmonic series diverges:
\[
\sum_{k=1}^\infty \eta_k = \eta_0 \sum_{k=1}^\infty \frac{1}{k} = \infty
\]
while the series of squared step sizes converges:
\[
\sum_{k=1}^\infty \eta_k^2 = \eta_0^2 \sum_{k=1}^\infty \frac{1}{k^2} 
= \eta_0^2  \frac{\pi^2}{6} < \infty
\]
Therefore, the Robbins-Monro conditions are fulfilled. Consequently, when using $\eta_k = \tfrac{\eta_0}{k}$ in the projected gradient updates of the mixture weights, the stochastic approximation converges to a local minimum of the $\mathrm{KL}(q_{\mathbf{w}}\Vert p)$ objective, ensuring stability and asymptotic consistency of the algorithm.

\textit{Remark.} other diminishing step-size schedules are sometimes employed, e.g. $\eta_k=\tfrac{\eta_0}{\sqrt{k}}$. However, in this case $\sum_k \eta_k^2 = \infty$, which violates the Robbins-Monro conditions and may cause oscillatory behaviour. On the other hand, schedules with faster decay, such as $\eta_k=\tfrac{\eta_0}{k^{1+\epsilon}}$ for $\epsilon>0$, ensure square-summability but also yield $\sum_k \eta_k < \infty$, which halts progress prematurely. The choice $\eta_k=\tfrac{\eta_0}{k}$ is therefore considered the canonical balance: it decays slowly enough to guarantee continual exploration (divergent sum), yet fast enough to control noise accumulation (convergent squared sum).

Fig.\ref{fig:robbins-munro-schedules} compares three common diminishing learning-rate schedules: (i) the \textit{harmonic decay} $\eta_k=\eta_0/k$, (ii) the \textit{square-root decay} $\eta_k=\eta_0/\sqrt{k}$, and (iii) a \textit{superlinear decay} $\eta_k=\eta_0/k^{1.1}$. The harmonic schedule is the unique case that satisfies both Robbins-Monro conditions, striking the right balance between persistent exploration and noise control. In contrast, the square-root decay decreases too slowly, leading to $\sum_k \eta_k^2 = \infty$ and potential oscillations in weights evolution, while the superlinear decay decreases too quickly, leading to $\sum_k \eta_k < \infty$ and premature stagnation. This illustrates why the harmonic choice $\eta_k=\eta_0/k$ is widely regarded as the canonical schedule for stochastic approximation algorithms.

\begin{figure}[H]
    \centering
    \includegraphics[width=0.65\linewidth]{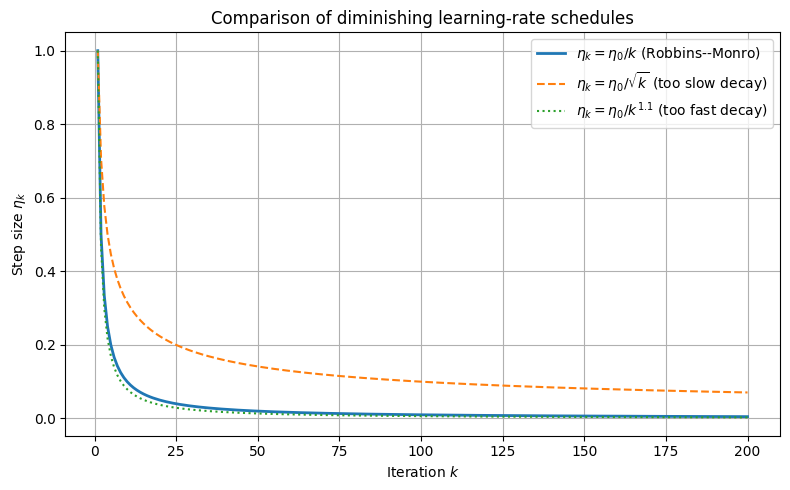}
    \caption{Comparison of diminishing learning-rate schedules. Only the harmonic schedule $\eta_k=\eta_0/k$ satisfies both Robbins-Monro conditions, ensuring convergence while controlling variance.}
    \label{fig:robbins-munro-schedules}
\end{figure}

\section{Computation of distributional distances} \label{app:distributional_distance_metrics}
Here we briefly introduce the 5 distributional discrepancy measures \cite{nguyen_convergence_2013} used in our experiments: the Kullback-Leibler (KL) divergence, 1D Wasserstein distance, Kolmogorov-Smirnov (KS) statistic, maximum mean discrepancy (MMD), and total variation (TV) distance.

\paragraph{KL divergence} 
The Kullback-Leibler (KL) divergence \cite{kullback_information_1951} measures the information lost when one distribution is used to approximate another. For two sets of samples, $\hat{q}_1$ and $\hat{q}_2$, this first estimate their respective probability mass functions (PMFs), generally using histograms over a shared set of bins. Let $\hat{p}_i$ be the proportion of samples from $\hat{q}_1$ and $\hat{q}_i$ be the proportion from $\hat{q}_2$ that fall into the $i$-th bin, the KL divergence is the sum of the log differences in probabilities, weighted by the probabilities of the first distribution \footnote{The KL divergence is not a true distance metric as it's asymmetric ($D_{KL}(\hat{p}||\hat{q}) \neq D_{KL}(\hat{q}||\hat{p})$). An symmetric alternative is the \textit{Jensen-Shannon divergence}. For the value to be finite, any bin with samples from $\hat{q}_1$ must also contain samples from $\hat{q}_2$ to avoid taking the log of zero.}:
\[
KL(\hat{p} || \hat{q}) = \sum_{i=1}^K \hat{p}_i \log\left(\frac{\hat{p}_i}{\hat{q}_i}\right)
\]
this \textit{inclusive} KL divergence gives the 'information gain' from using the true distribution $\hat{p}$ instead of the approximating distribution $\hat{q}$. The \textit{exclusive} KL divergence, also known as the reverse KL divergence, is expressed by swapping the roles of the two distributions:
\[
KL(\hat{q} || \hat{p}) = \sum_{i=1}^K \hat{q}_i \log\left(\frac{\hat{q}_i}{\hat{p}_i}\right)
\]

\paragraph{Wasserstein distance}
To compute the 1D Wasserstein distance \footnote{In 1D, the Wasserstein distance is a simple function of the order statistics.} (also known as earth mover's distance, EMD \cite{vapnik_necessary_1982}) between two sets of samples, we sort both sets and compare their empirical cumulative distribution functions (ECDFs). The 1-Wasserstein distance integrates the area between the ECDFs, which can be approximated by summing the absolute differences between the sorted samples of equal length\footnote{For unequal lengths, one can consider a new set combining sorted unique values from both sets.} $M$:
\[
W_1(\hat{q}_1, \hat{q}_2) = \frac{1}{M} \sum_{i=1}^M \left| x_{(i)} - y_{(i)} \right|
\]
where $x_{(1)}, \ldots, x_{(M)}$ are the sorted samples from $\hat{q}_1$, $y_{(1)}, \ldots, y_{(M)}$ are the sorted samples from $\hat{q}_2$.

\paragraph{KS statistic}
To compute the 1D KS statistic \footnote{The KS statistic is the basis for the KS test, which is a non-parametric test for the equality of continuous, one-dimensional probability distributions.} \cite{smirnov1939estimation}, we also first sort both sets and compare their empirical cumulative distribution functions (ECDFs). However, rather than integrating the area between the ECDFs, the KS statistic is simply the maximum absolute vertical distance between them at any point:
\[
KS(\hat{q}_1,\hat{q}_2) = \sup_x \left| F_1(x) - F_2(x) \right|
\]
where $F_1(x)$ and $F_2(x)$ are the ECDFs of the first and second samples, respectively, and $\sup_x$ gives the largest value of the absolute difference across all possible samples - resulting a value between\footnote{A value of 0 indicates identical ECDFs, while a value of 1 indicates no overlap between the ranges of the two samples.} 0 and 1.

\paragraph{MMD}
To compute MMD \cite{gretton2007kernel}, we compare two distributions by mapping them into a high-dimensional feature space using a kernel function \footnote{A kernel function, such as the Gaussian RBF kernel, computes the similarity between two points in a high-dimensional feature space without explicitly performing the mapping. The choice of kernel is crucial to the metric's performance.}. The MMD is then the distance between the mean embeddings of the two distributions in that space. For two sets of samples, $\{x_1, \ldots, x_M\} \sim \hat{q}_1$ and $\{y_1, \ldots, y_M\} \sim \hat{q}_2$, the squared MMD can be directly calculated from samples \footnote{This empirical $MMD^2$ is biased; an unbiased version of $MMD^2$ avoids using the same sample point twice when calculating the within-sample similarities, which prevents a positive bias that occurs from the $k(x, x)$ $k(y,y)$ terms.}:
\[
\text{MMD}^2(\hat{q}_1, \hat{q}_2) = \frac{1}{M^2} \sum_{i,j=1}^M k(x_i, x_j) + \frac{1}{M^2} \sum_{i,j=1}^M k(y_i, y_j) - \frac{2}{M^2} \sum_{i=1}^M \sum_{j=1}^M k(x_i, y_j)
\]
where $k(\cdot, \cdot)$ is the chosen kernel function. MMD balances the average similarity within each sample against the average similarity between the two samples. If the distributions are similar, these values will be close, and the MMD will be near zero.

\paragraph{Total variation (TV) distance} To compute the TV \cite{cam_sufficiency_1964}, we compare the probability mass functions (PMFs) of two distributions to find the largest possible divergence between the probabilities they assign to any single event. For two sets of samples, $\hat{q}_1$ and $\hat{q}_2$, this first requires estimating their PMFs, normally by creating histograms over a shared set of bins \footnote{The TV distance ranges from 0 (identical distributions) to 1 (distributions with no overlap). For continuous data, the choice and number of bins can affect the final estimated distance.}. Let $\hat{p}_i$ and $\hat{q}_i$ be the proportion of samples from each set that fall into the $i$-th bin, the TV distance is half the sum of the absolute differences of these proportions across all $B$ bins:
\[
d_{TV}(\hat{q}_1, \hat{q}_2) = \frac{1}{2} \sum_{i=1}^B \left| \hat{p}_i - \hat{q}_i \right|
\]
which quantifies the total difference between the two empirical distributions across all possible outcomes.

Of the five measures, the 1D Wasserstein distance, KS statistic, and TV distance are true distance metrics, as they are symmetric and satisfy the triangle inequality. MMD is also symmetric, but it only qualifies as a true distance metric when a characteristic (injective mapping) kernel (e.g. RBF) is used. In contrast, the KL divergence is not a true distance metric (what's why it is referred to as a divergence rather than a distance) - it is asymmetric and does not satisfy the triangle inequality.

\section{Stratified sampling} \label{app:simple_and_stratified_random_sampling}

As stratified sampling is concerned in the second stage of our GMA sampling method, here we present some properties of stratified sampling from a mathematical statistics perspective. Following discussions are mainly based on evidences from the three classic books: Cochran \cite{cochran_sampling_1977}, Boucheron \cite{boucheron_concentration_2013} and Rice \cite{rice_mathematical_2007}.

Compared to naive, simple random sampling, stratified Monte Carlo sampling reduces variance via sub-region sampling. It divides the sample space into non-overlapping strata, samples from each (proportionally to the strata probability), and combines the weighted results, preserving unbiasedness while reducing variance by ensuring all regions are represented. 
For example, we want to estimate $\mathbb{E}[Y]$ for a real-valued random variable $Y$, straightforwardly, we can either use \textit{simple random sampling} or \textit{stratified sampling} to sample from its underlying distribution or drawing samples from an existing population space. 
Simple random sampling randomly draws $n$ samples $\{Y_i\}_{i=1}^n$ from the sample space of $Y$. Let $N_0$ be the population size, we have the variance of a random sampler \cite{rice_mathematical_2007}:

\begin{theorem} [\textit{variance of simple random sampler, Rice \cite{rice_mathematical_2007}}]
\label{thm:variance_of_random_sampler}
With simple random sampling, we have
\[
\mathrm{Var}(\bar{Y}) =
\frac{\sigma^2}{n} \left( \frac{N_0 - n}{N_0 - 1} \right)
\]
\[
= \frac{\sigma^2}{n} \left( 1 - \frac{n - 1}{N_0 - 1} \right)
\]
where $\sigma^2 = \mathrm{Var}(Y)$.
\end{theorem}
\begin{proof}
    See proof of Theorem B on pp.208 in \cite{rice_mathematical_2007}.
\end{proof}

The factor $\left( 1 - \frac{n - 1}{N_0 - 1} \right)$ is termed \textit{finite population correction}, the ratio $n/N_0$ is called \textit{the sampling fraction}. When the sampling fraction is small, the standard error (standard deviation) of $\bar{Y}$ becomes \cite{rice_mathematical_2007}:
\begin{equation} \label{eq:approx_variance_of_simple_random_sampling}
    \sigma_{\bar{Y}} \approx \frac{\sigma}{\sqrt{n}}
\end{equation}
that is, the standard error of the random sample mean is inversely proportional to the square root of the sample size. To double the accuracy, the sample size must be quadrupled.

Stratified sampling partitions the sample space into $N$ disjoint strata $A_1, \dots, A_N$ such that $P\left( Y \in \bigcup_{i=1}^N A_i \right) = 1$. 
Let $p_i=P(Y\in A_i)=N_i/N_0$ be the proportion (weight) of strata $i$, $\mu_i=\mathbb{E}[Y\mid Y\in A_i]$, $\sigma_i^2=\mathrm{Var}(Y\mid Y\in A_i)$ be the population mean and variance of strata $i$. Let $N_i$ denote the population size of strata $i$, and the total population size is $N_0=N_1+N_2+...+N_N$; we draw $n_i$ i.i.d. samples from the $N_i$ samples in this stratum, with $\sum_i n_i=n$. We assume that the samples from different strata are independent of one another (i.e. sample independence across strata) \cite{rice_mathematical_2007}.

On the population level, by the law of total expectation (linearility of expectation), we have \cite{rice_mathematical_2007}:
\[
\mathbb{E}[Y]
= \sum_{i=1}^N P(Y \in A_i)   \mathbb{E}[Y \mid Y \in A_i]
= \sum_{i=1}^N p_i   \mathbb{E}[Y \mid Y \in A_i]
\]

If we denote the population means as $\mu$, and we have already the mean of each strata $\mu_i, i=1,2,...,N$, the above equation states 
\begin{equation}
    \mu=\sum_{i=1}^N p_i \mu_i
\end{equation}

As within each strata, $n_i$ samples $\{Y_{ij}\}_{j=1}^{n_i}$ are taken, the sample mean $\bar{Y}_i$ in strata $i$ is:
\begin{equation} \label{eq:sample_mean_of_stratified_sampler}
    \bar{Y}_i=\frac{1}{n_i}\sum_{j=1}^{n_i} Y_{ij}
\end{equation}

In analogy to the preceding relationship between the overall population mean and the population means of the various strata, we have the relation between sample means \cite{rice_mathematical_2007, kennedy2016montecarlo}:
\begin{equation} \label{eq:sample_mean_stratified_sampling}
\bar{Y}= \sum_{j=1}^{n_i} p_i \bar{Y}_i
\end{equation}

\begin{theorem}[\textit{Unbiasedness of stratified sampler, Rice \cite{rice_mathematical_2007}}]
\label{thm:unbiasedness_of_stratified_sampler}
The stratified estimate, $\bar{Y}$, of the population mean is unbiased.
\end{theorem}
\begin{proof}
\begin{align*}
E(\bar{Y}) &= \sum_{i=1}^N p_i E(\bar{Y}_i) \\
&=\sum_{i=1}^N p_i \mu_i \\
&= \frac{1}{N_0} \sum_{i=1}^N n_i \mu_i \\
&= \mu
\end{align*}
\end{proof}

Since we assume sample independence across strata and within each stratum a simple random sample is taken, the variance of $\bar{Y}$ can be calculated \cite{rice_mathematical_2007}:

\begin{theorem}[\textit{Variance of stratified sampler, Rice \cite{rice_mathematical_2007}}]
\label{thm:variance_of_stratified_sampler}
The variance of the stratified sample mean is 
\[
\mathrm{Var}(\bar{Y}) =
\sum_{i=1}^N p_i^2 \left( \frac{1}{n_i} \right)
\left( 1 - \frac{n_i - 1}{N_i - 1} \right) \sigma_i^2
\]
\end{theorem}

\begin{proof}
Since the $\bar{Y}_i$ are independent,
\[
\mathrm{Var}(\bar{Y}) = \sum_{i=1}^N p_i^2 \mathrm{Var}(\bar{Y}_i)
\]
Within each strata, simple random sampling is performed, from Theorem.\ref{thm:variance_of_random_sampler}, we have
\[
\mathrm{Var}(\bar{Y}_i) =
\frac{1}{n_i} \left( 1 - \frac{n_i - 1}{N_i - 1} \right) \sigma_i^2
\]
Therefore, the desired result follows.
\end{proof}

If the sampling fractions $n_i/N_i$ within all strata are small, we have \footnote{The variance of the stratified sampler (estimator) is also derived as Eq.(5.39) in \cite{kennedy2016montecarlo}.}:
\begin{equation} \label{eq:sample_variance_of_stratified_sampler}
\mathrm{Var}(\bar{Y}) \approx \sum_{i=1}^N \frac{p_i^2 \sigma_i^2}{n_i}
= \frac{1}{n}\sum_{i=1}^N \frac{p_i^2 \sigma_i^2}{\alpha_i},
\qquad \alpha_i:=\frac{n_i}{n}
\end{equation}
Therefore, $\mathrm{RMSE}(\bar{Y})=\mathcal{O}(n^{-1/2})$, which is same $1/\sqrt{n}$ rate as simple random sampling, but typically with a smaller constant (we shall elaborate in later sections).
Compared to simple random sampling which draws $Y_1, \dots, Y_n$ i.i.d. from $Y$, stratified sampling pre-specifies the fraction of samples to draw from each $A_i$, then generate them from the conditional distribution $p(Y \mid Y \in A_i)$.
This guarantees that each stratum is represented according to $p_i$, preserving unbiasedness while typically reducing variance \cite{glasserman_monte_2003}.

The question now being, if the resources allow only a total of $n$ samples to be sampled, how to choose $n_1, \ldots, n_L$ to minimize $\mathrm{Var}(\bar{Y})$ subject to the constraint $n_1 + \cdots + n_L = n$? In general, there are two strategies used in stratified sampling to allocate the sample draws: \textit{proportional allocation} and \textit{Neyman-optimal allocation}.

\paragraph{Neyman-optimal allocation} we allocate as per $n_i \propto p_i\sigma_i$:

\begin{theorem}[\textit{Variance-based, Neyman optimal allocation, Rice \cite{rice_mathematical_2007}}]
\label{thm:Neyman_allocation}
The sample sizes $n_1, \ldots, n_N$ that minimize $\mathrm{Var}(\bar{Y})$, subject to the constraint $n_1 + \cdots + n_N = n$, are given by
\[
n_i = n \frac{p_i\sigma_i}{\sum_{j=1}^N p_j\sigma_j}
\]
yielding
\[
\mathrm{Var}(\bar{Y})
= \frac{1}{n}\Big(\sum_{i=1}^N p_i\sigma_i\Big)^2
\]
neglecting the finite population correction.
\end{theorem}
\begin{proof}
    See Theorem A on pp.232 and Corollary A on pp.233 in \cite{rice_mathematical_2007}.
\end{proof}

This theorem gives the (optimal) allocation strategy which minimizes variance for a given $n$ (equal per-sample cost). 

\paragraph{Proportional allocation} we allocate number of samples as per $n_i=n p_i$:

\begin{theorem}[\textit{Proportional allocation, Rice \cite{rice_mathematical_2007}}]
\label{thm:proportional_allocation}
Using the proportional allocation strategy $n_i=n p_i$ with the constraint $n_1 + \cdots + n_N = n$, we have
\[
\mathrm{Var}(\bar{Y})=\frac{1}{n}\sum_{i=1}^N p_i \sigma_i^2
\le \frac{1}{n} \mathrm{Var}(Y)
\]
\end{theorem}
\begin{proof}
    Proof of the equality can be found in the proof of Theorem B on pp.234 in \cite{rice_mathematical_2007}.
    The inequality can be proved as follows. First, we prove this relation \footnote{This relation is also derived as Eq.(5.44) in \cite{kennedy2016montecarlo}.} between the overall population variance $\sigma^2=\mathrm{Var}(Y)$ and the strata variances $\sigma_i^2$: $\sigma^2=\mathrm{Var}(Y)=\sum_i p_i\sigma_i^2+\mathrm{Var}(\mu_i)$.

    The overall population variance can be expressed as
    \[
    \sigma^2 = \frac{1}{N_0} \sum_{i=1}^N \sum_{j=1}^{N_i} (x_{ij} - \mu)^2
    \]
    with $N_0=N_1+N_2+...+N_N$.
    As
    \[
    (x_{ij} - \mu)^2 
    = \left[(x_{ij} - \mu_i) + (\mu_i - \mu)\right]^2
    = (x_{ij} - \mu_i)^2 + 2(x_{ij} - \mu_i)(\mu_i - \mu) + (\mu_i - \mu)^2
    \]
    When this expression is summed over $j$, the middle term on the RHS becomes zero since 
    $N_i \mu_i = \sum_{j=1}^{N_i} x_{ij}$, so we have
    \[
    \sum_{j=1}^{N_i} (x_{ij} - \mu)^2
    = \sum_{j=1}^{N_i} (x_{ij} - \mu_i)^2 + N_i (\mu_i - \mu)^2
    \]
    \[
    = N_i \sigma_i^2 + N_i (\mu_i - \mu)^2
    \]
    Dividing both sides by $N_0$ and summing over $i$, we have
    \[
    \sigma^2 = \sum_{i=1}^N p_i \sigma_i^2 + \sum_{i=1}^N p_i (\mu_i - \mu)^2
    =\sum_{i=1}^N p_i\sigma_i^2+\mathrm{Var}(\mu_i)
    \]
    which gives:
    $
    \frac{1}{n} \sum_{i=1}^N p_i \sigma_i^2 = \frac{1}{n} \sigma^2 - \frac{1}{n} \mathrm{Var}(\mu_i)
    = \frac{1}{n} \mathrm{Var}(Y) - \frac{1}{n} \mathrm{Var}(\mu_i)
    $.
    Therefore, we have:
    $\mathrm{Var}(\bar{Y})=\frac{1}{n}\sum_{i=1}^N p_i \sigma_i^2 \le \frac{1}{n} \mathrm{Var}(Y)$
\end{proof}

Comparing Neyman optimal allocation and proportional allocation, we have their variance difference:

\begin{theorem}[\textit{Variance difference of Neyman optimal allocation and proportional allocation, Rice \cite{rice_mathematical_2007}}]
\label{thm:comparing_variances_of_Neyman_and_proportional_stratified_sampling}
With stratified random sampling, the difference between the variance of the estimate of the population mean based on proportional allocation and the variance of that estimate based on optimal allocation is
\[
\mathrm{Var}(\bar{Y}_{proportional}) - \mathrm{Var}(\bar{Y}_{Neyman}) 
= \frac{1}{n} \sum_{i=1}^N p_i (\sigma_i - \bar{\sigma})^2
\]
where $\bar{\sigma} = \sum_{i=1}^N p_i \sigma_i$, ignoring the finite population correction.
\end{theorem}
\begin{proof}
We can submit the results for variance from Theorem.\ref{thm:Neyman_allocation} and Theorem.\ref{thm:proportional_allocation} and obtain:
\[
\mathrm{Var}(\bar{Y}_{proportional}) - \mathrm{Var}(\bar{Y}_{Neyman}) 
= \frac{1}{n} \left[ \sum_{i=1}^N p_i \sigma_i^2 
- \left( \sum_{i=1}^N p_i \sigma_i \right)^2 \right]
\]
The term within the large brackets equals $\sum_{i=1}^N p_i (\sigma_i - \bar{\sigma})^2$, which can be verified by expanding the square and collecting terms.
\end{proof}

Theorem.\ref{thm:comparing_variances_of_Neyman_and_proportional_stratified_sampling} essentially hints that\footnote{Note that, this conclusion can also be derived using weighted Cauchy-Schwarz inequality: 
take the vectors$u_i = \sigma_i,\quad v_i = 1$ in the weighted inner product$\langle u, v\rangle_p := \sum_{i=1}^N p_i u_i v_i$, then the weighted Cauchy-Schwarz inequality says: $\left( \sum_{i=1}^N p_i u_i v_i \right)^2 \le \left( \sum_{i=1}^N p_i u_i^2 \right) \left( \sum_{i=1}^N p_i v_i^2 \right)$. Substitute $u_i = \sigma_i$, $v_i = 1$: $\left( \sum_{i=1}^N p_i \sigma_i \right)^2 \le \left( \sum_{i=1}^N p_i \sigma_i^2 \right) \left( \sum_{i=1}^N p_i \cdot 1^2 \right)$. As $\sum_{i=1}^N p_i = 1$, so: $\left( \sum_{i=1}^N p_i \sigma_i \right)^2 \le \sum_{i=1}^N p_i \sigma_i^2$, with equality iff all $\sigma_i$ are equal (or all but one $p_i$ are zero).}, $\mathrm{Var}(\bar{Y}_{proportional}) \leq \mathrm{Var}(\bar{Y}_{Neyman})$. Equality holds when the variances of each strata $\sigma_i, i=1,2,...,N$ are all the same, then proportional allocation yields the same results as optimal allocation. The more variable these variances are, the better it is to use Neyman optimal allocation to yield smaller variance \cite{rice_mathematical_2007}.
 
We can also compare the variance under simple random sampling with the variance under proportional allocation. The variance of simple random sampling, ignoring the finite sample correction factor, is $\mathrm{Var}(\bar{Y}) = \frac{\sigma^2}{n}$ (Theorem.\ref{thm:variance_of_random_sampler}), and the variance of the stratified sampling using proportional allocation strategy is $\mathrm{Var}(\bar{Y})=\frac{1}{n}\sum_{i=1}^N p_i \sigma_i^2$ (Theorem.\ref{thm:proportional_allocation}), together they give:

\begin{theorem}[\textit{Variance difference of simple random sampling and proportional stratified sampling, Rice \cite{rice_mathematical_2007}}]
\label{thm:comparing_variances_of_simpleRandom_and_proportional_stratified_sampling}
    The difference between the variance of the mean of a simple random sample and the variance of the mean of a stratified random sample based on proportional allocation is
    \[
    \mathrm{Var}(\bar{Y}) - \mathrm{Var}(\bar{Y}_{proportional}) 
    = \frac{1}{n} \sum_{i=1}^N p_i (\mu_i - \mu)^2
    \]
    , neglecting the finite population correction.
\end{theorem}
\begin{proof}
    We have already obtained 
    $
    \frac{1}{n} \sum_{i=1}^N p_i \sigma_i^2 = \frac{1}{n} \sigma^2 - \frac{1}{n} \mathrm{Var}(\mu_i)
    $
    in the proof of Theorem.\ref{thm:proportional_allocation}.
    Therefore, we have
    $
    \mathrm{Var}(\bar{Y}) - \mathrm{Var}(\bar{Y}_{proportional}) = \frac{1}{n} \mathrm{Var}(\mu_i)
    = \frac{1}{n} \sum_{i=1}^N p_i (\mu_i - \mu)^2
    $.   
\end{proof}

Combining Theorem.\ref{thm:comparing_variances_of_Neyman_and_proportional_stratified_sampling} and Theorem.\ref{thm:comparing_variances_of_simpleRandom_and_proportional_stratified_sampling}, we therefore note that, stratified random sampling with proportional allocation always gives a \textit{smaller} variance than does simple random sampling, providing that the finite population correction is ignored \cite{rice_mathematical_2007}. Further, in general we have:
\[
 \mathrm{Var}(\bar{Y}_{Neyman}) \leq \mathrm{Var}(\bar{Y}_{proportional}) \leq \mathrm{Var}(\bar{Y})
\]

\paragraph{Asymptotic normality of stratified samples} 
Noting the variance of stratified samples (Eq.\ref{eq:sample_variance_of_stratified_sampler}):
\[
\mathrm{Var}(\bar{Y}) \approx \sum_{i=1}^N \frac{p_i^2 \sigma_i^2}{n_i}
= \frac{1}{n}\sum_{i=1}^N \frac{p_i^2 \sigma_i^2}{\alpha_i},
\qquad \alpha_i:=\frac{n_i}{n} \in (0,1)
\]
as per \textit{central limit theorem}, we have
\[
    \sqrt{n} (\bar{Y} - \mu) \xrightarrow{d} 
\mathcal{N}  \left(0, \sum_{i=1}^N \frac{p_i^2\sigma_i^2}{\alpha_i}\right)
\]

\paragraph{A simple concentration bound (bounded strata)}
If $Y \in [a_i, b_i]$ almost surely on stratum $i$, then applying 
\textit{Hoeffding's inequality} \cite{hoeffding_probability_1963} 
for independent bounded means to the stratified sample mean yields the following 
direct proposition (cf.\ \cite[Sec.2]{boucheron_concentration_2013}; see also \cite{cochran_sampling_1977} for the variance case).

\begin{proposition}[\textit{Hoeffding-type concentration bound for stratified sampling}]
\label{prop:concentration_bound_for_stratified_sampling}
Let $\mu = \mathbb{E}[Y]$ and $\bar{Y}$ be the stratified sample mean based on 
$n_i$ i.i.d. samples from each stratum $i$ of probability mass 
$p_i = P(Y \in A_i)$. 
If $Y \in [a_i, b_i]$ almost surely on $A_i$, then for any error $\varepsilon > 0$,
\[
\Pr  \left(|\bar{Y} - \mu| \ge \varepsilon\right)
\le 
2\exp  \left(
-\frac{2 \varepsilon^2}{\sum_{i=1}^N \frac{p_i^2 (b_i-a_i)^2}{n_i}}
\right).
\]
\end{proposition}
\begin{proof}
This is an application of Hoeffding’s inequality for weighted sums of independent bounded variables.  
The steps follow \cite[Theorem.2.2]{boucheron_concentration_2013}, specialised to the case of independent strata means with weights $p_i$.
 
On the population level, we have
\[
\mu = \mathbb{E}[Y] = \sum_{i=1}^N p_i \mu_i, 
\quad \mu_i := \mathbb{E}[Y \mid Y \in A_i]
\]
Draw $n_i$ i.i.d. samples $Y_{i1},\dots,Y_{in_i}$ from $Y \mid Y \in A_i$, independently across strata, and assume
\[
Y_{ij} \in [a_i,b_i] \quad \text{a.s. for all } i,j
\]
The stratified estimator is
\[
\bar{Y} = \sum_{i=1}^N p_i \bar{Y}_i,
\qquad
\bar{Y}_i := \frac{1}{n_i}\sum_{j=1}^{n_i} Y_{ij},
\]
with $\mathbb{E}[\bar{Y}] = \mu$.

The estimation error can be written as
\[
\bar{Y} - \mu
= \sum_{i=1}^N p_i(\bar{Y}_i - \mu_i)
= \sum_{i=1}^N \sum_{j=1}^{n_i} \frac{p_i}{n_i} (Y_{ij} - \mu_i)
=: \sum_{i=1}^N \sum_{j=1}^{n_i} X_{ij}
\]
The variables $X_{ij}$ are independent, mean-zero, and bounded:
\[
Y_{ij} - \mu_i \in [a_i - \mu_i,  b_i - \mu_i]
\quad \Rightarrow \quad
X_{ij} \in \left[ -\frac{p_i}{n_i}(b_i-a_i),  \frac{p_i}{n_i}(b_i-a_i) \right]
\]
so each has range length
\[
\mathrm{range}(X_{ij}) = \frac{2 p_i}{n_i}(b_i-a_i)
\]

Hoeffding’s inequality for sums of independent bounded mean-zero variables states that \footnote{
Here we used a flattened form of Hoeffding’s inequality. The original Hoeffding’s inequality \cite{boucheron_concentration_2013}, which applies to independent sequences, states that \cite{boucheron_concentration_2013}:
If $X_1, \dots, X_n$ are independent, with $X_k \in [a_k, b_k]$ a.s.
Let$S = \sum_{k=1}^n (X_k - \mathbb{E}X_k)$, then for any $t > 0$, we have $\Pr(S \ge t)  \le  \exp  \left( -\frac{2t^2}{\sum_{k=1}^n (b_k - a_k)^2} \right)$.
It does not care how the $X_k$ are labeled, i.e. we can have double summation running over two indices, flatten it into single summation and re-label it (which just expands the terms for summation), then apply Hoeffding’s inequality.
For our case, we further need the two-sided Hoeffding’s inequality: 
if $X_k$ are bounded in $[a_k, b_k]$, then $-X_k$ are bounded in $[-b_k, -a_k]$ with the same range lengths $b_k-a_k$.
The one-sided bound applies equally to $\Pr(S \le -t)$. Therefore, 
$\Pr(|S| \ge t)  \le  \exp  \left( -\frac{2t^2}{\sum_{k=1}^n (b_k - a_k)^2} \right)
 +  \exp  \left( -\frac{2t^2}{\sum_{k=1}^n (b_k - a_k)^2} \right)
=2 \exp  \left( -\frac{2t^2}{\sum_{k=1}^n (b_k - a_k)^2} \right)
$.
}, for any $\varepsilon>0$,
\[
\Pr  \left( \left| \sum_{i,j} X_{ij} \right| \ge \varepsilon \right)
\le 2\exp  \left( -\frac{2 \varepsilon^2}{\sum_{i,j} \mathrm{range}(X_{ij})^2} \right)
\]
Summing over $j$ in each stratum yields
\[
\sum_{i,j} \mathrm{range}(X_{ij})^2
= \sum_{i=1}^N n_i \left( \frac{2 p_i}{n_i}(b_i-a_i) \right)^2
= \sum_{i=1}^N \frac{4 p_i^2 (b_i-a_i)^2}{n_i}
\]
Substituting into the Hoeffding bound gives
\[
\Pr  \left(|\bar{Y} - \mu| \ge \varepsilon\right)
\le 
2\exp  \left(
-\frac{2 \varepsilon^2}{\sum_{i=1}^N \frac{4 p_i^2 (b_i-a_i)^2}{n_i}}
\right)
= 
2\exp  \left(
-\frac{\varepsilon^2}{2 \sum_{i=1}^N \frac{p_i^2 (b_i-a_i)^2}{n_i}}
\right)
\]
\end{proof}

This theorem gives tighter \footnote{
If we apply Hoeffding's inequality directly to the $N$ terms $p_i(\bar Y_i-\mu_i)$, we only know $\bar{Y}_i\in[a_i,b_i]$, so each term has range $p_i(b_i-a_i)$. Hoeffding then gives the (valid but loose) bound$\Pr  \big(|\bar Y-\mu|\ge \varepsilon\big) \le 2\exp  \Bigg(-\frac{2\varepsilon^2}{\sum_{i=1}^N p_i^2(b_i-a_i)^2}\Bigg)$, which does not improve with $n_i$. That’s because bounding $\bar Y_i$ by its range ignores that $\bar{Y}_i$ is an average of $n_i$ bounded i.i.d. variables and therefore concentrates faster.} concentration when the within-stratum ranges (or variances) are small and/or allocation $n_i$ is chosen well.

\paragraph{Convergence rate from the concentration bound.}
Let $Z = \bar{Y} - \mu$. With a fixed allocation $n_i = \alpha_i n$ ($\alpha_i>0,\ \sum_{i=1}^N \alpha_i=1$), Proposition.\ref{prop:concentration_bound_for_stratified_sampling} implies
\begin{equation} \label{eq:range_based_constant}
    \Pr(|Z| \ge t) \le 2\exp  \left(-\frac{n  t^2}{C}\right),
    \qquad
    C := 2 \sum_{i=1}^N \frac{p_i^2 (b_i-a_i)^2}{\alpha_i}.
\end{equation}
Using the tail-integral identity for the nonnegative variable $X = Z^2$:
\[
\mathrm{RMSE}^2 = \mathbb{E}[Z^2] 
= \int_{0}^{\infty} \Pr(|Z| \ge t) 2t  dt
\]
and plugging in the Hoeffding bound gives
\[
\mathrm{RMSE}^2
 \le  \int_{0}^{\infty} 2 e^{-\frac{n t^2}{C}} \cdot 2t   dt
= 4 \int_{0}^{\infty} e^{-\frac{n t^2}{C}}   t   dt.
\]
With the change of variables $u = \tfrac{n}{C}t^2$, $du = \tfrac{2n}{C} t  dt$, we get $t dt = \tfrac{C}{2n}   du$, so
\[
\mathrm{RMSE}^2
 \le  4 \cdot \frac{C}{2n} \int_{0}^{\infty} e^{-u} du
= \frac{2C}{n}
\]
Therefore,
\begin{equation} \label{eq:range_based_RMSE_general}
    \mathrm{RMSE} \le \sqrt{\frac{2C}{n}} = \mathcal{O}(n^{-1/2}).
\end{equation}
Thus, stratified sampling matches the standard Monte Carlo convergence rate \footnote{
In Monte Carlo methods, the estimation error is $\mathcal{O}(\sigma / \sqrt{n})$, where $\sigma$ is the standard deviation of the quantity being estimated. For direct Monte Carlo, $\sigma$ is simply the standard deviation of the target variable, while in MCMC it is more complex. Accuracy can be improved by increasing the sample size $n$ or reducing $\sigma$ through variance reduction techniques \cite{kennedy2016montecarlo}.
} but improves the constant $\sqrt{C}$ whenever strata are homogeneous and sample allocation is well chosen. 
'Improves the constant means: stratification + good allocation does not change the slope of the error decay curve on a log-log plot (still slope $-1/2$), but it lowers the curve vertically because $\sqrt{C}$ is smaller.
$C$ aggregates the within-stratum ranges $(b_i-a_i)$, the strata probabilities $p_i$, and the allocation fractions $\alpha_i$. If strata are homogeneous (i.e. small within-stratum ranges), then each $(b_i-a_i)$ is small, and so $C$ is small.
By choosing $n_i = \alpha_i n$ to minimize $C$ (e.g. $\alpha_i \propto p_i(b_i-a_i)$ for range-based bounds, or $\alpha_i \propto p_i \sigma_i$ for variance-based Neyman allocation), we further reduce $C$ compared to naive proportional allocation.
When $C$ is much smaller due to homogeneity and optimal allocation, the same $n$ gives a much smaller error bound than standard Monte Carlo.

\begin{corollary}[Range-based optimal allocation]
\label{cor:optimal_range_allocation}
Let $C(\alpha)=2 \sum_{i=1}^N \dfrac{p_i^2 (b_i-a_i)^2}{\alpha_i}$ 
with $\alpha_i>0$ and $\sum_{i=1}^N \alpha_i=1$. 
Define $c_i:=p_i (b_i-a_i)$. 
Then $C(\alpha)$ is minimized by
\[
\alpha_i^\star = \frac{c_i}{\sum_{j=1}^N c_j},
\]
with minimal value
\[
C_{\min} = 2 \Big(\sum_{i=1}^N p_i(b_i-a_i)\Big)^2,
\]
yielding
\[
\mathrm{RMSE} \le \frac{2 \sum_{i=1}^N p_i(b_i-a_i)}{\sqrt{n}}.
\]
\end{corollary}
\begin{proof} (an alternative proof routine is via Lagrange multipliers.)

Note
\[
C(\alpha)=2\sum_{i=1}^N \frac{c_i^2}{\alpha_i}
\quad\text{with}\quad
\alpha_i>0,\ \sum_{i=1}^N \alpha_i=1.
\]
By Cauchy-Schwarz applied to the vectors 
$\big(\tfrac{c_i}{\sqrt{\alpha_i}}\big)_{i=1}^N$ and $\big(\sqrt{\alpha_i}\big)_{i=1}^N$,
\[
\Big(\sum_{i=1}^N c_i\Big)^2
=\Big(\sum_{i=1}^N \tfrac{c_i}{\sqrt{\alpha_i}}\cdot \sqrt{\alpha_i}\Big)^2
\le \Big(\sum_{i=1}^N \tfrac{c_i^2}{\alpha_i}\Big)\Big(\sum_{i=1}^N \alpha_i\Big)
=\sum_{i=1}^N \tfrac{c_i^2}{\alpha_i}.
\]
Multiplying both sides by $2$ gives the universal lower bound
\[
C(\alpha) \ge  2\Big(\sum_{i=1}^N c_i\Big)^2.
\]
Equality in Cauchy-Schwarz holds \textit{iff} there exists $\lambda>0$ such that
$\tfrac{c_i}{\sqrt{\alpha_i}}=\lambda\sqrt{\alpha_i}$ for all $i$, i.e.
$\alpha_i \propto c_i$. Enforcing $\sum_i \alpha_i=1$ yields the minimizer
\[
\alpha_i^\star=\frac{c_i}{\sum_{j=1}^N c_j}
\quad (i=1,\dots,N),
\]
and the minimal value
\[
C_{\min}=2\Big(\sum_{i=1}^N c_i\Big)^2
=2\Big(\sum_{i=1}^N p_i(b_i-a_i)\Big)^2.
\]
Using the bound $\mathrm{RMSE}\le \sqrt{2C/n}$ from the concentration argument then gives
\begin{equation} \label{eq:range_based_RMSE_optimal}
\mathrm{RMSE} \le \sqrt{\frac{2C_{\min}}{n}}
=\frac{2\sum_{i=1}^N p_i(b_i-a_i)}{\sqrt{n}}.
\end{equation}
\textit{Edge case:} if some $c_i=0$, the optimal rule sets $\alpha_i^\star=0$ for those $i$
(achieved as a limit within the constraint $\alpha_i>0$), and the same formulas hold.
\end{proof}

\begin{remark}[Equal-range case: proportional allocation is range-optimal]
Assume all strata have the same bounded range, $b_i-a_i \equiv r$. 
Then in the range-based Hoeffding constant
\[
C(\alpha)=2\sum_{i=1}^N \frac{p_i^2(b_i-a_i)^2}{\alpha_i},
\]
we have $c_i:=p_i(b_i-a_i)=p_i r$, so the minimizer in Corollary.\ref{cor:optimal_range_allocation} is 
$\alpha_i^\star=\dfrac{c_i}{\sum_j c_j}=\dfrac{p_i r}{r}=p_i$, i.e. proportional allocation. 
The minimal constant is
\[
C_{\min}=2\Big(\sum_{i=1}^N p_i r\Big)^2=2r^2,
\]
so the Hoeffding-based bound yields 
$\mathrm{RMSE}\le \sqrt{2C_{\min}/n}= \dfrac{2r}{\sqrt{n}}$.

\textit{Alternative check from proportional variance:}
under proportional allocation, Theorem.\ref{thm:proportional_allocation} gives 
$\mathrm{Var}(\bar Y)=\frac{1}{n}\sum_i p_i\sigma_i^2$. 
With bounded support and equal range, Popoviciu's inequality on variances implies 
$\sigma_i^2\le r^2/4$, hence 
$\mathrm{Var}(\bar Y)\le \frac{1}{n}\sum_i p_i (r^2/4)=\frac{r^2}{4n}$, 
which is consistent with the Hoeffding-based constant $C_{\min}=2r^2$ since 
$\Pr(|\bar Y-\mu|\ge t)\le 2\exp  \big(-{nt^2}/{(2r^2)}\big)$ integrates to an 
RMSE of order $2r/\sqrt{n}$.
\end{remark}

\begin{corollary}[Variance-based, Neyman optimal allocation \cite{neyman_two_1992}]
\label{cor:neyman_allocation_RMSE}
Under the variance-based, Neyman optimal allocation (Theorem.\ref{thm:Neyman_allocation}), 
stratified sampling achieves the standard Monte Carlo rate $\mathcal{O}(n^{-1/2})$ 
with the smallest possible RMSE constant $\sum_{i=1}^N p_i \sigma_i$:
\[
\mathrm{RMSE} \approx \frac{\sum_{i=1}^N p_i \sigma_i}{\sqrt{n}}.
\]
\end{corollary}
\begin{proof}
From Theorem.\ref{thm:Neyman_allocation}, the variance of the stratified sample mean 
under the Neyman optimal allocation is
\[
\mathrm{Var}(\bar{Y}) = \frac{1}{n} \left( \sum_{i=1}^N p_i \sigma_i \right)^2.
\]
Taking square roots gives the root-mean-square error
\begin{equation} \label{eq:Neyman_RMSE}
\mathrm{RMSE} = \sqrt{\mathrm{Var}(\bar{Y})} 
= \frac{\sum_{i=1}^N p_i \sigma_i}{\sqrt{n}}.
\end{equation}

The $n^{-1/2}$ dependence confirms the standard Monte Carlo convergence rate, and 
$\sum_{i=1}^N p_i\sigma_i$ is the minimal RMSE constant achievable under any allocation, 
by optimality of the Neyman allocation.
\end{proof}

\begin{remark}[Range-based \textit{vs} variance-based allocation]
Both Corollary.\ref{cor:optimal_range_allocation} and Corollary.\ref{cor:neyman_allocation_RMSE}
minimize $\sum_i w_i^2 / \alpha_i$ subject to $\sum_i \alpha_i=1$,
with $\alpha_i \propto w_i$.  
The difference lies in the choice of $w_i$:
\begin{itemize}
    \item \textit{Hoeffding-bound (range-based):} $w_i = p_i (b_i-a_i)$, robust when only range bounds are known.
    \item \textit{Variance-based (Neyman):} $w_i = p_i \sigma_i$, optimal when variances are available.
\end{itemize}
\end{remark}

\begin{remark}[Proportional allocation as a special case]
From Corollary.\ref{cor:optimal_range_allocation}, the range-based optimal allocation satisfies 
$\alpha_i^\star \propto p_i (b_i-a_i)$.  
If all strata have the same range, $b_i-a_i \equiv r$ for all $i$, then
$\alpha_i^\star \propto p_i r = p_i$, 
so the range-based optimal allocation reduces to the proportional allocation 
of Theorem.\ref{thm:proportional_allocation}.  
Similarly, in the variance-based Neyman allocation 
(Theorem.\ref{thm:Neyman_allocation}), if all standard deviations 
$\sigma_i$ are equal, then $\alpha_i^\star \propto p_i\sigma_i = p_i$, 
again recovering proportional allocation.  
Thus, proportional allocation is a special case of both optimal rules when strata are homogeneous.
\end{remark}

We now illustrate the theoretical results from Corollaries.\ref{cor:optimal_range_allocation} and.\ref{cor:neyman_allocation_RMSE}
using a synthetic stratified sampling experiment.  
4 strata are specified with probabilities $p_i$, ranges $(b_i-a_i)$, and within-stratum standard deviations $\sigma_i$ given in Table.\ref{tab:strata_params}.  
For each allocation rule, i.e. proportional (Theorem.\ref{thm:proportional_allocation}), range-optimal (Corollary.\ref{cor:optimal_range_allocation}), and variance-based Neyman (Corollary.\ref{cor:neyman_allocation_RMSE}), we compute the RMSE constant from the corresponding closed-form formula, as summarised in Table.\ref{tab:rmse_constants}.  
We then plot the predicted $\mathrm{RMSE} \approx \text{RMSE const}/\sqrt{n}$ curves \footnote{
The RMSE constant here refers to the denominator $\sqrt{2C}$ in Eq.\ref{eq:range_based_RMSE_general}, $2\sum_{i=1}^N p_i(b_i-a_i)$ in Eq.\ref{eq:range_based_RMSE_optimal} and $\sum_{i=1}^N p_i \sigma_i$ in Eq.\ref{eq:Neyman_RMSE}.
} over sample sizes $n\in[10^2,10^5]$ on both linear and log-log axes to compare their absolute accuracy levels (vertical offsets) and verify the common $\mathcal{O}(n^{-1/2})$ convergence rate.

\begin{figure}[H]
    \centering
    \includegraphics[width=0.95\linewidth]{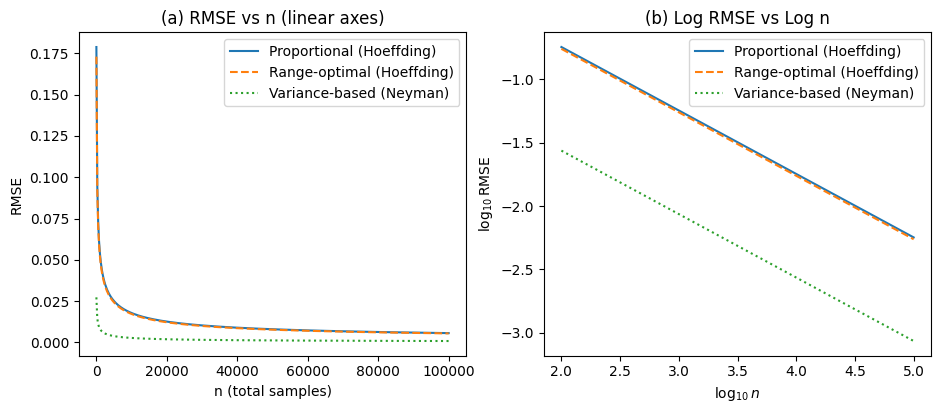}
    \caption{RMSE (=$\text{RMSE const}/\sqrt{n}$) values under different stratified allocations computed using the strata parameters in Table.\ref{tab:strata_params} and constants in Table.\ref{tab:rmse_constants}. 
    Panel (a) shows RMSE decay versus total sample size $n$ on linear axes; 
    panel (b) shows the same data on a log-log scale, highlighting the $\mathcal{O}(n^{-1/2})$ convergence rate with different vertical offsets determined by the RMSE constants. 
    Variance-based (Neyman) allocation achieves the smallest constant, followed by range-optimal allocation, with proportional allocation having the largest.}
    \label{fig:rmse_allocation_dual}
\end{figure}

\begin{table}[H]
\centering
\caption{Example strata parameters used in Fig..\ref{fig:rmse_allocation_dual}.}
\label{tab:strata_params}
\begin{tabular}{lccc}
\toprule
Stratum $i$ & $p_i$ & range $(b_i-a_i)$ & $\sigma_i$ \\
\midrule
1 & 0.10 & 0.50 & 0.12 \\
2 & 0.25 & 1.20 & 0.40 \\
3 & 0.35 & 0.70 & 0.22 \\
4 & 0.30 & 0.90 & 0.28 \\
\bottomrule
\end{tabular}
\end{table}

\begin{table}[H]
\centering
\small
\caption{RMSE constants under different allocations (same $n$). 
Two-sided Hoeffding is used for bounded-range cases, so 
$\mathrm{RMSE} \lesssim \text{RMSE const}/\sqrt{n}$.}
\label{tab:rmse_constants}
\begin{tabular}{lcc}
\toprule
Allocation & Constant & RMSE const. \\
\midrule
Proportional (Hoeffding) 
& $C_{\mathrm{H,prop}}=2  \sum_i p_i(b_i-a_i)^2=1.5990$
& $2\sqrt{\sum_i p_i(b_i-a_i)^2}=1.7882$ \\
Range-optimal (Hoeffding) 
& $C_{\mathrm{H,min}}=2\big(\sum_i p_i(b_i-a_i)\big)^2=1.49645$
& $2\sum_i p_i(b_i-a_i)=1.7300$ \\
Variance-based (Neyman) 
& $V_{\mathrm{Neyman}}=(\sum_i p_i\sigma_i)^2=0.074529$
& $\sum_i p_i\sigma_i=0.2730$ \\
\bottomrule
\end{tabular}
\end{table}

\noindent\textbf{Derivation of constants in Table.\ref{tab:rmse_constants}}
For the two-sided Hoeffding bound with allocation $\alpha=\{\alpha_i\}_{i=1}^N$, the tail constant is (Eq.\ref{eq:range_based_constant})
\[
C(\alpha) = 2\sum_{i=1}^N \frac{p_i^2(b_i-a_i)^2}{\alpha_i},
\]
and integrating the bound yields the RMSE estimate (Eq.\ref{eq:range_based_RMSE_general})
\[
\mathrm{RMSE}  \le  \sqrt{\frac{2 C(\alpha)}{n}}.
\]
(i) \textit{Proportional (Hoeffding)}: Setting $\alpha_i=p_i$ in the above $C(\alpha)$ gives
\[
C_{\mathrm{H,prop}}=2\sum_{i=1}^N p_i(b_i-a_i)^2.
\]
Using Table.\ref{tab:strata_params}, $\sum_i p_i(b_i-a_i)^2 = 0.7995$, hence
$C_{\mathrm{H,prop}}=1.5990$ and RMSE constant $=2\sqrt{0.7995}\approx 1.7882$.

(ii) \textit{Range-optimal (Hoeffding)}: From Corollary.\ref{cor:optimal_range_allocation} (also Eq.\ref{eq:range_based_RMSE_optimal}),
\[
C_{\mathrm{H,min}} = 2\big(\sum_{i=1}^N p_i(b_i-a_i)\big)^2.
\]
Here $\sum_i p_i(b_i-a_i) = 0.8650$, so $C_{\mathrm{H,min}}=1.49645$ and RMSE constant $=1.7300$.

(iii) \textit{Variance-based (Neyman)}: From Theorem.\ref{thm:Neyman_allocation} (also Eq.\ref{eq:Neyman_RMSE}),
\[
V_{\mathrm{Neyman}} = \big(\sum_{i=1}^N p_i\sigma_i\big)^2.
\]
Using Table.\ref{tab:strata_params}, $\sum_i p_i\sigma_i = 0.2730$, giving $V_{\mathrm{Neyman}}=0.074529$ and RMSE constant $=0.2730$ (Corollary.\ref{cor:neyman_allocation_RMSE}).

\begin{remark}[Link to allocation theory]
The constants in Table.\ref{tab:rmse_constants} correspond to the 
$\sqrt{C}$ and $\sqrt{V}$ terms from Corollaries.\ref{cor:optimal_range_allocation} 
and.\ref{cor:neyman_allocation_RMSE}.  
Proportional allocation is a special case of the range-based optimal allocation when all strata have identical ranges, $b_i-a_i \equiv r$, in which case $\alpha_i^\star \propto p_i$ and $C_{\mathrm{H,prop}} = 2 r^2$.  
Figure.\ref{fig:rmse_allocation_dual} confirms the theory: 
all allocations decay as $\mathcal{O}(n^{-1/2})$, but with different vertical offsets, smallest for Neyman allocation and largest for proportional allocation.
\end{remark}

\section{Improved WGMA sampling} \label{app:improved_WGMA_algorithm}

Following the discussion in Section.\ref{sec:methodology}, here we implement some efficient strategies for improving the GMA sampling algorithm, for example: (1) pre-computes the GMM density values $\mathcal{N}(\mathbf{s}_{i,j}; \boldsymbol{\mu}_i, \Sigma_i)$, and target density values $\log\bar{p}(\mathbf{s}_{i,j})$. (2) Replace the gradient (\ref{eq:KL_divergence_objective4_gradient_GMM1_MC_cc}) by its Monte Carlo estimator (\ref{eq:KL_divergence_objective4_gradient_GMM1_MC_cc}). The first improved GMA variant is presented in Section.\ref{app:precomputing_density_values}, the second variant is presented in Section.\ref{app:MC_gradient_estimator}. We further combine both strategies, yielding the optimal GMA sampler, in Section.\ref{app:combine_precomputing_and_MC_gradient_estimator}. As a companion, we also present a naive, heuristic, GD based GMA sampling method in Section.\ref{app:GMA_with_GD_and_heuristics}.

\subsection{WGMA sampling algorithm with GD and heuristics} \label{app:GMA_with_GD_and_heuristics}
Below we present the GMA sampling which uses standard GD with heuristic clipping and re-normalisation in each iteration.

The overall computational complexity of this heuristic, GD based GMA-sampling algorithm is the same as the pGD based GMA algorithm in Algo.\ref{algo:WGMA-sampling-pgd}, i.e. $\mathcal{O}(K N^2 M d^2)$, dominated by the nested loops within the iterative weight update (Step 4).

\begin{algorithm}[H]
\footnotesize
\caption{WGMA-sampling: sampling via Gaussian mixture approximation (with GD and heuristics)}
\label{algo:GMA-sampling-heuristic}
\textbf{Input:} Number of Gaussian components $N$; number of samples per component $M$; number of iterations $K$; target unnormalised density \(\bar{p}(\mathbf{z})\); initial means \(\{\boldsymbol{\mu}_i\}_{i=1}^N\); initial covariance matrices \(\{\Sigma_i\}_{i=1}^N\); initial learning rate \(\eta_0\). \\
\textbf{Output:} Ensemble of selected samples \(\{\mathbf{s}_{\text{selected}}\}\) approximating samples from $p(\mathbf{z})$.

\vspace{1mm}\hrule\vspace{1mm}

\begin{algorithmic}[1]
\STATE Initialize an empty set \(\mathcal{S} = \{\}\) to store selected samples. \hfill\textit{\(\mathcal{O}(1)\)}
\STATE Initialize weight vector \(\mathbf{w}^{(0)} = [w_1^{(0)}, \ldots, w_N^{(0)}]^\top\) with $w_i^{(0)} \sim \mathcal{U}(0, 1)$ and normalize \(\sum_{i=1}^N w_i^{(0)} = 1\). \hfill\textit{\(\mathcal{O}(N)\)}
\STATE Draw $M$ samples \(\{\mathbf{s}_{i,j}\}_{j=1}^M\) from each Gaussian \(\mathcal{N}(\boldsymbol{\mu}_i, \Sigma_i)\) for $i = 1, 2, \ldots, N$ using standard Gaussian sampling (e.g. \(\boldsymbol{\epsilon} \sim \mathcal{N}(0, I)\), \(\mathbf{s}_{i,j} = \boldsymbol{\mu}_i + L_i \boldsymbol{\epsilon}\), where $L_i L_i^\top = \Sigma_i$). \hfill\textit{\(\mathcal{O}(N M d^2)\), where $d$ is dimension}
\FOR{$k = 1$ to $K$}
    \FOR{$i = 1$ to $N$}
        \STATE Compute GMM density \(q_{\mathbf{w}}^{(k-1)}(\mathbf{s}_{i,j}) = \sum_{i=1}^N w_l^{(k-1)} \cdot \mathcal{N}(\mathbf{s}_{i,j}; \boldsymbol{\mu}_l, \Sigma_l)\) for all $j = 1, 2, \ldots, M$. \hfill\textit{\(\mathcal{O}(N M d^2)\)}
        \STATE Compute gradient component for $w_i$ (\ref{eq:KL_divergence_objective4_gradient_GMM1_cc}): 
        \[
        g_i = 1 +  \sum_{j=1}^M \mathcal{N}(\mathbf{s}_{i,j}; \boldsymbol{\mu}_i, \Sigma_i) [\log q_{\mathbf{w}}^{(k-1)}(\mathbf{s}_{i,j}) - \log \bar{p}(\mathbf{s}_{i,j})]
        \]
        \hfill\textit{\(\mathcal{O}(M d^2)\)}
    \ENDFOR
    \STATE Update weight vector \(\mathbf{w}^{(k)} = \mathbf{w}^{(k-1)} - \frac{\eta_0}{k} \cdot \nabla_{\mathbf{w}} KL(q_{\mathbf{w}}(\mathbf{z}) \| p(\mathbf{z}))\), where \(\nabla_{\mathbf{w}} KL = [g_1, g_2, \ldots, g_N]^\top\). Ensure $0 \leq w_i \leq 1$ via e.g. clipping.  \hfill\textit{\(\mathcal{O}(N)\)}
    \STATE Re-normalize weights: \(\mathbf{w}^{(k)} = \frac{\mathbf{w}^{(k)}}{\sum_{i=1}^N w_i^{(k)}}\) to ensure \(\sum_{i=1}^N w_i^{(k)} = 1\). \hfill\textit{\(\mathcal{O}(N)\)}
\ENDFOR
\STATE Set final weights \(\mathbf{w}^* = \mathbf{w}^{(K)}\). \hfill\textit{\(\mathcal{O}(1)\)}
\STATE Compute component selection probabilities \(\mathbf{p} = \mathbf{w}^* / \sum_{i=1}^N w_i^*\). \hfill\textit{\(\mathcal{O}(N)\)}
\STATE Generate ensemble samples: for each $m = 1$ to $N \cdot M$, draw index $i_m \sim \text{Categorical}(\mathbf{p})$ and append $\mathbf{s}_{i_m, j_m}$ (where $j_m \sim \text{Uniform}(\{1, 2, \ldots, M\})$) to \(\mathcal{S}\). \hfill\textit{\(\mathcal{O}(N M)\)}
\STATE Return the set \(\mathcal{S}\). \hfill\textit{\(\mathcal{O}(1)\)}
\end{algorithmic}
\end{algorithm}

\subsection{Pre-computing density values} \label{app:precomputing_density_values}

\begin{algorithm}[H]
\footnotesize
\caption{WGMA-sampling: sampling via Gaussian mixture approximation (with pGD and pre-computing density values)}
\label{algo:GMA-sampling-precomputing-pgd}
\textbf{Input:} Number of Gaussian components $N$; number of samples per component $M$; number of iterations $K$; target unnormalised density \(\bar{p}(\mathbf{z})\); initial means \(\{\boldsymbol{\mu}_i\}_{i=1}^N\); initial covariance matrices \(\{\Sigma_i\}_{i=1}^N\); initial learning rate \(\eta_0\). \\
\textbf{Output:} Ensemble of selected samples \(\{\mathbf{s}_{\text{selected}}\}\) approximating samples from $p(\mathbf{z})$.

\vspace{1mm}\hrule\vspace{1mm}

\begin{algorithmic}[1]
\STATE Initialize an empty set \(\mathcal{S} = \{\}\) for storing selected samples. \hfill\textit{\(\mathcal{O}(1)\)}
\STATE Initialize weight vector \(\mathbf{w}^{(0)}\) on the probability simplex, e.g. $w_i^{(0)} \sim \mathcal{U}(0, 1)$ and normalize. \hfill\textit{\(\mathcal{O}(N)\)}
\STATE Draw $M$ samples \(\{\mathbf{s}_{i,j}\}_{j=1}^M\) from each Gaussian \(\mathcal{N}(\boldsymbol{\mu}_i, \Sigma_i)\) for $i = 1, 2, \ldots, N$ using standard Gaussian sampling, e.g. \(\boldsymbol{\epsilon} \sim \mathcal{N}(0, I)\), \(\mathbf{s}_{i,j} = \boldsymbol{\mu}_i + L_i \boldsymbol{\epsilon}\), where $L_i L_i^\top = \Sigma_i$. \hfill\textit{\(\mathcal{O}(N M d^2)\)}
\STATE \textbf{Pre-compute GMM PDFs}: create a matrix \(\mathbf{P} \in \mathbb{R}^{(NM) \times N}\). For each sample \(\mathbf{s}_{i,j}\) and each component \(l\), compute \(P_{(i-1)M+j, l} = \mathcal{N}(\mathbf{s}_{i,j}; \boldsymbol{\mu}_l, \Sigma_l)\). \hfill\textit{\(\mathcal{O}(N^2 M d^2)\)}
\STATE \textbf{Pre-compute target densities}: create a vector \(\mathbf{p}_{\text{target}} \in \mathbb{R}^{NM}\). For each sample \(\mathbf{s}_{i,j}\), compute \(\left(\mathbf{p}_{\text{target}}\right)_{(i-1)M+j} = \log\bar{p}(\mathbf{s}_{i,j})\). \hfill\textit{\(\mathcal{O}(NM \cdot C_p)\)}
\FOR{$k = 1$ to $K$}
    \STATE Compute mixture densities for all samples: \(\mathbf{q}^{(k-1)} = \mathbf{P} \cdot \mathbf{w}^{(k-1)}\). \hfill\textit{\(\mathcal{O}(N^2 M)\)}
    \STATE Compute gradient vector \(\mathbf{g} = [g_1, \ldots, g_N]^\top\):
    \FOR{$i = 1$ to $N$}
        \STATE \(g_i = 1 +  \sum_{j=1}^M P_{(i-1)M+j, i} \cdot [\log q^{(k-1)}_{(i-1)M+j} - (\mathbf{p}_{\text{target}})_{(i-1)M+j}]\)
    \ENDFOR
    \hfill\textit{\(\mathcal{O}(N M)\)}
    \STATE Take gradient descent: \(\mathbf{v}^{(k)} = \mathbf{w}^{(k-1)} - \frac{\eta_0}{k} \cdot \mathbf{g}\). \hfill\textit{\(\mathcal{O}(N)\)}
    \STATE Project onto simplex: \(\mathbf{w}^{(k)} = \text{Proj}_{\Delta}(\mathbf{v}^{(k)})\), where \(\text{Proj}_{\Delta}\) is the projection onto the set \(\{\mathbf{w} | w_i \ge 0, \sum w_i = 1\}\). \hfill\textit{\(\mathcal{O}(N \log N)\)}
\ENDFOR
\STATE Set final weights \(\mathbf{w}^* = \mathbf{w}^{(K)}\), and component selection probabilities \(\mathbf{p} = \mathbf{w}^*\).  \hfill\textit{\(\mathcal{O}(1)\)}
\STATE Generate ensemble samples: for each $m = 1$ to $N \cdot M$, draw index $i_m \sim \text{Categorical}(\mathbf{p})$ and append $\mathbf{s}_{i_m, j_m}$ (where $j_m \sim \text{Uniform}(\{1, \ldots, M\})$) to \(\mathcal{S}\). \hfill\textit{\(\mathcal{O}(N M)\)}
\STATE Return the set \(\mathcal{S}\). \hfill\textit{\(\mathcal{O}(1)\)}
\end{algorithmic}
\end{algorithm}

The complexity is now best described by separating the one-time pre-computation cost from the iterative cost.

\paragraph{Pre-computation cost}
This is a significant, one-time cost incurred before the optimisation loop begins.
Sample generation (Step 3): generating $N \cdot M$ samples, where each generation involves a matrix-vector product of cost $\mathcal{O}(d^2)$, results in a total cost of $\mathcal{O}(N M d^2)$. 
PDF matrix (Step 4): this is the most expensive single step. We evaluate the PDF of each of the $N \cdot M$ samples under each of the $N$ Gaussian components. This requires $N^2 M$ evaluations of the Gaussian PDF, each costing $\mathcal{O}(d^2)$. The total cost is $\mathcal{O}(N^2 M d^2)$. 
In Step 5, we evaluate the unnormalised, target density $\bar{p}(\mathbf{z})$ at the $N M$ sample positions. This costs $\mathcal{O}(NM \cdot C_p)$, where the constant $C_p$ represents the one-time computational cost of evaluating $\bar{p}(\mathbf{z})$ for a single sample point $\mathbf{z}$ - it depends entirely on how complex the function $\bar{p}(\mathbf{z})$ is.
The total pre-computation cost is dominated by the PDF matrix calculation, making it $\mathcal{O}(N^2 M d^2)$.

\paragraph{Iterative cost}
By using the pre-computed matrix $\mathbf{P}$, the cost of each iteration (Steps 6-12) is drastically reduced.
Mixture density (Step 7): computing the vector $\mathbf{q}$ of mixture densities for all $N \cdot M$ samples is now a matrix-vector product between $\mathbf{P}$ ($(NM) \times N$) and $\mathbf{w}$ ($N \times 1$). This costs $\mathcal{O}((NM) \times N) = \mathcal{O}(N^2 M)$.
Gradient calculation (Steps 8-11): calculating the gradient vector $\mathbf{g}$ involves a loop over $N$ components, with an inner summation over $M$ samples. All values inside the sum are now simple lookups in the pre-computed tables. This step costs $\mathcal{O}(N M)$.
The cost per iteration is dominated by the mixture density calculation, making it $\mathcal{O}(N^2 M)$. The cost of the gradient projection step ($\mathcal{O}(N \log N)$) is subsumed by the much more expensive mixture density calculation ($\\mathcal{O}(N^2 M)$) that occurs within each iteration.

\paragraph{Overall complexity}
The total complexity of the optimised algorithm is the sum of the pre-computation and the total iterative costs:
\[
\text{Overall computational complexity} = \underbrace{\mathcal{O}(N^2 M d^2)}_{\text{Pre-computation}} + \underbrace{\mathcal{O}(K \cdot N^2 M)}_{\text{K iterations}}
\]

This is a significant improvement over the original algorithm's $\mathcal{O}(K N^2 M d^2)$. The expensive dependency on dimensionality $d^2$ has been moved outside the main loop and is no longer multiplied by the number of iterations $K$. This will result in a much faster implementation, especially for problems with high dimensionality or requiring many iterations.

\subsection{Using Monte Carlo gradient estimator} \label{app:MC_gradient_estimator}

Introducing the MC estimator (Eq.\ref{eq:KL_divergence_objective4_gradient_GMM1_MC_cc}) for gradient is theoretically insightful, it also marginally reduces the cost of the gradient calculation step in Line 8: originally in Algo.\ref{algo:WGMA-sampling-pgd} it was $\mathcal{O}(M d^2)$ because it evaluated $M$ Gaussian PDFs; the new Monte Carlo estimator has a cost of $\mathcal{O}(M \cdot C\_p)$, as it no longer evaluates those PDFs. $C_p$ again represents the one-time computational cost of evaluating $\bar{p}(\mathbf{z})$ for a single sample point $\mathbf{z}$ - it depends entirely on how complex the function $\bar{p}(\mathbf{z})$ is. 

The asymptotic complexity, however, doesn't change as the computational bottleneck has not been addressed. The main computational bottleneck is in Line 7, where the GMM density $q_{\mathbf{w}}$ is computed for all $M$ samples. This step, which costs $\mathcal{O}(N M d^2)$ for each component $i$, is still inside the $i$-loop. This means the total cost for each iteration $k$ is still dominated by $N \times \mathcal{O}(N M d^2) = \mathcal{O}(N^2 M d^2)$.

\begin{algorithm}[H]
\footnotesize
\caption{WGMA-sampling: sampling via Gaussian mixture approximation (with pGD and MC gradient estimator)}
\label{algo:GMA-sampling-MCEstimator}
\textbf{Input:} Number of Gaussian components $N$; number of samples per component $M$; number of iterations $K$; target unnormalised density \(\bar{p}(\mathbf{z})\); initial means \(\{\boldsymbol{\mu}_i\}_{i=1}^N\); initial covariance matrices \(\{\Sigma_i\}_{i=1}^N\); initial learning rate \(\eta_0\). \\
\textbf{Output:} Ensemble of selected samples \(\{\mathbf{s}_{\text{selected}}\}\) approximating samples from $p(\mathbf{z})$.

\vspace{1mm}\hrule\vspace{1mm}

\begin{algorithmic}[1]
\STATE Initialize an empty set \(\mathcal{S} = \{\}\) to store selected samples. \hfill\textit{\(\mathcal{O}(1)\)}
\STATE Initialize weight vector \(\mathbf{w}^{(0)}\) on the probability simplex, e.g. $w_i^{(0)} \sim \mathcal{U}(0, 1)$ and normalize. \hfill\textit{\(\mathcal{O}(N)\)}
\STATE Draw $M$ samples \(\{\mathbf{s}_{i,j}\}_{j=1}^M\) from each Gaussian \(\mathcal{N}(\boldsymbol{\mu}_i, \Sigma_i)\) for $i = 1, 2, \ldots, N$ using standard Gaussian sampling (e.g. \(\boldsymbol{\epsilon} \sim \mathcal{N}(0, I)\), \(\mathbf{s}_{i,j} = \boldsymbol{\mu}_i + L_i \boldsymbol{\epsilon}\), where $L_i L_i^\top = \Sigma_i$). \hfill\textit{\(\mathcal{O}(N M d^2)\), where $d$ is dimension}
\FOR{$k = 1$ to $K$}
    \STATE Compute gradient vector \(\mathbf{g} = [g_1, g_2, \ldots, g_N]^\top\):
    \FOR{$i = 1$ to $N$}
        \STATE Compute GMM density \(q_{\mathbf{w}}^{(k-1)}(\mathbf{s}_{i,j}) = \sum_{i=1}^N w_l^{(k-1)} \cdot \mathcal{N}(\mathbf{s}_{i,j}; \boldsymbol{\mu}_l, \Sigma_l)\) for all $j = 1, \ldots, M$. \hfill\textit{\(\mathcal{O}(N M d^2)\)}
        \STATE Compute gradient component using Monte Carlo estimator (\ref{eq:KL_divergence_objective4_gradient_GMM1_MC_cc}): 
        \[
        g_i = 1 + \frac{1}{M} \sum_{j=1}^M \left[ \log q_{\mathbf{w}}^{(k-1)}(\mathbf{s}_{i,j}) - \log \bar{p}(\mathbf{s}_{i,j}) \right]
        \]
        \hfill\textit{\(\mathcal{O}(M \cdot C_p)\)}
    \ENDFOR
    \STATE Take gradient descent: \(\mathbf{v}^{(k)} = \mathbf{w}^{(k-1)} - \frac{\eta_0}{k} \cdot \mathbf{g}\). \hfill\textit{\(\mathcal{O}(N)\)}
    \STATE Project onto simplex: \(\mathbf{w}^{(k)} = \text{Proj}_{\Delta}(\mathbf{v}^{(k)})\), where \(\text{Proj}_{\Delta}\) is the projection onto the set \(\{\mathbf{w} | w_i \ge 0, \sum w_i = 1\}\). \hfill\textit{\(\mathcal{O}(N \log N)\)}
\ENDFOR
\STATE Set final weights \(\mathbf{w}^* = \mathbf{w}^{(K)}\), and component selection probabilities \(\mathbf{p} = \mathbf{w}^*\). \hfill\textit{\(\mathcal{O}(1)\)}
\STATE Generate ensemble samples: for each $m = 1$ to $N \cdot M$, draw index $i_m \sim \text{Categorical}(\mathbf{p})$ and append $\mathbf{s}_{i_m, j_m}$ (where $j_m \sim \text{Uniform}(\{1, \ldots, M\})$) to \(\mathcal{S}\). \hfill\textit{\(\mathcal{O}(N M)\)}
\STATE Return the set \(\mathcal{S}\). \hfill\textit{\(\mathcal{O}(1)\)}
\end{algorithmic}
\end{algorithm}

\subsection{Combining pre-computation and MC gradient estimation} \label{app:combine_precomputing_and_MC_gradient_estimator}

To achieve a significant speedup, we combine the density pre-computation strategy with MC gradient estimation strategy. Pre-computing the PDF matrix (as in Section.\ref{app:precomputing_density_values}) removes the expensive $\mathcal{O}(K N^2 M d^2)$ term, while using the correct Monte Carlo gradient (as in Section.\ref{app:MC_gradient_estimator}) ensures the algorithm is theoretically sound.

\begin{algorithm}[H]
\footnotesize
\caption{WGMA-sampling: sampling via Gaussian mixture approximation (with pGD, pre-computing density values and MC estimator)}
\label{algo:GMA-sampling-optimal}
\textbf{Input:} Number of Gaussian components $N$; number of samples per component $M$; number of iterations $K$; target unnormalised density \(\bar{p}(\mathbf{z})\); initial means \(\{\boldsymbol{\mu}_i\}_{i=1}^N\); initial covariance matrices \(\{\Sigma_i\}_{i=1}^N\); initial learning rate \(\eta_0\). \\
\textbf{Output:} Ensemble of selected samples \(\{\mathbf{s}_{\text{selected}}\}\) approximating samples from $p(\mathbf{z})$.

\vspace{1mm}\hrule\vspace{1mm}

\begin{algorithmic}[1]
\STATE Initialize an empty set \(\mathcal{S} = \{\}\) for storing selected samples. \hfill\textit{\(\mathcal{O}(1)\)}
\STATE Initialize weight vector \(\mathbf{w}^{(0)}\) on the probability simplex, e.g. $w_i^{(0)} \sim \mathcal{U}(0, 1)$ and normalize. \hfill\textit{\(\mathcal{O}(N)\)}
\STATE Draw $M$ samples \(\{\mathbf{s}_{i,j}\}_{j=1}^M\) from each Gaussian \(\mathcal{N}(\boldsymbol{\mu}_i, \Sigma_i)\) for $i = 1, 2, \ldots, N$ using standard Gaussian sampling (e.g. \(\boldsymbol{\epsilon} \sim \mathcal{N}(0, I)\), \(\mathbf{s}_{i,j} = \boldsymbol{\mu}_i + L_i \boldsymbol{\epsilon}\), where $L_i L_i^\top = \Sigma_i$). \hfill\textit{\(\mathcal{O}(N M d^2)\), where $d$ is dimension}
\STATE \textbf{Pre-compute GMM PDFs}: create a matrix \(\mathbf{P} \in \mathbb{R}^{(NM) \times N}\). For each sample \(\mathbf{s}_{i,j}\) and each component \(l\), compute \(P_{(i-1)M+j, l} = \mathcal{N}(\mathbf{s}_{i,j}; \boldsymbol{\mu}_l, \Sigma_l)\). \hfill\textit{\(\mathcal{O}(N^2 M d^2)\)}
\STATE \textbf{Pre-compute target densities}: create a vector \(\mathbf{p}_{\text{target}} \in \mathbb{R}^{NM}\). For each sample \(\mathbf{s}_{i,j}\), compute \(\left(\mathbf{p}_{\text{target}}\right)_{(i-1)M+j} = \log\bar{p}(\mathbf{s}_{i,j})\). \hfill\textit{\(\mathcal{O}(NM \cdot C_p)\)}
\FOR{$k = 1$ to $K$}
    \STATE Compute mixture densities for all samples: \(\mathbf{q}^{(k-1)} = \mathbf{P} \cdot \mathbf{w}^{(k-1)}\). \hfill\textit{\(\mathcal{O}(N^2 M)\)}
    \STATE Compute gradient vector \(\mathbf{g} = [g_1, \ldots, g_N]^\top\):
    \FOR{$i = 1$ to $N$}
        \STATE \(g_i = 1 + \frac{1}{M} \sum_{j=1}^M \left[ \log q^{(k-1)}_{(i-1)M+j} - (\mathbf{p}_{\text{target}})_{(i-1)M+j} \right]\)
    \ENDFOR
    \hfill\textit{\(\mathcal{O}(N M)\)}
    \STATE Take gradient step: \(\mathbf{v}^{(k)} = \mathbf{w}^{(k-1)} - \frac{\eta_0}{k} \cdot \mathbf{g}\). \hfill\textit{\(\mathcal{O}(N)\)}
    \STATE Project onto simplex: \(\mathbf{w}^{(k)} = \text{Proj}_{\Delta}(\mathbf{v}^{(k)})\), where \(\text{Proj}_{\Delta}\) is the projection onto the set \(\{\mathbf{w} | w_i \ge 0, \sum w_i = 1\}\). \hfill\textit{\(\mathcal{O}(N \log N)\)}
\ENDFOR
\STATE Set final weights \(\mathbf{w}^* = \mathbf{w}^{(K)}\), and component selection probabilities \(\mathbf{p} = \mathbf{w}^*\). \hfill\textit{\(\mathcal{O}(1)\)}
\STATE Generate ensemble samples: for each $m = 1$ to $N \cdot M$, draw index $i_m \sim \text{Categorical}(\mathbf{p})$ and append $\mathbf{s}_{i_m, j_m}$ (where $j_m \sim \text{Uniform}(\{1, \ldots, M\})$) to \(\mathcal{S}\). \hfill\textit{\(\mathcal{O}(N M)\)}
\STATE Return the set \(\mathcal{S}\). \hfill\textit{\(\mathcal{O}(1)\)}
\end{algorithmic}
\end{algorithm}

The complexity of this optimal algorithm is identical to the pre-computation strategy, as the change to the gradient estimator does not affect the dominant term in the iterative loop.

\paragraph{Density pre-computation cost}
The one-time cost is dominated by the creation of the PDF matrix $P$ (Step 4), which requires evaluating each of the $N \cdot M$ samples under each of the $N$ Gaussian components. This cost is $\mathcal{O}(N^2 M d^2)$.

\paragraph{Iterative cost}
Within each of the $K$ iterations, the cost is dominated by the matrix-vector product used to compute the mixture densities for all samples (Step 7). This costs $\mathcal{O}(N^2 M)$. The gradient calculation (Steps 8-11) is now a highly efficient $\mathcal{O}(NM)$ operation, as it only involves lookups and summations.

\paragraph{Overall complexity}
The total complexity is the sum of the pre-computation and the total iterative costs:
\[
\text{Overall computational complexity} = \underbrace{\mathcal{O}(N^2 M d^2)}_{\text{Pre-computation}} + \underbrace{\mathcal{O}(K \cdot N^2 M)}_{\text{K iterations}}
\]
This combined algorithm represents the best of all strategies: it is computationally efficient by moving the expensive, $d^2$-dependent calculations outside the main loop, and it is theoretically sound by using the unbiased Monte Carlo estimator for the gradient.

Memory scales as $\mathcal{O}(NM d + N d^2 + N^2 M)$ for storing the bank, covariances, and the PDF matrix $P$.

\subsubsection*{\textit{Anti-mode-collapse schedules for pGD-GMA sampling}}

When implementing pGD-GMA in Algo.\ref{algo:GMA-sampling-optimal}, we introduce annealing schedules to mitigate premature collapse of the mixture weights onto a single component. Unlike the mirror descent variant, projected gradient descent applies a direct Euclidean update followed by projection back onto the probability simplex. Nevertheless, stabilization is achieved through tempering and entropy regularization.

\begin{itemize}
    \item \textbf{Tempering $\beta_k$:}  
    during gradient computation, the log-likelihood terms are scaled by a tempering factor $\beta_k$:
    \begin{equation} \label{eq:tempering_pGD_GMA}
    \text{diffs} = \log q_w(x_j) - \frac{\beta_k   \log p(x_j) - \mu_{\ell}}{\sigma_{\ell}}
    \end{equation}
    Early in training, $\beta_k$ is set to a small value (e.g. $0.3$), reducing the influence of noisy gradients. As iterations progress, $\beta_k \to 1$, recovering the true posterior objective.

    \item \textbf{Entropy regularization $\lambda_k$:}  
    to discourage degeneracy, the gradient is augmented with a decaying entropy term:
    \begin{equation} \label{eq:entropy_regularisation_pGD_GMA}
    g^{(k)}_i  \leftarrow  g^{(k)}_i + \lambda_k \big(1 + \log w^{(k)}_i\big) 
    \end{equation}
    This pushes up very small weights and suppresses overly large ones, spreading probability mass across multiple components. As $k$ increases, $\lambda_k$ is annealed toward zero, allowing the final distribution to concentrate freely.
\end{itemize}

Together, the tempering and entropy regularization schedules stabilize the projected gradient updates by encouraging broad exploration at early iterations and sharper posterior concentration at later stages. Unlike mirror descent (see below section and Appendix.\ref{app:mirror_descent}), pGD-GMA does not employ explicit temperature rescaling and convex mixing, relying instead on simplex projection combined with these anti-collapse schedules.

\subsubsection*{\textit{Improve numerical stability via log-sum-exp}}
When evaluating mixture densities of the form
\[
q(x) = \sum_{j=1}^N w_j  \mathcal{N}(x \mid \mu_j, \Sigma_j)
\]
direct computation may suffer from \textit{numerical underflow} in high dimensions, since each Gaussian density can be extremely small when $x$ lies far from most component means. To stabilise, we instead compute
\[
\log q(x) = \log \Bigg( \sum_{j=1}^N w_j \exp\big(\log \mathcal{N}(x \mid \mu_j, \Sigma_j)\big) \Bigg)
\]
which can still be underflow; we apply the \textit{log-sum-exp} trick:
\[
\log q(x) = m + \log \Bigg( \sum_{j=1}^N w_j \exp\big(\log \mathcal{N}(x \mid \mu_j, \Sigma_j) - m\big) \Bigg)
\]
where $m = \max_j \log \mathcal{N}(x \mid \mu_j, \Sigma_j)$. This ensures at least one term inside the exponential equals $1$, while all others are bounded by $1$, thereby avoiding underflow and preserving stable evaluations of mixture densities even in high-dimensional parameter spaces.

\subsection{Mirror descent (MD) based GMA sampling} \label{app:mirror_descent_GMA}

An alternative to the projected gradient descent (pGD) step is to employ a mirror descent update with KL geometry, also known as the multiplicative weights update. This avoids the explicit Euclidean projection step and ensures that the weights remain strictly positive and normalized on the simplex throughout the optimization. 

Specifically, instead of taking a step 
\[
\mathbf{v}^{(k)} = \mathbf{w}^{(k-1)} - \eta_k   \mathbf{g}^{(k)}, \quad 
\mathbf{w}^{(k)} = \text{Proj}_{\Delta}(\mathbf{v}^{(k)}),
\]
we perform the mirror descent update
\[
\tilde{w}_i^{(k)}  \propto  w_i^{(k-1)} \exp  \left(-\eta_k   g_i^{(k)}\right), 
\quad 
w_i^{(k)} = \frac{\tilde{w}_i^{(k)}}{\sum_{l=1}^N \tilde{w}_l^{(k)}}
\]

This update can be seen as performing gradient descent in the dual space defined by the negative entropy potential, yielding a natural geometry for probability vectors. It automatically enforces the non-negativity and normalization constraints without requiring an explicit projection step. For detailed derivation of the MD, using a negative entropy function, can be found in Appendix.\ref{app:mirror_descent}. The step size $\eta_k$ can follow a decaying schedule such as $\eta_k = \eta_0 / \sqrt{k + k_0}$. Note that this square-root decay no longer satisfies the Robbins-Monro conditions, but is often preferred in practice due to its slower decay and improved robustness to gradient noise.

\begin{algorithm}[H]
\footnotesize
\caption{WGMA-sampling: sampling via Gaussian mixture approximation (with mirror descent, pre-computing density values and MC estimator)}
\label{algo:GMA-sampling-mirror}
\textbf{Input:} Number of Gaussian components $N$; number of samples per component $M$; number of iterations $K$; target unnormalised density \(\bar{p}(\mathbf{z})\); initial means \(\{\boldsymbol{\mu}_i\}_{i=1}^N\); initial covariance matrices \(\{\Sigma_i\}_{i=1}^N\); initial learning rate \(\eta_0\). \\
\textbf{Output:} Ensemble of selected samples \(\{\mathbf{s}_{\text{selected}}\}\) approximating samples from $p(\mathbf{z})$.

\vspace{1mm}\hrule\vspace{1mm}

\begin{algorithmic}[1]
\STATE Initialize an empty set \(\mathcal{S} = \{\}\). 
\STATE Initialize weight vector \(\mathbf{w}^{(0)}\) on the probability simplex, e.g. uniform.
\STATE Draw \(M\) samples \(\{\mathbf{s}_{i,j}\}_{j=1}^M\) from each Gaussian component \(i=1,\dots,N\). 
\STATE Pre-compute the Gaussian PDF matrix \(\mathbf{P} \in \mathbb{R}^{(NM)\times N}\) with entries \(P_{(i-1)M+j, l} = \mathcal{N}(\mathbf{s}_{i,j}; \boldsymbol{\mu}_l,\Sigma_l)\). 
\STATE Pre-compute the target log-densities \(\mathbf{p}_{\text{target}} \in \mathbb{R}^{NM}\), with \((\mathbf{p}_{\text{target}})_{(i-1)M+j} = \log \bar{p}(\mathbf{s}_{i,j})\).
\FOR{$k = 1$ to $K$}
    \STATE Compute mixture densities for all samples: \(\mathbf{q}^{(k-1)} = \mathbf{P} \cdot \mathbf{w}^{(k-1)}\).
    \STATE For each component \(i\), compute the Monte Carlo gradient:
    \[
    g_i^{(k)} = 1 + \frac{1}{M} \sum_{j=1}^M \Big[ \log q^{(k-1)}_{(i-1)M+j} - (\mathbf{p}_{\text{target}})_{(i-1)M+j} \Big].
    \]
    \STATE Mirror descent update:
    \[
    \tilde{w}_i^{(k)}  \propto  w_i^{(k-1)} \exp  \left(-\eta_k   g_i^{(k)}\right), \qquad
    w_i^{(k)} = \frac{\tilde{w}_i^{(k)}}{\sum_{l=1}^N \tilde{w}_l^{(k)} }.
    \]
\ENDFOR
\STATE Set final weights \(\mathbf{w}^* = \mathbf{w}^{(K)}\).
\STATE Sample ensemble: for each output sample, draw index \(i\sim \text{Categorical}(\mathbf{w}^*)\) and pick a random local sample \(\mathbf{s}_{i,j}\).
\STATE Return \(\mathcal{S}\).
\end{algorithmic}
\end{algorithm}

\paragraph{Complexity.}
The pre-computation and gradient estimation costs remain unchanged from the pGD variant:
\[
\mathcal{O}(N^2 M d^2)  +  \mathcal{O}(K \cdot N^2 M).
\]
Replacing projection with the multiplicative weights update reduces the per-iteration overhead from \(\mathcal{O}(N \log N)\) to \(\mathcal{O}(N)\), while guaranteeing that weights remain in the simplex. The mirror descent update is therefore not only more natural for probability vectors, but also slightly more efficient.

When implementing MD-GMA, anti-mode-collapse tricks, similar to those used in pGD-GMA but with enhanced techniques, can be applied, we discuss these in the following section.

\section{Mirror descent for mixture weights optimisation}
\label{app:mirror_descent}

Here we explain the mirror descent (MD \cite{blair_problem_1985,nemirovski_robust_2009}) method used for optimizing the mixture weights in the GMA sampler. We also contrast MD with projected gradient descent (pGD). The derivation closely follows the updates implemented in our algorithm.

\paragraph{Proximal view of mirror descent on the simplex}
We aim to minimize an objective $f(w)$ over the probability simplex
\[
\Delta := \{ w \in \mathbb{R}^N : w_i \ge 0, \ \sum_{i=1}^N w_i = 1 \}.
\]
A single MD step solves the KL-regularized proximal subproblem
\[
w^{(k+1)} = \arg\min_{w \in \Delta} 
\Big\{ \langle g^{(k)}, w \rangle + \tfrac{1}{\eta_k} D_{\psi}(w \| w^{(k)}) \Big\},
\]
where $g^{(k)} = \nabla f(w^{(k)})$, $\eta_k>0$ is the step size, and $D_{\psi}$ is the Bregman divergence induced by the mirror map $\psi$.

\begin{definition}[Bregman divergence]
Given a strictly convex and differentiable mirror map $\psi : \mathbb{R}^d \to \mathbb{R}$, 
the \textit{Bregman divergence} between $w, u \in \mathbb{R}^d$ with respect to $\psi$ is
\[
D_{\psi}(u  \|  w) := \psi(u) - \psi(w) - \langle \nabla \psi(w),  u - w \rangle.
\]
\end{definition}

\paragraph{Negative-entropy mirror map}
For problems on the simplex, the canonical mirror map is the negative entropy
\[
\psi(w) = \sum_{i=1}^N w_i \log w_i
\]
whose Bregman divergence is the forward KL divergence \cite{cmu15850notes2020}
\[
D_{\psi}(u \| v) = \sum_{i=1}^N u_i \log \frac{u_i}{v_i}
\]
The proximal subproblem then becomes
\[
w^{(k+1)} = \arg\min_{w\in\Delta} 
\Big\{ \langle g^{(k)}, w\rangle + \tfrac{1}{\eta_k} \mathrm{KL}(w \| w^{(k)}) \Big\}.
\]

\paragraph{Closed-form solution: exponentiated Gradient}
To solve the KL-regularized proximal subproblem
\[
w^{(k+1)} = \arg\min_{w\in\Delta} 
\Big\{ \langle g^{(k)}, w\rangle + \tfrac{1}{\eta_k} \mathrm{KL}(w \| w^{(k)}) \Big\}
\]
we expand the KL divergence:
\[
\mathrm{KL}(w \| w^{(k)}) = \sum_{i=1}^N w_i \log \frac{w_i}{w^{(k)}_i}
\]
The optimization problem becomes
\[
\min_{w \in \Delta}   \sum_{i=1}^N \Big( g^{(k)}_i w_i + \tfrac{1}{\eta_k} w_i \log \tfrac{w_i}{w^{(k)}_i} \Big).
\]
Introducing a Lagrange multiplier $\lambda$ for the simplex constraint $\sum_i w_i = 1$, the Lagrangian is
\[
\mathcal{L}(w,\lambda) = \sum_{i=1}^N \Big( g^{(k)}_i w_i + \tfrac{1}{\eta_k} w_i \log \tfrac{w_i}{w^{(k)}_i} \Big)
+ \lambda\Big(\sum_{i=1}^N w_i - 1\Big)
\]
Setting derivatives w.r.t.~$w_i$ to zero gives
\[
g^{(k)}_i + \tfrac{1}{\eta_k}\Big(\log \tfrac{w_i}{w^{(k)}_i} + 1\Big) + \lambda = 0
\]
Re-arranging:
\[
\log \tfrac{w_i}{w^{(k)}_i} = -\eta_k g^{(k)}_i - 1 - \eta_k \lambda
\]
Exponentiating both sides yields
\[
w_i = w^{(k)}_i \exp(-\eta_k g^{(k)}_i) \cdot \exp(-1 - \eta_k \lambda)
\]
The factor $\exp(-1 - \eta_k \lambda)$ is the same for all $i$ and is fixed by the constraint $\sum_i w_i=1$. Thus we obtain the normalized \textit{exponentiated gradient} update:
\begin{equation} \label{eq:MD_exponential_update}
    w^{(k+1)}_i = \frac{w^{(k)}_i \exp(-\eta_k g^{(k)}_i)}{\sum_j w^{(k)}_j \exp(-\eta_k g^{(k)}_j)}
\end{equation}
This multiplicative update preserves nonnegativity and automatically enforces the simplex constraint.

We use mirror descent (MD) for GMM weights optimisation based on following considerations:
\begin{itemize}
    \item \textit{Simplex constraints:} MD is a natural fit when optimizing distributions over the simplex, as in GMA mixture weights.  
    \item \textit{Better conditioning:} in high dimensions, the entropy geometry exploited by MD often provides more stable updates compared with Euclidean projection.  
    \item \textit{Sparser solutions:} exponentiated-gradient style updates in MD tend to promote sparsity in the learned weight distribution, effectively pruning away redundant mixture components.  
\end{itemize}

\subsubsection*{\textit{Anti-mode-collapse schedules for MD-GMA}}

When implementing MD-GMA, we further introduce several annealing schedules to prevent the mixture weights from collapsing onto a single component too early. These schedules are applied directly in the weight update step.

\begin{itemize}
    \item \textbf{Temperature $\tau_k \ge 1$:}  
    In the exponentiated-gradient update, the raw update is scaled by a temperature factor:
    \begin{equation} \label{eq:temperature_factor_MD_GMA}
    \log w^{\text{prop}} = \frac{\log w^{(k)} - \eta_k g^{(k)}}{\tau_k}
    \end{equation}
    When $\tau_k > 1$, the distribution is flattened, corresponding to stronger KL regularization that prevents sharp peaks. As $k$ increases, $\tau_k$ is annealed toward $1$, allowing the weights to concentrate more sharply on high-probability components.

    \item \textbf{Convex mixing $\alpha_k$ (with uniform shrinkage as a special case).}
    After computing the proposed update $w^{\text{prop}}$ via exponentiated gradient, form a convex combination\footnote{The superscript $\text{prop}$ denotes the proposed (intermediate) update.}:
    \begin{equation}\label{eq:convex_mixing_MD_GMA}
      \tilde w^{(k+1)} = (1-\alpha_k)  w^{(k)}  +  \alpha_k  w^{\text{prop}}
    \end{equation}
    This mixing smooths fluctuations and prevents abrupt changes in the weights that can trigger premature collapse. Typically $\alpha_k$ starts around $0.2$ and decays toward a small value such as $0.05$.
    
    \textit{Uniform shrinkage as a special case.}  
    Let $\mathbf{u}   =   \frac{1}{N}\mathbf{1}$ be the uniform distribution on the simplex. Taking $w^{\text{prop}}=\mathbf{u}$ in Eq.\ref{eq:convex_mixing_MD_GMA} yields the uniform-mixing update
    \begin{equation}\label{eq:uniform_shrink}
      w^{(k+1)} = (1-\tau_k)  w^{(k)}  +  \tau_k  \mathbf{u}
    \end{equation}
    which guarantees a per-component floor $w^{(k+1)}_i \ge \tau_k/N$ and thus mitigates mode collapse. In practice we may apply both steps: first Eq.\ref{eq:convex_mixing_MD_GMA} with $\alpha_k$, then Eq.\ref{eq:uniform_shrink} with a small $\tau_k   \in   [10^{-3},10^{-2}]$. Both updates are convex combinations, so they preserve non-negativity and the simplex constraint $\sum_i w_i = 1$.

    \item \textbf{Tempering $\beta_k$:}  
    The log-likelihood contribution of each local sample is multiplied by $\beta_k$ when computing the standardized target:
    \begin{equation} \label{eq:tempering_MD_GMA}
    \text{diffs} = \log q_w(x_j) - \frac{\beta_k   \log p(x_j) - \mu_{\ell}}{\sigma_{\ell}}
    \end{equation}
    where $\mu_{\ell}, \sigma_{\ell}$ are running statistics for standardization. Early in training, $\beta_k$ is set to a smaller value (e.g. $0.3$) so that the algorithm downweights the noisy, high-variance target log-likelihood. As $k$ increases, $\beta_k$ anneals toward $1$, gradually shifting the optimization focus toward the true posterior.

    \item \textbf{Entropy regularization $\lambda_k$:}
    We augment the objective with a (negative) entropy penalty
    \[
    \mathcal{R}(w) = \sum_{i=1}^N w_i \log w_i
    \quad\text{(with the convention } 0\log 0 := 0\text{)}
    \]
    A single coordinate’s contribution is $r(w_i) = w_i\log w_i$, whose derivative is
    \[
    \frac{\partial}{\partial w_i}  r(w_i)
    = \frac{\partial}{\partial w_i}  \big( w_i \log w_i \big)
    = \log w_i + 1.
    \]
    Therefore,
    \[
    \nabla \mathcal{R}(w) = \big(1+\log w_1,  1+\log w_2, \ldots,  1+\log w_N\big)^\top,
    \]
    and adding the term $\lambda_k \mathcal{R}(w)$ to the objective contributes the explicit gradient
    \begin{equation} \label{eq:entropy_regularisation_MD_GMA}
    g^{(k)}_i  \leftarrow  g^{(k)}_i  +  \lambda_k \big(1 + \log w^{(k)}_i\big)
    \end{equation}
    Intuitively, this pushes up very small weights (for which $\log w_i \ll 0$) and pushes down very large ones, encouraging spread spread out across components, preventing degeneracy where a single weight dominates (i.e. mode collapse). As training progresses, $\lambda_k$ is reduced toward zero, allowing the final distribution to sharpen.
    
    \textit{Remarks:}
    \begin{itemize}
    \item \textit{Relation to KL to uniform.} Up to an additive constant, $\mathcal{R}(w)$ is the forward KL divergence to the uniform distribution: $\mathrm{KL}(w\|u) = \sum_i w_i \log\frac{w_i}{1/N} = \mathcal{R}(w) + \log N$. Thus $\lambda_k \mathcal{R}(w)$ encourages $w$ to stay closer to uniform early on.
    \item \textit{Curvature.} The Hessian of $r(w_i)$ is $r''(w_i) = 1/w_i$, so $\nabla^2 \mathcal{R}(w) = \mathrm{diag}(1/w_i)$ on the interior of the simplex; small $w_i$ imply large curvature that resists further shrinkage.
    \item \textit{Numerics.} To avoid $\log 0$, we use a floor such as $\log(\max\{w_i,\varepsilon\})$ with $\varepsilon \in [10^{-12},10^{-8}]$ in code.
    \item \textit{Annealing.} We start with a larger $\lambda_k$ to promote exploration and reduce it toward zero as $k$ grows, allowing the final solution to sharpen once collapse is no longer a risk.
    \end{itemize}
\end{itemize}

Together, these schedules (temperature, convex mixing, tempering, and entropy regularization) provide a controlled annealing mechanism. Early iterations emphasize stability and exploration across mixture components, while later iterations allow the weights to focus more precisely on high-likelihood regions.

\paragraph{Gradient estimation}
Given pre-drawn local samples $\{x_j\}_{j=1}^{NM}$ around each component mean, the mixture log-density is
\[
\log q_w(x_j) = \log \sum_{i=1}^N w_i   \mathcal{N}(x_j;\mu_i, \sigma^2 I).
\]
The stochastic gradient estimator is
\[
g^{(k)}_i \approx 1 + \frac{1}{M}\sum_{j \in \text{component}(i)} \Big( \log q_w(x_j) - \text{standardized target}(x_j) \Big)
\]
with an additional entropy term $\lambda_k (1 + \log w_i^{(k)})$. The standardized form of the (tempered) target log-density:
$$
\text{standardized target}(x_j) = \frac{\beta_k   \log p(x_j) - \mu_{\ell}}{\sigma_{\ell}}
$$

\paragraph{Final optimisation objective with anti-collapse schedules}
Combining the mechanisms above, the effective optimization problem at iteration $k$ can be expressed as
\[
\min_{w \in \Delta}  
\Big\langle g^{(k)}_{\text{temp}, \beta},   w \Big\rangle
 +  \tfrac{1}{\eta_k} D_{\psi}(w  \|  w^{(k)})
 +  \lambda_k \sum_{i=1}^N w_i \log w_i
\]
where:
\begin{itemize}
    \item $g^{(k)}_{\text{temp}, \beta}$ is the tempered and standardized gradient estimate at iteration $k$, incorporating the tempering factor $\beta_k$ that downweights the target log-likelihood in early iterations.
    \item $D_{\psi}$ is the Bregman divergence induced by the negative entropy mirror map, with an additional \textit{temperature scaling} $\tau_k$ applied inside the proximal step:
    \[
    \log w^{\text{prop}} = \frac{\log w^{(k)} - \eta_k g^{(k)}_{\text{temp}, \beta}}{\tau_k}
    \]
    \item $\lambda_k \sum_i w_i \log w_i$ is the explicit entropy penalty that pushes weights toward uniformity early on.
\end{itemize}
The proposed multiplicative update $w^{\text{prop}}$ is then combined with the previous iterate using convex mixing:
\[
w^{(k+1)} = (1-\alpha_k)  w^{(k)} + \alpha_k  w^{\text{prop}}
\]

The final MD step in our GMA algorithm jointly incorporates these schedules, producing a robust and stable annealed optimization path on the simplex, with
\begin{itemize}
    \item $\tau_k$ flattens the exponentiated-gradient update, enforcing stronger KL regularization early on.
    \item $\alpha_k$ convexly mixes the proposed update with the previous iterate to smooth the trajectory of weights.
    \item $\beta_k$ tempers the contribution of the noisy target log-likelihood, gradually shifting attention toward the true posterior.
    \item $\lambda_k$ penalizes low-entropy solutions, explicitly adding an entropy gradient to discourage premature collapse.
\end{itemize}

\paragraph{Comparison to projected gradient descent (pGD)}
For reference, the pGD update takes an additive step followed by Euclidean projection:
\[
\tilde{w} = w^{(k)} - \eta_k g^{(k)}, \qquad
w^{(k+1)} = \Pi_{\Delta}(\tilde{w})
\]
where $\Pi_{\Delta}$ is projection onto the simplex.  
In contrast, MD yields the multiplicative update
\[
w^{(k+1)} \propto w^{(k)} \odot \exp(-\eta_k g^{(k)})
\]
pGD uses Euclidean geometry, while MD uses KL (entropy) geometry, which is often more natural for probability vectors. MD automatically maintains positivity and tends to preserve entropy, avoiding the spiky solutions pGD can create after projection.

In short,
\begin{itemize}
    \item MD is a KL-proximal update that produces exponentiated-gradient steps.
    \item Temperature $\tau_k$ and convex mixing $\alpha_k$ prevent weight collapse.
    \item Tempering $\beta_k$ and entropy penalty $\lambda_k$ further stabilize learning.
    \item Compared with PGD, MD is better suited to simplex constraints, since it preserves positivity and maintains spread across components.
\end{itemize}

\paragraph{Workflow of MD with stabilization strategies}

The whole flow of the MD method, with stabilization strategies applied to GMM weights optimisation, is shown in Fig.\ref{fig:mirror-descent-flow}.

\begin{figure}[H]
\centering
\begin{tikzpicture}[node distance=1.8cm, auto, >=Latex, thick]
\tikzset{
  block/.style = {rectangle, draw, fill=blue!10, rounded corners,
                  minimum height=1cm, minimum width=4.8cm,
                  align=center},
  side/.style  = {rectangle, draw, fill=green!10, rounded corners,
                  minimum height=1cm, minimum width=4.8cm,
                  align=center},
  line/.style  = {draw, -{Latex[length=3mm]}}
}

% Nodes
\node[block] (init)
  {Initialize $w^{(0)}  \in  \Delta$; draw locals;\\
   precompute $\log\mathcal{N}(x;\mu_i,\sigma^2 I)$ and stats of $\log\bar{p}(x)$};

\node[block, below=of init] (grad)
  {Compute MC gradient $g^{(k)}$ on tempered target ($\beta_k$)\\
   and add entropy term:\\[2pt]
   $g^{(k)} \leftarrow g^{(k)} + \lambda_k\big(1+\log w^{(k)}\big)$};

\node[block, below=of grad] (md)
  {Mirror descent with temperature $\tau_k$ (proposal):\\[4pt]
   $z = \dfrac{\log w^{(k)} - \eta_k g^{(k)}}{\tau_k}$\\[3pt]
   $w^{\text{prop}} = \mathrm{softmax}(z)$\\[6pt]
   \textit{(PGD alternative):}\quad
   $w^{\text{prop}} = \Pi_{\Delta}  \big(w^{(k)}-\eta_k g^{(k)}\big)$};

\node[block, below=of md] (mix)
  {Convex mixing (stabilization):\\[2pt]
   $\displaystyle w^{(k+1)}=(1-\alpha_k) w^{(k)}+\alpha_k w^{\text{prop}}$};

\node[block, below=of mix] (polyak)
  {Polyak/tail averaging over last $L$ iterates:\\[2pt]
   $\displaystyle \bar{w}^{(K)}=\frac{1}{L}\sum_{t=K-L+1}^{K} w^{(t)}$};

\node[block, below=0.5cm of polyak] (final)
  {Final weights $\bar{w}^{(K)}$ (or $w^{(K)}$); sample ensemble};

% Side node for tempering
\node[side, right=of grad] (temp)
  {Tempering: scale target $\log \bar{p}(x)$ \\ by $\beta_k\in(0,1]$ before gradient};

% Arrows
\path[line] (init) -- (grad);
\path[line] (grad) -- (md);
\path[line] (md) -- (mix);
\path[line] (mix) -- (polyak);
\path[line] (polyak) -- (final);

% Side arrow for tempering -> gradient
% \draw[line] (temp.west) -- ++(-2.8,0) |- (grad.east);
\draw[line, shorten >=6pt, shorten <=6pt] (temp.west) -- (grad.east);

\end{tikzpicture}
\caption{Stabilized mirror-descent GMA workflow. Tempering ($\beta_k$) modifies the target inside the gradient; entropy regularization is added to the gradient; the MD step uses temperature $\tau_k$; convex mixing stabilizes updates; Polyak averaging reduces variance. pGD shown as an alternative to the MD proposal.}
\label{fig:mirror-descent-flow}
\end{figure}
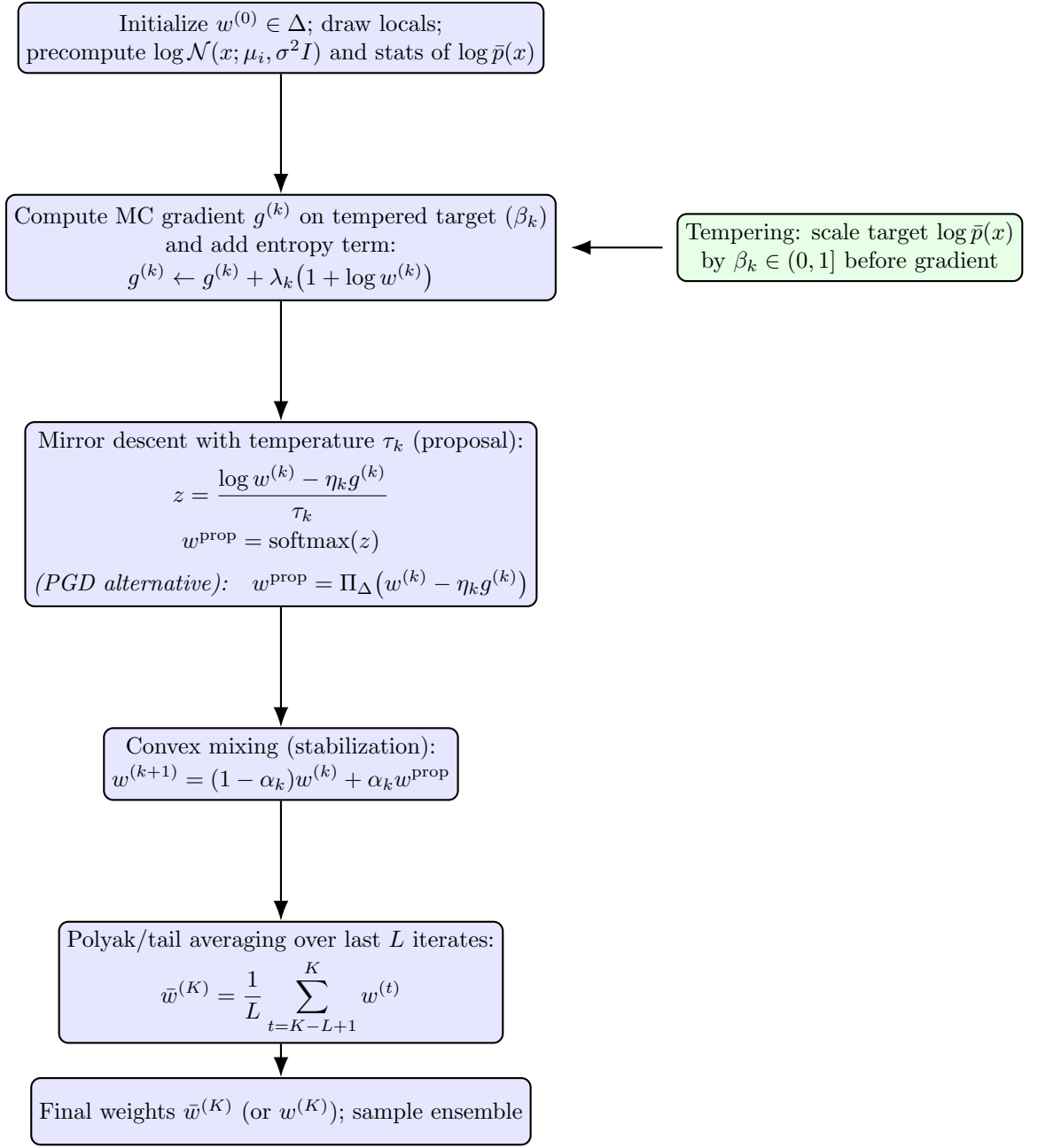

\section{Laplace mixture approximation} \label{app:LMA}

Here we extent our discussion of LMA in Section.\ref{subsec:LMA}. As discussed, one way to improve GMM sampling is to fit a GMM to the target distribution via Laplace approximation. Classic Laplace approximation fits a uni-modal Gaussian at the maximum of the posterior \cite{Mackay1998Choice}; here we propose fitting a GMM to multi-modes of the posterior. Discovery of these multi-modes are out of the scope; it can be done via e.g. a multi-start optimisation procedure. 
The Laplace approximation approach places the GMM centered around the posterior MAP modes, with precision equaling to the observed Fisher information.
After fitting this GMM (we call it \textit{Laplace mixtures}) to the target (in a similar fashion as GMM weights optimisation), we can directly sample from each component using stratified sampling, or use the fitted Laplace mixture as a warm start for GMA sampling (i.e. finer weights optimisation and stratified sampling).

\subsection{Laplace approximation (LA)}
Here we use a Bayesian posterior distribution $p(\theta \mid \mathcal{D})$ as our target. Let $\bar{p}(\theta)$ denote the \textit{unnormalised} target density on $\theta\in\mathbb{R}^d$:
\[
p(\theta \mid \mathcal{D}) \propto \bar{p}(\theta) = p(\mathcal{D}\mid\theta) p(\theta)
\]
and let $\ell(\theta)=\log \bar{p}(\theta) = \log p(\mathcal{D}, \theta) = \log p(\mathcal{D} \mid \theta) + \log p(\theta)$ be the (unnormalised) log-posterior density. 

Let $\hat{\theta}$ be a (local) MAP estimate, found by maximizing $\ell(\theta)$:
\[
\hat{\theta} = \arg\max_{\theta} \ell(\theta)
\]
We Taylor-expand $\ell(\theta)$ to the second order around $\hat{\theta}$. Since $\hat{\theta}$ is a local \textit{maximum}, the first-order term vanishes, and we arrive at:
\[
\ell(\theta) \approx \ell(\hat{\theta}) - \frac{1}{2} (\theta - \hat{\theta})^\top H(\hat{\theta}) (\theta - \hat{\theta}), \quad 
H(\hat{\theta}) = -\nabla^2 \ell(\theta)\big|_{\theta=\hat{\theta}}
\]
Since $\hat{\theta}$ is a local maximum, the Hessian $H(\hat{\theta})$ is positive definite. Exponentiating this quadratic form gives a Gaussian approximation centered around the mode $\hat{\theta}$:
\[
p(\theta \mid \mathcal{D}) \approx \mathcal{N}(\theta \mid \hat{\theta}, H^{-1}(\hat{\theta}))
\]
The exact form of this approximation can be derived as follows.
First, exponentiating the $\ell(\theta)$ gives:
\[
p(\theta \mid \mathcal{D}) \propto \exp(\ell(\hat{\theta})) \exp\left( - \frac{1}{2} (\theta - \hat{\theta})^\top H(\hat{\theta}) (\theta - \hat{\theta}) \right)
\]
where $\exp(\ell(\hat{\theta}))$ is a constant, we have found the unnormalised form of our Gaussian approximation.
For the approximation to be a valid probability density, it must integrate to 1. We find the \textit{normalising constant} $Z$ by integrating the functional part of the expression over the entire parameter space:
\[
Z = \int_{\mathbb{R}^d} \exp\left( - \frac{1}{2} (\theta - \hat{\theta})^\top H(\hat{\theta}) (\theta - \hat{\theta}) \right)   d\theta
\]
This is a standard Gaussian integral, whose general solution is $(2\pi)^{d/2} |\Sigma|^{1/2}$, where $\Sigma$ is the covariance matrix. In our case, the inverse covariance (precision) matrix is $H(\hat{\theta})$, so the covariance matrix is $\Sigma = H^{-1}(\hat{\theta})$. Plugging in this in gives:
\[
Z = (2\pi)^{d/2} |H^{-1} (\hat{\theta})|^{1/2} = (2\pi)^{d/2} |H(\hat{\theta})|^{-1/2}
\]

Finally, the normalised probability density is constructed by dividing the unnormalised form by the constant $Z$:
\begin{align*}
p(\theta \mid \mathcal{D}) & \approx \frac{1}{Z} \exp\left( - \frac{1}{2} (\theta - \hat{\theta})^\top H(\hat{\theta}) (\theta - \hat{\theta}) \right) \\
& = \frac{|H(\hat{\theta})|^{1/2}}{(2\pi)^{d/2}} \exp\left( - \frac{1}{2} (\theta - \hat{\theta})^\top H(\hat{\theta}) (\theta - \hat{\theta}) \right) \\
& = \mathcal{N}(\theta \mid \hat{\theta}, H^{-1}(\hat{\theta}))
\end{align*}
which is the \textit{classic Laplace approximation}. We have two immediate observations:
\begin{itemize}
    \item \textbf{Local covariance:} $\Sigma \approx H^{-1}(\hat{\theta})$ captures curvature (and thus uncertainty) near the mode.
    \item \textbf{Laplace approximation of the evidence} (marginal likelihood, normalizing constant)
    \footnote{Derivation of the Laplace evidence: let $\ell(\theta)=\log p(\mathcal{D},\theta)$. Then $p(\mathcal{D})=\int p(\mathcal{D},\theta) d\theta=\int \exp\{\ell(\theta)\} d\theta$. Quadratically expand $\ell$ at the MAP $\hat\theta$: $\ell(\theta)\approx \ell(\hat\theta)-\tfrac12(\theta-\hat\theta)^\top H(\hat\theta)(\theta-\hat\theta)$ with $H(\hat\theta)=-\nabla^2\ell(\hat\theta)\succ 0$. Substituting gives
        $$
        p(\mathcal{D})\approx e^{\ell(\hat\theta)} \int \exp \Big\{-\tfrac12(\theta-\hat\theta)^\top H(\hat\theta)(\theta-\hat\theta)\Big\} d\theta
        $$
        Using the standard Gaussian integral $\int \exp\{-\tfrac12(\mathbf{x}-\mu)^\top \Sigma^{-1}(\mathbf{x}-\mu)\} d\mathbf{x}=(2\pi)^{d/2}|\Sigma|^{1/2}$ yields
        $$
        p(\mathcal{D})\approx (2\pi)^{d/2} |H(\hat\theta)|^{-1/2} \exp\{\ell(\hat\theta)\}
        $$
        \textit{Occam's razor intuition:} The term $e^{\ell(\hat\theta)} = p(\mathcal{D}|\hat{\theta})p(\hat{\theta})$ rewards data fit; $|H(\hat\theta)|^{-1/2}$ penalises sharp (complex) posteriors by shrinking the volume around the peak.
    }
    :
    \[
    p(\mathcal{D}) \approx (2\pi)^{d/2} \left| H(\hat{\theta}) \right|^{-1/2} \exp\{\ell(\hat{\theta})\}
    \]
\end{itemize}
This “local evidence” is useful for weighting multiple modes.

\subsection{Laplace-mixture approximation (LMA)}

While the \textit{classical} Laplace approximation replaces the target locally by a Gaussian, here we aim to approximate the unnormalised posterior with \textit{Laplacian mixtures}.
Suppose we locate $J\ge 1$ local modes $\{\hat\theta_j\}_{j=1}^J$ (e.g. via multi-start optimisation of $\ell$). For each mode, we form the observed Fisher information, which is the negative Hessian of the log-posterior $\ell$:
\[
H_j = - \nabla^2 \ell(\theta)\big|_{\theta=\hat\theta_j}
\]
The local covariance for mode $j$ is $H_j^{-1}$.

To avoid the overly concentrated nature of Laplace approximations, we define an inflated covariance matrix $\Sigma_j$ for each mode. The base covariance $H_j^{-1}$ is broadened by an inflation factor\footnote{Since we aim to scale the standard deviations by $\kappa$, the covariance matrix is scaled by $\kappa^2$.} $\kappa \ge 1$:
\[
\Sigma_j = \kappa^2 H_j^{-1}
\]
A principled set of \textit{preliminary} mixture weights is then obtained from the local evidence formula for each component:
\[
\tilde w_j \propto \bar{p}(\hat\theta_j) (2\pi)^{d/2} |\Sigma_j|^{1/2},
\qquad
w_j = \frac{\tilde w_j}{\sum_{r=1}^J \tilde w_r}
\]

The design of above weights is to make each weight $\tilde w_j$ proportional to the \textit{local evidence} ($p(\mathcal{D})_j$) of that component. In the above, the unnormalised posterior at the mode is the exponentiated log-posterior: $\exp\{\ell(\hat{\theta}_j)\} = \bar{p}(\hat\theta_j)$, and the determinant of the covariance is proportional to the inverse determinant of the Hessian: $|\Sigma_j|^{1/2} = |\kappa^2 H_j^{-1}|^{1/2} \propto |H_j|^{-1/2}$.
Observing these, we see that the weight formula is a direct rearrangement of the evidence formula, representing a product of the peak's ``height'' and ``volume'':
\[
\tilde w_j \propto \underbrace{\bar{p}(\hat\theta_j)}_{\substack{\text{Height of} \\ \text{the peak}}} \times \underbrace{(2\pi)^{d/2} |\Sigma_j|^{1/2}}_{\substack{\text{Volume of} \\ \text{the peak}}} \approx p(\mathcal{D})_j
\]
In essence, the weight for each Gaussian component is determined by approximating the total probability mass in its local region of the posterior. This is done by multiplying the height of the posterior peak ($\bar{p}(\hat\theta_j)$) by its effective volume or ``width'' ($(2\pi)^{d/2} |\Sigma_j|^{1/2}$), giving a principled way to decide how much influence each component should have in the final mixture.

This yields the final \textbf{Laplace mixture}, which is a Gaussian Mixture Model (GMM):
\[
q_0(\theta) = \sum_{j=1}^{J} w_j \mathcal{N} \big(\theta; \hat\theta_j,\Sigma_j\big)
\]
which can be used for direct sampling (stratified sampling) or as an initialiser for GMA sampling. The LMA procedure is presented in the following:

\begin{tcolorbox}[boxrule=1pt, colback=white]
\small
\centering
\textit{\textbf{Procedure: Laplace mixture approximation (LMA)}}
\vspace{2mm}
\vspace{2mm}
\begin{enumerate}
    \item \textbf{Find Modes:} Given a target log-posterior density $\ell(\theta)$, find a set of $J$ distinct local maxima (modes), $\{\hat\theta_1, \dots, \hat\theta_J\}$, typically via multi-start optimisation.

    \item \textbf{Compute Hessians:} At each mode $\hat\theta_j$, compute the negative Hessian matrix (i.e. the observed Fisher information):
    \[ H_j = - \nabla^2 \ell(\theta)\big|_{\theta=\hat\theta_j} \]

    \item \textbf{Set Inflated Covariances:} For each mode, define an inflated covariance matrix $\Sigma_j$ by scaling the inverse Hessian with a factor $\kappa \ge 1$:
    \[ \Sigma_j = \kappa^2 H_j^{-1} \]

    \item \textbf{Calculate Mixture Weights:} Compute a weight $w_j$ for each component based on its local evidence. First, find the un-normalised weights $\tilde w_j$:
    \[ \tilde w_j \propto \exp\{\ell(\hat\theta_j)\} (2\pi)^{d/2} |\Sigma_j|^{1/2} \]
    where $\ell(\theta) = \log p(\mathcal{D} \mid \theta) + \log p(\theta)$ is the sum of the log-likelihood and log-prior.
    Then, normalise them to sum to one:
    \[ w_j = \frac{\tilde w_j}{\sum_{r=1}^J \tilde w_r} \]

    \item \textbf{Construct GMM:} Combine the modes (as means), the inflated covariances, and the weights to form the final Gaussian Mixture Model (GMM) proposal distribution:
    \[ q_0(\theta) = \sum_{j=1}^{J} w_j \mathcal{N} \big(\theta; \hat\theta_j,\Sigma_j\big) \]
\end{enumerate}
\end{tcolorbox}

\section{EM refinement of GMM approximation} \label{app:em_GMA}

As proposed in Section.\ref{subsec:em_GMA}, here we describe how to refine a Gaussian mixture $q_{\boldsymbol{\theta}}(\mathbf{z})=\sum_{i=1}^N w_i \mathcal{N}(\mathbf{z};\boldsymbol{\mu}_i,\Sigma_i)$ against an \textit{unnormalised} target $\bar{p}(\mathbf{z})$ by (population) EM. Unlike the weights-only reverse-KL update used in GMM approximation, EM targets the \textit{inclusive} KL:
\[
\boldsymbol{\theta}^\star  \in \arg\min_{\boldsymbol{\theta}}  KL \big(p  \|  q_{\boldsymbol{\theta}}\big)
= \arg\max_{\boldsymbol{\theta}}  \mathbb{E}_{p} \left[\log q_{\boldsymbol{\theta}}(\mathbf{Z})\right],
\quad p(\mathbf{z})=\bar{p}(\mathbf{z})/Z_p
\]
which encourages \textit{mass covering} \cite{Le2017reversekl}. Expectations under $p$ are intractable but can be approximated by self-normalised importance sampling (SNIS) from any proposal $r(\mathbf{z})$.

\subsection{Latent-variable augmentation and EM objective}
Introduce a component label $C\in\{1,\dots,N\}$ with
\(
q_{\boldsymbol{\theta}}(\mathbf{z},C=i)=w_i \mathcal{N}(\mathbf{z};\boldsymbol{\mu}_i,\Sigma_i).
\)
The standard EM auxiliary function (population form) is
\begin{align}
Q(\boldsymbol{\theta}\mid\boldsymbol{\theta}^{(t)})
&= \mathbb{E}_{p(\mathbf{z})} \mathbb{E}_{q(C\mid \mathbf{z};\boldsymbol{\theta}^{(t)})}
\left[\log q_{\boldsymbol{\theta}}(\mathbf{z},C)\right] \nonumber\\
&= \mathbb{E}_{p(\mathbf{z})}\Big[\sum_{i=1}^N r_i^{(t)}(\mathbf{z})\big(\log w_i + \log \mathcal{N}(\mathbf{z};\boldsymbol{\mu}_i,\Sigma_i)\big)\Big], \label{eq:Qfun}
\end{align}
where the current \textit{responsibilities} are
\[
r_i^{(t)}(\mathbf{z}) = q(C=i\mid \mathbf{z};\boldsymbol{\theta}^{(t)})
= \frac{w_i^{(t)} \mathcal{N}(\mathbf{z};\boldsymbol{\mu}_i^{(t)},\Sigma_i^{(t)})}
{\sum_{\ell=1}^N w_\ell^{(t)} \mathcal{N}(\mathbf{z};\boldsymbol{\mu}_\ell^{(t)},\Sigma_\ell^{(t)})}.
\]

\subsection{SNIS approximation of the $p$-expectation}
Let $\{\mathbf{z}_m\}_{m=1}^M \sim r(\mathbf{z})$ be a fixed \textit{bank} (e.g. the existing GMA bank, or fresh draws from $q_{\boldsymbol{\theta}^{(t)}}$). Define unnormalised IS weights
\(
\tilde{\omega}_m = \bar{p}(\mathbf{z}_m)/r(\mathbf{z}_m)
\)
and normalised weights
\(
\omega_m=\tilde{\omega}_m / \sum_{s=1}^M \tilde{\omega}_s.
\)
Then
\[
\mathbb{E}_{p}[f(\mathbf{Z})]  \approx  \sum_{m=1}^M \omega_m  f(\mathbf{z}_m),
\qquad
Q(\boldsymbol{\theta}\mid\boldsymbol{\theta}^{(t)})
 \approx 
\sum_{m=1}^M \omega_m \sum_{i=1}^N r_{i m}^{(t)} \big(\log w_i + \log \mathcal{N}(\mathbf{z}_m;\boldsymbol{\mu}_i,\Sigma_i)\big),
\]
where $r_{i m}^{(t)} \equiv r_i^{(t)}(\mathbf{z}_m)$.

\paragraph{Choice of proposal $r$.}
Common and convenient choices are: (i) the current mixture $r=q_{\boldsymbol{\theta}^{(t)}}$ (“self-consistency”, often called population EM
\footnote{
\textit{Population EM} maximises the population objective $\theta^{(t+1)}=\arg\max_\theta \mathbb{E}_{p}[\log q_\theta(\mathbf Z)]$, equivalently minimising the inclusive $\mathrm{KL}(p\|q_\theta)$, where $p(\mathbf z)=\bar p(\mathbf z)/Z_p$. Since direct expectations under $p$ are intractable, we use self-normalised importance sampling (SNIS): at sweep $t$, draw a bank $\{\mathbf z_m\}_{m=1}^M\sim r^{(t)}(\mathbf z)$ (commonly $r^{(t)}=q_{\theta^{(t)}}$), set unnormalised weights $\tilde\omega_m=\bar p(\mathbf z_m)/r^{(t)}(\mathbf z_m)$, normalise $\omega_m=\tilde\omega_m/\sum_s\tilde\omega_s$, and approximate $\mathbb{E}_p[f(\mathbf Z)]\approx\sum_m \omega_m f(\mathbf z_m)$. With responsibilities $r_{im}^{(t)}=\dfrac{w_i^{(t)}\mathcal N(\mathbf z_m;\mu_i^{(t)},\Sigma_i^{(t)})}{\sum_{\ell} w_\ell^{(t)}\mathcal N(\mathbf z_m;\mu_\ell^{(t)},\Sigma_\ell^{(t)})}$, the EM auxiliary function satisfies $Q(\theta\mid\theta^{(t)})\approx \sum_m \omega_m \sum_i r_{im}^{(t)}\big(\log w_i+\log \mathcal N(\mathbf z_m;\mu_i,\Sigma_i)\big)$, yielding the standard M-step updates with $p$-weighted sufficient statistics. Optionally, the bank is \textit{refreshed} each sweep (population/self-consistent EM) by resampling from $q_{\theta^{(t)}}$.}
), or (ii) a fixed, wider bank distribution (e.g. the Laplace-mixture initialiser). In case (i), the IS weights simplify to
\(
\tilde{\omega}_m = \bar{p}(\mathbf{z}_m)/q_{\boldsymbol{\theta}^{(t)}}(\mathbf{z}_m).
\)

\subsection{E-step (responsibilities)}
Given $\boldsymbol{\theta}^{(t)}$ and the bank $\{\mathbf{z}_m\}$,
\[
r_{i m}^{(t)}
= \frac{w_i^{(t)} \mathcal{N}(\mathbf{z}_m;\boldsymbol{\mu}_i^{(t)},\Sigma_i^{(t)})}
{\sum_{\ell=1}^N w_\ell^{(t)} \mathcal{N}(\mathbf{z}_m;\boldsymbol{\mu}_\ell^{(t)},\Sigma_\ell^{(t)})},
\qquad
\omega_m \propto \frac{\bar{p}(\mathbf{z}_m)}{r(\mathbf{z}_m)}
\]
Define the IS-weighted effective counts
\(
N_i^{(t)} = \sum_{m=1}^M \omega_m  r_{i m}^{(t)}.
\)

\subsection{M-step (closed-form updates)}
Maximising the SNIS approximation to Eq.\ref{eq:Qfun} yields the usual mixture updates with \textit{$p$-weighted} sufficient statistics:

\begin{equation} \label{eq:EM_Mstep}
\begin{aligned}
    w_i^{(t+1)} &= \frac{N_i^{(t)}}{\sum_{j=1}^N N_j^{(t)}} = N_i^{(t)} \quad (\text{since } \sum_i N_i^{(t)}=1),\\[2pt]
    \boldsymbol{\mu}_i^{(t+1)} &= \frac{1}{N_i^{(t)}} \sum_{m=1}^M \omega_m  r_{i m}^{(t)}  \mathbf{z}_m,\\[2pt]
    \Sigma_i^{(t+1)} &= \frac{1}{N_i^{(t)}} \sum_{m=1}^M \omega_m  r_{i m}^{(t)}  (\mathbf{z}_m-\boldsymbol{\mu}_i^{(t+1)})(\mathbf{z}_m-\boldsymbol{\mu}_i^{(t+1)})^\top  +  \lambda I
\end{aligned}
\end{equation}

A small ridge $\lambda I$ (covariance inflation) stabilises updates when $N_i^{(t)}$ is small. Diagonal or tied-covariance constraints are obtained by projecting $\Sigma_i^{(t+1)}$ accordingly.

\subsection{The EM-based GMA sampler}

The EM-based procedure is described below:

\begin{tcolorbox}[boxrule=1pt, colback=white]
\small
\centering
\textit{\textbf{Procedure: GMM refinement via population EM}}
\vspace{2mm}

\begin{enumerate}
    \item \textbf{Step 1 (optional): Bank refresh.} Draw a new bank $\{\mathbf{z}_m\}_{m=1}^M \sim r^{(t)}(\mathbf{z})$ (e.g. $r^{(t)}=q_{\boldsymbol{\theta}^{(t)}}$), then set
    \[
    \tilde{\omega}_m = \frac{\bar{p}(\mathbf{z}_m)}{r^{(t)}(\mathbf{z}_m)}, 
    \qquad
    \omega_m = \frac{\tilde{\omega}_m}{\sum_{s=1}^M \tilde{\omega}_s}.
    \]

    \item \textbf{Step 2: E-step (responsibilities) and effective counts.} For the (refreshed or fixed) weighted bank $\{(\mathbf{z}_m,\omega_m)\}_{m=1}^M$ and current $\boldsymbol{\theta}^{(t)}$,
    \[
    r_{im}^{(t)} = \frac{w_i^{(t)} \mathcal{N}(\mathbf{z}_m;\boldsymbol{\mu}_i^{(t)},\Sigma_i^{(t)})}
    {\sum_{\ell=1}^N w_\ell^{(t)} \mathcal{N}(\mathbf{z}_m;\boldsymbol{\mu}_\ell^{(t)},\Sigma_\ell^{(t)})},
    \qquad
    N_i^{(t)} = \sum_{m=1}^M \omega_m  r_{im}^{(t)}
    \]

    \item \textbf{Step 3: M-step (updates).} Update parameters as per Eq. Eq.\ref{eq:EM_Mstep}:
    \[
    \begin{aligned}
    w_i^{(t+1)} &= \frac{N_i^{(t)}}{\sum_{j=1}^N N_j^{(t)}} = N_i^{(t)},\\[2pt]
    \boldsymbol{\mu}_i^{(t+1)} &= \frac{1}{N_i^{(t)}} \sum_{m=1}^M \omega_m  r_{im}^{(t)}  \mathbf{z}_m,\\[2pt]
    \Sigma_i^{(t+1)} &= \frac{1}{N_i^{(t)}} \sum_{m=1}^M \omega_m  r_{im}^{(t)}  
    (\mathbf{z}_m-\boldsymbol{\mu}_i^{(t+1)})(\mathbf{z}_m-\boldsymbol{\mu}_i^{(t+1)})^\top + \lambda I
    \end{aligned}
    \]

    \item \textbf{Step 4: Use the refined proposal.} After $\tau$ sweeps, use $q_{\boldsymbol{\theta}^{(\tau)}}$ for (i) direct sampling, or (ii) importance sampling / resampling with
    \(
    w(\mathbf{z}) \propto \bar{p}(\mathbf{z})/q_{\boldsymbol{\theta}^{(\tau)}}(\mathbf{z})
    \),
    or (iii) as a warm-up for reverse-KL GMA weight sharpening.
\end{enumerate}
\end{tcolorbox}

After a few EM iterations, we can use the resulting GMM approximator $q_{\boldsymbol{\theta}^{(\tau)}} \approx p(\theta)$ to perform:
\begin{enumerate}
\item \textbf{Direct sampling:} draw $\mathbf{z}\sim q_{\boldsymbol{\theta}^{(\tau)}}$ for posterior approximation.
\item \textbf{Importance sampling / resampling:} use $q_{\boldsymbol{\theta}^{(\tau)}}$ as proposal, with weights
\(
w(\mathbf{z}) \propto \bar{p}(\mathbf{z})/q_{\boldsymbol{\theta}^{(\tau)}}(\mathbf{z})
\)
to obtain weighted or resampled draws.
\item \textbf{Hybrid with weights-only GMA:} keep $\{\Sigma_i\}$ fixed (e.g. Laplace covariances), alternate a few EM steps (updating $\{w_i,\boldsymbol{\mu}_i\}$) with the reverse-KL weights-only GMA sampling (e.g. with pGD or MD) to sharpen modes.
\end{enumerate}

\subsection{Some notes}
\begin{itemize}
\item \textit{Objective.} EM with exact $p$-expectations monotonically decreases $KL(p\|q_{\boldsymbol{\theta}})$; with SNIS it is approximate and improves with larger, better-covering banks.
\item \textit{Refreshing sample banks.} Choosing $r=q_{\boldsymbol{\theta}^{(t)}}$ and regenerating the bank each iteration yields a \textit{population EM} loop; using a fixed bank is cheaper but may limit exploration if the bank is narrow.
\item \textit{Mode coverage \textit{vs} sharpness.} Inclusive (forward) KL (EM) is mass-covering and tends to broaden components; the reverse/exclusive-KL weights-only update is mode-seeking. In practice, we found an \textit{EM warm-up} from GMM (e.g. a Laplace-mixture initialiser) followed by reverse-KL weight sharpening to be effective.
\item \textit{Complexity.} Each EM sweep costs $\mathcal{O}(NMd^2)$ for Gaussian PDFs plus covariance updates; using precomputed $\log \mathcal{N}(\mathbf{z}_m;\boldsymbol{\mu}_i,\Sigma_i)$ and Cholesky factors reduces overhead (cf. Appendix.\ref{app:precomputing_density_values} on pre-computation of GMM density).
\item \textit{Data-driven nature of EM.}
EM is inherently \textit{data-driven}: it maximises $\mathbb{E}_{p}[\log q_{\theta}(\mathbf Z)]$ and thus requires samples to approximate expectations. In classical settings these are observed data; in our posterior-approximation setting we instead rely on a \textit{bank} of samples with importance weights (e.g. SNIS) to stand in for draws from $p(\mathbf z)=\bar p(\mathbf z)/Z_p$. Consequently, the quality of EM updates is limited by the coverage of this bank: if it undersamples high-probability regions, EM will not “discover” them and can converge to biased, over-narrow solutions. Practical safeguards include refreshing the bank from a broader proposal (e.g. $q_{\theta^{(t)}}$), monitoring \textit{effective sample size} (ESS), using covariance inflation/ridge ($\lambda I$) when effective counts are small, and when needed, tempering or broadening the proposal to ensure adequate exploration.
\item \textit{EM-GMA vs traditional EM.} Traditional EM (on data) maximises the empirical log-likelihood $\tfrac{1}{N}\sum_{n}\log q_{\theta}(x_n)$ using responsibilities $r_{nk}$ over the observed dataset; EM-GMA instead minimises the inclusive $\mathrm{KL}(p\|q_{\theta})$ (equivalently maximises $\mathbb{E}_{p}[\log q_{\theta}(Z)]$) by SNIS on a proposal-drawn bank $\{(z_m,\omega_m)\}$ (typically from $q_{\theta^{(t)}}$), replacing sums over data by importance-weighted sums $\sum_m \omega_m r_{mk}$. EM-GMA needs only the unnormalised target $\bar p(z)$ (no normalising constant), is mass-covering, and shares the same E/M structure and moment updates as EM, with $p$-weighted sufficient statistics.
\end{itemize}

\section{LSTM for mortality modelling: hyper-parameter settings and further results of Bayesian LSTM using pGD-GMA and MD-GMA sampling} \label{app:Bayesian_LSTM_further_results}

Here we present the settings of hyper-parameters used in our mortality modelling experiments, and the results from pGD-GAM with multi-step rolling forecasting, and MD-GMA with single and multi-step rolling forecasting schemes.

\subsection{LSTM for mortality modelling: hyper-parameter settings}

\paragraph{Common data and model settings (all experiments).}
We fix the lookback window to $L=52$ weeks and the forecast horizon to $H=52$ weeks. Features and targets are standardised with \texttt{StandardScaler} fitted on the training split only. The PyTorch LSTM has \texttt{input\_dim}$=1$ (univariate series), \texttt{hidden\_dim}$=16$, \texttt{num\_layers}$=3$, \texttt{output\_dim}$=1$, \texttt{batch\_first}$=\texttt{True}$; hidden/cell states are re-initialised to zeros in each forward pass. Random seeds are fixed to $(111)$ (my favorite seed for Python programming) for NumPy and PyTorch. The train/test split uses the last $H=52$ sequences as test data.

\paragraph{Classic LSTM hyper-parameters:}
\begin{itemize}\setlength{\itemsep}{2pt}
  \item \textit{Network:} LSTM$(\text{input\_dim}=1,\ \text{hidden\_dim}=16,\ \text{num\_layers}=3)$ + Linear$(16  \to  1)$; $\approx 5,585$ trainable parameters.
  \item \textit{Loss \& optimiser:} MSE loss; \textit{Adam} \cite{kingma_adam_2017} with learning rate $0.01$.
  \item \textit{Training schedule:} $1,000$ epochs, full-batch updates (all sequences per epoch).
  \item \textit{Forecasting modes:} both \textit{one-step} and \textit{multi-step rolling} were run:
    \begin{itemize}\setlength{\itemsep}{1pt}
      \item One-step: uses true inputs only.
      \item Multi-step: autoregressive rollout. For \textit{train} rolling fits we seed the window with the \textbf{first} $52$ observed points; for \textit{test} rolling forecasts we seed with the \textbf{last} $52$ observed points from the training span (to avoid peeking).
    \end{itemize}
  \item \textit{Standardisation:} features ( $X$ ) scaled per time step using the training set; targets ( $y$ ) scaled and then inverted to report results in log-index and index space.
\end{itemize}

\paragraph{Bayesian LSTM with pGD-GMA hyper-parameters.}
\begin{itemize}\setlength{\itemsep}{2pt}
  \item \textit{Model/prior:} same LSTM architecture as above; elementwise Gaussian prior on $\theta$: $\theta \sim \mathcal{N}(0,1)$; observation scale prior $\sigma \sim \mathrm{HalfNormal}(1)$.
  \item \textit{Posterior target:} log-likelihood $\sum_t \log \mathcal{N}(y^{\text{(scaled)}}_t \mid f_\theta(X_t), \sigma)$ computed on the training split.
  \item \textit{Warm start:} Gaussian component centers are initialised to the point estimate $\theta^\star$ of weights from the trained classic LSTM; $\log \sigma$ initialised by $\log \widehat{\sigma}$ with $\widehat{\sigma}=\sqrt{\mathrm{MSE}_{\text{train}}}$ (on scaled targets). We infer over $(\theta,\log\sigma)$.
  \item \textit{Local proposal cloud:} $N=200$ Gaussian components; $M=30$ local samples per component; isotropic covariance $\sigma^2_{\text{loc}} I$ with $\sigma^2_{\text{loc}}=5\times 10^{-4}$.
  \item \textit{Initialisation around warm start:} jitter scales \texttt{init\_scale\_theta}$=0.02$, \texttt{init\_scale\_logsig}$=0.05$.
  \item \textit{Optimisation over mixture weights (pGD):}
    \begin{itemize}\setlength{\itemsep}{1pt}
      \item Iterations $K=100$; step size $\eta_k=\eta_0/\sqrt{k+k_0}$ with $\eta_0=0.05$, $k_0=800$.
      \item Gradient uses precomputed log-PDFs and a tempered, standardised target; tempering schedule $\beta_k$ linearly increases from $\beta_{\min}=0.30$ to $1.0$.
      \item Euclidean step followed by projection to the probability simplex.
      \item Entropy penalty coefficient annealed: $\lambda_k = \lambda_0 (1-k/K)$ with $\lambda_0=10^{-2}$.
      \item \textit{Note:} Unlike MD-GMA, pGD-GMA does \textbf{not} use temperature $\tau_k$ or convex mixing $\alpha_k$; only $\beta_k$ and $\lambda_k$ schedules are active.
    \end{itemize}
  \item \textit{Tail averaging:} final weights are averaged over the last $L_{\text{avg}}=75$ iterations.
  \item \textit{Posterior sampling:} draw \texttt{TOTAL\_SAMPLES}$=3000$ parameter vectors by first sampling a component with probability $w_i$ and then a local sample within that component.
  \item \textit{Forecasting:} per-draw predictive paths computed with up to \texttt{max\_draws}$=500$ posterior samples; noise injection during rollout is disabled (mean-path summaries); both \textit{one-step} and \textit{multi-step} modes as described above (same seeding rule for rolling).
\end{itemize}

\paragraph{Bayesian LSTM with MD-GMA (with stabilisation) hyper-parameters.}
\begin{itemize}\setlength{\itemsep}{2pt}
  \item \textit{Model/prior, warm start, local cloud, initialisation:} identical to pGD-GMA (values as above): $N=200$, $M=30$, $K=100$, $\sigma^2_{\text{loc}}=5  \times  10^{-4}$, \texttt{init\_scale\_theta}$=0.02$, \texttt{init\_scale\_logsig}$=0.05$.
  \item \textit{Mirror-descent update (entropy geometry):}
    \begin{itemize}\setlength{\itemsep}{1pt}
      \item Step size $\eta_k=\eta_0/\sqrt{k+k_0}$ with $\eta_0=0.05$, $k_0=800$.
      \item \textit{Temperature} schedule $\tau_k = 1 + (\tau_0-1)(1-k/K)$ with $\tau_0=1.8$ (flattens early updates).
      \item \textit{Convex mixing} coefficient $\alpha_k=\max\{0.05,\ \alpha_0(1-k/K)\}$ with $\alpha_0=0.20$ (smooths trajectories).
      \item \textit{Tempering} $\beta_k$ increases linearly from $\beta_{\min}=0.30$ to $1.0$ (robustifies early gradients).
      \item \textit{Entropy regularisation} $\lambda_k = \lambda_0 (1-k/K)$ with $\lambda_0=10^{-2}$ (prevents weight collapse early).
      \item Update: $\log w^{\text{prop}} = (\log w^{(k-1)} - \eta_k g^{(k)})/\tau_k$, then $w^{(k)}=(1-\alpha_k) w^{(k-1)} + \alpha_k  \mathrm{softmax}(\log w^{\text{prop}})$.
    \end{itemize}
  \item \textit{Tail averaging, sampling, forecasting:} same as pGD-GMA ($L_{\text{avg}}=75$, \texttt{TOTAL\_SAMPLES}$=3000$, \texttt{max\_draws}$=500$, no noise injection in rollout; both one-step and multi-step modes with the same seeding protocol).
\end{itemize}

\subsection{Classic LSTM and Bayesian LSTM (pGD-GMA) with \textit{multi-step} rolling forecasting}

When switching from the \textit{one-step} scheme to the more challenging \textit{multi-step} rolling forecasting mode, both classic and Bayesian LSTM models show a notable drop in accuracy. In this setting, forecasts are generated recursively by feeding predictions back into the input sequence rather than using the true observed values. For training forecasts, only the first $L=52$ observed weeks are provided to initialize the recursion, and for test forecasts, the last $L=52$ points from the training set are used as the starting window. This setup amplifies the risk of error propagation across long horizons. The classic LSTM benchmark achieves $\text{train RMSE} = 0.1828$ and $\text{test RMSE} = 0.1030$, indicating significant error accumulation. The Bayesian LSTM trained with pGD-GMA sampling performs slightly better in terms of robustness, with posterior mean forecasts reaching $\text{train RMSE} = 0.1681$ and $\text{test RMSE} = 0.1196$. While both models under rolling forecasts exhibit greater error than in the one-step setting, the Bayesian formulation maintains coherent uncertainty intervals that widen appropriately across the forecast horizon, reflecting compounding uncertainty. These results highlight the challenges of long-range autoregressive forecasting, while underscoring the value of Bayesian inference in quantifying forecast reliability under demanding conditions.

\begin{figure}[H]
    \centering
    \begin{minipage}{0.43\linewidth}
        \centering
        \includegraphics[width=\linewidth]{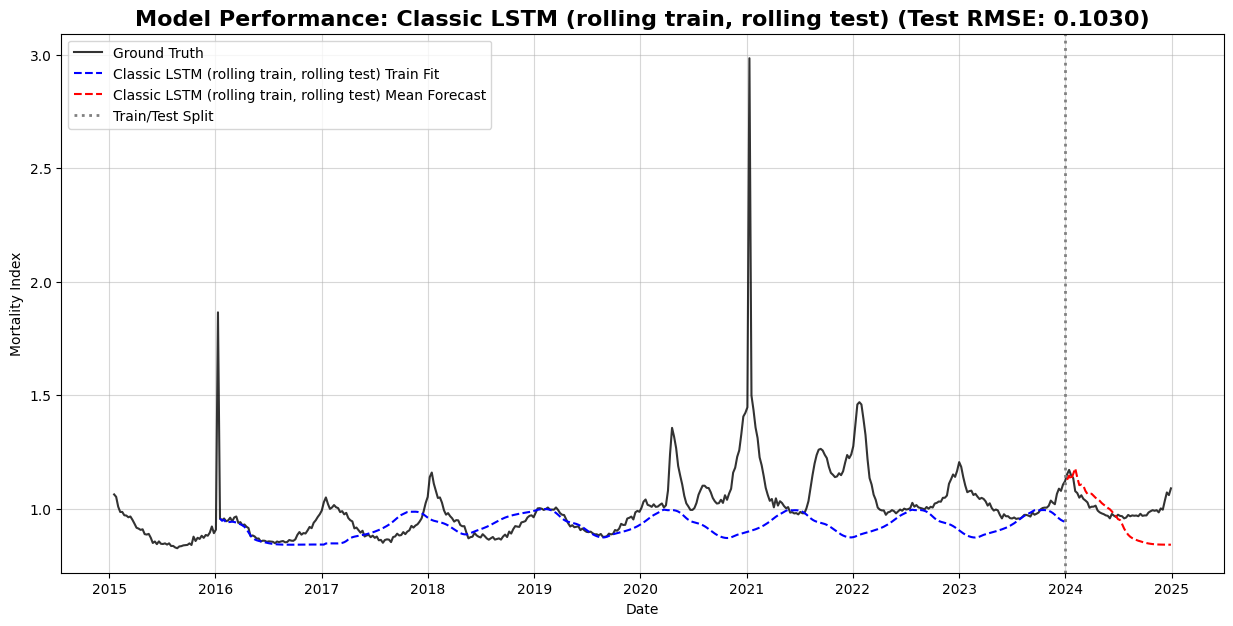}
        \subcaption{Classic LSTM multi-step rolling forecasts.}
        \label{fig:classic_lstm_multiStep}
    \end{minipage}
    \hfill
    \begin{minipage}{0.48\linewidth}
        \centering
        \includegraphics[width=\linewidth]{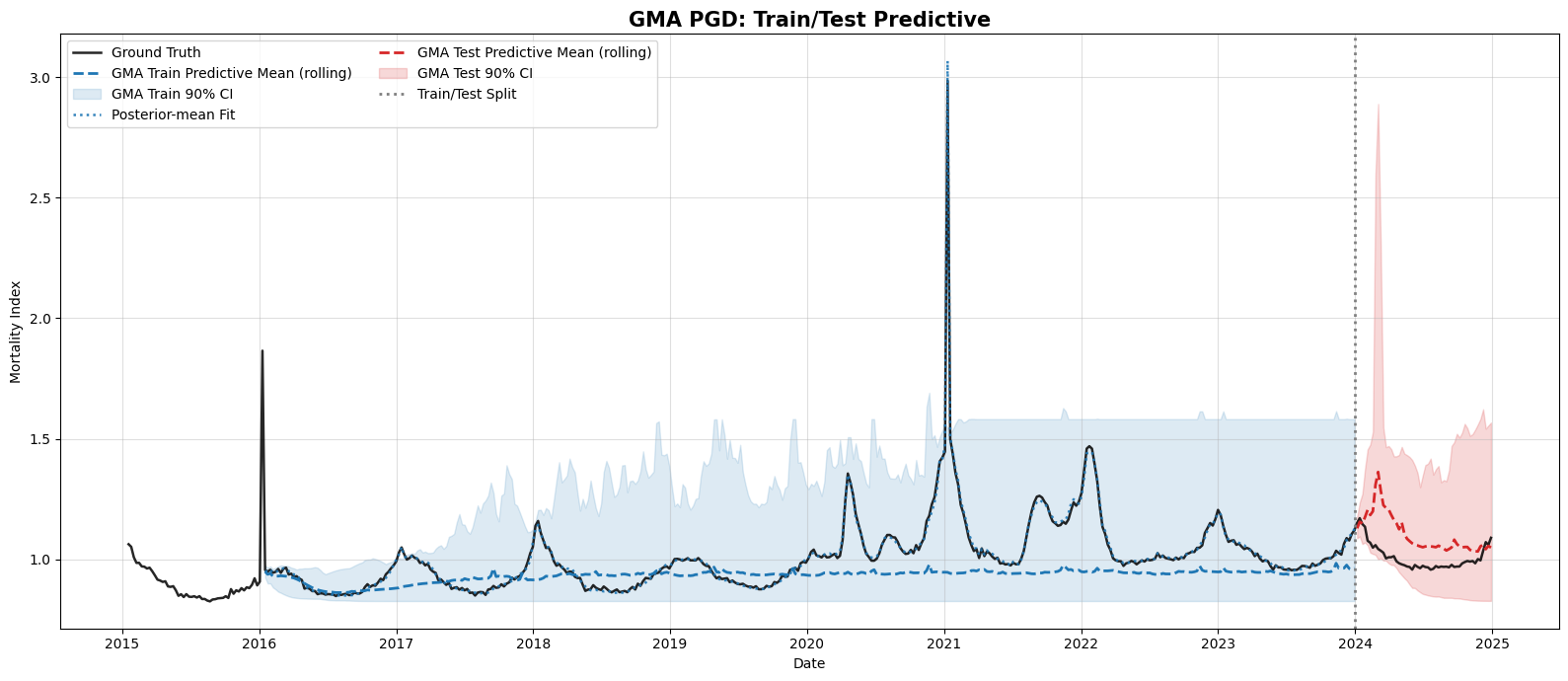}
        \subcaption{Bayesian LSTM multi-step rolling forecasts (pGD-GMA).}
        \label{fig:bayesian_lstm_multiStep}
    \end{minipage}
    \caption{Comparison of multi-step rolling forecasts for U.S. weekly mortality using (a) classic LSTM and (b) Bayesian LSTM with pGD-GMA sampling. Forecasts are initialized with $L=52$ observed weeks (black), with predictions recursively added to the input sequence and rolled forward. Bayesian LSTM additionally provides credible intervals that widen with horizon length.}
    \label{fig:rolling_forecast_results}
\end{figure}

\subsection{Bayesian LSTM (MD-GMA) with \textit{one-step} rolling forecasting}

\paragraph{Results.}
We next evaluate the Bayesian LSTM trained with the mirror descent variant of GMA (MD-GMA) on the U.S.\ mortality index data under the \textit{one-step rolling} forecasting scheme. As shown in Fig.\ref{fig:gma_md_diagnostics}, the MD-GMA algorithm converges rapidly: by iteration $k=100$ (again, runtime $\approx 1$ second), the mixture entropy had stabilized at $H(w)\approx 5.30$, corresponding to an effective number of components $\text{eff}\approx 200$, i.e. dense weights and nearly the full mixture support. Unlike the pGD-GMA sampler, which concentrated most of the posterior mass on a small handful of Gaussian components (effective $\approx 10$), the MD-GMA sampler maintained a much broader distribution of mixture weights. This is reflected in the nearly uniform top-10 component weights (Fig.\ref{fig:gma_md_diagnostics}, right), each accounting for only about $0.5\%$ of the mixture. Such behaviour indicates that MD-GMA emphasizes posterior diversity, avoiding mode-collapse and ensuring exploration across a wider set of parameter hypotheses.

Forecasting performance under MD-GMA remains strong. The posterior mean forecasts achieve $\text{train RMSE} = 0.0215$ and $\text{test RMSE} = 0.0256$, very close to those obtained under the pGD-GMA sampler ($\text{train RMSE} = 0.0170$, $\text{test RMSE} = 0.0255$). Both Bayesian formulations outperform the classic LSTM benchmark in terms of test error: while the classic, deterministic LSTM achieves $\text{test RMSE} = 0.0276$ under one-step forecasting, its lack of uncertainty quantification limits interpretability. By contrast, the MD-GMA posterior predictive intervals (Fig.\ref{fig:Bayesian_LSTM_forecast_MD}) captures both aleatoric variability and epistemic uncertainty.

Overall, these results suggest a clear trade-off between the two optimisation schemes: pGD-GMA yields sharper posterior concentration and more efficient dimensionality reduction (smaller effective component count), whereas MD-GMA favours robustness and diversity by spreading mass across nearly all mixture components. Both approaches deliver comparable point forecast accuracy, but the MD-GMA sampler may be particularly advantageous in regimes where preserving posterior diversity is crucial for downstream decision-making.

\begin{figure}[ht!]
    \centering
    \begin{subfigure}[b]{0.32\textwidth}
        \centering
        \includegraphics[width=\linewidth]{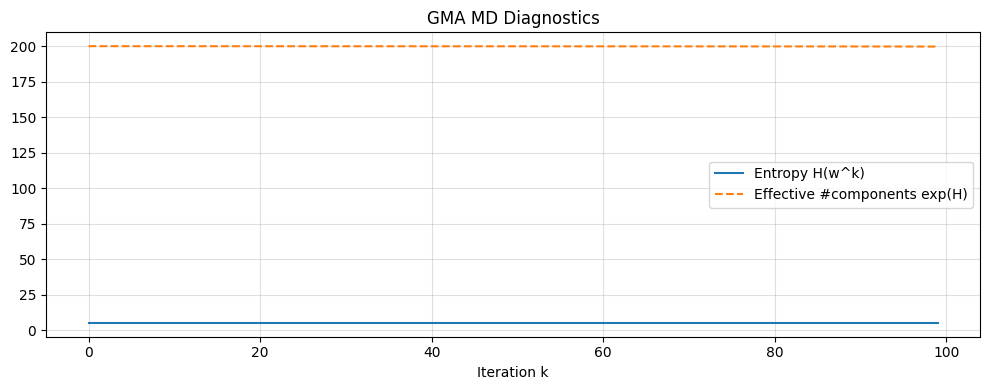}
    \end{subfigure}
    \hfill
    \begin{subfigure}[b]{0.32\textwidth}
        \centering
        \includegraphics[width=\linewidth]{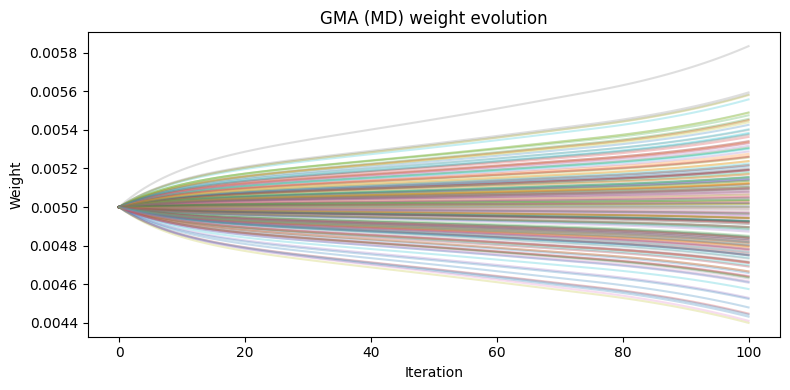}
    \end{subfigure}
    \hfill
    \begin{subfigure}[b]{0.32\textwidth}
        \centering
        \includegraphics[width=\linewidth]{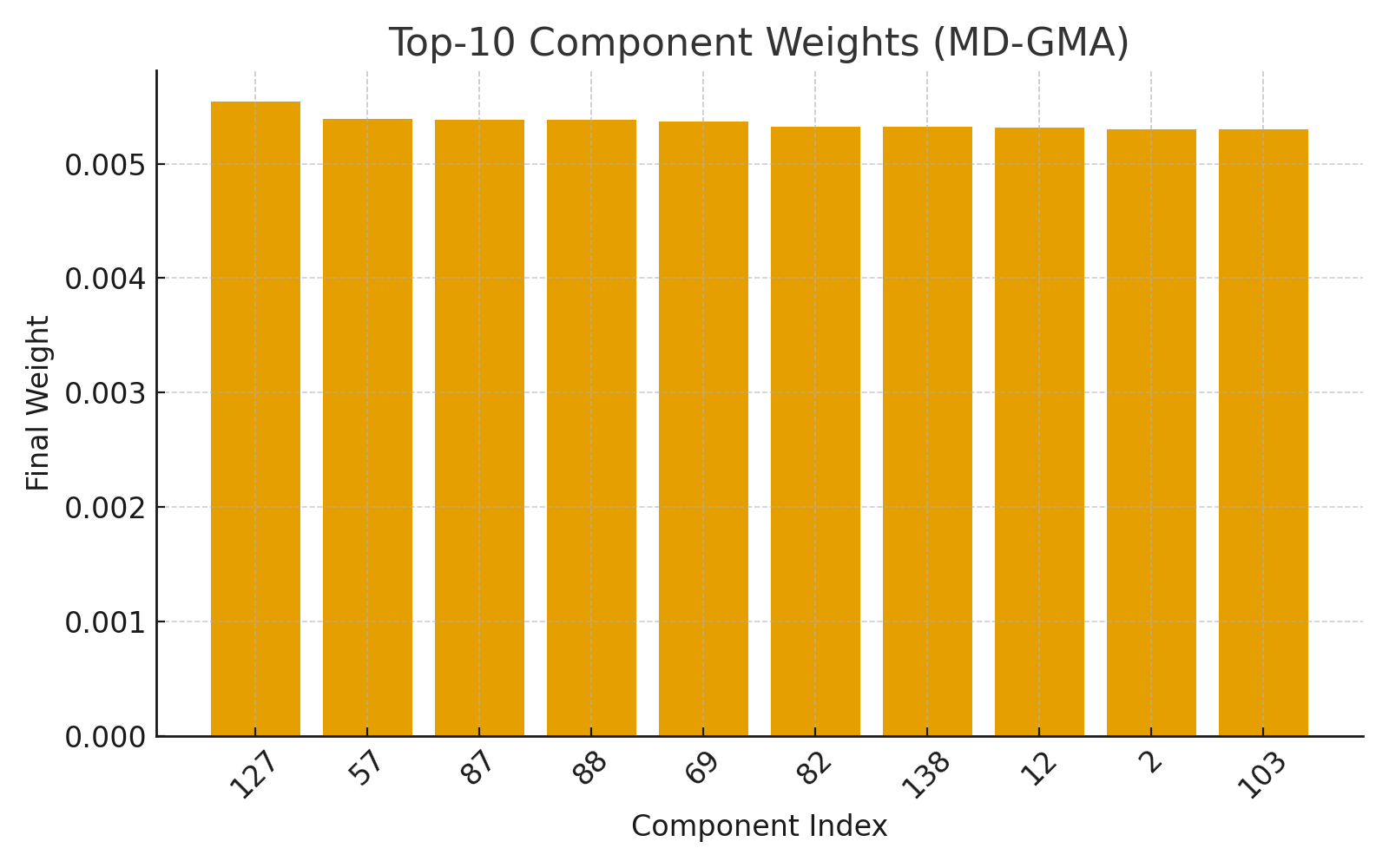}
    \end{subfigure}
    \caption{MD-GMA sampling diagnostics. Left: entropy \& effective component count, middle: weight evolution, right: final top-10 component weights. Compared to pGD-GMA, MD-GMA maintains a nearly uniform allocation across components, avoiding mode collapse.}
    \label{fig:gma_md_diagnostics}
\end{figure}

\begin{figure}[H]
    \centering
    \includegraphics[width=0.65\linewidth]{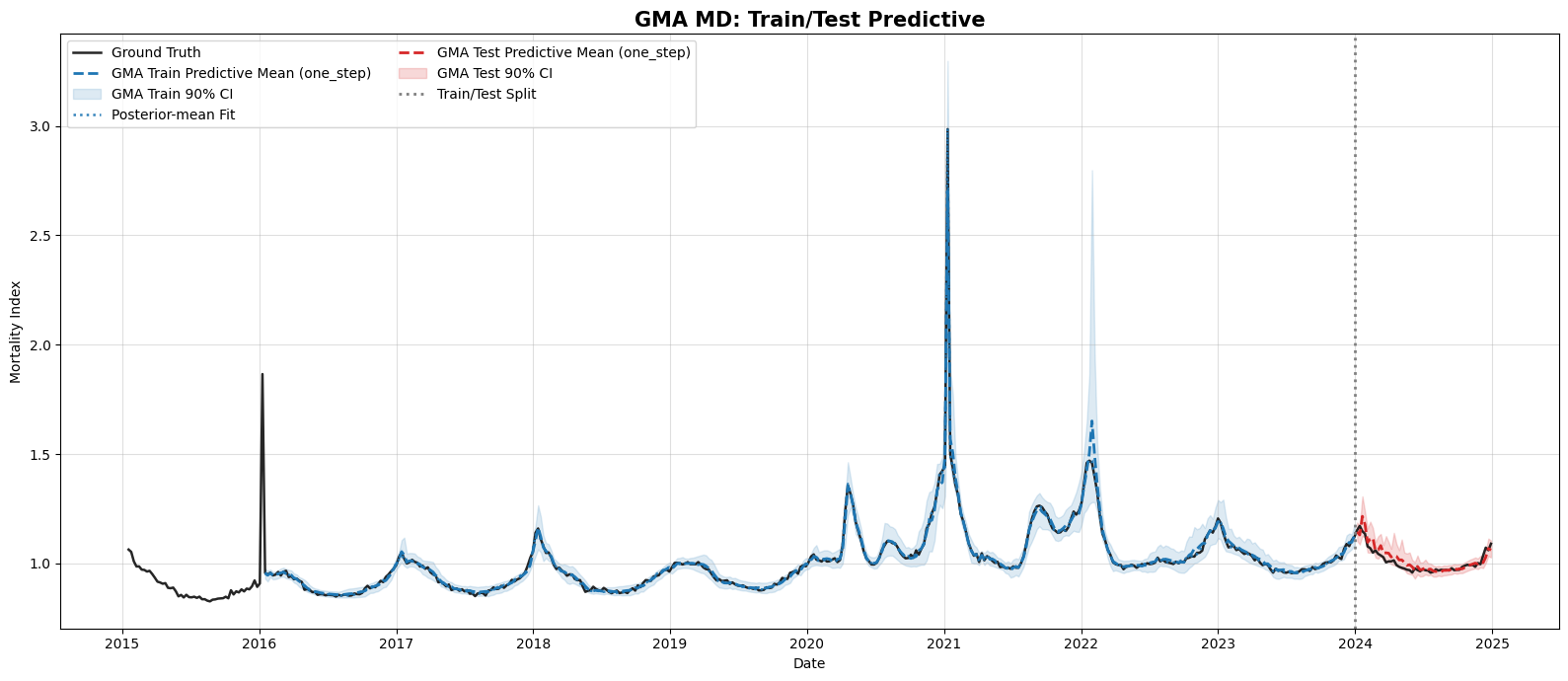}
    \caption{Bayesian LSTM (\textit{one-step rolling}) forecasts of weekly U.S.\ mortality index using MD-GMA sampling. Predictive mean trajectories (red) track the observed values (black), while the predictive intervals (shaded) remain broad, reflecting the higher entropy posterior induced by MD-GMA.}
    \label{fig:Bayesian_LSTM_forecast_MD}
\end{figure}

\subsection{Bayesian LSTM (MD-GMA) with \textit{multi-step} rolling forecasting}

When applying the MD-GMA sampler under the more demanding \textit{multi-step rolling} forecasting scheme, the Bayesian LSTM again demonstrates competitive performance. The posterior mean forecasts yield $\text{train RMSE} = 0.1775$ and $\text{test RMSE} = 0.1154$. Compared with the one-step MD-GMA results ($\text{train RMSE} = 0.0215$, $\text{test RMSE} = 0.0256$), this represents a substantial increase in error, reflecting the cumulated forecasting errors and the difficulty of recursively propagating uncertainty over longer horizons. Nevertheless, MD-GMA remains broadly comparable to the multi-step pGD-GMA sampler (whose $\text{train RMSE} = 0.1681$, $\text{test RMSE} = 0.1196$), and both Bayesian approaches outperform the classic LSTM benchmark ($\text{train RMSE} = 0.1828$, $\text{test RMSE} = 0.1030$) in terms of robustness, particularly by providing coherent uncertainty quantification. 

Relative to pGD-GMA, the MD-GMA intervals (Fig.\ref{fig:Bayesian_LSTM_forecast_multiStep_MD}) are slightly wider, consistent with the higher entropy posterior maintained by MD optimisation. This broader spread mitigates overconfidence and leads to marginally smaller test RMSE. Overall, the results confirm that both Bayesian formulations are more resilient than the classic deterministic LSTM when faced with compounding error in multi-step forecasts, with MD-GMA offering a more conservative (wider), but still accurate, quantification of forecast uncertainty.

\begin{figure}[H]
    \centering
    \includegraphics[width=0.65\linewidth]{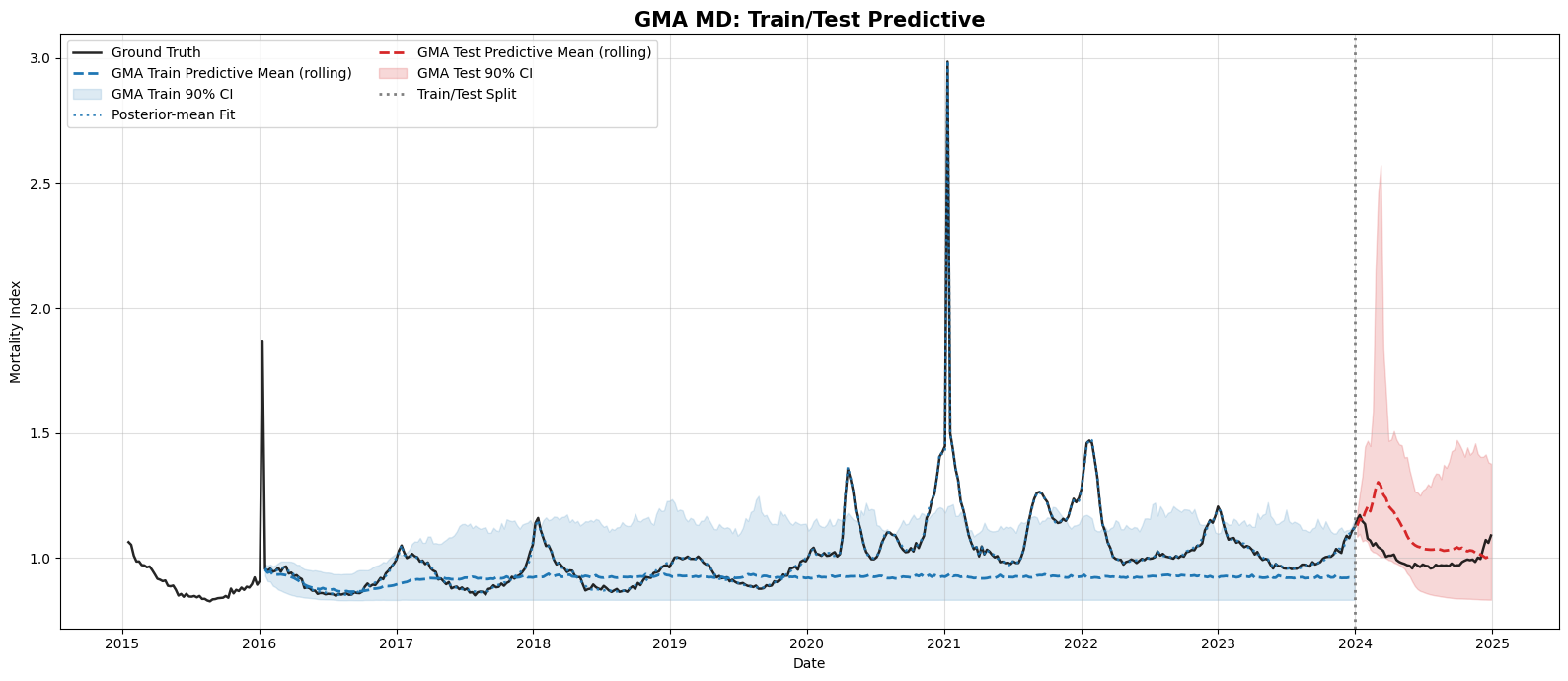}
    \caption{Bayesian LSTM (\textit{multi-step rolling}) forecasts of weekly U.S.\ mortality index using MD-GMA sampling. The predictive mean trajectory (red) captures long-term trends, while the shaded predictive intervals reflect compounding uncertainty across the forecast horizon.}
    \label{fig:Bayesian_LSTM_forecast_multiStep_MD}
\end{figure}

\section{Metrics used in language modelling experiments}
\label{app:LLM_metrics}

For completeness, we expand the metric definitions used in the language modelling experiments in Section.\ref{sec:BLM}. 
All metrics are computed at the \textit{token level}, i.e. with respect to the final tokenizer outputs.

\begin{itemize}
    \item \textbf{Negative log-likelihood (NLL):} the average of the negative logarithm of the predicted probability assigned to the correct token.
    \[
    \mathrm{NLL} = - \frac{1}{T}\sum_{t=1}^T \log p_\theta(y_t \mid x_{\leq t})
    \]

    \item \textbf{Perplexity:} the exponential of the cross-entropy. 
    For two distributions $p$ (true) and $q$ (model), cross-entropy is 
    \[
    H(p,q) = -\sum_y p(y)\log q(y).
    \]
    In language modeling the true distribution $p$ is unknown, but in supervised training/evaluation we only observe one true token $y_t$ at each step. 
    Thus we approximate $p(y)$ with a one-hot empirical distribution,
    \[
    p(y) = 
    \begin{cases}
    1 & y = y_t, \\
    0 & y \neq y_t
    \end{cases}
    \]
    which reduces the cross-entropy to 
    \[
    H(p,q) = -\log q(y_t)
    \]
    Averaging across all tokens $t=1,\ldots,T$ gives
    \[
    \frac{1}{T}\sum_{t=1}^T - \log q(y_t \mid x_{\leq t})
    \]
    which is exactly the \textit{negative log-likelihood} (NLL). 
    Hence, in our case $\mathrm{NLL} \equiv H(p,q)$, and perplexity is a monotone transform:
    \[
    \mathrm{PPL} = \exp(\mathrm{NLL}).
    \]
    Perplexity can be interpreted as the effective average branching factor of the model’s predictions (e.g. $\mathrm{NLL}=2$ nats implies $\mathrm{PPL}\approx 7.4$, meaning the model behaves as if choosing uniformly among about seven tokens per step).

    \item \textbf{Accuracy:} the fraction of correct top-1 predictions across tokens.
    \[
    \mathrm{Accuracy} = \frac{1}{T}\sum_{t=1}^T \mathbf{1}\{\arg\max_v p_\theta(v\mid x_{\leq t}) = y_t\}
    \]

    \item \textbf{Brier score:} the mean squared error of predicted probabilities, computed as the squared difference between the predicted probability assigned to the correct token and the actual outcome indicator, averaged over all tokens.
    \[
    \mathrm{Brier} = \frac{1}{T}\sum_{t=1}^T \Big( p_\theta(y_t \mid x_{\leq t}) - 1 \Big)^2 + \sum_{v\neq y_t} \Big(p_\theta(v\mid x_{\leq t}) - 0\Big)^2
    \]
\end{itemize}

\section{Constructing LMA from empirical samples}
\label{sec:constructing_LMA_from_empirical_samples}

Here we present the detailed process of constructing an Laplace mixture, via clustering and GMA refining, from samples draw from a target distribution, as used in the BOED example in Section.\ref{subsec:BOED}.

\paragraph{Inputs.}
Imagine that we have some empirical samples \footnote{For example, in Section.\ref{subsec:BOED}, we obtained these empirical samples by fitting a separate (penalized) logistic model per cell line $i$ on centered log-dose in the historical data $\mathcal{H}$, yielding one MAP estimate $\hat\theta_i=(\hat\alpha_i,\hat\beta_i)^\top$; repeating this yields multiple empirical samples of $\theta$.} $\{\hat\theta_i\}$ drawn from the prior $p(\theta)$.

\paragraph{Step 1: find local ``modes'' via $k$-means.}
Cluster $\{\hat\theta_i\}$ into $J$ groups using $k$-means (any Lloyd-style algorithm suffices).  
Let $\mathcal{C}_j=\{i:\text{label}(i)=j\}$ be the index set for cluster $j$.
Define the component location and covariance
\[
\mu_j
=\frac{1}{|\mathcal{C}_j|}\sum_{i\in\mathcal{C}_j}\hat\theta_i,
\qquad
\Sigma_j
=\underbrace{\frac{1}{\max(|\mathcal{C}_j|-1,1)}
\sum_{i\in\mathcal{C}_j}(\hat\theta_i-\mu_j)(\hat\theta_i-\mu_j)^\top}_{\text{empirical covariance}}
 + \lambda I_2
\]
with a small ridge $\lambda>0$ to ensure $\Sigma_j\succ 0$.  
Initialize weights by cluster size: $w_j^{(0)} \propto |\mathcal{C}_j|$, then normalize $\sum_j w_j^{(0)}=1$.

\paragraph{Step 2: weights-only Gaussian mixture refinement (WGMA).}
Keep $\{(\mu_j,\Sigma_j)\}_{j=1}^J$ fixed and refine the mixture weights $\mathbf{w}$ on the simplex to better match the empirical samples.  
Write the LMA surrogate as
\[
q_{\mathrm{LMA}}(\theta;\mathbf{w})
=\sum_{j=1}^J w_j \mathcal{N}(\theta;\mu_j,\Sigma_j),
\quad
w_j\ge 0, \sum_j w_j=1
\]
Using the empirical support $\{\theta_m\}_{m=1}^M=\{\hat\theta_i\}$ with uniform masses $w_m=\tfrac{1}{M}$,
minimize the reverse-KL (equivalently the cross-entropy) loss
\[
\mathcal{L}(\mathbf{w})
=-\frac{1}{M}\sum_{m=1}^M
\log \Bigg(\sum_{j=1}^J w_j \phi_j(\theta_m)\Bigg),
\qquad
\phi_j(\theta)=\mathcal{N}(\theta;\mu_j,\Sigma_j)
\]
The gradient is
\[
\frac{\partial \mathcal{L}}{\partial w_j}
= -\frac{1}{M}\sum_{m=1}^M
\frac{\phi_j(\theta_m)}{\sum_{\ell=1}^J w_\ell \phi_\ell(\theta_m)}
\]
Update $\mathbf{w}$ with exponentiated-gradient (MD) steps on the simplex:
\[
w_j \leftarrow
\frac{ w_j \exp \big(-\eta \partial \mathcal{L}/\partial  w_j\big)}
{\sum_{\ell=1}^J  w_\ell \exp \big(-\eta \partial \mathcal{L}/\partial  w_\ell\big)},
\quad j=1,\dots,J
\]
with stepsize $\eta>0$, iterating until convergence.

\begin{algorithm}[H]
\caption{WGMA (weights-only) refinement for LMA initialiser}
\begin{algorithmic}[1]
\STATE \textbf{Input:} empirical support $\{\theta_m\}_{m=1}^M$, fixed $\{(\mu_j,\Sigma_j)\}_{j=1}^J$, init $\mathbf{w}^{(0)}$, stepsize $\eta$.
\STATE Precompute $\Phi_{mj} \leftarrow \phi_j(\theta_m)$ for all $m,j$.
\FOR{$t=0,1,2,\dots$}
  \STATE $d_j \leftarrow -\frac{1}{M}\sum_{m=1}^M \frac{\Phi_{mj}}{\sum_{\ell=1}^J  w_\ell^{(t)} \Phi_{m\ell}}$ \hfill (gradient component)
  \STATE $\tilde w_j \leftarrow  w_j^{(t)} \exp(-\eta d_j)$
  \STATE $ w_j^{(t+1)} \leftarrow \tilde w_j \Big/ \sum_{\ell=1}^J \tilde w_\ell$ \hfill (renormalize)
\ENDFOR
\STATE \textbf{Output:} refined weights $\mathbf{w}$; mixture $q_{\mathrm{LMA}}(\theta)=\sum_j  w_j \mathcal{N}(\theta;\mu_j,\Sigma_j)$.
\end{algorithmic}
\end{algorithm}

\paragraph{Practical notes.}
(i) Use a small ridge $\lambda$ (e.g. $10^{-2}$) to stabilize $\Sigma_j$.  
(ii) If a cluster size is small, reseed or borrow a few nearest neighbors before computing its moments.  
(iii) $J$ can be selected by simple heuristics (elbow plot) or held fixed across runs.  
(iv) Because WGMA optimizes only $\mathbf{w}$, the procedure is fast and robust, yet it captures multi-modality and skew present in the empirical cloud.

\section{Inference of the star-shaped density: settings and implementation}
\label{app:star_shape_implementations}

Here we provide the implementations details of the experiment conducted in Section.\ref{subsec:star_shaped_density}.

\paragraph{Target.}
Five-component equal-weight Gaussian mixture on a radius-$1.5$ circle; each arm is anisotropic with covariance
$\Sigma_0=\mathrm{diag}(1, 1/\text{skewness})$ with $\text{skewness}=100$, rotated by $2\pi/5$ between arms. The normalized log-density is
\(
\log p(\mathbf z)=\log\Big[\tfrac15\sum_{k=1}^5 \mathcal N(\mathbf z;\boldsymbol\mu_k,\Sigma_k)\Big].
\)
Reference set: $N_\text{ref}=2000$ i.i.d. draws from the ground-truth mixture.

\paragraph{Evaluation.}
We report time-to-samples and the unbiased squared Maximum Mean Discrepancy (MMD):
\[
\widehat{\mathrm{MMD}}^2(X,Y)  = 
\frac{1}{m(m-1)}  \sum_{i\neq j} k(x_i,x_j)
+
\frac{1}{n(n-1)}  \sum_{i\neq j} k(y_i,y_j)
-
\frac{2}{mn} \sum_{i,j} k(x_i,y_j)
\]
where \(k(x,y)=\sum_{\ell=1}^3 \exp(-\gamma_\ell \|x-y\|^2)\).
We set \(\gamma_0 = 1/(2 \mathrm{median}\{ \|x_i-x_j\|^2\})\) from the reference set (upper triangular pairs), and use \(\gamma_\ell\in\{0.5,1,2\}\gamma_0\).
Each method outputs exactly \(N_\text{draw}=2000\) samples.

\paragraph{EM-GMA (ata-free, population EM).}
We maximize \(\mathbb E_{p}[\log q_\theta(\mathbf Z)]\) for a 5-component GMM \(q_\theta\) via EM, where expectations are approximated by self-normalized importance sampling using a bank of \(M_\text{bank}=8192\) draws from the current \(q\); 80 EM sweeps; ridge \(10^{-5}\); initialization: weights uniform, means on a random ring of radius 2.0, covariances identity. After fitting, we draw \(N_\text{draw}\) samples by stratified sampling from the learned mixture.

\paragraph{MH.}
Random-walk Gaussian proposal with covariance \(0.2 I_2\); 1,000 burn-in steps then 2,000 kept draws.

\paragraph{NUTS (PyMC).}
We target the unnormalized mixture via a Potential that cancels a standard-normal base prior:
\(\text{posterior} \propto \exp(\log p(\mathbf z))\).
PyMC \texttt{NUTS} with \texttt{target\_accept}=0.9, \texttt{tune}=1000, 1 chain, 2,000 kept draws.

\paragraph{LMC.}
Overdamped Langevin updates with step size \(10^{-2}\), initialized at \(\mathbf 0\), for 2,000 steps; gradients via PyTorch \textit{autodiff} on the closed-form mixture.

\paragraph{SVGD (JAX).}
2,000 particles initialized i.i.d. standard normal; 1,000 SVGD steps with step size \(10^{-2}\); RBF kernel bandwidth via median heuristic; gradients from the analytic log-density.

\paragraph{MFVI-ADVI (PyMC).}
Mean-field Gaussian variational family on \(\mathbf z\in\mathbb R^2\) with standard Normal prior replaced by a Potential \(\log p(\mathbf z)\). Fit for 20,000 steps (\texttt{pm.fit(method='advi')}), then draw 2,000 posterior samples.

\paragraph{GM-ADVI (JAX; stabilized).}
Diagonal-covariance GMM variational family with \(K_\text{mix}=40\). Parameters: mixture logits, means, and \(\texttt{softplus}\)-parameterized scales with a floor \(10^{-3}\).
Initialization: means copied from the five arm means and tiled to 40 with small Gaussian jitter; logits zero; scales \(\approx 0.3\).
Objective: SIWAE with \(T_\text{siwae}=8\) per component. Regularizers: entropy on logits (weight \(10^{-3}\)) and quadratic scale penalty (weight \(10^{-4}\)). Optimization: \textit{Adam} \cite{kingma_adam_2017} with global-norm clipping 5.0, lr \(10^{-2}\), for 3,000 steps. Exactly 2,000 samples drawn from the final variational mixture.

\paragraph{EVI (particle-based energetic variational inference \cite{Wang2021EVI}).}
2,000 particles initialized i.i.d. \(\mathcal N(0,I)\). Per outer loop: Adagrad updates of the particle flow field
\(
\mathbf v(\theta)=\tfrac{1}{\tau}(\theta-\theta_0)-\nabla_\theta \log p(\theta) - \nabla_\theta \log k_\text{rbf}(\theta)
\)
with RBF kernel interactions (bandwidth via median heuristic), learning rate \(0.1\), \(\tau=0.1\); 20 inner iterations and 5 outer iterations.

\paragraph{Timing.}
We include warm-up/optimization time for NUTS/ADVI/GM-ADVI/EVI and the full iteration budget for SVGD/LMC/MH/EM-GMA. All methods emit exactly 2,000 samples used for \(\widehat{\mathrm{MMD}}^2\).

\end{document}